\documentclass[11pt,a4paper]{book}
\usepackage{anuthesis}

\title{Nonparametric General Reinforcement Learning}
\author{Jan Leike}

\def\keywords{%
Bayesian methods,
sequence prediction,
merging,
general reinforcement learning,
universal artificial intelligence,
AIXI,
Thompson sampling,
knowledge-seeking agents,
Pareto optimality,
intelligence,
asymptotic optimality,
computability,
reflective oracle,
grain of truth problem,
Nash equilibrium%
}

\date{\today}

\usepackage[utf8]{inputenc}
\usepackage[T1]{fontenc}
\usepackage{hyperref}
\usepackage{url}
\usepackage{enumerate}
\usepackage[labelfont=bf]{caption}
\usepackage{subcaption}
\usepackage[usenames,dvipsnames]{xcolor}

\usepackage{amsmath}
\usepackage{amssymb}
\usepackage{amsthm}
\usepackage{mathtools}

\usepackage[flowcharts]{aixi}

\usepackage{longtable} 
\usepackage{paralist}  
\usepackage{booktabs}  
\usepackage{eso-pic}   
\usepackage{forloop}   

\usepackage{natbib}
\bibpunct{(}{)}{;}{a}{,}{,}
\usepackage{textcomp} 
\usepackage{bibentry} 
\nobibliography*

\usepackage{makeidx}
\makeindex

\newif\ifprint
\printfalse

\usepackage{algorithm}
\usepackage[noend]{algpseudocode}

\usepackage{aliascnt}

\newcommand{\mynewtheorem}[2]{
	\newaliascnt{#1}{dummy}
	\newtheorem{#1}[#1]{#2}
	\aliascntresetthe{#1}
	\expandafter\def\csname #1autorefname\endcsname{#2}
}
\theoremstyle{plain}
\mynewtheorem{theorem}{Theorem}
\mynewtheorem{lemma}{Lemma}
\mynewtheorem{proposition}{Proposition}
\mynewtheorem{corollary}{Corollary}
\mynewtheorem{assumption}{Assumption}
\mynewtheorem{problem}{Problem}
\theoremstyle{definition}
\mynewtheorem{definition}{Definition}
\mynewtheorem{remark}{Remark}
\mynewtheorem{example}{Example}

\usepackage{etoolbox}
\AtBeginEnvironment{example}{%
	\pushQED{\qed}%
}
\AtEndEnvironment{example}{\popQED}
\AtBeginEnvironment{remark}{%
	\pushQED{\qed}%
}
\AtEndEnvironment{remark}{\popQED}

\AtBeginDocument{

}

\colorlet{mycolor}{purple!70!black}

\colorlet{myblue}{blue!70!black}
\ifprint
\def\linkcolor{black}
\else
\def\linkcolor{myblue}
\fi
\hypersetup{
  colorlinks=true,
  linkcolor=\linkcolor,
  citecolor=\linkcolor,
  urlcolor=\linkcolor
}

\makeatletter
\hypersetup{
  pdftitle={\@title},
  pdfauthor={\@author},
  pdfkeywords={\keywords},
  pdfsubject={AI},
  pdflang={en-US}
}
\makeatother

\usepackage{noindentafter}
\newcommand{\falsequote}[2]{%
	\vspace{-5em}
	\textit{#2}
	\hfill --- #1
	\vspace{5em}
	\NoIndentAfterThis
}

\newcommand{\assref}[2]{\hyperref[#2]{\autoref*{#1}\ref*{#2}}}

\usepackage{dsfont} 
\def\one{{\mathds{1}}}           
\DeclareMathOperator*{\xor}{xor} 
\def\F{\mathcal{F}}              
\def\Foo{\F_\infty}              
\def\Km{K\!m}                    
\def\MDL{\mathrm{MDL}}           
\def\Mrefl{\M_{\mathrm{refl}}}
\def\Mcomp{\M_{\mathrm{comp}}^{\mathrm{CCM}}}   
\def\Mlsc{\M_{\mathrm{LSC}}}     
\def\eps{\varepsilon}
\def\S{\mathcal{S}}              
\def\Ent{{\mathrm{Ent}}}         
\def\IG{{\mathrm{IG}}}           
\def\BayesExp{{\mbox{BayesExp}}} 
\def\Bernoulli{{\mathrm{\mbox{Bernoulli}}}} 

\renewcommand{\thepage}{\roman{page}}

\begin{document}


\addcontentsline{toc}{chapter}{Title Page}
\pagestyle{empty}
\thispagestyle{empty}

\begin{titlepage}
  \enlargethispage{2cm}
  \begin{center}
    \makeatletter
    \Huge\textbf{\@title} \\[.4cm]
    \Huge\textbf{\thesisqualifier} \\[2.5cm]
    \huge\textbf{\@author} \\[9cm]
    \makeatother
    \LARGE A thesis submitted for the degree of \\
    Doctor of Philosophy \\
    at the \\
    Australian National University \\[2cm]
    \thismonth
  \end{center}
\end{titlepage}


\cleardoublepage
\newpage

\vspace*{14cm}
\begin{center}
  \makeatletter
  \copyright\ \@author
  \makeatother \\[2em]
  This work is licensed under the \\
  \href{http://creativecommons.org/licenses/by/4.0/}{Creative Commons Attribution 4.0 International License} \\[1.5em]
  \includegraphics[height=10mm,natwidth=701,natheight=247]{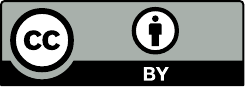}
\end{center}
\noindent

\newpage

\vspace*{14cm}
\begin{center}
No reinforcement learners were harmed in the making of this thesis.
\end{center}

\newpage

\vspace*{7cm}
\begin{center}
  Except where otherwise indicated, this thesis is my own original
  work.
\end{center}

\ifprint
\vspace*{2cm}
\hspace{8cm}\includegraphics[height=2cm]{signature.png}
\else
\vspace*{4cm}
\fi

\hspace{8cm}\makeatletter\@author\makeatother\par
\hspace{8cm}\today

\cleardoublepage
\pagestyle{empty}
\section*{Acknowledgements}

There are many without whom this thesis would not have been possible.
I sincerely hope that
this page is not the way they learn how grateful I am to them.
I thank in particular \dots

\begin{itemize}[\dots]
\item first and foremost, Marcus Hutter:
	he is an amazing supervisor;
	always very supportive of my (unusual) endeavors,
	spent countless hours reading my drafts
	with a impressive attention to detail.
	I am also grateful to him for forcing me to be absolutely rigorous
	in my mathematical arguments, and, of course,
	for developing the theory of universal AI
	without which this thesis would not have existed.
	I could not have picked a better supervisor.
\item the Australian National University
	for granting me scholarships
	that let me pursue my academic interests unrestricted
	and without any financial worries.
\item Csaba Szepesvári and the University of Alberta
	for hosting me for three months.
\item Matthias Heizmann and the University of Freiburg
	for hosting me while I was traveling in Europe.
\item the Machine Intelligence Research Institute
	for enabling me to run MIRIx research workshops.
\item CCR, UAI, Google DeepMind, ARC, MIRI, and FHI for supporting my travel.
\item Tor Lattimore for numerous explanations, discussions, and pointers
	that left me with a much deeper understanding
	of the theory of reinforcement learning.
\item Laurent Orseau for interesting discussions, encouragement,
	and for sharing so many intriguing ideas.
\item my fellow students:
	Mayank Daswani, Tom Everitt, Daniel Filan, Roshan \text{Shariff}, Tian Kruger,
	Emily Cutts Worthington, Buck Shlegeris, Jarryd Martin, John Aslanides,
	Alexander Mascolo, and Sultan Javed for so many interesting discussions
	and for being awesome friends.
	I especially thank Daniel, Emily, Mayank, and Buck
	for encouraging me to read more of Less Wrong and Slate Star Codex.
\item Tosca Lechner for studying statistics with me
	despite so many scheduling difficulties across all these time zones.
\item Tom Sterkenburg, Christian Kamm, Alexandra Surdina,
	Freya Fleckenstein, Peter Sunehag, Tosca Lechner, Ines Nikolaus,
	Laurent Orseau, John Aslanides,
	and especially Daniel Filan
	for proofreading parts of this thesis.
\item the CSSA for being a lovely bunch
	that made my stay in Australia feel less isolated.
\item my family for lots of love and support,
	and for tolerating my long absences from Europe.
\end{itemize}

\cleardoublepage
\pagestyle{headings}
\chapter*{Abstract}
\label{cha:abstract}
\addcontentsline{toc}{chapter}{Abstract}


Reinforcement learning problems
are often phrased in terms of Markov decision processes~(MDPs).
In this thesis we go beyond MDPs and consider
reinforcement learning in environments that are non-Markovian,
non-ergodic and only partially observable.
Our focus is not on practical algorithms,
but rather on the fundamental underlying problems:
How do we balance exploration and exploitation?
How do we explore optimally?
When is an agent optimal?
We follow the nonparametric realizable paradigm:
we assume the data is drawn from an unknown source
that belongs to a known countable class of candidates.

First,
we consider the passive (sequence prediction) setting,
learning from data that is not independent and identically distributed.
We collect results from artificial intelligence,
algorithmic information theory, and game theory
and put them in a reinforcement learning context:
they demonstrate how an agent can learn the value of its own policy.

Next,
we establish negative results on Bayesian reinforcement learning agents,
in particular AIXI.
We show that
unlucky or adversarial choices of the prior
cause the agent to misbehave drastically.
Therefore
Legg-Hutter intelligence and balanced Pareto optimality,
which depend crucially on the choice of the prior,
are entirely subjective.
Moreover,
in the class of all computable environments every policy is Pareto optimal.
This undermines all existing optimality properties for AIXI.

However, there are Bayesian approaches to general reinforcement learning
that satisfy objective optimality guarantees:
We prove that Thompson sampling
is asymptotically optimal in stochastic environments
in the sense that
its value converges to the value of the optimal policy.
We connect asymptotic optimality to regret
given a recoverability assumption on the environment
that allows the agent to recover from mistakes.
Hence Thompson sampling achieves sublinear regret in these environments.

AIXI is known to be incomputable.
We quantify this using the arithmetical hierarchy,
and establish upper and corresponding lower bounds for incomputability.
Further, we show that AIXI is not limit computable,
thus cannot be approximated using finite computation.
However there are limit computable $\eps$-optimal approximations to AIXI.
We also derive computability bounds for knowledge-seeking agents, and
give a limit computable weakly asymptotically optimal reinforcement learning agent.

Finally,
our results culminate in a formal solution to the grain of truth problem:
A Bayesian agent acting in a multi-agent environment learns to predict
the other agents' policies if its prior assigns positive probability to them%
~(the prior contains a grain of truth).
We construct a large but limit computable class containing a grain of truth
and show that agents based on Thompson sampling over this class
converge to play $\eps$-Nash equilibria
in arbitrary unknown computable multi-agent environments.

\paragraph{Keywords.}
\keywords.

\cleardoublepage
\pagestyle{headings}
\markboth{Contents}{Contents}
\setcounter{tocdepth}{2} 
\tableofcontents
\addcontentsline{toc}{chapter}{Contents} 
\listoffigures
\addcontentsline{toc}{chapter}{List of Figures}
\listoftables
\addcontentsline{toc}{chapter}{List of Tables}


\mainmatter


\chapter{Introduction}
\label{cha:introduction}

\falsequote{Albert Einstein}{Everything I did was for the glamor, the money, and the sex.}

After the early enthusiastic decades,
research in artificial intelligence~(AI) now mainly aims at specific domains:
playing games, mining data, processing natural language,
recognizing objects in images, piloting robots, filtering email,
and many others~\citep{RN:2010}.
Progress on particular domains has been remarkable,
with several high-profile breakthroughs:
The chess world champion Garry Kasparov was defeated by the computer program Deep Blue in 1997~\citep{DeepBlue}.
In 2011 the world's best Jeopardy!\ players were defeated by the computer program Watson~\citep{Watson}.
As of 2014 Google's self-driving cars completed over a million kilometers autonomously on public roads~\citep{GoogleCar}.
Finally, in 2016 Google DeepMind's AlphaGo beat Lee Sedol,
one of the world's best players, at the board game Go~\citep{AlphaGo}.

While these advancements are very impressive,
they are highly-specialized algorithms tailored to their domain of expertise.
Outside that domain these algorithms perform very poorly:
AlphaGo cannot play chess, Watson cannot drive a car, and DeepBlue
cannot answer natural language queries.
Solutions in one domain typically do not generalize to other domains
and no single algorithm performs well in more than one of them.
We classify these kinds of algorithms as \emph{narrow AI}\index{AI!narrow}.

This thesis is not about narrow AI.
We expect progress on narrow AI to continue and even accelerate,
taking the crown of human superiority in domain after domain.
But this is not the ultimate goal of artificial intelligence research.
The ultimate goal is to engineer a mind---%
to build a machine that can learn to do all tasks that humans can do,
at least as well as humans do them.
We call such a machine
\emph{human-level AI}\index{AI!human-level}~(HLAI)
if it performs at human level and
\emph{strong AI}\index{AI!strong}
if it surpasses human level.
This thesis is about strong AI.

The goal of developing HLAI has a long tradition in AI research and
was explicitly part of the 1956 Dartmouth conference
that gave birth to the field of AI~\citep{Dartmouth:1955}:
\begin{quote}
We propose that a 2 month, 10 man study of artificial intelligence be carried out during the summer of 1956 at Dartmouth College in Hanover, New Hampshire. The study is to proceed on the basis of the conjecture that every aspect of learning or any other feature of intelligence can in principle be so precisely described that a machine can be made to simulate it. An attempt will be made to find how to make machines use language, form abstractions and concepts, solve kinds of problems now reserved for humans, and improve themselves. We think that a significant advance can be made in one or more of these problems if a carefully selected group of scientists work on it together for a summer.
\end{quote}
In hindsight this proposal reads vastly overconfident,
and disappointment was inevitable.
Making progress on these problems turned out to be a lot harder than promised,
and over the last decades any discussion of research targeting HLAI
has been avoided by serious researchers in the field.
This void was filled mostly by crackpots,
which tainted the reputation of HLAI research even further.
However, this trend has recently been reverted:
\citet{Chalmers:2010singularity},
\citet{Hutter:2012explosion}, \citet{Schmidhuber:2012singularity},
\citet{Bostrom:2014}, \citet*{HTRW:2014huffington},
\citet{Shanahan:2015singularity}, and \citet{Walsh:2016singularity}
are well-known scientists discussing the prospect of HLAI seriously.
Even more: the explicit motto of Google DeepMind,
one of today's leading AI research centers,
is to ``solve intelligence.''

\section{Reinforcement Learning}
\label{sec:reinforcement-learning}

The best formal model for strong AI we currently have
is reinforcement learning~(RL).
Reinforcement learning studies algorithms
that learn to act in an unknown environment through trial and error%
~\citep{SB:1998,Szepesvari:2010,WvO:2012}.
Without knowing the structure of the environments or the goal\index{goal},
an agent has to learn what to do through the carrot-and-stick approach:
it receives a reward in form of a numeric feedback
signifying how well it is currently doing;
from this signal the agent has to figure out autonomously what to do.
More specifically,
in a \emph{general reinforcement learning problem}%
\index{general reinforcement learning problem}
an \emph{agent}\index{agent} interacts sequentially with
an unknown \emph{environment}\index{environment}:
in every time step the agent chooses an \emph{action}\index{action} and
receives a \emph{percept}\index{percept}
consisting of an \emph{observation}\index{observation}
and a real-valued \emph{reward}\index{reward}.
The sequence of past actions and percepts is the \emph{history}\index{history}.
The goal in reinforcement learning is
to maximize cumulative (discounted) rewards
(this setup is described formally in \autoref{sec:general-rl}).

A central problem in reinforcement learning is the balance between
\emph{exploration} and \emph{exploitation}\index{exploration vs.\ exploitation}:
should the agent harvest rewards in the regions of the environment
that it currently knows~(exploitation) or
try discovering more profitable regions~(exploration)?
Exploration is costly and dangerous:
it forfeits rewards that could be had right now,
and it might lead into traps\index{trap} from which the agent cannot recover.
However, exploration may pay off in the long run.
Generally, it is not clear how to make this tradeoff
(see \autoref{sec:discussion-optimality}).

Reinforcement learning algorithms can be categorized by whether they
learn \emph{on-policy}\index{on-policy} or \emph{off-policy}\index{off-policy}.
Learning on-policy means learning the value of
the policy that the agent currently follows.
Typically, the policy is slowly improved while learning,
like \emph{SARSA}\index{SARSA}~\citep{SB:1998}.
In contrast, learning off-policy means following one policy
but learning the value of another policy (typically the optimal policy),
like \emph{$Q$-learning}\index{Q-learning@$Q$-learning}%
~\citep{WD:1992Qlearning}.
Off-policy methods are more difficult to handle in practice
(see the discussion on function approximation below)
but tend to be more data-efficient
since samples from an old policy do not have to be discarded.

Reinforcement learning has to be distinguished from \emph{planning}\index{planning}.
In a planning problem we are provided with
the true environment and are tasked with finding an optimal policy.
Mathematically it is clear what the optimal policy is,
the difficulty stems from finding a reasonable solution with
limited computation.
Reinforcement learning is fundamentally more difficult
because the true environment is unknown and has to be learned from observation.
This enables two approaches:
we could learn a model of the true environment and then
use planning techniques within that model;
this is the \emph{model-based}\index{model-based} approach.
Alternatively, we could learn an optimal policy directly
or through an intermediate quantity
(typically the value function);
this is the \emph{model-free}\index{model-free} approach.
Model-based methods tend to be more data-efficient
but also computationally more expensive.
Therefore most algorithms used in practice~($Q$-learning and SARSA)
are model-free.

\subsection{Narrow Reinforcement Learning}
\label{ssec:narrow-reinforcement-learning}

In the reinforcement learning literature it is typically assumed that
the environment is a \emph{Markov decision process}\index{MDP}~(MDP),
i.e., the next percept only depends on the last percept and action and
is independent of the rest of the history
(see \autoref{ssec:typical-environment-classes}).
In an MDP, percepts are usually called \emph{states}\index{state}.
This setting is well-analyzed~\citep{Puterman:1994,BT:1995,SB:1998},
and there is a variety of algorithms that are known to
learn the MDP asymptotically,
such as \emph{TD learning}\index{TD-learning}~\citep{Sutton:1988}
and $Q$-learning\index{Q-learning@$Q$-learning}~\citep{WD:1992Qlearning}.

Moreover,
for MDPs various learning guarantees have been proved in the literature.
First, there are bounds on the agent's \emph{regret}\index{regret},
the difference between the obtained rewards and
the rewards of the optimal policy.
\citet{AJO:2009} derive the regret\index{regret} bound $\tilde O(dS\sqrt{At})$
for ergodic\index{ergodic} MDPs where
$d$ is the \emph{diameter}\index{diameter} of the MDP
(how many steps a policy needs on average to get from one state of the MDP
to any other),
$S$ is the number of states,
$A$ is the number of actions,
and $t$ is the number of time steps the algorithm runs.
Second,
given $\varepsilon$ and $\delta$, a reinforcement learning algorithm is said
to have \emph{sample complexity}\index{PAC} $C(\varepsilon, \delta)$ iff
it is $\varepsilon$-suboptimal for at most $C(\varepsilon, \delta)$ time steps
with probability at least $1 - \delta$~(probably approximately correct, PAC).
For MDPs the first sample complexity bounds were due to \citet{Kakade:2003}.
\citet{LH:2012PAC} use the algorithm UCRL$\gamma$~\citep{AJO:2009}
with geometric discounting\index{discounting!geometric}
with discount rate $\gamma$ and
derive the currently best-known PAC\index{PAC} bound of
$\tilde O( -T / (\varepsilon^2 (1-\gamma)^3) \log \delta)$
where $T$ is the number of non-zero transitions in the MDP.

Typically, algorithms for MDPs rely on visiting every state multiple times
(or even infinitely often),
which becomes infeasible for large state spaces
(e.g.\ a video game screen consisting of millions of pixels).
In these cases, \emph{function approximation}\index{function approximation}
can be used
to learn an approximation to the value function~\citep{SB:1998}.
Linear function approximation is known to converge for
several on-policy algorithms~\citep{TVR:1997fapprox,Sutton:1988,Gordon:2001},
but proved tricky for off-policy algorithms~\citep{Baird:1995}.
A recent breakthrough was made by \citet{MYWS:2015emphatic} and
\citet{Yu:2015emphatic}
with their emphatic TD algorithm that converges off-policy.
For nonlinear function approximation no convergence guarantee is known.

Among the historical successes of reinforcement learning is
autonomous helicopter piloting~\citep{KJSN:2003} and TD-Gammon,
a backgammon algorithm that learned through self-play~\citep{Tesauro:1995},
similar to AlphaGo~\citep{DeepMind:2016Go}.

\subsection{Deep Q-Networks}
\label{ssec:DQN}

The current state of the art in reinforcement learning
challenges itself to playing simple video games.
Video games are an excellent benchmark because
they come readily with the reward structure provided:
the agent's rewards are the change in the game score.
Without prior knowledge of any aspect of the game,
the agent needs to learn to score as many points in the game as possible
from looking only at raw pixel data~(sometimes after some preprocessing).

This approach to general AI is
in accordance with the definition of intelligence given by
\citet{LH:2007int}:
\begin{quote}
Intelligence\index{intelligence}
measures an agent's ability to achieve goals\index{goal}
in a wide range of environments.
\end{quote}
In reinforcement learning the definition of the goal is very flexible,
and provided by the rewards.
Moreover,
a diverse selection of video games
arguably constitutes a `wide range of environments.'

A popular such selection is the Atari 2600\index{Atari 2600} video game console%
~\citep{BNVB:2013ale}.
There are hundreds of games released for this platform,
with very diverse challenges:
top-down shooting games such as Space Invaders,
ball games such as Pong,
agility-based games such as Boxing or Gopher,
tactical games such as Ms.\ Pac-Man, and
maze games such as Montezuma's Revenge.
An overview over some of the games is given in
\autoref{fig:Atari} on page~\pageref*{fig:Atari}.

\begin{figure}[t]
\centering
\begin{subfigure}[b]{0.45\textwidth}
\includegraphics[width=\textwidth]{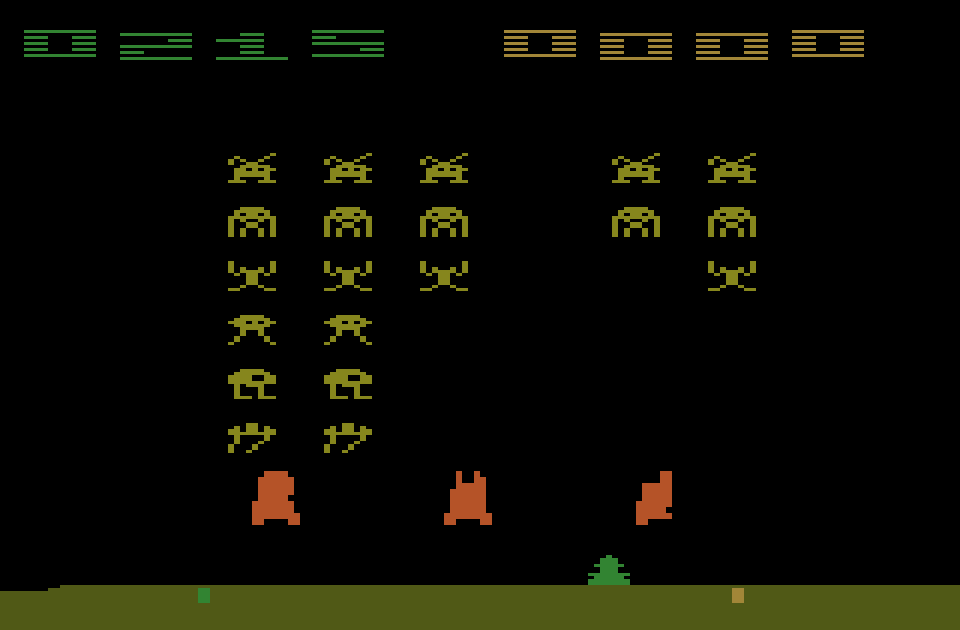}
\caption{Space Invaders:
the player controls the green cannon on the bottom of the screen
and fires projectiles at the yellow ships at the top.
The red blobs can be used as cover, but also fired through.
}
\label{fig:space-invaders}
\end{subfigure}
~
\begin{subfigure}[b]{0.45\textwidth}
\includegraphics[width=\textwidth]{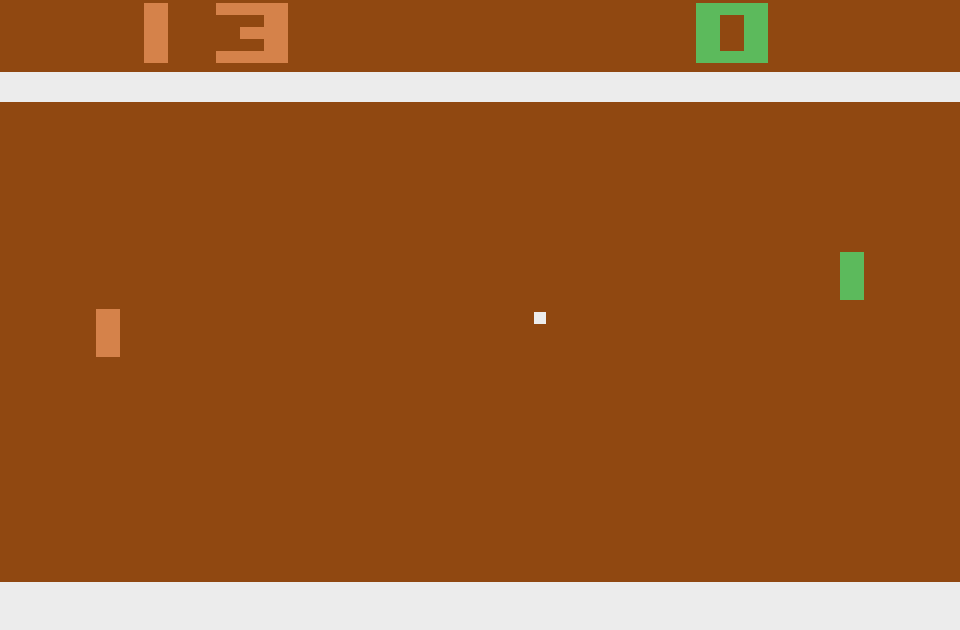}
\caption{Pong:
the player controls the green paddle on the right of the screen
and needs to hit the white ball such that
the computer opponent controlling the red paddle on the left
fails to hit the ball back.
}
\label{fig:pong}
\end{subfigure}
\\[1em]
\begin{subfigure}[b]{0.45\textwidth}
\includegraphics[width=\textwidth]{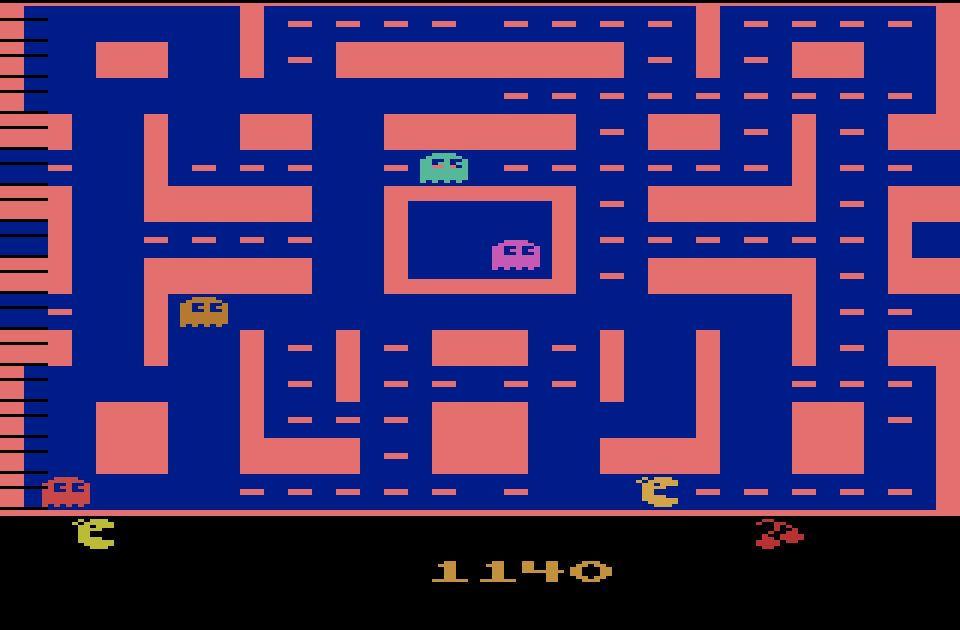}
\caption{Ms.\ Pac-Man:
the player controls the yellow mouth and
needs to eat all the red pallets in the maze.
The maze is roamed by ghosts that occasionally hunt the player and
kill her on contact unless a `power pill' was consumed recently.
}
\label{fig:ms-pac-man}
\end{subfigure}
~
\begin{subfigure}[b]{0.45\textwidth}
\includegraphics[width=\textwidth]{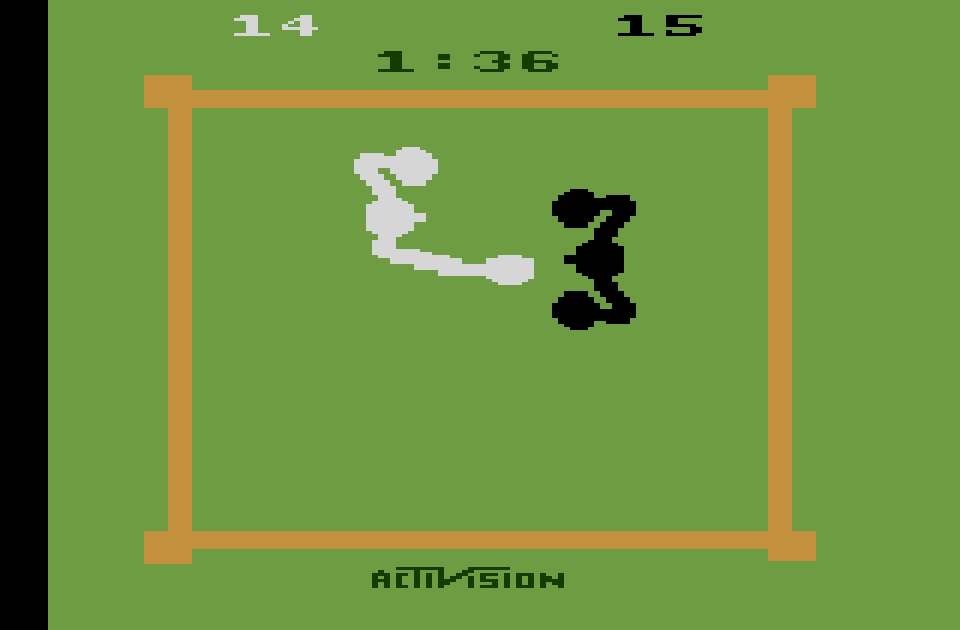}
\caption{Boxing:
the player controls the white figure on the screen
and extends their arms to throw a punch.
The aim is to hit the black figure that is controlled by the computer
and dodge their punches.
(I'm sure the choice of color was by accident.)
}
\label{fig:boxing}
\end{subfigure}
\\[1em]
\begin{subfigure}[b]{0.45\textwidth}
\includegraphics[width=\textwidth]{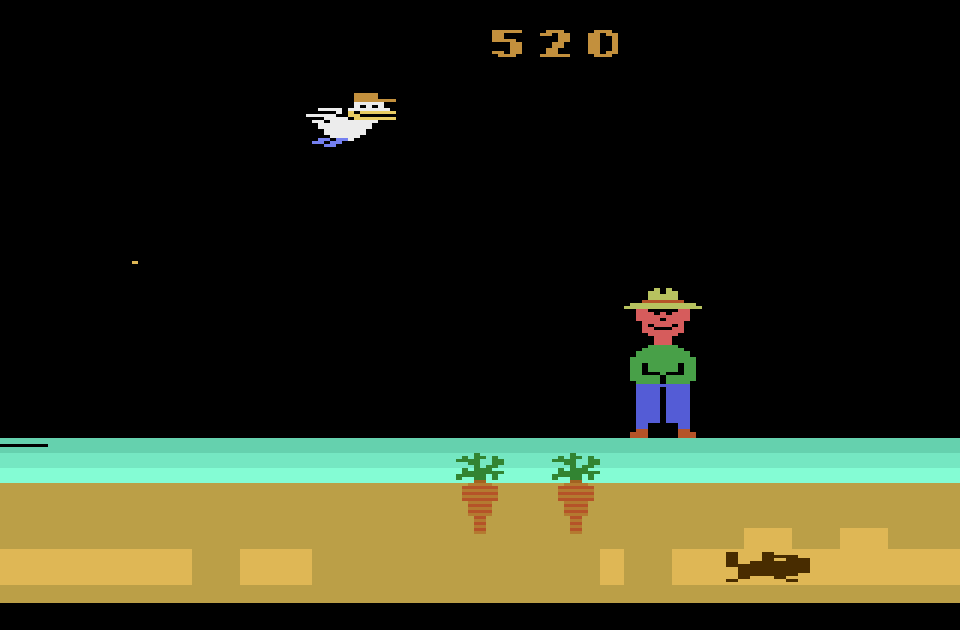}
\caption{Gopher:
a hungry rodent attempts to dig to the surface and steal the vegetables.
The player controls the farmer who protects them
by filling the rodent's holes.
}
\label{fig:Gopher}
\end{subfigure}
~
\begin{subfigure}[b]{0.45\textwidth}
\includegraphics[width=\textwidth]{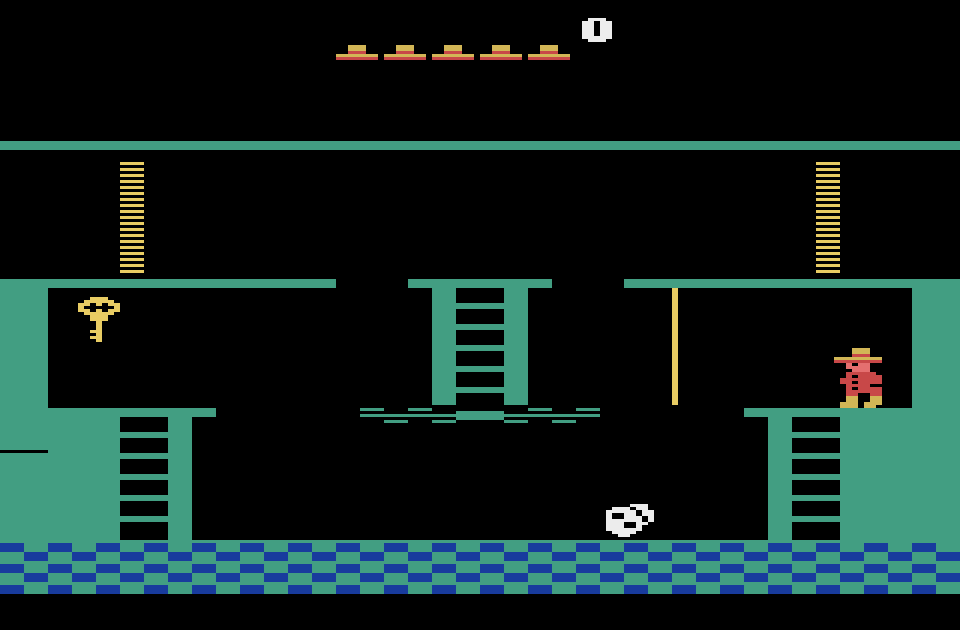}
\caption{Montezuma's Revenge:
the player controls the red adventurer.
The aim is to navigate a maze of deadly traps,
use keys to open doors, and
collect artifacts.
}
\label{fig:Montezuma}
\end{subfigure}
\caption[Selection of Atari 2600 video games]{%
A selection of Atari 2600 video games.
}
\label{fig:Atari}
\index{Atari 2600}
\end{figure}

\citet{MKSGAWR:2013DQN,MKSRV+:2015deepQ}\index{DQN}
introduce the deep $Q$-network~(DQN) algorithm,
combining $Q$-learning\index{Q-learning@$Q$-learning} with
nonlinear function approximation through convolutional neural networks.
DQN achieves 75\% of the performance of a human game tester
on 29 of 49 Atari games.
The two innovations that made this breakthrough possible are
(1) using a not so recent target $Q$-function in the TD update and
(2) experience replay.
For experience replay, a set of recent state transitions is retained
and the network is regularly retrained on
random samples from these old transitions.%
\footnote{The slogan for experience replay should be
`regularly retrained randomly on retained rewards'.}

DQN rides the wave of success of \emph{deep learning}\index{deep learning}%
~\citep{LBH:2015deep,Schmidhuber:2015deep,BGC:2016deep}.
Deep learning refers to the training of artificial neural networks
with several layers.
This allows them to automatically learn higher-level abstractions from data.
Deep neural networks are conceptionally simple
and have been studied since the inception of AI;
only recently has computation power become cheap enough
to train them effectively.
Recently deep neural networks have taken
the top of the machine learning benchmarks by storm%
~\citep[and references therein]{LBH:2015deep}:
\begin{quote}
These methods have dramatically improved the state-of-the-art
in speech recognition, visual object recognition, object detection
and many other domains such as drug discovery and genomics.
\end{quote}

Since the introduction of DQN
there have been numerous improvements on this algorithm:
increasing the gap on the Q-values of different actions~\citep{BOGTM:2016DQN},
training in parallel~\citep{NSBA+:2015DQN,MBMG+2016DQN},
improvements to the experience replay mechanism~\citep{SQAS:2016DQN},
generalization to continuous action spaces~\citep{LHPH+:2016DQN},
solve the overestimation problem~\citep{HGS:2016doubleQ}, and
improvements to the neural network architecture~\citep{WFL:2015DQN}.
The $Q$-values learned by DQN's neural networks are
intransparent to inspection;
\citet{ZZM:2016DQN} use visualization techniques on the $Q$-value networks.
Finally,
\citet{LMTB:2015shallow} managed to reproduce DQN's success
using only \emph{linear} function approximation\index{function approximation}%
~(no neural networks).
The key is a selection of features
similar to the ones produced by DQN's convolutional neural networks.

Regardless of its success,
the DQN algorithm fundamentally falls short of
the requirements for strong AI:
$Q$-learning\index{Q-learning@$Q$-learning} with function approximation
is targeted at solving
large-state~(fully observable) Markov decision processes.
In particular, it does not address the following challenges.

\begin{itemize}
\item \emph{Partial observability.}\index{partially observable}
	All games in the ATARI framework are fully observable%
	~(except for Montezuma's revenge):
	all information relevant to the state of the game
	is visible on the screen at all times
	(when using the four most recent frames).

	However, the real world is only partially observable.
	For example, when going to the supermarket
	you have to remember what you wanted to buy
	because you currently cannot observe
	which items you are missing at home.
	A strong AI needs to have memory and be able to remember
	things that happened in the past~(rather than only learning from it).

	An obvious approach to equip DQN with memory is
	to use recurrent neural networks
	instead of simple feedforward neural networks~\citep{HHLS:2015DQRN}.
	\citet{HS:2015DRQN} show that
	this enables the agent to play the games when using only
	a single frame as input.
	However, it is currently unclear whether recurrent neural networks
	are powerful enough to learn long-term dependencies in the data%
	~\citep{BSF:1994RNN}.

\item \emph{Directed exploration.}
	DQN fails in games with delayed rewards.
	For example,
	in Montezuma's Revenge the agent needs to avoid several obstacles
	to get to a key before receiving the first reward.
	DQN fails to score any rewards in this environment.
	This is not surprising:
	the typical approach for reinforcement learning,
	to use $\eps$-exploration
	for which the agent chooses actions at random with a certain probability,
	is insufficient for exploring complex environments;
	the probability of random walking into the first reward is just too low.

	Instead we need a more targeted exploration approach that
	aims at understanding the environment in a structured manner.
	Theoretical foundations are provided by
	\emph{knowledge-seeking agents}\index{knowledge-seeking}%
	~\citep{Orseau:2011ksa,Orseau:2014ksa,OLH:2013ksa}.
	\citet{KNST:2016DQN} introduce a hierarchical approach
	based on intrinsic motivation to improve DQN's exploration and
	manage to score points in Montezuma's Revenge.
	However, their approach relies on quite a bit of visual preprocessing and
	domain knowledge.

\item \emph{Non-ergodicity.}\index{recoverability}\index{ergodic}
	When losing in an Atari game,
	the agent always gets to play the same game again.
	From the agent's perspective, it has not actually failed,
	it just gets transported back to the starting state.
	Because of this,
	there are no strong incentives to be careful when exploring the environment:
	there can be no bad mistakes that make recovery impossible.

	However, in the real world some actions are irreversibly bad.
	If the robot drives off a cliff it can be fatally damaged
	and cannot learn from the mistake.
	The real world is full of potentially fatal mistakes
	(e.g.\ crossing the street at the wrong time) and
	for humans, natural reflexes and training by society
	make sure that we are very confident of what situations to avert.
	This is crucial,
	as some mistakes must be avoided without any training examples.
	Current reinforcement learning algorithms
	only learn about bad states by visiting them.

\item \emph{Wireheading.}\index{wireheading}
	The goal\index{goal} of reinforcement learning is to maximize rewards.
	When playing a video game the most efficient way to get rewards is
	to increase the game score.
	However, when a reinforcement learning algorithm is acting in the real world,
	theoretically it can change its own hard- and software.
	In this setting, the most efficient way to get rewards is
	to modify the reward mechanism to always provide the maximal reward%
	~\citep{Omohundro:2008,RO:2011delusion,Bostrom:2014}.
	Consequently the agent no longer pursues
	the designers' originally intended goals
	and instead only attempts to protect its own existence.
	The name \emph{wireheading} was established by analogy to
	a biology experiment by \citet{OM:1954wireheading}
	in which rats had a wire embedded into the reward center of their brain
	that they could then stimulate by the push of a button. 

	Today's reinforcement learning algorithms usually do not have access to
	their own internal workings, but more importantly they are not smart enough
	to understand their own architecture.
	They simply lack the capability to wirehead.
	But as we increase their capability,
	wireheading will increasingly become
	a challenge for reinforcement learning.
\end{itemize}

\subsection{General Reinforcement Learning}
\label{ssec:general-reinforcement-learning}

\begin{table}[t]
\begin{center}
\begin{tabular}{ll}
\toprule
Assumption & Description \\
\midrule
Full observability & the agent needs no memory to act optimally \\
Finite state & the environment has only finitely many states \\
Ergodicity & the agent can recover from any mistakes \\
Computability & the environment is computable \\
\bottomrule
\end{tabular}
\end{center}
\index{partially observable}\index{ergodic}\index{computable}\index{state}
\caption[Assumptions in reinforcement learning]{%
List of assumptions from the reinforcement learning literature.
In this thesis, we only make the computability assumption
which is important for \autoref{cha:computability}
and \autoref{cha:grain-of-truth-problem}.
}
\label{tab:assumptions-in-rl}
\end{table}

A theory of strong AI
cannot make some of the typical assumptions.
Environments are partially observable\index{partially observable},
so we are dealing with
\emph{partially observable Markov decision processes}\index{POMDP}~(POMDPs).
The POMDP's state space does not need to be finite.
Moreover, the environment may not allow recovery from mistakes:
we do not assume ergodicity or weak communication%
~(not every POMDP state has to be reachable from every other state).
So in general, our environments are infinite-state non-ergodic POMDPs.
\autoref{tab:assumptions-in-rl} lists the assumptions that are typical
but we do not make.

Learning POMDPs is a lot harder,
and only partially successful attempts have been made:
through predictive state representations~\citep{SLJPS:2003psr,SJR:2004psr},
and Bayesian methods~\citep{Doshi-Velez:2012}.
A general approach is
\emph{feature reinforcement learning}\index{feature reinforcement learning}%
~\citep{Hutter:2009phidbn,Hutter:2009phimdpx},
which aims to reduce the general reinforcement learning problem to an MDP
by aggregating histories into states.
The quest for a good cost function for feature maps
remains unsuccessful thus far~\citep{SH:2010frl,Daswani:2015}.
However, \citet{Hutter:2014aggr} managed to derive
strong bounds relating the optimal value function of the aggregated MDP
to the value function of the original process
even if the latter violates the Markov condition.

A full theoretical approach to the general reinforcement learning problem
is given by \citet{Hutter:2000,Hutter:2001aixi,Hutter:2002,
Hutter:2003AIXI,Hutter:2005,Hutter:2007aixi,Hutter:2012uai}.
He introduces the Bayesian RL agent \emph{AIXI}\index{AIXI}
building on the theory of sequence prediction
by \citet{Solomonoff:1964,Solomonoff:1978}.
Based in algorithmic information theory,
Solomonoff's prior\index{Solomonoff!prior} draws from famous insights by
William of Ockham, Sextus Epicurus, Alan Turing, and Andrey Kolmogorov%
~\citep{RH:2011}.
AIXI uses Solomonoff's prior over the class of all computable environments
and acts to maximize Bayes-expected rewards.
We formally introduce
Solomonoff's theory of induction in \autoref{cha:learning} and
AIXI in \autoref{ssec:AIXI}.
See also \citet{Legg:2008} for an accessible introduction to AIXI.

A typical optimality property in general reinforcement learning is
\emph{asymptotic optimality}\index{optimality!asymptotic}~\citep{LH:2011opt}:
as time progresses the agent converges to achieve the same rewards
as the optimal policy.
Asymptotic optimality is usually what is meant by
``$Q$-learning\index{Q-learning@$Q$-learning} converges''%
~\citep{WD:1992Qlearning} or
``TD learning converges''~\citep{Sutton:1988}.
\citet{Orseau:2010,Orseau:2013} showed that
AIXI is not asymptotically optimal.
Yet asymptotic optimality in the general setting can be achieved
through optimism~\citep{SH:12optopt,SH:2012optimistic,SH:2015opt},
Thompson sampling~(\autoref{ssec:asymptotic-optimality-TS}), or
an extra exploration component on top of AIXI~\citep[Ch.~5]{Lattimore:2013}.

In our setting,
learning the environment
does not just involve learning a fixed finite set of parameters;
the real world is too complicated to fit into a template.
Therefore we fall back on the \emph{nonparametric}\index{nonparametric} approach
where we start with an infinite but countable class of candidate environments.
Our only assumption is that the true environment
is contained in this class~(the \emph{realizable case}\index{realizable case}).
As long as this class of environments is large enough
(such as for the class of all computable environments),
this assumption is rather weak.

\section{Contribution}
\label{sec:contribution}

The goal of this thesis is not to increase AI capability.
As such, we are not trying to improve on the state of the art,
and we are not trying to derive practical algorithms.
Instead, the emphasis of this thesis is
to further our \emph{understanding} of general reinforcement learning
and thus strong AI.
How a future implementation of strong AI will actually work
is in the realm of speculation at this time.
Therefore we should make as few and as weak assumptions as possible.

We disregard computational constraints
in order to focus on the fundamental underlying problems.
This is unrealistic, of course.
With unlimited computation power many traditional AI problems become trivial:
playing chess, Go, or backgammon
can be solved by exhaustive expansion of the game tree.
But the general RL problem does not become trivial:
the agent has to learn the environment and
balance between exploration and exploitation%
\index{exploration vs.\ exploitation}.
That being said,
the algorithms that we study do have a relationship with
algorithms being used in practice and
our results can and should educate implementation.

On a high level,
our insights can be viewed from three different perspectives.
\begin{itemize}
\item \emph{Philosophically.}
	Concisely, our understanding of strong AI can be summarized as follows.
	\begin{equation}\label{eq:learning+acting}
	\text{intelligence} ~=~ \text{learning} ~+~ \text{acting}
	\end{equation}
	Here, \emph{intelligence}\index{intelligence} refers to an agent that
	optimizes towards some goal\index{goal}
	in accordance with the definition by \citet{LH:2007int}.
	For \emph{learning} we distinguish two~(very related) aspects:
	\begin{inparaenum}[(1)]
	\item arriving at accurate beliefs about the future and
	\item making accurate predictions about the future.
	\end{inparaenum}
	Of course, the former implies the latter:
	if you have accurate beliefs,
	then you can also make good predictions.
	For RL accurate beliefs is what we care about
	because they enable us to plan for the future.
	Learning is a passive process that only observes the data
	and does not interfere with its generation.
	In particular, learning does not require a goal.
	With \emph{acting} we mean the selection of actions in pursuit of some goal.
	This goal can be reward maximization as in reinforcement learning,
	understanding the environment as for knowledge-seeking agents, or
	something else entirely.
	Together they enable an agent to learn the environment's behavior
	in response to itself~(on-policy learning) and
	to choose a policy that furthers its goal.
	We discuss the formal aspects of learning in \autoref{cha:learning}
	and some approaches to acting in \autoref{cha:acting}.

	Given infinite computational resources, learning is easy
	and Solomonoff induction provides a complete theoretical solution.
	However, acting is not straightforward.
	We show that in contrast to popular belief, AIXI,
	the natural extension of Solomonoff induction to reinforcement learning,
	does not provide the objectively best answer to this question.
	We discuss some alternatives and their problems in \autoref{cha:optimality}.
	Unfortunately, the general question of how to act optimally remains open.

	AIXI$tl$\index{AIXItl}~\citep[Ch.~7.2]{Hutter:2005} is
	often mentioned as a computable approximation to AIXI.
	But AIXI$tl$ does not converge to AIXI in the limit.
	Inspired by Hutter search~\citep{Hutter:2002search},
	it relies on an automated theorem prover
	to find the provably best policy
	computable in time $t$ with a program of length $\leq l$.
	In contrast to AIXI,
	which only requires the choice of universal Turing machine,
	proof search requires an axiom system
	that must not be too weak or too strong.
	In \autoref{ssec:Gödel-prior} we discuss some of the
	problems with AIXI$tl$.
	Moreover, in \autoref{cor:complexity-eps-aixi} we show that
	$\varepsilon$-optimal AIXI is limit computable,
	which shows that AIXI can be computably approximated
	by running this algorithm for a fixed number of time steps
	or until a timeout is reached.
	While neither AIXI$tl$ nor this AIXI approximation algorithm
	are practically feasible,
	the latter is a better example for a computable strong AI.

	In our view,
	AIXI should be taken as a descriptive
	rather than prescriptive model.
	It is descriptive as an abstraction from
	an actual implementation of strong AI
	where we ignore all the details of the learning algorithm and
	the computational approximations of choosing how to act.
	It should not be viewed as a prescription of how strong AI should be built
	and AIXI approximations~\citep{VNHUS:2011,VBHCD:2014}
	are easily outperformed by neural-network-based approaches%
	~\citep{MKSRV+:2015deepQ}.

\item \emph{Mathematically.}
	Some of the proof techniques we employ are novel and could be used to
	analyze other algorithms.
	Examples include the proofs for the lower bounds on
	the computability results~(\autoref{ssec:lower-bounds}) and
	to a lesser extent the upper bounds~(\autoref{ssec:upper-bounds}),
	which should work analogously for a wide range of algorithms.
	Furthermore, the proof of the asymptotic optimality of Thompson sampling%
	~(\autoref{thm:Thompson-sampling-aoim})
	brings together a variety of mathematical tools
	from measure theory, probability theory, and stochastic processes.

	Next,
	the \emph{recoverability assumption}\index{recoverability}%
	~(\autoref{def:recoverability})
	is a novel technical assumption on the environment akin to
	ergodicity and weak communication in finite-state environments.
	It is more general, yet mathematically simple and
	works for arbitrary environments.
	This assumption turns out to be what we need to prove the connection
	from asymptotic optimality to sublinear regret in \autoref{sec:regret}.

	Moreover,
	we introduce the use of the recursive
	instead of the iterative value function%
	~(\autoref{sec:iterative-value-function}).
	The iterative value function is the natural extension of expectimax search
	to the sequential setting and
	was originally used by \citet[Sec.~5.5]{Hutter:2005}.
	Yet it turned out to be an incorrect and inconvenient definition:
	it does not correctly maximize expected rewards%
	~(\autoref{prop:iterative-AINU-is-not-a-reward-maximizer}) and
	it is not limit computable~(\autoref{thm:iterative-AINU-is-Sigma3-hard} and
	\autoref{thm:iterative-eps-AIXI-is-Pi2-hard}).
	However, this is only a minor technical correction.

	Finally, this work raises new mathematically intriguing questions
	about the properties of reflective oracles~(\autoref{sec:reflective-oracles}).

\item \emph{Practically.}
	One insight from this thesis is regarding the effective horizon.
	In practice geometric discounting is ubiquitous which has a
	constant effective horizon.
	However, when facing a finite horizon problem or an episodic task,
	sometimes the effective horizon changes.
	One lesson from our result on Thompson sampling%
	~(\autoref{ssec:asymptotic-optimality-TS} and \autoref{sec:regret})
	is that you should explore for an effective horizon
	instead of using $\eps$-greedy.
	While the latter exploration method is often used in practice,
	it has proved ineffective in environments with delayed rewards%
	~(see \autoref{ssec:DQN}).

	Furthermore,
	our application of reinforcement learning results to game theory
	in \autoref{cha:grain-of-truth-problem}
	reinforces this trend to solve game theory problems%
	~\citep[and many more]{Tesauro:1995,BV:2001,BBDS:2008,DeepMind:2016Go,%
	Heinrich:2016,FAdFW:2016}.
	In particular, the approximation algorithm for reflective oracles%
	~(\autoref{ssec:lc-reflective-oracle-proof})
	could guide future applications for computing Nash equilibria%
	~(see also \citealp{FTC:2015reflection}).
\end{itemize}

On a technical level,
we advance the theory of general reinforcement learning.
In its center is the Bayesian reinforcement learning agent AIXI.
AIXI is meant as an answer to the question of how to do general RL
disregarding computational constraints.
We analyze the computational complexity of AIXI and related agents in
\autoref{cha:computability}
and show that
even with an infinite horizon
AIXI can be computationally approximated
with a regular Turing machine~(\autoref{ssec:upper-bounds}).
We also derive corresponding lower bounds for most of our upper bounds%
~(\autoref{ssec:lower-bounds}).

\autoref{cha:optimality}
is about notions of optimality in general reinforcement learning.
We dispel AIXI's status as
the gold standard for reinforcement learning.
\citet{Hutter:2002} showed that AIXI is
Pareto optimal, balanced Pareto optimal, and self-optimizing.
\citet{Orseau:2013} established that
AIXI does not achieve asymptotic optimality in all computable environments
(making the self-optimizing result inapplicable to this general environment class).
In \autoref{sec:Pareto-optimality} we show that every policy is Pareto optimal
and in \autoref{sec:Bayes-optimality} we show that
balanced Pareto optimality is highly subjective,
depending on the choice of the prior;
bad choices for priors are discussed in \autoref{sec:bad-priors}.
Notable is the \emph{dogmatic prior}\index{prior!dogmatic} that
locks a Bayesian reinforcement learning agent into a particular~(bad) policy
as long as this policy yields some rewards.
Our results imply that there are no known
nontrivial and non-subjective optimality results for AIXI.
We have to regard AIXI as a \emph{relative} theory of intelligence.
More generally, our results imply that
general reinforcement learning is difficult
\emph{even when disregarding computational costs}.

But this is not the end to Bayesian methods in general RL.
We show in \autoref{sec:asymptotic-optimality} that
a Bayes-inspired algorithm called \emph{Thompson sampling}%
\index{Thompson sampling}
achieves asymptotic optimality.
Thompson sampling, also known as \emph{posterior sampling}%
\index{posterior!sampling|see {Thompson sampling}} or
the \emph{Bayesian control rule}%
\index{Bayesian!control rule|see {Thompson sampling}}
repeatedly draws one environment from the posterior distribution and then
acts as if this was the true environment for a certain period of time
(depending on the discount function\index{discount!function}).
Moreover, given a recoverability assumption on the environment
and some mild assumptions on the discount function,
we show in \autoref{sec:regret}
that Thompson sampling achieves sublinear regret.

Finally,
we tie these results together to solve an open problem in game theory:
When acting in a multi-agent environment with other Bayesian agents,
each agent needs to assign positive prior probability
to the other agents' actual policies%
~(they need to have a \emph{grain of truth}\index{grain of truth}).
Finding a reasonably large class of policies
that contains the Bayes optimal policies with respect to this class
is known as the \emph{grain of truth problem}\index{grain of truth!problem}%
~\citep[Q.~5j]{Hutter:2009open}.
Only small classes are known to have a grain of truth
and the literature contains several related impossibility results%
~\citep{Nachbar:1997,Nachbar:2005,FY:2001impossibility}.
Moreover,
while AIXI assumes the environment to be computable,
our computability results on AIXI confirm that it is incomputable%
~(\autoref{thm:AIXI-is-not-computable} and
\autoref{thm:eps-AIXI-is-Sigma1-hard}).
This asymmetry elevates AIXI above its environment computationally,
and prevents the environment from containing other AIXIs.

In \autoref{cha:grain-of-truth-problem}
we give a formal and general solution to the grain of truth problem:
we construct a class of policies
that avoid this asymmetry.
This class contains all computable policies as well as
Bayes optimal policies for every lower semicomputable prior over the class.
When the environment is unknown,
our dogmatic prior from \autoref{sec:bad-priors} makes
Bayes optimal agents fail to act optimally even asymptotically.
However, our convergence results on Thompson sampling%
~(\autoref{ssec:asymptotic-optimality-TS}) imply that
Thompson samplers converge to play $\varepsilon$-Nash equilibria
in arbitrary unknown computable multi-agent environments.
While these results are purely theoretical,
we use techniques from \autoref{cha:computability}
to show that they can be computationally approximated arbitrarily closely.


\section{Thesis Outline}
\label{sec:thesis-outline}

\begin{table}[t]
\begin{center}
\begin{tabular}{ll}
\toprule
Chapter & Publication(s) \\
\midrule
\autoref{cha:introduction} & - \\
\autoref{cha:preliminaries} & - \\
\autoref{cha:learning} & with links to \citet{LH:2014martosc,LH:2015ravens,FLH:2016speed} \\
\autoref{cha:acting} & - \\
\autoref{cha:optimality} & \citet{LH:2015priors,LLOH:2016Thompson} \\
\autoref{cha:computability} & \citet{LH:2015computability2,LH:2015computability,LH:2016computability} \\
\autoref{cha:grain-of-truth-problem} & \citet{LTF:2016} \\
\autoref{cha:conclusion} & - \\
\hyperref[cha:measures-martingales]{Appendix A} & \citet{LH:2014martoscx} \\
\bottomrule
\end{tabular}
\end{center}
\caption[List of publications by chapter]{%
List of publications by chapter.
}
\label{tab:publications-chapters}
\end{table}

This thesis is based on the papers
\citet{LH:2014martosc,LH:2014martoscx,LH:2015computability,LH:2015computability2,
LH:2015priors,LH:2016computability,LLOH:2016Thompson,LTF:2016}.
During my PhD, I was also involved in the publications
\citet{LH14TACAS,LH14WST,LH15LMCS,HDLMP15,HDGLMSP16} based on
my research in termination analysis~(in collaboration with Matthias Heizmann),
\citet{DL:2015happiness}~(co-authored with Mayank Daswani in equal parts),
\citet{ELH:2015sdt}~(co-authored with Tom Everitt in equal parts),
\citet{FLH:2016speed}~(written by Daniel Filan as part of his honour's thesis
supervised by Marcus Hutter and me).
\citet{LH:2016computability} is still under review.
\citet{LH:2014martosc,LH:2015ravens} are tangential to this thesis' main thrust,
so the results are mentioned only in passing.
A list of papers written during my PhD is given in
\autoref{tab:publications} on page~\pageref*{tab:publications},
with a corresponding chapter outline in \autoref{tab:publications-chapters}.
The core of our contribution is found in chapters
\ref{cha:optimality}, \ref{cha:computability}, and
\ref{cha:grain-of-truth-problem}.

\begin{table}[t]
\begin{center}
\small
\begin{NoHyper}
\begin{enumerate}[{[}1{]}]
\item \bibentry{LH:2014martosc}
\item \bibentry{LH15LMCS}
\item \bibentry{DL:2015happiness}
\item \bibentry{ELH:2015sdt}
\item \bibentry{LH:2015computability}
\item \bibentry{LH:2015computability2}
\item \bibentry{LH:2015priors}
\item \bibentry{LH:2015ravens}
\item \bibentry{HDLMP15}
\item \bibentry{HDGLMSP16}
\item \bibentry{FLH:2016speed}
\item \bibentry{LH:2016computability}
\item \bibentry{LLOH:2016Thompson}
\item \bibentry{LTF:2016}
\item \bibentry{LH16GNTA}
\end{enumerate}
\end{NoHyper}
\end{center}
\caption[List of publications]{%
List of publications.
}
\label{tab:publications}
\end{table}

Every thesis chapter starts with a quote.
In case this is not blatantly obvious: \emph{these are false quotes},
a desperate attempt to make the thesis less dry and humorless.
None of the quotes were actually stated by the person they are attributed to
(according to our knowledge).


\chapter{Preliminaries}
\label{cha:preliminaries}

\falsequote{Leonhard Euler}{Mathematics is a waste of time.}

This chapter establishes the notation and background material
that is used throughout this thesis.
\autoref{sec:measure-theory} is about probability and measure theory,
\autoref{sec:stochastic-processes} is about stochastic processes,
\autoref{sec:information-theory} is about information theory, and
\autoref{sec:AIT} is about algorithmic information theory.
We defer the formal introduction to reinforcement learning
to \autoref{cha:acting}.
Additional preliminary notation and terminology is also established
in individual chapters wherever necessary.
A \hyperref[cha:notation]{list of notation} is provided
in the appendix on page~\pageref*{cha:notation}.

Most of the content from this chapter can be found in
standard textbooks and reference works.
We recommend to consult
\citet{Wasserman:2004} on statistics,
\citet{Durrett:2010} on probability theory and stochastic processes,
\citet{CT:2006} on information theory,
\citet{LV:2008} on algorithmic information theory,
\citet{RN:2010} on artificial intelligence,
\citet{Bishop:2006} and \citet{HTF:2009} on machine learning,
\citet{SB:1998} on reinforcement learning, and
\citet{Hutter:2005} and \citet{Lattimore:2013} on general reinforcement learning.

We understand definitions to follow natural language;
e.g., when defining the adjective `continuous',
we define at the same time the noun `continuity' and the adverb `continuously'
wherever appropriate.

\paragraph{Numbers.}
$\mathbb{N} := \{ 1, 2, 3, \ldots \}$ denotes
the set of natural numbers (starting from $1$),
$\mathbb{Q} := \{ \pm p/q \mid p \in \mathbb{N} \cup \{ 0 \}, q \in \mathbb{N} \}$ denotes
the set of rational numbers, and
$\mathbb{R}$ denotes the set of real numbers.
For two real numbers $r_1, r_2$,
the set $[r_1, r_2] := \{ r \in \mathbb{R} \mid r_1 \leq r \leq r_2 \}$ denotes
the closed interval with end points $r_1$ and $r_2$;
the sets $(r_1, r_2] := [r_1, r_2] \setminus \{ r_1 \}$ and
$[r_1, r_2) := [r_1, r_2] \setminus \{ r_2 \}$
denote half-open intervals; 
the set $(r_1, r_2) := [r_1, r_2] \setminus \{ r_1, r_2 \}$ denotes
an open interval.

\paragraph{Strings.}
Fix $\X$ to be a finite nonempty set,
called \emph{alphabet}\index{alphabet|textbf}.
We assume that $\X$ contains at least two distinct elements.
The set $\X^* := \bigcup_{n=0}^\infty \X^n$ is
the set of all finite strings over the alphabet $\X$,
the set $\X^\infty$ is
the set of all infinite strings
over the alphabet $\X$, and
the set $\X^\sharp := \X^* \cup \X^\infty$ is their union.
The empty string is denoted by $\epsilon$, not to be confused
with the small positive real number $\varepsilon$.
Given a string $x \in \X^\sharp$, we denote its length by $|x|$.
For a (finite or infinite) string $x$ of length $\geq k$,
we denote with $x_k$ the $k$-th character of $x$,
with $x_{1:k}$ the first $k$ characters of $x$,
and with $x_{<k}$ the first $k - 1$ characters of $x$.
The notation $x_{1:\infty}$ stresses that $x$ is an infinite string.
We use $x \sqsubseteq y$ to denote that $x$ is a prefix of $y$, i.e.,
$x = y_{1:|x|}$.
Our examples often (implicitly) involve the binary alphabet $\{ 0, 1 \}$.
In this case we define the functions
$\mathrm{ones}, \mathrm{zeros}: \X^* \to \mathbb{N}$
that count the number of ones and zeros in a string respectively.

\paragraph{Computability.}
\index{lower semicomputable}%
A function $f: \X^* \to \mathbb{R}$ is
\emph{lower semicomputable}\index{lower semicomputable|textbf} iff
the set $\{ (x, q) \in \X^* \times \mathbb{Q} \mid f(x) > q \}$
is recursively enumerable.
\index{computable}%
If $f$ and $-f$ are lower semicomputable,
then $f$ is called \emph{computable}\index{computable|textbf}.
See \autoref{ssec:computability-of-real-valued-functions}
for more computability definitions.

\paragraph{Asymptotic Notation.}
Let $f, g: \mathbb{N} \to \mathbb{R}_{\geq 0}$.
We use $f \in O(g)$ to denote that
there is a constant $c$ such that $f(t) \leq c g(t)$
for all $t \in \mathbb{N}$.
We use $f \in o(g)$ to denote that
$\limsup_{t \to \infty} f(t) / g(t) = 0$.
For functions on strings $P, Q: \X^* \to \mathbb{R}$
we use $Q \timesgeq P$ to denote that
there is a constant $c > 0$ such that
$Q(x) \geq cP(x)$ for all $x \in \X^*$.
We also use $Q \timesleq P$ for $P \timesgeq Q$ and
$Q \timeseq P$ for $Q \timesleq P$ and $P \timesleq Q$.
Note that $Q \timeseq P$ \emph{does not imply} that
there is a constant $c$ such that $Q(x) = cP(x)$ for all $x \in \X^*$.
For a sequence $(a_t)_{t \in \mathbb{N}}$
with limit $\lim_{t \to \infty} a_t = a$
we also write $a_t \to a$ as $t \to \infty$.
If no limiting variable is provided,
we mean $t \to \infty$ by convention.

\paragraph{Other Conventions.}
Let $A$ be some set.
We use $\#A$ to denote the cardinality of the set $A$,
i.e., the number of elements in $A$,
and $2^A$ to denote the power set\index{power set} of $A$,
i.e., the set of all subsets of $A$.
We use $\log$ to denote the binary logarithm and
$\ln$ to denote the natural logarithm.

\section{Measure Theory}
\label{sec:measure-theory}

For a countable set $\Omega$,
we use $\Delta\Omega$ to denote
the set of probability distributions over $\Omega$.
If $\Omega$ is uncountable
(such as the set of all infinite strings $\X^\infty$),
we need to use the machinery of measure theory.
This section provides a concise introduction to measure theory;
see \citet{Durrett:2010} for an extensive treatment.

\begin{definition}[$\sigma$-algebra]
\label{def:sigma-algebra}
\index{s-algebra@$\sigma$-algebra|textbf}
Let $\Omega$ be a set.
The set $\mathcal{F} \subseteq 2^\Omega$ is a
\emph{$\sigma$-algebra over $\Omega$} iff
\begin{enumerate}[(a)]
\item $\Omega \in \mathcal{F}$,
\item $A \in \mathcal{F}$ implies $\Omega \setminus A \in \mathcal{F}$, and
\item for any countable number of sets $A_0, A_1, \ldots, \in \mathcal{F}$,
  the union $\bigcup_{i \in \mathbb{N}} A_i \in \mathcal{F}$.
\end{enumerate}
\end{definition}

For a set $\mathcal{A} \subseteq 2^\Omega$,
we define $\sigma(\mathcal{A})$ to be
the smallest (with respect to set inclusion)
$\sigma$-algebra containing $\mathcal{A}$.

For the real numbers, the default $\sigma$-algebra
(used implicitly)
is the \emph{Borel $\sigma$-algebra $\mathcal{B}$}%
\index{s-algebra@$\sigma$-algebra!Borel|textbf}
generated by the open sets of the usual topology. Formally,
$\mathcal{B} := \sigma(\{ (a, b) \mid a, b \in \mathbb{R} \})$.

A set $\Omega$ together with a $\sigma$-algebra $\mathcal{F}$
forms a \emph{measurable space}\index{measurable!space|textbf}.
The sets from the $\sigma$-algebra $\mathcal{F}$
are called \emph{measurable sets}\index{measurable!set}.
A function $f:\Omega_1 \to \Omega_2$ between two measurable spaces
is called \emph{measurable} iff
any preimage of an (in $\Omega_2$) measurable set is
measurable (in $\Omega_1$).

\begin{definition}[Probability Measure]
\label{def:probability-measure}
\index{measure!probability|textbf}
Let $\Omega$ be a measurable space with $\sigma$-algebra $\mathcal{F}$.
A \emph{probability measure} on the space $\Omega$ is
a function $\mu: \mathcal{F} \to [0,1]$ such that
\begin{enumerate}[(a)]
\item $\mu(\Omega) = 1$ (\emph{normalization}), and
\item $\mu(\bigcup_{i \in \mathbb{N}} A_i) = \sum_{i \in \mathbb{N}} \mu(A_i)$
	for any collection $\{ A_i \mid i \in \mathbb{N} \} \subseteq \mathcal{F}$
	that is pairwise disjoint (\emph{$\sigma$-additivity}).
\end{enumerate}
\end{definition}

A probability measure $\mu$ is \emph{deterministic}%
\index{measure!deterministic|textbf} iff
it assigns all probability mass to a single element of $\Omega$,
i.e., iff there is an $x \in \Omega$ with $\mu(\{ x \}) = 1$.

We define the \emph{conditional probability} $\mu(A \mid B)$
for two measurable sets $A, B \in \mathcal{F}$ with $\mu(B) > 0$ as
$\mu(A \mid B) := \mu(A \cap B) / \mu(B)$.

\begin{definition}[Random Variable]
\label{def:random-variable}
\index{random variable|textbf}
Let $\Omega$ be a measurable space with probability measure $\mu$.
A \emph{(real-valued) random variable}
is a measurable function $X: \Omega \to \mathbb{R}$.
\end{definition}

We often (but not always) denote random variables with uppercase Latin letters.

Given a $\sigma$-algebra $\F$, a probability measure $P$ on $\F$, and
an $\F$-measurable random variable $X$,
the \emph{conditional expectation}\index{conditional expectation|textbf}
$\EE[ X \mid \F ]$ of $X$ given $\F$
is a random variable $Y$ such that
\begin{inparaenum}[(1)]
\item $Y$ is $\F$-measurable and
\item $\int_A X dP = \int_A Y dP$ for all $A \in \F$.
\end{inparaenum}
The conditional expectation exists and is unique up to a set of $P$-measure $0$%
~\citep[Sec.~5.1]{Durrett:2010}.
Intuitively, if $\F$ describes the information we have at our disposal, then
$\EE[ X \mid \F ]$ denotes the expectation of $X$ given this information.

We proceed to define the $\sigma$-algebra on $\X^\infty$
(the $\sigma$-algebra on $\X^\sharp$ is defined analogously).
For a finite string $x \in \X^*$,
the \emph{cylinder set}\index{cylinder set}
\[
\Gamma_x := \{ xy \mid y \in \X^\infty \}
\]
is the set of all infinite strings of which $x$ is a prefix.
Furthermore, we fix the $\sigma$-algebras
\[
\F_t := \sigma \big(\{ \Gamma_x \mid x \in \X^t \} \big)
\qquad\text{ and }\qquad
\Foo := \sigma \left( \bigcup_{t=1}^\infty \F_t \right).
\]
The sequence $(\F_t)_{t \in \mathbb{N}}$ is a \emph{filtration}%
\index{filtration|textbf}:
from $\Gamma_x = \bigcup_{a \in \X} \Gamma_{xa}$
follows that $\F_t \subseteq \F_{t+1}$ for every $t \in \mathbb{N}$,
and all $\F_t \subseteq \Foo$ by the definition of $\Foo$.

For our purposes, the $\sigma$-algebra $\F_t$ means
`all symbols up to and including time step $t$.'
So instead of conditioning an expectation on $\F_t$,
we can just as well condition it on the sequence $x_{1:t}$ drawn at time $t$.
Hence we write $\EE[ X \mid x_{1:t}]$ instead of $\EE[ X \mid \F_t]$.
Moreover, for conditional probabilities we also write
$Q(x_t \mid x_{<t})$ instead of $Q(x_{1:t} \mid x_{<t})$.

In the context of probability measures,
a measurable set $E \in \Foo$ is also called an
\emph{event}\index{event|textbf}.
The event $E^c := \X^\infty \setminus E$ denotes the complement of $E$.
In case the event $E$ is defined by a predicate $Q$
dependent on the random variable $X$, $E = \{ x \in \Omega \mid Q(X(x)) \}$,
we also use the shorthand notation
\[
P[ Q(X) ] := P(\{ x \in \Omega \mid Q(X(x)) \}) = P(E).
\]

We assume all sets to be measurable\index{measurable!set}%
\index{measurable!set|textbf};
when we write $P(A)$ for some set $A \subseteq \X^\infty$,
we understand implicitly that $A$ be measurable.
This is not true: not all subsets of $\X^\infty$ are measurable
(assuming the axiom of choice).
While we choose to do this for readability purposes,
note that under some axioms compatible with Zermelo-Fraenkel set theory,
notably the axiom of determinacy,
all subsets of $\X^\infty$ are measurable.

\section{Stochastic Processes}
\label{sec:stochastic-processes}

This section introduces some notions about sequences of random variables.

\begin{definition}[Stochastic Process]
\label{def:stochastic-process}\index{stochastic process|textbf}
$(X_t)_{t \in \mathbb{N}}$ is a \emph{stochastic process} iff
$X_t$ is a random variable for every $t \in \mathbb{N}$.
\end{definition}

A stochastic process $(X_t)_{t \in \mathbb{N}}$ is \emph{nonnegative} iff
$X_t \geq 0$ for all $t \in \mathbb{N}$.
The process is \emph{bounded} iff
there is a constant $c \in \mathbb{R}$ such that
$|X_t| \leq c$ for all $t \in \mathbb{N}$.

In the real numbers, a sequence $(z_t)_{t \in \mathbb{N}}$
converges if and only if it is a Cauchy sequence,
i.e., iff $|z_{t+1} - z_t| \to 0$ as $t \to \infty$.
For sequences of random variables
convergence is a lot more subtle and
there are several different notions of convergence.

\begin{definition}[Stochastic Convergence]
\label{def:stochastic-convergence}
Let $P$ be a probability measure.
A stochastic process $(X_t)_{t \in \mathbb{N}}$
\emph{converges to the random variable $X$}
\begin{itemize}
\item \emph{in $P$-probability}\index{convergence!in probability|textbf} iff
for every $\eps > 0$,
\[
P\big[ |X_t - X| > \eps \big] \to 0 \text{ as $t \to \infty$};
\]
\item \emph{in $P$-mean}\index{convergence!in mean|textbf} iff
\[
\EE_P \big[ | X_t - X| \big] \to 0 \text{ as $t \to \infty$};
\]
\item \emph{$P$-almost surely}\index{convergence!almost surely|textbf} iff
\[
P\left[ \lim_{t \to \infty} X_t = X \right] = 1.
\]
\end{itemize}
\end{definition}

Almost sure convergence and convergence in mean both imply
convergence in probability~\citep[Thm.~5.17]{Wasserman:2004}.
If the stochastic process is bounded,
then convergence in probability implies
convergence in mean~\citep[Thm.~5.19]{Wasserman:2004}.

A sequence of real numbers $(a_t)_{t \in \mathbb{N}}$ converges
\emph{in Cesàro average}\index{Cesàro average|textbf} to $a \in \mathbb{R}$ iff
$1/t \sum_{k=1}^t a_k \to a$ as $t \to \infty$.
The definition for sequences of random variables is analogous.

\begin{definition}[Martingale]
\label{def:martingale}
\index{martingale|textbf}
Let $P$ be a probability measure over $(\X^\infty, \Foo)$.
A stochastic process $(X_t)_{t \in \mathbb{N}}$ is a
\emph{$P$-supermartingale ($P$-submartingale)}
iff
\begin{enumerate}[(a)]
\item each $X_t$ is $\F_t$-measurable, and
\item $\EE[X_t \mid \F_s] \leq X_s$ ($\EE[X_t \mid \F_s] \geq X_s$)
	$P$-almost surely
	for all $s, t \in \mathbb{N}$ with $s < t$.
\end{enumerate}
A \emph{$P$-martingale} is a process that is both
a $P$-supermartingale and a $P$-submartingale.
\end{definition}

\begin{example}[Fair Gambling]
\label{ex:fair-gambling}
Suppose Mary bets on the outcome of a fair coin flip.
If she predicts correctly, her wager is doubled and otherwise it is lost.
Let $X_t$ denote Mary's wealth at time step $t$.
Since the game is fair,
$\EE[ X_{t+1} \mid \F_t ] = X_t$
where $\F_t$ represents the information available at time step $t$.
Hence $\EE[X_t] = X_1$,
so in expectation she never loses money regardless of her betting strategy.
\end{example}

For martingales the following famous convergence result
was proved by \citet{Doob:1953}. 

\begin{theorem}[{Martingale Convergence; \citealp[Thm.~5.2.9]{Durrett:2010}}]
\label{thm:martingale-convergence}
\index{convergence!martingale}\index{martingale}
If $(X_t)_{t \in \mathbb{N}}$ is a nonnegative supermartingale,
then it converges almost surely to a limit $X$ with $\EE[X] \leq \EE[X_1]$.
\end{theorem}

By \autoref{thm:martingale-convergence}
the martingale from \autoref{ex:fair-gambling}
representing Mary's wealth converges almost surely,
regardless of her betting strategy.
Either she refrains from betting at some point
(assuming she cannot place smaller and smaller bets)
or she cannot play anymore because her wealth is $0$.
Is there a lesson to learn here about gambling?

\section{Information Theory}
\label{sec:information-theory}

This section introduces the notions of \emph{entropy} and
two notions of distance between probability measures:
\emph{KL-divergence} and \emph{total variation distance}.

\begin{definition}[Entropy]
\label{def:entropy}
\index{entropy|textbf}
Let $\Omega$ be a countable set.
For a probability distribution $p \in \Delta\Omega$,
the \emph{entropy} of $p$ is defined as
\[
   \Ent(p)
:= -\sum_{\mathclap{x \in \Omega:\, p(x) > 0}} p(x) \log p(x).
\]
\end{definition}

\begin{definition}[KL-Divergence]
\label{def:KL-divergence}
\index{KL-divergence|textbf}
Let $P, Q$ be two measures and
let $m \in \mathbb{N}$ be a lookahead time step.
The \emph{Kullback-Leibler-divergence} (\emph{KL-divergence}) of $P$ and $Q$
between time steps $t$ and $m$
is defined as
\[
   \KL_m(P, Q \mid x_{<t}) \\
:= \sum_{x_{t:m} \in \X^{m-t+1}} P(x_{1:m} \mid x_{<t})
     \log \frac{P(x_{1:m} \mid x_{<t})}{Q(x_{1:m} \mid x_{<t})}.
\]
Moreover, we define $\KL_\infty(P, Q \mid x_{<t})
:= \lim_{m \to \infty} \KL_m(P, Q \mid x_{<t})$.
\end{definition}

KL-divergence is also known as
\emph{relative entropy}\index{entropy!relative|see {KL-divergence}}.
KL-divergence is always nonnegative by Gibbs' inequality%
\index{Gibb's inequality},
but it is not a distance since it is not symmetric.
If the alphabet $\X$ is finite,
then $\KL_m(P, Q \mid x)$ is always finite.
However, $\KL_\infty(P, Q \mid x)$ may be infinite.

\begin{definition}[Total Variation Distance]
\label{def:total-variation-distance}
\index{total variation distance|textbf}
Let $P, Q$ be two measures and
let $1 \leq m \leq \infty$ be a lookahead time step.
The \emph{total variation distance} between $P$ and $Q$
between time steps $t$ and $m$
is defined as
\[
   D_m(P, Q \mid x)
:= \sup_{A \subseteq \X^m} \Big|
     P(A \mid x_{<t}) - Q(A \mid x_{<t})
   \Big|.
\]
\end{definition}

Total variation distance is always bounded between $0$ and $1$
since $P$ and $Q$ are probability measures.
Moreover,
in contrast to KL-divergence total variation distance
satisfies the axioms of distance: symmetry ($D(P, Q) = D(Q, P)$),
identity of indiscernibles ($D(P, Q) = 0$ if and only if $P = Q$),
and the triangle inequality ($D(P, Q) + D(Q, R) \geq D(P, R)$).

The following lemma shows that
total variation distance can be used to bound
differences in expectation.

\begin{lemma}[Total Variation Bound on the Expectation]
\label{lem:expectation-total-variation}
\index{total variation distance}
For a random variable $X$ with $0 \leq X \leq 1$
and two probability measures $P$ and $Q$
\[
       \big| \EE_P[X] - \EE_Q[X] \big|
~\leq~ D(P, Q).
\]
\end{lemma}

KL-divergence and total variation distance
are linked by the following inequality.

\begin{lemma}[{Pinsker's inequality; \citealp[Lem.~2.5i]{Tsybakov:2008}}]
\label{lem:Pinsker}\index{Pinsker's inequality}
\index{KL-divergence}\index{total variation distance}
For all probability measures $P$ and $Q$ on $\X^\infty$,
for every $x \in \X^*$, and for every $m \in \mathbb{N}$
\[
     D_m(P, Q \mid x)
\leq \sqrt{\frac{1}{2} \KL_m(P, Q \mid x)}
\]
\end{lemma}

\section{Algorithmic Information Theory}
\label{sec:AIT}

A \emph{universal Turing machine}\index{Turing machine!universal|textbf} (UTM)
is a Turing machine that can simulate all other Turing machines.
Formally, a Turing machine $U$ is a UTM iff
for every Turing machine $T$ there is a binary string $p$
(called \emph{program}\index{program})
such that $U(p, x) = T(x)$ for all $x \in \X^*$,
i.e., the output of $U$ when run on $(p, x)$ is the same as
the output of $T$ when run on $x$.
We assume the set of programs on $U$ is prefix-free.
The \emph{Kolmogorov complexity $K(x)$}\index{Kolmogorov complexity|textbf}
of a string $x$ is the length of the shortest program on $U$
that prints $x$ and then halts:
\[
K(x) := \min \{ |p| \mid U(p) = x \}.
\]

A \emph{monotone Turing machine}\index{Turing machine!monotone|textbf}
is a Turing machine
with a one-way read-only input tape,
a one-way write-only output tape, and
a read/write work tape.
Monotone Turing machines sequentially read symbols from their input tape
and write to their output tape.
Interpreted as a function,
a monotone Turing machine $T$ maps a string $x$ to
the longest string that $T$ writes to the output tape
while reading $x$ and no more from the input tape%
~\citep[Ch.~4.5.2]{LV:2008}.

We also use $U$ to denote
a universal monotone Turing machine
(programs on the universal monotone Turing machine
do not have to be prefix-free).
The \emph{monotone Kolmogorov complexity $\Km(x)$}%
\index{Kolmogorov complexity!monotone|textbf} denotes
the length of the shortest program on the monotone machine $U$
that prints a string starting with $x$~\citep[Def.~4.5.9]{LV:2008}:
\begin{equation}\label{eq:Km}
\Km(x) := \min \{ |p| \mid x \sqsubseteq U(p) \}.
\end{equation}
Since monotone complexity does not require the machine to halt,
there is a constant $c$ such that $\Km(x) \leq K(x) + c$ for all $x \in X^*$.

The following notion of a (semi)measure
is particular to algorithmic information theory.

\begin{definition}[{Semimeasure; \citealp[Def.~4.2.1]{LV:2008}}]
\label{def:semimeasure}
\index{semimeasure|textbf}
\index{measure|textbf}
A \emph{semimeasure} over the alphabet $\X$ is
a function $\nu: \X^* \to [0,1]$ such that
\begin{enumerate}[(a)]
\item $\nu(\epsilon) \leq 1$, and
\item $\nu(x) \geq \sum_{a \in \X} \nu(xa)$ for all $x \in \X^*$.
\end{enumerate}
A semimeasure is a \emph{(probability) measure} iff
equalities hold in (a) and (b) for all $x \in \X^*$.
\end{definition}

Semimeasures are not probability measures
in the classical measure theoretic sense.
However,
semimeasures correspond canonically to
classical probability measures on the probability space
$\X^\sharp = \X^* \cup X^\infty$
whose $\sigma$-algebra is generated by the cylinder sets\index{cylinder set}%
~(\citealp[Ch.~4.2]{LV:2008} and \citealp{Hay:2007}).

Lower semicomputable\index{lower semicomputable} semimeasures
correspond naturally to
monotone Turing machines\index{Turing machine!monotone}%
~\citep[Thm.~4.5.2]{LV:2008}:
for a monotone Turing machine $T$,
the semimeasure $\lambda_T$ maps a string $x$ to
the probability that $T$ outputs something starting with $x$
when fed with fair coin flips as input (and vice versa).
Hence we can enumerate all lower semicomputable semimeasures
$\nu_1, \nu_2, \ldots$ by enumerating all monotone Turing machines.
We define the Kolmogorov complexity $K(\nu)$ of
a lower semicomputable semimeasure $\nu$ as
the Kolmogorov complexity of the index of $\nu$ in this enumeration.

We often mix the (semi)measures of algorithmic information theory
with concepts from probability theory.
For convenience, we identify a finite string $x \in \X^*$
with its cylinder set $\Gamma_x$.
Then $\nu(x)$ in the algorithmic information theory sense
coincides with $\nu(\Gamma_x)$ in the measure theory sense
if we use the identification of semimeasures with probability measures above.

\begin{example}[Lebesgue Measure]
\label{ex:Lebesgue-measure}\index{measure!Lebesgue|textbf}
\index{measure!uniform|see {Lebesgue measure}}
The \emph{Lebesgue measure} or \emph{uniform measure} is defined as
\[
\lambda(x) := (\#\X)^{-|x|}.
\qedhere
\]
\end{example}

The following definition turns a semimeasure into a measure,
preserving the predictive ratio $\nu(xa)/\nu(xb)$ for $a, b \in \X$.

\begin{definition}[Solomonoff Normalization]
\label{def:Solomonoff-normalization}
\index{Solomonoff!normalization|textbf}
The \emph{Solomonoff normalization} $\nu\norm$ of a semimeasure $\nu$
is defined as
$\nu\norm(\epsilon) := 1$ and
for all $x \in \X^*$ and $a \in \X$,
\begin{equation}\label{eq:normalization}
   \nu\norm(xa)
:= \nu\norm(x) \frac{\nu(xa)}{\sum_{b \in \X} \nu(xb)}.
\end{equation}
\end{definition}

By definition, $\nu\norm$ is a measure.
Moreover, $\nu\norm$ dominates $\nu$ according to the following lemma.

\begin{lemma}[$\nu\norm \geq \nu$]
\label{lem:nu-norm-dominates-nu}
$\nu\norm(x) \geq \nu(x)$ for all $x \in \X^*$ and all semimeasures $\nu$.
\end{lemma}
\begin{proof}
We use induction on the length of $x$:
if $x = \epsilon$ then $\nu\norm(\epsilon) = 1 = \nu(\epsilon)$,
and otherwise
\[
     \nu\norm(xa)
=    \frac{\nu\norm(x) \nu(xa)}{\sum_{b \in \X} \nu(xb)}
\geq \frac{\nu(x) \nu(xa)}{\sum_{b \in \X} \nu(xb)}
\geq \frac{\nu(x) \nu(xa)}{\nu(x)}
=    \nu(xa).
\]
The first inequality holds by induction hypothesis and
the second inequality uses the fact that $\nu$ is a semimeasure.
\end{proof}


\chapter{Learning}
\label{cha:learning}

\def\BR{{B\!R}}                         
\def\nBR{{\overline{B}\!R}}             
\def\BnR{{B\!\overline{R}}}             
\def\nBnR{{\overline{B}\!\overline{R}}} 

\falsequote{David Hume}{The problem of induction is essentially solved.}

\emph{Machine learning} refers to the process of
learning models of and/or making predictions about (large) sets of data points
that are typically independent and identically distributed (i.i.d.);
see \citet{Bishop:2006} and \citet{HTF:2009}.
In this chapter we do \emph{not} make the i.i.d.\ assumption.
Instead, we aim more generally at the
theoretical fundamentals of the \emph{sequence prediction problem}:
how will a sequence of symbols generated by an unknown stochastic process
be continued?
Given a finite string $x_{<t} = x_1 x_2 \ldots x_{t-1}$ of symbols,
what is the next symbol $x_t$?
How likely does a given property hold for the entire sequence $x_{1:\infty}$?
Arguably, any learning or prediction problem can be phrased in this fashion:
anything that can be stored on a computer can be turned into a sequence of bits.

We distinguish two major elements of learning.
First, the process of converging to accurate beliefs,
called \emph{merging}\index{merging}.
Second, the process of making accurate forecasts about the next symbol,
called \emph{predicting}\index{prediction}.
These two notions are not distinct:
if you have accurate beliefs about the unseen data,
then you can make good predictions,
but not necessarily vice versa~(see \autoref{ex:prediction-no-merging}).
We discuss different notions of merging in \autoref{sec:merging} and
state bounds on the prediction regret in \autoref{sec:predicting}.

In the general reinforcement learning problem we target in this thesis,
the environment is unknown and the agent needs to learn it.
The literature on non-i.i.d.\ learning has focused on
predicting individual symbols and bounds on the number of prediction errors%
~\citep{Hutter:2001error,Hutter:2005,CBL:2006},
and the results on merging are from the game theory literature%
~\citep{BD:1962,KL:1994,LS:1996}.
However, we argue that merging\index{merging}
is the essential property for general AI.
In order to make good decisions,
the agent needs to have accurate beliefs about what its actions will entail.
On a technical level,
merging leads to on-policy value convergence%
~(\autoref{ssec:on-policy-value-convergence}), the fact that
the agents learns to estimate the values for its own policy correctly.

The setup we consider is the \emph{realizable case}\index{realizable case}:
we assume that the data is generated by an unknown probability distribution
that belongs to a known (countable) class of distributions.
In contrast, the \emph{nonrealizable case}
allows no assumptions on the underlying process that generates the data.
A well-known approach to the nonrealizable case
is \emph{prediction with expert advice}\index{prediction!with expert advice}%
~\citep{CBL:2006},
which we do not consider here.
Generally, the nonrealizable case is harder,
but \citet{Ryabko:2011} argues that for some problems,
both cases coincide.

After introducing the formal setup in \autoref{sec:learning-setup},
we discuss several examples for learning distributions
and notions that relate the learning distribution
with the process generating the data in \autoref{sec:compatibility}.
In \autoref{sec:martingales} we connect these notions
to the theory of martingale processes.

\autoref{sec:learning-AIT} connects the results from the first sections
to the learning framework developed by \citet{Solomonoff:1964,Solomonoff:1978},
\citet{Hutter:2001error,Hutter:2005,Hutter:2007universal}, and
\citet{Schmidhuber:2002} (among others).
This framework relies on results from
algorithmic information theory and computability theory
to learn any computable distribution quickly and effectively.
It is incomputable~(see \autoref{sec:complexity-Solomonoff}),
but can serve as a gold standard for learning.

Most of this chapter echoes the literature.
We collect results from economics and computer science
that previously had not been assembled in one place.
We provide proofs that connect the various properties
(\autoref{prop:martingales-and-compatibility},
\autoref{prop:compatibility}, and \autoref{prop:merging}), and
we fill in a few gaps in the picture:
the prediction bounds for absolute continuity%
~(\autoref{ssec:prediction-absolute-continuity})
and the improved regret bounds for nonuniform measures%
~(\autoref{thm:expected-prediction-regret-nonuniform} and
\autoref{thm:prediction-regret-nonuniform}).
\autoref{sec:learning-summary} summarizes the results
in \autoref{tab:learning-overview}
on page~\pageref*{tab:learning-overview}
as well as \autoref{fig:learning-overview}
on page~\pageref*{fig:learning-overview}.

\section{Setup}
\label{sec:learning-setup}

For the rest of this chapter,
fix $P$ and $Q$ to be two probability measures over
the measurable space of infinite sequences $(\X^\infty, \Foo)$.
We think of $P$ as the \emph{true distribution}
from which the data sequence $x_{1:\infty}$ is drawn, and
of $Q$ as our belief distribution or learning algorithm.
In other words,
we use the distribution $Q$
to learn a string drawn from the distribution $P$.

Let $H$ denote a \emph{hypothesis}\index{hypothesis|textbf},
i.e., any measurable set from $\Foo$.
Our \emph{prior belief}\index{prior} in the hypothesis $H$ is $Q(H)$.
In each time step $t$,
we make one observation $x_t \in \X$.
Our \emph{history}\index{history} $x_{<t} = x_1 x_2 \ldots x_{t-1}$
is the sequence of all previous observations.
We update our belief
in accordance with Bayesian learning;
our \emph{posterior belief}\index{posterior} in the hypothesis $H$ is
\[
    Q(H \mid x_{1:t})
~=~ \frac{Q(H \cap x_{1:t})}{Q(x_{1:t})}.
\]
The observation $x_t$ \emph{confirms}\index{confirmation}
the hypothesis $H$ iff
$Q(H \mid x_{1:t}) > Q(H \mid x_{<t})$ (the belief in $H$ increases), and
the observation $x_t$ \emph{disconfirms}\index{disconfirmation}
the hypothesis $H$ iff
$Q(H \mid x_{1:t}) < Q(H \mid x_{<t})$ (the belief in $H$ decreases).
If $Q(H \mid x_{1:t}) = 0$, then $H$ is \emph{refuted}\index{refute}
or \emph{falsified}\index{falsify}.

\index{compatibility}
When we assign a prior belief of $0$ to a hypothesis $H$,
this means that we think that $H$ is impossible;
it is refuted from the beginning.
If $Q(H) = 0$, then the posterior $Q(H \mid x_{<t}) = 0$,
so no evidence whatsoever can change our mind that $H$ is impossible.
This is bad if the hypothesis $H$ is actually true.

To be able to learn
we need to make some assumptions on the learning distribution $Q$:
we need to have an open mind about anything that might actually happen,
i.e., $Q(H) > 0$ on any hypothesis $H$ with $P(H) > 0$.
This property is called \emph{absolute continuity}\index{absolute continuity}.
We discuss this and other notions of \emph{compatibility of $P$ and $Q$}%
\index{compatibility}
in \autoref{sec:compatibility}.

We motivate this chapter with the following example.

\begin{example}[{The Black Ravens; \citealp[Sec.~7.4]{RH:2011}}]
\label{ex:black-ravens}\index{black ravens}
If we live in a world in which all ravens are black,
how can we learn this fact?
Since at every time step we have observed only a finite subset of
the (possibly infinite) set of all ravens,
how can we confidently state anything about all ravens?

We formalize this problem in line with \citet[Sec.~7.4]{RH:2011}
and \citet{LH:2015ravens}.
We define two predicates, blackness $B$ and ravenness $R$.
There are four possible observations:
a black raven $\BR$,
a non-black raven $\nBR$,
a black non-raven $\BnR$, and
a non-black non-raven $\nBnR$.
Therefore our alphabet consists of
four symbols corresponding to each of the possible observations,
$\X := \{ \BR, \nBR, \BnR, \nBnR \}$.

We are interested in the hypothesis\index{hypothesis} `all ravens are black'.
Formally, it corresponds to the measurable set
\begin{equation}\label{eq:hypothesis-all-ravens-are-black}
     H
~:=~ \{ x \in \X^\infty \mid x_t \neq \nBR \;\forall t \}
~ =~ \{ \BR, \BnR, \nBnR \}^\infty,
\end{equation}
the set of all infinite strings
in which the symbol $\nBR$ does not occur.

If we observe a non-black raven, $x_t = \nBR$,
the hypothesis $H$ is refuted
since $H \cap x_{1:t} = \emptyset$
and this implies $Q(H \mid x_{1:t}) = 0$.
In this case, our inquiry regarding $H$ is settled.
The interesting case is when the hypothesis $H$ is in fact true ($P(H) = 1$),
i.e., $P$ does not generate any non-black ravens.
The property we desire is that in a world in which all ravens are black,
we arrive at this belief:
$P(H) = 1$ implies $Q(H \mid x_{<t}) \to 1$ as $t \to \infty$.
\end{example}

\section{Compatibility}
\label{sec:compatibility}
\index{compatibility}\index{measure!compatibility|see {compatibility}}

In this section
we define \emph{dominance}\index{dominance},
\emph{absolute continuity}\index{absolute continuity},
\emph{dominance with coefficients}\index{dominance!with coefficients},
\emph{weak dominance}\index{dominance!weak}, and
\emph{local absolute continuity}\index{absolute continuity!local},
in decreasing order of their strength.
These notions make the relationship
of the two probability measures $P$ and $Q$ precise.
We also give examples for various choices for the learning algorithm $Q$.

In our examples, we frequently rely on the following process.

\begin{example}[Bernoulli Process]
\label{ex:Bernoulli}\index{measure!Bernoulli}
Assume $\X = \{ 0, 1 \}$.
For a real number $r \in [0, 1]$ we define the
\emph{Bernoulli process with parameter $r$} as the measure
\[
   \Bernoulli(r)(x)
:= r^{\mathrm{ones}(x)}
    (1 - r)^{\mathrm{zeros}(x)}.
\]
Note that $\Bernoulli(1/2) = \lambda$,
the Lebesgue measure from \autoref{ex:Lebesgue-measure}.
\end{example}

\begin{definition}[Dominance]
\label{def:dominance}\index{dominance|textbf}
The measure $Q$ \emph{dominates} $P$
($Q \timesgeq P$)
iff there is a constant $c > 0$ such that
$Q(x) \geq c P(x)$ for all finite strings $x \in \X^*$.
\end{definition}

Dominance is also called \emph{having a grain of truth}\index{grain of truth}%
~(\citealp[Def.~2a]{LS:1996} and \citealp{KL:1993});
we discuss this property in the context of game theory
in \autoref{cha:grain-of-truth-problem}.

\begin{example}[Bayesian mixture]
\label{ex:Bayesian-mixture}\index{Bayesian!mixture|textbf}
Let $\M$ be a countable set of probability measures
on $(\X^\infty, \Foo)$ and
let $w \in \Delta\M$ be a prior over $\M$.
If $w(P) > 0$ for all $P \in \M$,
the prior $w$ is called \emph{positive}\index{prior!positive|textbf}
or \emph{universal}\index{prior!universal|see {prior!positive}}.
Then the Bayesian mixture $\xi := \sum_{P \in \M} w(P) P$
dominates each $P \in \M$.
\end{example}

The Bayesian mixture is a mathematically simple yet very powerful concept.
It is very easy to derive from a countable set of distributions,
and it has been considered extensively in the literature%
~\citep[\dots]{Solomonoff:1964,Jaynes:2003,Hutter:2005}.
\citet{Ryabko:2009} shows that
even for uncountably infinite classes,
if there are good predictors,
then a Bayesian mixture over a countable subclass
asymptotically also does well.

\begin{example}[Solomonoff Prior]
\label{ex:Solomonoff-prior}
\index{Solomonoff!prior|textbf}
\index{Turing machine!universal}
\citet{Solomonoff:1964} defines a distribution $M$ over $\X^\sharp$
that assigns to a string $x$
the probability that the universal monotone Turing machine $U$ outputs $x$
when fed with fair coin flips on the input tape.
Formally,
\begin{equation}\label{eq:def-M}
M(x) := \sum_{p:\, x \sqsubseteq U(p)} 2^{-|p|}
\end{equation}
where $p$ is a binary string.\footnote{%
We use the name \emph{Solomonoff prior} for both a distribution over
$\X^\infty$ and a distribution over a computably enumerable set $\M$.
Maybe $M$ should better be called
\emph{Solomonoff mixture}\index{Solomonoff!mixture}
to avoid confusion.
}
The function $M$ is a lower semicomputable semimeasure,
but not computable and not a measure~\citep[Lem.~4.5.3]{LV:2008};
see \autoref{sec:complexity-Solomonoff}
for the computability properties of $M$.
More importantly, $M$ dominates every lower semicomputable semimeasure%
~\citep[Thm.~4.5.1]{LV:2008}.

Solomonoff's prior $M$ has
a number of appealing philosophical properties.
In line with Ockham's razor\index{Ockham's razor}
it favors simple environments over complex ones:
Turing machines that have a short program on the UTM $U$
have a higher contribution in the sum \eqref{eq:def-M}.
In line with Epicurus' principle\index{Epicurus' principle}
it never discards possible explanations:
every program that produces the string $x$ contributes to the sum.
See \citet{RH:2011} for a discussion on
the philosophical underpinnings of Solomonoff's prior.
\end{example}

\citet{WSH:2011} show that
the Solomonoff prior $M$ can equivalently be defined as
a Bayesian mixture over all lower semicomputable semimeasures
with a prior $w(P) \propto 2^{-K(P)}$.
(If we use $w(P) = 2^{-K(P)}$ we get a \emph{semiprior}\index{semiprior}
because $\sum_{P \in \M} 2^{-K(P)}$ can be less than $1$.
This prior also carries the name
\emph{Solomonoff prior}\index{Solomonoff!prior}.)

\begin{definition}[Absolute Continuity]
\label{def:absolute-continuity}\index{absolute continuity|textbf}
The measure $P$ is \emph{absolutely continuous with respect to $Q$}
($Q \gg P$)
iff $Q(A) = 0$ implies $P(A) = 0$ for all measurable sets $A$.
\end{definition}

\begin{remark}[Absolute Continuity $\nRightarrow$ Dominance]
\label{rem:absolute-continuity-neq-dominance}
\index{absolute continuity}\index{dominance}
Absolute continuity is strictly weaker than dominance:
let $\X := \{ 0, 1 \}$ and
define a probability measure $P$ that assigns probability $2/3$ to $1$
and probability $1/3$ to $0$ until seeing the first $0$,
then $P$ behaves like the Lebesgue measure $\lambda$.
Formally,
\[
P(x_{1:t}) :=
\begin{cases}
\left(\tfrac{2}{3} \right)^t
	&\text{if } x_{1:t} = 1^t, \text{ and} \\
\left(\tfrac{2}{3} \right)^n \tfrac{1}{3} \lambda(x_{n+2:t})
	&\text{if } \exists n \geq 0.\; 1^n 0 \sqsubseteq x_{1:t}.
\end{cases}
\]
Since $\lambda(1^t)/P(1^t) = (3/4)^t \to 0$ as $t \to \infty$,
there is no constant $c$ such that
$\lambda(x)/P(x) > c > 0$ for all finite strings $x \in \X^*$,
hence $\lambda$ does not dominate $P$.
But $P$ is absolutely continuous with respect to $\lambda$
because $P$-almost surely we draw a $0$ eventually,
and then $P$ behaves like $\lambda$.
Hence $P$-almost surely $\lambda / P \not\to 0$.
The claim now follows from
\autoref{prop:martingales-and-compatibility}\ref{itm:martingale-absolute-continuity}.
\end{remark}

The idea of \autoref{rem:absolute-continuity-neq-dominance}
is to `punch a hole into $\lambda$' at the infinite string $1^\infty$.
This infinite string has probability $0$,
hence this hole does not break absolute continuity.
But it breaks dominance on this infinite string.
Analogously we could punch countably many holes into a probability measure
without breaking absolute continuity.

\begin{definition}[Weak Dominance]
\label{def:weak-dominance}\index{dominance!weak|textbf}
The measure $Q$ \emph{weakly dominates} $P$
($Q \timesgeq_W P$)
iff
\[
\lim_{t \to \infty} \frac{1}{t} \log \frac{Q(x_{1:t})}{P(x_{1:t})} = 0
\text{ with $P$-probability $1$}.
\]
\end{definition}

\citet[Rem.~8]{LS:1996} point out that for any $P$ and $Q$,
\[
\limsup_{t \to \infty} \frac{1}{t} \log \frac{Q(x_{1:t})}{P(x_{1:t})} \leq 0
\text{ $P$-almost surely},
\]
so crucial is whether the $\liminf$ is also $0$.

\begin{remark}[Weak Dominance]
\label{rem:weak-dominance}\index{dominance!weak}
The measure $Q$ weakly dominates $P$ if and only if
$P$-almost surely $\log (P(x) / Q(x)) \in o(t)$.
\end{remark}

\citet{RH:2007,RH:2008} consider the following definition.
It is analogous to \autoref{def:dominance},
except that the constant $c$ is allowed to depend on time.

\begin{definition}[{Dominance with Coefficients;
\citealp[Def.~2]{RH:2008}}]
\label{rem:dominance-with-coefficients}
\index{dominance!with coefficients|textbf}
The measure $Q$ \emph{dominates $P$ with coefficients $f$}
($Q \geq P/f$) iff
$Q(x) \geq P(x) / f(|x|)$ for all $x \in \X^*$.
\end{definition}

If $Q$ dominates $P$ with coefficients $f$ and
$f$ grows subexponentially ($f \in o(\exp)$),
then $Q$ weakly dominates $P$ by \autoref{rem:weak-dominance}.

\begin{example}[Speed Prior]
\label{ex:speed-prior}\index{prior!speed|textbf}
\citet{Schmidhuber:2002} defines a variant of Solomonoff's prior $M$
that penalizes programs by their running time, called the \emph{speed prior}.
Consider the speed prior
\[
S_{Kt}(x) :=
\sum_{p:\,x \sqsubseteq U(p)} \frac{2^{-|p|}}{t(U, p, x)}
\]
where $t(U, p, x)$ is the number of time steps the Turing machine $U$
takes to produce $x$ from the program $p$.
For any deterministic measure $P$ computable in time $q$
we have $S_{Kt}(x) \timesgeq P(x) / q(|x|)$.
Therefore $S_{Kt}$ dominates $P$ with coefficients $O(q)$.
If $q$ is a polynomial ($P$ is computable in polynomial time),
then it grows subexponentially and thus $S_{Kt}$ weakly dominates $P$.
\end{example}

The semimeasure loss $S_{Kt}(x) - \sum_{a \in \X} S_{Kt}(xa)$
in the speed prior is quite substantial:
since it takes at least $n$ steps to output a string of length $n$,
$M(x) \geq |x| S_{Kt}(x)$.

\begin{example}[Laplace Rule]
\label{ex:Laplace-rule}\index{Laplace rule|textbf}
The \emph{Laplace rule} $\rho_L$ is defined by
\[
\rho_L(x_t \mid x_{<t}) := \frac{\#\{ i < t \mid x_i = x_t \}}{t + \#\X}.
\]

For $\X = \{ 0, 1 \}$ and $r \in [0, 1]$
the measure $\rho_L$ dominates $\Bernoulli(r)$
with coefficients $f(t) = t^{-\#\X+1}$~\citep[Prop.~3]{RH:2008}.
\end{example}

\begin{definition}[Local Absolute Continuity]
\label{def:local-absolute-continuity}
\index{absolute continuity!local|textbf}
The measure $P$ is \emph{locally absolutely continuous with respect to $Q$}
($Q \gg_L P$)
iff $Q(x) = 0$ implies $P(x) = 0$
for all finite strings $x \in \X^*$.
\end{definition}

The notable difference
between local absolute continuity and absolute continuity is that
\autoref{def:absolute-continuity} talks about arbitrary measurable sets
while \autoref{def:local-absolute-continuity} only talks about finite strings.
The former is a much stronger property.

For example,
every measure is locally absolutely continuous with respect to
the Lebesgue measure
since $\lambda(x) > 0$ for all finite strings $x \in \X^*$.

Local absolute continuity is an extremely weak property.
If it is not satisfied, we have to be very careful when using $Q$ for prediction:
then there is a positive probability that
we have to condition on a probability zero event.

\begin{example}[The Minimum Description Length Principle;
\citealp{Gruenwald:2007}]
\label{ex:MDL}\index{MDL|textbf}
Let $\M$ be a countable set of probability measures
on $(\X^\infty, \Foo)$ and
let $K: \mathcal{M} \to [0, 1]$ be a function
such that $\sum_{P \in \M} 2^{-K(P)} \leq 1$ called \emph{regularizer}.
Following notation from \citet{Hutter:2009MDL},
we define for each $x \in \X^*$ the \emph{minimal description length} model as
\[
\MDL^x := \argmin_{P \in \M} \{ -\log P(x) + K(P) \}.
\]
$-\log P(x)$ is the (arithmetic) code length of $x$ given model $P$,
and $K(P)$ is a complexity penalty for $P$.
Given data $x \in \X^*$,
$\MDL^x$ is the measure $P \in \M$
that minimizes the total code length of data and model.

Note that the Lebesgue measure is not
locally absolutely continuous with respect to
the MDL distribution $Q(x) := \MDL^x(x)$:
for some $x \in \X^*$ the minimum description $P \in \M$ may assign
probability zero to a continuation $xy \in \X^*$.
\end{example}

\begin{remark}[{MDL is Inductively Inconsistent;
\citealp[Cor.~13]{LH:2014martosc}}]
\label{rem:MDL-is-inductively-inconsistent}
\index{MDL}
The MDL estimator for countable classes as defined in \autoref{ex:MDL}
is inductively inconsistent:
the selected model $P \in \M$ can change infinitely often
and thus the limit $\lim_{t \to \infty} \MDL^{x_{<t}}$ may not exist.
This can be a major obstacle for using MDL for prediction,
since the model used for prediction has to be changed over and over again,
incurring the corresponding computational cost.
\end{remark}

The following proposition establishes the relationship between
our notions of compatibility;
see also \autoref{fig:learning-overview}
on page~\pageref*{fig:learning-overview}.

\begin{proposition}[Relationships between Compatibilities] ~
\label{prop:compatibility}
\index{dominance}
\index{absolute continuity}
\index{dominance!weak}
\index{dominance!with coefficients}
\index{absolute continuity!local}
\begin{enumerate}[(a)]
\item\label{itm:dominance=>absolute-continuity}
	If $Q \timesgeq P$,
	then $Q \gg P$.
\item\label{itm:absolute-continuity=>weak-dominance}
	If $Q \gg P$,
	then $Q \timesgeq_W P$.
\item\label{itm:dominance=>dominance-with-coefficients}
	If $Q \timesgeq P$,
	then $Q$ dominates $P$ with coefficients $f$
	for a constant function $f$.
\item\label{itm:dominance-with-coefficients=>weak-dominance}
	If $Q$ dominates $P$ with coefficients $f$ and $f \in o(\exp)$,
	then $Q \timesgeq_W P$.
\item\label{itm:weak-dominance=>local-absolute-continuity}
	If $Q \timesgeq_W P$,
	then $Q \gg_L P$.
\end{enumerate}
\end{proposition}
\begin{proof}
\begin{enumerate}[(a)]  
\item From \autoref{prop:martingales-and-compatibility}
	(\ref{itm:martingale-dominance}) and
	(\ref{itm:martingale-absolute-continuity}).
\item From \autoref{prop:martingales-and-compatibility}\ref{itm:martingale-absolute-continuity} and \citet[Prop.~3a]{KL:1994}.
\item Follows immediately from the definitions.
\item From \autoref{rem:weak-dominance}.
\item Follows immediately from the definitions.
\qedhere
\end{enumerate}
\end{proof}

Note that the converse of \autoref{prop:compatibility}%
\ref{itm:dominance-with-coefficients=>weak-dominance} is false:
in \autoref{rem:absolute-continuity-neq-dominance} we defined a measure $P$ that
is absolutely continuous with respect to $\lambda$
(and hence is weakly dominated by $\lambda$),
but the coefficients for $P/\lambda$ grow exponentially on the string $1^t$.
This infinite string has $P$-probability $0$,
but dominance with coefficients demands
the inequality $Q \geq P/f$ to hold for all strings.

\begin{remark}[Local Absolute Continuity $\nRightarrow$ Absolute Continuity]
\label{rem:absolute-continuity}
\index{absolute continuity!local}
\index{absolute continuity}
Define $P := \Bernoulli(2/3)$ and $Q := \Bernoulli(1/3)$.
Both measures $P$ and $Q$ are nonzero on all cylinder sets\index{cylinder set}:
$Q(x) \geq 3^{-|x|} > 0$ and $P(x) \geq 3^{-|x|} > 0$
for every $x \in \X^*$.
Therefore
$Q$ is locally absolutely continuous with respect to $P$.
However, $Q$ is \emph{not} absolutely continuous with respect to $P$:
define
\[
A := \left\{ x \in \X^\omega
     \;\middle|\;
        \limsup_{t \to \infty} \tfrac{1}{t} \mathrm{ones}(x_{1:t})
          \leq \tfrac{1}{2}
     \right\}.
\]
The set $A$ is $\Foo$-measurable
since $A = \bigcap_{n = 1}^\infty \bigcup_{x \in U_n} \Gamma_x$ with
$U_n := \{ x \in \X^* \mid
	|x| \geq n \text{ and } \mathrm{ones}(x) \leq |x| / 2 \}$,
the set of all finite strings of length at least $n$
that have at least as many zeros as ones.
We have that $P(A) = 0$ and $Q(A) = 1$,
hence $Q$ is not absolutely continuous with respect to $P$.
\end{remark}

\section{Martingales}
\label{sec:martingales}\index{martingale}

The following two theorems
state the connection between
probability measures on infinite strings and martingales.
For two probability measures $P$ and $Q$
the quotient $Q/P$ is a nonnegative $P$-martingale
if $Q$ is locally absolutely continuous with respect to $P$.
Conversely, for every nonnegative $P$-martingale
there is a probability measure $Q \gg_L P$ such that
the martingale is $P$-almost surely a multiple of $Q/P$.

\begin{theorem}[{Measures $\mapsto$ Martingales;
\citealp[II§7 Ex.~3]{Doob:1953}}]
\label{thm:measure-martingale}
\index{absolute continuity!local}
\index{martingale}\index{measure}
Let $Q$ and $P$ be two probability measures on $(\X^\infty, \Foo)$ such that
$Q$ is locally absolutely continuous with respect to $P$.
Then the stochastic process $(X_t)_{t \in \mathbb{N}}$,
\[
X_t(x) := \frac{Q(x_{1:t})}{P(x_{1:t})}
\]
is a nonnegative $P$-martingale
with $\EE[X_t] = 1$.
\end{theorem}

\begin{theorem}[Martingales $\mapsto$ Measures]
\label{thm:martingale-measure}
\index{absolute continuity!local}
\index{martingale}\index{measure}
Let $P$ be a probability measure on $(\X^\infty, \Foo)$ and
let $(X_t)_{t \in \mathbb{N}}$ be a nonnegative $P$-martingale
with $\EE[X_t] = 1$.
There is a probability measure $Q$ on $(\X^\infty, \Foo)$
that is locally absolutely continuous with respect to $P$ and
for all $x \in \X^\infty$ and
all $t \in \mathbb{N}$ with $P(x_{1:t}) > 0$,
\[
X_t(x) = \frac{Q(x_{1:t})}{P(x_{1:t})}.
\]
\end{theorem}

The proofs for
\autoref{thm:measure-martingale} and \autoref{thm:martingale-measure}
are provided in the \hyperref[cha:measures-martingales]{Appendix}.

\begin{example}[The Posterior Martingale]
\label{ex:posterior-martingale}
\index{martingale}\index{black ravens}\index{hypothesis}
Suppose we are interested in a hypothesis $H \subseteq \X^\infty$
(such as the proposition `all ravens are black' in \autoref{ex:black-ravens}).
If $Q(H) = \sum_{P \in \M} w(P) P(H)$
is a Bayesian mixture over a set of probability distributions $\M$
with prior weights $w \in \Delta\M$~(see \autoref{ex:Bayesian-mixture}),
then the posterior belief
$Q(H \mid x) = \sum_{P \in \M} w(P \mid x) P(H \mid x)$.
The weights $w(P \mid x)$ are called \emph{posterior weights}\index{posterior},
and satisfy the identity
\begin{equation}\label{eq:posterior-weight}
w(P \mid x) = w(P) \frac{P(x)}{Q(x)}
\end{equation}
since
\begingroup
\allowdisplaybreaks
\begin{align*}
   Q(H \mid x)
&= \frac{Q(H \cap x)}{Q(x)} \\
&= \frac{1}{Q(x)} \sum_{P \in \M} w(P) P(H \cap x) \\
&= \sum_{P \in \M} w(P) \frac{P(x) P(H \cap x)}{Q(x) P(x)} \\
&= \sum_{P \in \M} w(P \mid x) P(H \mid x).
\end{align*}
\endgroup

According to \autoref{thm:measure-martingale}
the posterior weight $w(P \mid x)$ is a $Q$-martingale with expectation $w(P)$.
In particular, this means that
the posterior weights converge $Q$-almost surely
by the martingale convergence theorem~(\autoref{thm:martingale-convergence})%
\index{convergence!martingale}.
Since $Q$ dominates $P$,
by \autoref{prop:compatibility}\ref{itm:dominance=>absolute-continuity}
$P$ is absolutely continuous with respect to $Q$
and hence the posterior also converges $P$-almost surely.
\end{example}

\begin{remark}[Martingales and Absolute Continuity]
\label{rem:martingales-ac}\index{absolute continuity}\index{martingale}
While \autoref{thm:measure-martingale} trivially also holds
if $Q$ is absolutely continuous with respect to $P$,
\autoref{thm:martingale-measure} does not imply that
$Q$ is absolutely continuous with respect to $P$.

Let $P$ and $Q$ be defined as in \autoref{rem:absolute-continuity}.
Consider the process $X_0(x) := 1$,
\[
X_{t+1}(x) :=
\begin{cases}
2 X_t, &\text{if } x_{t+1} = 0, \text{ and} \\
\tfrac{1}{2} X_t, &\text{if } x_{t+1} = 1.
\end{cases}
\]
The process $(X_t)_{t \in \mathbb{N}}$ is a nonnegative $P$-martingale since
every $X_t$ is $\F_t$-measurable and for $x = y_{1:t}$ we have
\begin{align*}
   \EE[X_{t+1} \mid \F_t ](y)
&= P(x0 \mid x) 2 X_t(y)
   + P(x1 \mid x) \tfrac{1}{2} X_t(y) \\
&= \tfrac{1}{3} 2 X_t(y) + \tfrac{2}{3} \cdot \tfrac{1}{2} X_t(y) = X_t(y).
\end{align*}
Moreover,
\[
  Q(x)
= \left( \tfrac{1}{3} \right)^{\mathrm{ones}(x)}
  \left( \tfrac{2}{3} \right)^{\mathrm{zeros}(x)}
= \left( \tfrac{2}{3} \right)^{\mathrm{ones}(x)}
  \left( \tfrac{1}{3} \right)^{\mathrm{zeros}(x)}
  2^{-\mathrm{ones}(x)} 2^{\mathrm{zeros}(x)}
= P(x) X_t(y).
\]
Hence
$X_t(y) = Q(y_{1:t}) / P(y_{1:t})$ $P$-almost surely.
The measure $Q$ is uniquely defined by its values on the cylinder sets,
and as shown in \autoref{rem:absolute-continuity},
$Q$ is not absolutely continuous with respect to $P$.
\end{remark}

\begin{theorem}[Radon-Nikodym Derivative]
\label{thm:Radon-Nikodym-derivative}
If $Q \gg P$,
then there is a function $dP/dQ: \X^\infty \to [0, \infty)$
called the \emph{Radon-Nikodym derivative}\index{Radon-Nikodym derivative}
such that
\[
\int f dP = \int f \frac{dP}{dQ} dQ
\]
for all measurable functions $f$.
\end{theorem}

This function $dP/dQ$ can be seen as a density function of $P$
with respect to the background measure $Q$.
Moreover, $dP/dQ$ is the limit of the martingale $P/Q$%
~\citep[Sec.~5.3.3]{Durrett:2010}
which exists $Q$-almost surely
according to \autoref{thm:martingale-convergence}\index{convergence!martingale}.

The following proposition characterizes
the notions of compatibility from \autoref{sec:compatibility}
in terms of the martingale $Q/P$.

\begin{proposition}[Martingales and Compatibility]
\label{prop:martingales-and-compatibility}
\index{martingale}\index{compatibility}
The following relationships hold
between $Q$, $P$, and
the $P$-martingale $Y_t := Q(x_{1:t}) / P(x_{1:t})$.
\begin{enumerate}[(a)]
\item\label{itm:martingale-dominance}\index{dominance}
	$Q \timesgeq P$
	if and only if
	$Y_t \geq c > 0$ for all $t \in \mathbb{N}$.
\item\label{itm:martingale-absolute-continuity}\index{absolute continuity}
	$Q \gg P$
	if and only if
	$P$-almost surely $Y_t \not\to 0$ as $t \to \infty$.
\item\label{itm:martingale-dominance-with-coefficients}
	\index{dominance!with coefficients}
	$Q$ dominates $P$ with coefficients $f$ if and only if
	$Y_t \geq 1/f(t)$ for all $t$.
\item\label{itm:martingale-weak-dominance}\index{dominance!weak}
	$Q \timesgeq_W P$
	if and only if
	$P$-almost surely $\log(Y_{t+1}/Y_t) \to 0$ in Cesàro average.
\item\label{itm:martingale-local-absolute-continuity}
	\index{absolute continuity!local}
	$Q \gg_L P$
	if and only if
	$P$-almost surely $Y_t > 0$ for all $t \in \mathbb{N}$.
\end{enumerate}
\end{proposition}
\begin{proof}
$(Y_t)_{t \in \mathbb{N}}$ is a $P$-martingale
according to \autoref{thm:measure-martingale}.
\begin{enumerate}[(a)]
\item $Q(x) \geq c P(x)$ with $c > 0$ for all $x \in \X^*$
	is equivalent to $Q(x) / P(x) \geq c > 0$ for all $x \in \X^*$.
\item Proved by \citet[Lem.~3i]{Hutter:2009MDL}.
	%
	%
\item Analogously to the proof of (\ref{itm:martingale-dominance}).
\item If $Q$ weakly dominates $P$,
	we get $-\log Y_t \in o(t)$ according to \autoref{rem:weak-dominance}.
	Together with $Y_0 = 1$ we get
	$-\log Y_t = \sum_{k=0}^{t-1} -\log(Y_{k+1}/Y_k) \in o(t)$,
	therefore
	$t^{-1} \sum_{k=0}^{t-1} -\log(Y_{k+1}/Y_k) \to 0$ as $t \to \infty$.
	Conversely, if the Cesàro average converges to $0$,
	then $t^{-1} \log Y_t \to 0$, hence $-\log Y_t \in o(t)$.
\item Let $x \in \X^*$ be any finite string.
	If $Q \gg_L P$ and $P(x) > 0$, then $Q(x) > 0$, and hence $Q(x) / P(x) > 0$.
	Conversely, if $P(x) > 0$ then $Y_t$ is well-defined,
	so if $Y_{|x|}(x) > 0$ then $Q(x) > 0$.
\qedhere
\end{enumerate}
\end{proof}

From \hyperref[itm:martingale-absolute-continuity]{%
\autoref*{prop:martingales-and-compatibility}%
\ref*{itm:martingale-absolute-continuity}}
and \autoref{thm:Radon-Nikodym-derivative} we get that
$Q \gg P$ if and only if
the Radon-Nikodym derivative $dQ/dP$ is positive on a set of $P$-measure $1$.

\section{Merging}
\label{sec:merging}\index{merging}

If $Q$ is capable of learning,
it should use the sequence $x$ drawn from $P$
to change its opinions more in the direction of $P$.
More precisely,
we want $Q(\,\cdot \mid x_{<t}) \approx P(\,\cdot \mid x_{<t})$ for large $t$.
In the rest of this chapter,
we make this notion of closeness precise and
discuss different conditions on $Q$ that are sufficient for learning.

\emph{Strong merging}\index{merging!strong}
implies that the belief of \emph{any} hypothesis merges.
This is very strong, as hypotheses can talk about \emph{tail events}%
\index{tail event}:
events that are independent of any finite initial part of the infinite sequence
(such as the event $A$ in \autoref{rem:absolute-continuity}).
\emph{Weak merging}\index{merging!weak}
only considers hypothesis about the next couple of symbols,
and \emph{almost weak merging}\index{merging!almost weak}
allows $Q$ to deviate from $P$
in a vanishing fraction of the time.
Much of this section is based on \citet{KL:1994} and \citet{LS:1996}.

\subsection{Strong Merging}
\label{ssec:strong-meging}

\begin{definition}[Strong Merging]
\label{def:strong-merging}\index{merging!strong|textbf}
$Q$ \emph{merges strongly with $P$} iff
$D_\infty(P, Q \mid x_{<t}) \to 0$ as $t \to \infty$ $P$-almost surely.
\end{definition}

The following theorem is
the famous merging of opinions theorem by \citet{BD:1962}.

\begin{theorem}[{Absolute Continuity $\Rightarrow$ Strong Merging;
\citealp{BD:1962}}]
\label{thm:Blackwell-Dubins}
\index{merging!of opinions}\index{absolute continuity}\index{merging!strong}
If $P$ is absolutely continuous with respect to $Q$,
then $Q$ merges strongly with $P$.
\end{theorem}

\begin{example}[{The Black Ravens 2; \citealp[Sec.~7.4]{RH:2011}}]
\label{ex:black-ravens2}\index{black ravens}
Recall the black raven problem from \autoref{ex:black-ravens}.
Let $Q$ be a learning distribution that dominates the true distribution $P$,
such as a Bayesian mixture~(\autoref{ex:Bayesian-mixture}).
By \autoref{prop:compatibility}\ref{itm:dominance=>absolute-continuity}
we get $Q \gg P$,
and hence $Q$ merges strongly to $P$ by \autoref{thm:Blackwell-Dubins}.
Thus we get as $t \to \infty$
that $P$-almost surely $|Q(H \mid x_{<t}) - P(H \mid x_{<t})| \to 0$
for the hypothesis $H$ that `all ravens are black'
defined in \eqref{eq:hypothesis-all-ravens-are-black}.
Thus if all ravens are black in the real world ($P(H) = 1$),
$Q$ learns this asymptotically ($Q(H \mid x_{<t}) \to 1$).
This is the solution we desired:
the learning distribution $Q$ converges to a true belief about an infinite set
by only looking from a finite (but growing) number of data points.
\end{example}



The following is the converse of \autoref{thm:Blackwell-Dubins}.

\begin{theorem}[{Strong Merging $\land$ Local Absolute Continuity $\Rightarrow$ Absolute Continuity;
\citealp[Thm.~2]{KL:1994}}]
\label{thm:Blackwell-Dubins-converse}
\index{absolute continuity}\index{merging!strong}
\index{absolute continuity!local}
If $Q$ is locally absolutely continuous with respect to $P$ and
$Q$ merges strongly with $P$,
then $P$ is absolutely continuous with respect to $Q$.
\end{theorem}

The following result shows that local absolute continuity
is not required for strong merging:
recall that according to \autoref{ex:MDL} the MDL distribution is not
locally absolutely continuous with respect to every $P$ from the class $\M$.

\begin{theorem}[{Strong Merging for MDL;
\citealp[Thm.~1]{Hutter:2009MDL}}]
\label{thm:MDL-merges-strongly}\index{MDL}\index{merging!strong}
If $P \in \M$, then
\[
D_\infty(P, \MDL^x \mid x) \to 0 \text{ as $|x| \to \infty$ $P$-almost surely.}
\]
\end{theorem}

Let $\M$ be a (possibly uncountable) set
of probability measures on $(\X^\infty, \Foo)$.
\citet[Thm.~4]{Ryabko:2010} shows that
if there is a $Q$ that merges strongly with every $P \in \M$,
then there is a Bayesian mixture over a countable subset of $\M$
that also merges strongly with every $P \in \M$.

\subsection{Weak Merging}
\label{ssec:weak-meging}\index{merging!weak}

In \autoref{def:strong-merging} the supremum ranges over
all measurable sets $A \in \Foo$ which includes tail events.
Instead, we may restrict the supremum to
the next few symbols.
This is known as \emph{weak merging}.

\begin{definition}[Weak Merging]
\label{def:weak-merging}\index{merging!weak|textbf}
$Q$ \emph{weakly merges with $P$} iff
for every $d \in \mathbb{N}$,
$D_{t+d}(Q, P \mid x_{<t}) \to 0$ as $t \to \infty$ $P$-almost surely.
\end{definition}

The following lemma gives an equivalent formulation of weak merging.

\begin{lemma}[{\citealp[Rem.~5]{LS:1996}}]
\label{lem:weak-merging}\index{merging!weak}
$Q$ weakly merges with $P$ if and only if
$D_t(Q, P \mid x_{<t}) \to 0$ as $t \to \infty$ $P$-almost surely.
\end{lemma}

Unfortunately, weak dominance is not sufficient for weak merging%
~\citep[Ex.~10]{LS:1996}.
We need the following stronger condition,
that turns out to be (almost) necessary.
In the following,
let $Y_t := Q(x_{1:t}) / P(x_{1:t})$
denote the $P$-martingale from \autoref{prop:martingales-and-compatibility}.

\begin{theorem}[{\citealp[Prop.~5a]{KL:1994}}]
\label{thm:weak-merging}\index{merging!weak}
If $P$-almost surely $Y_{t+1} / Y_t \to 1$,
then $Q$ merges weakly with $P$.
\end{theorem}

\begin{example}[Laplace Rule 2]
\label{ex:Laplace-rule2}\index{Laplace rule}
Suppose we use the Laplace rule from \autoref{ex:Laplace-rule}
to predict a $\Bernoulli(r)$ process.
By the strong law of large numbers,
$\rho_L(x_t \mid x_{<t}) \to r$ almost surely.
Therefore we can use \autoref{thm:weak-merging} to conclude that
$\rho_L$ merges weakly with $\Bernoulli(r)$ for all $r \in [0, 1]$.
(Note that \emph{strongly merging} with every Bernoulli process is impossible;
\citealp[p.~7]{Ryabko:2010})
\end{example}

The following is a converse to \autoref{thm:weak-merging}.

\begin{theorem}[{\citealp[Prop.~5b]{KL:1994}}]
\label{thm:weak-merging-converse}\index{merging!weak}
If $Q$ merges weakly with $P$,
then $Y_{t+1} / Y_t \to 1$ in $P$-probability.
\end{theorem}

Unfortunately, weak dominance is not enough to guarantee weak merging.

\begin{example}[{Weak Dominance $\nRightarrow$ Weak Merging;
\citealp[Prop.~7]{RH:2007}}]
\label{ex:weak-dominance-does-not-imply-weak-merging}
\index{dominance!weak}
\index{merging!weak}
Let $\X = \{ 0, 1 \}$ and
let $f$ be any arbitrarily slowly monotone growing function
with $f(t) \to \infty$.
Define $P(1^\infty) := 1$,
the sequence $(t_i)_{i \in \mathbb{N}}$ such that $f(t_{i+1}) \geq 2f(t_i)$, and
\[
Q(x_t \mid x_{<t}) :=
\begin{cases}
\tfrac{1}{2} &\text{if } t = t_i \text{ for some $i \in \mathbb{N}$}, \\
1            &\text{if } t \neq t_i \text{ and } x_t = 1, \text{ and} \\
0            &\text{otherwise}.
\end{cases}
\]
Now $Q$ dominates $P$ with coefficients $f$ by construction and
$Q$ weakly dominates $P$ if $f$ grows subexponentially.
However, $|Q(1 \mid 1^t) - P(1 \mid 1^t)| \geq 1/2$
for infinitely many $t \in \mathbb{N}$.
\end{example}

\subsection{Almost Weak Merging}
\label{ssec:almost-weak-meging}\index{merging!almost weak}

The following definition is due to \citet[Def.~10]{LS:1996}.

\begin{definition}[Almost Weak Merging]
\label{def:almost-weak-merging}\index{merging!almost weak|textbf}
$Q$ \emph{almost weakly merges with $P$} iff
for every $d \in \mathbb{N}$
\[
\frac{1}{t} \sum_{k=1}^t D_{t+d}(Q, P \mid x_{<t}) \to 0
\text{ as $t \to \infty$ $P$-almost surely}.
\]
\end{definition}

There is also an analogue of \autoref{lem:weak-merging} for
almost weakly merging in the sense that
we can equivalently set $d = 0$~\citep[Rem.~6]{LS:1996}.

\begin{remark}[Weak Merging and Merging in KL-Divergence]
\label{rem:weak-merging-and-merging-in-KL-divergence}
\index{merging!weak}\index{KL-divergence}
From \autoref{lem:Pinsker} follows that
weak merging is implied by $\KL_d(P, Q \mid x) \to 0$ $P$-almost surely and
almost weak merging is implied by
$\sum_{k=1}^t \KL_1(P, Q \mid x_{<k}) \in o(t)$
$P$-almost surely, i.e., $\KL_t(P, Q) \in o(t)$~\citep[Lem.~1]{RH:2008}.
The converse is generally false.
\end{remark}

The following proposition relates the three notions of merging.

\begin{proposition}[Strong Merging $\Rightarrow$ Weak Merging $\Rightarrow$ Almost Weak Merging]
\label{prop:merging}
\index{merging!strong}\index{merging!weak}\index{merging!almost weak}
If $Q$ merges strongly with $P$, then $Q$ merges weakly  with $P$.
If $Q$ merges weakly with $P$, then $Q$ merges almost weakly with $P$.
\end{proposition}
\begin{proof}
Follows immediately from the definitions.
\end{proof}

\begin{theorem}[{Weak Dominance $\Rightarrow$ Almost Weak Merging;
\citealp[Thm.~4]{LS:1996}}]
\label{thm:almost-weak-merging}
\index{dominance!weak}\index{merging!almost weak}
If $Q$ weakly dominates $P$,
then $Q$ merges almost weakly with $P$.
\end{theorem}

From \autoref{thm:almost-weak-merging} we get that
the speed prior~(\autoref{ex:speed-prior}) merges almost weakly
with any probability distribution estimable in polynomial time.

We also have the following converse to \autoref{thm:almost-weak-merging}.

\begin{theorem}[{Almost Weak Merging $\Rightarrow$ Weak Dominance;
\citealp[Cor.~7]{LS:1996}}]
\label{thm:almost-weakly-merging-converse}
\index{dominance!weak}\index{merging!almost weak}
If $Q$ is locally absolutely continuous with respect to $P$,
$Q$ merges almost weakly with $P$, and
$P$-almost surely $\liminf_{t \to \infty} Y_{t+1} / Y_t > 0$, then
$Q$ weakly dominates $P$.
\end{theorem}

\section{Predicting}
\label{sec:predicting}

In \autoref{sec:merging} we wanted $Q$ to acquire the correct beliefs about $P$.
In this section,
we exploit the accuracy of our beliefs for predicting individual symbols.
We derive bounds on
the number of errors $Q$ makes when trying to predict a string drawn from $P$.

Since the data drawn from $P$ is stochastic,
we cannot expect to make a finite number of errors.
Even the perfect predictor that knows $P$
generally makes an infinite number of errors.
For example, trying to predict the Lebesgue measure $\lambda$%
~(\autoref{ex:Lebesgue-measure}),
in expectation we make half an error in every time step.
So instead we are asking about the asymptotic error rate of
a predictor based on $Q$ compared to a predictor based on $P$,
the \emph{prediction regret}\index{prediction!regret}.

Let $x_t^R$ be the $t$-th symbol predicted by the probability measure $R$
according to the maximum likelihood estimator:
\begin{equation}\label{eq:ML-prediction}
x_t^R :\in \argmax_{a \in \X} R(x_{<t}a \mid x_{<t}).
\end{equation}

The \emph{instantaneous error}\index{error!instantaneous|textbf}
of a $R$-based predictor is defined as
\[
e_t^R :=
\begin{cases}
0 &\text{if } x_t = x_t^R, \text{ and} \\
1 &\text{otherwise}.
\end{cases}
\]
and the \emph{cumulative error}\index{error!cumulative|textbf} is
\[
E_t^R := \sum_{k=1}^t e_k^R.
\]
Note that both $e_t$ and $E_t$ are random variables.

\begin{definition}[Prediction Regret]
\label{def:prediction-regret}\index{prediction!regret|textbf}
\index{prediction!expected regret|textbf}
In time step $t$
the \emph{prediction regret} is $E_t^Q - E_t^P$ and
the \emph{expected prediction regret} is $\EE \left[ E_t^Q - E_t^P \right]$.
\end{definition}

\index{prediction!loss}
More generally, we could also follow \citet{Hutter:2001error}
and phrase predictive performance in terms of \emph{loss}:
given a loss function $\ell: \X \times \X \to \mathbb{R}$
the predictor $Q$ suffers an (instantaneous) loss
of $\ell(x_t^Q, x_t^P)$ in time step $t$.
If the loss function $\ell$ is bounded in $[0, 1]$,
many of the results for prediction regret also hold for
cumulative loss
(for \autoref{ssec:prediction-absolute-continuity}
we also need $\ell(a, a) = 0$ for all $a \in \X$).
In this chapter
we chose to phrase the results in terms of prediction errors instead of loss
because prediction errors are conceptionally simpler.

\begin{example}[Good Prediction Regret $\nRightarrow$ Merging/Compatibility]
\label{ex:prediction-no-merging}
\index{prediction!regret}\index{merging}
Good prediction regret does not imply (weak/strong) merging or (weak) dominance:
Let $P := \Bernoulli(1/3)$ and $Q := \Bernoulli(1/4)$.
Clearly $P$ and $Q$ do not merge (weakly) or (weakly) dominate each other.
However, a $P$-based predictor always predicts $0$,
and so does a $Q$-based predictor.
Therefore the prediction regret $E_t^Q - E_t^P$ is always $0$.
\end{example}

\begin{example}[{Adversarial Sequence; \citealp[Lem.~4]{Legg:2006}}]
\label{ex:adversarial-sequence}
\index{adversarial!sequence}
No learning distribution $Q$ will learn to predict everything.
We can always define a \emph{$Q$-adversarial sequence} $z_{1:\infty}$
recursively according to
\[
z_t :=
\begin{cases}
0 &\text{if } Q(0 \mid z_{<t}) < 1/2, \text{ and} \\
1 &\text{if } Q(0 \mid z_{<t}) \geq 1/2. \\
\end{cases}
\]
In every time step the probability that a $Q$-based predictor makes an error
is at least $1/2$, hence $e_t^Q \geq 1/2$ and $E_t^Q \geq t/2$.
But $z_{1:\infty}$ is a deterministic sequence,
thus an informed predictor makes zero errors.
Therefore the prediction regret of $Q$
on the sequence $z_{1:\infty}$ is linear.
\end{example}

\subsection{Dominance}
\label{ssec:prediction-dominance}

We start with the prediction regret bounds proved by \citet{Hutter:2001error}
in case the learning distribution $Q$ dominates the true distribution $P$.
In the following, let $c_P$ denote the constant from \autoref{def:dominance}.

\begin{theorem}[{\citealp[Eq.~5 \& 8]{Hutter:2007universal}}]
\label{thm:error-bound}
\index{KL-divergence}\index{error!cumulative}
For all $P$ and $Q$,
\[
\sqrt{\EE_P E^Q_n} - \sqrt{\EE_P E^P_n} \leq \sqrt{2\KL_n(P, Q)}.
\]
\end{theorem}

The following bound on prediction regret then follows easily,
but it is a factor of $\sqrt{2}$ worse than
the bound stated by \citet[Thm.~3.36]{Hutter:2005}.

\begin{corollary}[Expected Prediction Regret]
\label{cor:expected-prediction-regret}
\index{KL-divergence}\index{prediction!expected regret}
For all $P$ and $Q$,
\begin{align*}
      0
 \leq \EE_P \Big[ E^Q_n - E^P_n \Big]
&\leq 2\KL_n(P, Q) + 2 \sqrt{2\KL_n(P, Q) \EE_P E^P_n}.
\end{align*}
\end{corollary}
\begin{proof}
From \autoref{thm:error-bound} we get
\begin{align*}
      \EE_P \Big[ E_n^Q - E_n^P \Big]
&=    \left(\sqrt{\EE_P E_n^Q} + \sqrt{\EE_P E_n^P} \right)
      \left( \sqrt{\EE_P E_n^Q} - \sqrt{\EE_P E_n^P} \right) \\
&\leq \left( \sqrt{\EE_P E_n^Q} + \sqrt{\EE_P E_n^P} \right) \sqrt{2\KL_n(P, Q)} \\
&\leq \left( \sqrt{2\KL_n(P, Q)} + \sqrt{\EE_P E_n^P} + \sqrt{\EE_P E_t^P} \right) \sqrt{2\KL_n(P, Q)} \\
&=    2\KL_n(P, Q) + 2\sqrt{2\KL_n(P, Q) \EE_P E_n^P}.
\qedhere
\end{align*}
\end{proof}

If $Q$ dominates $P$,
then we have $\KL_n(P, Q) \leq -\ln c_P$:
\begin{equation}\label{eq:KL-bound}
     \KL_n(P, Q)
=    \sum_{x \in \X^n} P(x) \log \frac{P(x)}{Q(x)} \\
\leq \sum_{x \in \X^n} P(x) \log \frac{1}{c_P} \\
=    -\log c_P
\end{equation}
This invites the following corollary.

\begin{corollary}[{Prediction Regret for Dominance;
\citealp[Cor.~3.49]{Hutter:2005}}]
\label{cor:expected-prediction-regret-dominance}
\index{dominance}\index{prediction!expected regret}
If $Q$ dominates $P$,
then the following statements hold.
\begin{enumerate}[(a)]
\item $\EE_P E_\infty^Q$ is finite if and only if
	$\EE_P E_\infty^P$ is finite.
\item $\sqrt{\EE_P E_\infty^Q} - \sqrt{\EE_P E_\infty^P} \in O(1)$
\item $\EE_P E_t^Q / \EE_P E_t^P \to 1$
	for $\EE_P E_t^P \to \infty$.
\item $\EE_P \left[ E_t^Q - E_t^P \right]
	\in O \left(\sqrt{\EE_P E_t^P} \right)$.
\end{enumerate}
\end{corollary}

If the true distribution $P$ is deterministic,
we can improve on these bounds:

\begin{example}[Predicting a Deterministic Measure] 
\label{ex:predicting-deterministic-distribution}
\index{measure!deterministic}\index{prediction!expected regret}
Suppose we are predicting a deterministic measure $P$
that assigns probability $1$ to the infinite string $x_{1:\infty}$.
If $P$ is dominated by $Q$,
the total expected prediction regret $\EE_P E_\infty^Q$ is bounded by $-2\ln c_P$
by \autoref{cor:expected-prediction-regret}.
This is easy to see: every time we predict a wrong symbol $a \neq x_t$,
then $Q(a \mid x_{<t}) \geq Q(x_t \mid x_{<t})$,
so $Q(x_t \mid x_{<t}) \leq 1/2$.
Therefore $Y_t \leq Y_{t-1} / 2$ and by dominance $Y_t \geq c_P$.
Hence a prediction error can occur at most $-\log c_P$ times.
\end{example}

Generally, the $O(\EE_P E^P_t)$ bounds on expected prediction regret
given in \autoref{cor:expected-prediction-regret-dominance}
are essentially unimprovable:

\begin{example}[Lower Bounds on Prediction Regret]
\label{ex:prediction-lower-bounds}
\index{prediction!regret}
Set $\X := \{ 0, 1 \}$ and
consider the uniform measure $\lambda$\index{measure!Lebesgue}
from \autoref{ex:Lebesgue-measure}.
For each time step $t$,
we have $\lambda(0 \mid x_{<t}) = \lambda(1 \mid x_{<t}) = 1/2$,
so the argmax in \eqref{eq:ML-prediction} ties and hence
it does not matter whether we predict $0$ or $1$.
We take two predictors $P$ and $Q$, where $P$ always predicts $0$ and
$Q$ always predicts $1$.
Let $Z_t := E_t^Q - E_t^P$.
Since their predictions never match,
$Z_t$ is an ordinary random walk with step size $1$.
We have~\citep{Weisstein:2002}
\[
  \limsup_{t \to \infty} \frac{\EE_P[ E_t^Q - E_t^P ]}{\sqrt{t}}
= \sqrt{2/\pi}
\]
and for the law of the iterated logarithm~\citep[Thm.~8.8.3]{Durrett:2010}
\[
\limsup_{t \to \infty} \frac{E_t^Q - E_t^P}{\sqrt{2t \log \log t}} = 1
\text{ $P$-almost surely}.
\]
Both bounds are known to be asymptotically tight.
\end{example}

While \autoref{ex:prediction-lower-bounds} shows that the bounds
from \autoref{cor:expected-prediction-regret-dominance} are asymptotically tight,
they are misleading
because in most cases, we can do much better.
According to the following theorem,
the worst case bounds are only attained
if $P(x_t \mid x_{<t})$ is sufficiently close to $1/2$.

\begin{theorem}[Expected Prediction Regret for Nonuniform Measures]
\label{thm:expected-prediction-regret-nonuniform}
\index{KL-divergence}\index{prediction!expected regret}
If $\X = \{ 0, 1 \}$ and
there is an $\varepsilon > 0$ such that
$|P(x_t \mid x_{<t}) - 1/2| \geq \varepsilon$ for all $x_{1:t} \in \X^*$,
then
\[
     \EE_P \left[ E_t^Q - E_t^P \right]
\leq \frac{\KL_t(P, Q)}{\varepsilon}.
\]
\end{theorem}
\begin{proof}
Recall the definition of entropy\index{entropy} in nats:
\[
\Ent(p) := - p \ln p - (1 - p) \ln (1 - p).
\]
The second order Taylor approximation of $\Ent$ at $1/2$ is
\[
f(p) = \ln 2 - 2 (p - \tfrac{1}{2})^2.
\]
One can check that $f(p) \geq \Ent(p)$ for all $0 \leq p \leq 1$.
Define $p := P(x_t^P \mid x_{<t}) \geq 1/2$ and
$q := Q(x_t^Q \mid x_{<t}) \geq 1/2$ to ease notation.
Consider the function
\begin{align*}
      g(p,q,\eps)
&:=   p - (1 - p) - \varepsilon^{-1} \big(
        p \ln \tfrac{p}{1-q} + (1-p) \ln \tfrac{1 - p}{q}
      \big) \\
\intertext{which is strictly increasing as $q$ decreases,
so from $q \geq 1/2$ we get}
      g(p,q,\eps)
&\leq 2p - 1 - \varepsilon^{-1} \ln 2 + \varepsilon^{-1} \Ent(p) \ln 2 \\
&\leq 2p - 1 - \varepsilon^{-1} \ln 2 + \varepsilon^{-1} f(p) \ln 2 \\
&=    2p - 1 - \varepsilon^{-1} 2 (p - \tfrac{1}{2})^2,
\intertext{which decreases as $p$ increases,
hence it is maximized for $p = 1/2 + \varepsilon$,}
      g(p,q,\eps)
&\leq 2 \varepsilon - \varepsilon^{-1} 2 \varepsilon^2
=     0
\end{align*}
Therefore $g$ is nonpositive.
If $x_t^Q = x_t^P$, the one-step error is $0$.
Otherwise $\EE_P[e_t \mid x_{<t}] = p - (1 - p)$ and
$g(p, q, \eps) = \EE_P[e_t \mid x_{<t}] - \eps^{-1} \KL_1(P, Q \mid x_{<t})$,
so we get $\EE_P[e_t \mid x_{<t}] \leq \varepsilon^{-1} \KL_1(P, Q \mid x_{<t})$.
Summing this from $t = 1$ to $n$ yields the claim.
\[
     \EE \Big[ E_n^Q - E_n^P \Big]
\leq \varepsilon^{-1} \KL_n(P, Q).
\qedhere
\]
\end{proof}

\subsection{Absolute Continuity}
\label{ssec:prediction-absolute-continuity}

\begin{theorem}[Prediction with Absolute Continuity]
\label{thm:prediction-absolute-continuity}
\index{error!cumulative}\index{absolute continuity}
If $Q \gg P$, then
\[
\sqrt{E_t^Q} - \sqrt{E_t^P} \leq O \left( \sqrt{\log \log t} \right)
\text{ $P$-almost surely}.
\]
\end{theorem}

The proof idea is inspired by \citet{MS:1999merging}.
We think of $P$ and $Q$ as two players in a zero-sum betting game.
In every time step $t$,
the players will make a bet on the outcome of $x_t$.
If $x_t = x_t^Q \neq x_t^P$, then $Q$ wins \$1 from $P$,
if $x_t = x_t^P \neq x_t^Q$, then $Q$ loses \$1 to $P$.
Otherwise $x_t^Q = x_t^P$ or $x_t^Q \neq x_t \neq x_t^P$
and neither player gains or loses money.
Since $Q$ predicts according to the maximum likelihood principle \eqref{eq:ML-prediction},
it is rational to accept the bet from $Q$'s perspective.
In $Q$'s eyes, the worst case is a fair bet,
so $Q$ will not lose more money than it would lose on a random walk.
The law of the iterated logarithm gives a $Q$-probability one
statement about this bound, which transfers to $P$ by absolute continuity.

\begin{proof}
Define the stochastic process $Z_t := E_t^Q - E_t^P$.
Since $\EE_Q[ e_t^R ] = Q(x_t^R \mid x_{<t})$, we get
\begin{align*}
      \EE_Q[ Z_{t+1} \mid \mathcal{F}_t]
&=    Q(x_t^Q \mid x_{<t}) - Q(x_t^P \mid x_{<t}) + Z_t \\
&\geq Q(x_t^Q \mid x_{<t}) - Q(x_t^Q \mid x_{<t}) + Z_t
=     Z_t,
\end{align*}
hence $(Z_t)_{t \in \mathbb{N}}$ is a $Q$-submartingale.
In the worst case (for $Q$), $(Z_t)_{t \in \mathbb{N}}$ is just a random walk
with step size $1$.
But $Z_t$ can only move if $Q$ and $P$ predict a different symbol.
If this happens, at least one of them makes an error.
Let $m_t$ be the number of steps $Z_t$ has moved ($Z_{t+1} \neq Z_t$).
Then $m_t \leq E_t^Q + E_t^P$ and $m_t \leq t$.
By the law of the iterated logarithm~\citep[Thm.~8.8.3]{Durrett:2010},
\[
\liminf_{t \to \infty} \frac{Z_t}{\sqrt{2m_t \log \log m_t}} = -1
\]
$Q$-almost surely.
We define the event
\[
A := \left\{ \exists C\, \forall t.\; Z_t \geq -C \sqrt{m_t \log \log m_t} \right\}.
\]
Then $Q(A) = 1$,
hence $P(A) = 1$ by absolute continuity.
\[
     E_t^Q - E_t^P
=    Z_t
\leq C \sqrt{ (E_t^Q + E_t^P) \log \log t}
\leq C \left( \sqrt{E_t^Q} + \sqrt{E_t^P} \right) \sqrt{\log \log t}
\]
Dividing both sides by $\sqrt{E_t^Q} + \sqrt{E_t^P}$ yields that
there is a $P$-almost surely finite random variable $C$ such that
$\sqrt{E_t^Q} - \sqrt{E_t^P} \leq C \sqrt{\log \log t}$.
\end{proof}

This invites the following immediate corollary.

\begin{corollary}[Prediction Regret for Absolute Continuity]
\label{cor:prediction-absolute-continuity}
\index{absolute continuity}\index{prediction!regret}
If $Q \gg P$, then
\[
E_t^Q - E_t^P \in O\left(\log \log t + \sqrt{E_t^P \log \log t} \right)
\text{ $P$-almost surely}.
\]
\end{corollary}
\begin{proof}
Analogously to the proof of \autoref{cor:expected-prediction-regret}.
\end{proof}

While \autoref{cor:prediction-absolute-continuity}
establishes an almost sure prediction regret bound,
it is different from the bound on expected prediction regret
from \autoref{cor:expected-prediction-regret};
bounds on $\EE[E^Q_t - E^P_t]$
are incomparable to almost sure bound
given in \autoref{thm:prediction-absolute-continuity}:
for a sequence of nonnegative (unbounded) random variables
convergence in mean does not imply almost sure convergence%
~\citep[Sec.~14.7]{Stoyanov:2013counterexamples}
or vice versa~\citep[Sec.~14.8ii]{Stoyanov:2013counterexamples}.

We proceed to establish an improved prediction regret bound
in case $P$ is nonuniform
analogously to \autoref{thm:expected-prediction-regret-nonuniform}.

\begin{theorem}[Prediction Regret for Nonuniform Measures]
\label{thm:prediction-regret-nonuniform}
\index{absolute continuity}\index{prediction!regret}
If $Q \gg P$, $\X = \{ 0, 1\}$, and
there is an $\varepsilon > 0$ such that with $P$-probability $1$
\[
|P(x_t \mid x_{<t}) - 1/2| \geq \varepsilon
\]
for all $t \in \mathbb{N}$,
then $P$-almost surely $E_t^Q - E_t^P \in O(1)$.
\end{theorem}
\begin{proof}
If $|P(x_t \mid x_{<t}) - 1/2| \geq \varepsilon$,
then for large enough $t$, $Q$ will have merged with $P$
(\autoref{thm:Blackwell-Dubins}) and hence
$|Q(x_t \mid x_{<t}) - 1/2| \geq \varepsilon/2$ infinitely often.

Thus $Z_t$ has an expected gain of $\varepsilon/2$ if the predictors disagree.
Therefore $Z_t \to \infty$ $Q$-almost surely.
Consequently, the set
\[
A := \{ \exists t_0\, \forall t \geq t_0.\; Z_t \geq 0 \}
\]
has $Q$-measure $1$.
By absolute continuity,
it also has $P$-measure $1$,
hence there is a $P$-almost surely finite random variable $C$
such that for all $t$, $Z_t \geq -C$.
\end{proof}

There is another argument that we could use to show that
under the condition of \autoref{thm:prediction-regret-nonuniform}
$E_t^Q - E_t^P$ is almost surely finite:
If $P$ is absolutely continuous with respect to $Q$,
then $Q$ merges strongly with $P$ and hence $Q$ merges weakly with $P$.
Therefore almost surely there is a $t_0$ such that for all $t \geq t_0$ we have
$|Q(x_t^P \mid x_{<t}) - P(x_t^P \mid x_{<t})| < \eps$,
thus $x_t^Q = x_t^P$ for $t \geq t_0$.

\subsection{Dominance with Coefficients}
\label{ssec:prediction-dominance-with-coefficients}

\begin{lemma}[KL Divergence and Dominance With Coefficients]
\label{lem:KL-divergence-dominance-with-coefficients}
\index{dominance}\index{KL-divergence}
If $Q$ dominates $P$ with coefficients $f$,
then $\KL_t(Q, P) \leq \ln f(t)$.
\end{lemma}
\begin{proof}
Analogous to \eqref{eq:KL-bound}.
\end{proof}

This lets us derive an analogous regret bound to
\autoref{cor:expected-prediction-regret}.

\begin{corollary}[Expected Prediction Regret for Dominance With Coefficients]
\label{cor:expected-prediction-regret-dominance-with-coefficients}
\index{dominance}\index{prediction!expected regret}
If $Q$ dominates $P$ with coefficients $f$,
then
\[
     \EE_P \Big[ E^Q_n - E^P_n \Big]
\leq 2 \ln f(t) + 2 \sqrt{2\EE_P E^P_n \ln f(t)}.
\]
\end{corollary}
\begin{proof}
Apply \autoref{lem:KL-divergence-dominance-with-coefficients}
to \autoref{cor:expected-prediction-regret}.
\end{proof}

For weak dominance we get sublinear prediction regret.

\begin{corollary}[Sublinear Prediction Regret for Weak Dominance]
\label{cor:sublinear-prediction-regret-weak-dominance}
\index{dominance!weak}\index{prediction!expected regret}
If $Q$ weakly dominates $P$,
then $\EE_P[ E^Q_n - E^P_n ] \in o(t)$.
\end{corollary}
\begin{proof}
By \autoref{rem:weak-dominance}
$\ln f \in o(t)$.
Applying \autoref{cor:expected-prediction-regret-dominance-with-coefficients}
we get
\[
     \EE_P \Big[ E^Q_n - E^P_n \big]
\leq 2 o(t) + 2 \sqrt{2\EE_P E^P_n o(t)}
\leq 2 o(t) + 2 \sqrt{2O(t) o(t)}
\in  o(t).
\qedhere
\]
\end{proof}

\section{Learning with Algorithmic Information Theory}
\label{sec:learning-AIT}

Algorithmic information theory
provides a theoretical framework to apply
the probability theory results from the previous sections.
In the following
we discuss Solomonoff's famous theory of induction%
~(\autoref{ssec:Solomonoff-induction}),
the speed prior~(\autoref{ssec:speed-prior}), and
learning with a universal compression algorithm%
~(\autoref{ssec:universal-compression}).


\subsection{Solomonoff Induction}
\label{ssec:Solomonoff-induction}
\index{Solomonoff!induction}

\citet{Solomonoff:1964,Solomonoff:1978} proposed a theory of learning,
also known as \emph{universal induction}%
\index{universal!induction|see {Solomonoff induction}} or
\emph{Solomonoff induction}\index{Solomonoff!induction}.
It encompasses \emph{Ockham's razor}\index{Ockham's razor}
by favoring simple explanations over complex ones,
and \emph{Epicurus' principle of multiple explanations}\index{Epicurus' principle}
by never discarding possible explanations.
See \citet{RH:2011} for a very readable introduction to Solomonoff's theory
and its philosophical motivations and
\citet{Sterkenburg:2016} for a critique of its optimality.

At the core of this theory is \emph{Solomonoff's distribution $M$},
as defined in \autoref{ex:Solomonoff-prior}.
Since $M$ dominates all lower semicomputable semimeasures,
we get all the merging and prediction results
from \autoref{sec:merging} and \autoref{sec:predicting}:
when drawing a string from any computable measure $P$,
$M$ arrives at the correct belief for any hypothesis.

\begin{corollary}[Strong Merging for Solomonoff Induction]
\label{cor:strong-merging-M}
\index{merging!strong}\index{Solomonoff!prior}
$M$ merges strongly with every computable measure.
\end{corollary}
\begin{proof}
From \autoref{prop:compatibility}\ref{itm:dominance=>absolute-continuity}
and \autoref{thm:Blackwell-Dubins}.
\end{proof}

\begin{corollary}[Expected Prediction Regret for Solomonoff Induction]
\label{cor:expected-prediction-regret-M}
\index{prediction!expected regret}\index{Solomonoff!prior}
For all computable measures $P$,
\[
     \EE_P \left[ E_t^M - E_t^P \right]
\leq K(P) \ln 4 + \sqrt{2 \EE_P E_t^P K(P) \ln 16}.
\]
\end{corollary}
\begin{proof}
From \autoref{cor:expected-prediction-regret} and $c_P = 2^{-K(P)}$.
\end{proof}

\begin{remark}[Converging Fast and Slow]
\label{rem:converging-fast-and-slow}
\index{Solomonoff!prior}
The convergence of $M$ to a computable $P$ is fast
in the sense of \autoref{cor:expected-prediction-regret-M}:
$M$ cannot make many more prediction errors than $P$ in expectation.
When predicting an infinite computable sequence $x_{1:\infty}$,
the total number of prediction errors is bounded by $|p|2\ln 2 \approx 1.4|p|$
where $p$ is a program that generates $x_{1:\infty}$%
~(\autoref{ex:predicting-deterministic-distribution}).

The convergence of $M$ to $P$ is also slow in the sense that
$M(x_t \mid x_{<t}) \to 1$ slower than any computable function
since $1 - M(x_t \mid x_{<t}) \timesgeq 2^{-\min_{n \geq t} K(n)}$ for all $t$.
\end{remark}

The bound from \autoref{cor:expected-prediction-regret-M} is not optimal.
Even if we knew the program $p$ generating the sequence $x_{1:\infty}$,
there might be a shorter program $p'$ that computes $x_{1:\infty}$;
hence the improved bound $E_\infty^M \leq |p'|2\ln 2$ also holds.
Since Kolmogorov complexity is incomputable,
we can't find the `best' bound algorithmically.


Solomonoff induction may even converge on some incomputable measures.

\begin{example}[$M$ Converges on Some Incomputable Measures]
\label{ex:M-converges-on-some-incomputable-measures}
\index{Solomonoff!prior}
Let $r$ be an incomputable real number.
Then the measure $P := \Bernoulli(r)$ is not computable
and $M$ is not absolutely continuous with respect to $P$:
for
\[
A := \left\{ x \in \X^\infty \;\middle|\;
       \lim_{t \to \infty} \mathrm{ones}(x_{1:t}) = r
     \right\}
\]
we have $P(A) = 1$ but $M(A) = 0$.
Since $M \gg_L P$ we get from \autoref{thm:Blackwell-Dubins-converse}
that $M$ does not merge with $P$.
Nevertheless, $M$ still succeeds at prediction
because it dominates $\Bernoulli(q)$ for each rational $q$ and
the rationals are dense around $r$.
According to \citet[Lem.~3]{LS:1996},
this implies that $M$ weakly dominates $P$ and
by \autoref{thm:almost-weak-merging} $M$ almost weakly merges to $P$.
\end{example}

The fact that $M$ does not merge strongly with every $\Bernoulli(r)$ process
is not a failure of Solomonoff's prior.
\citet[p.~7]{Ryabko:2010} shows that
for the class of all Bernoulli measures
there is no probability measure that merges strongly with each of them.

\index{Turing machine!universal}\index{invariance theorem}
The definition of $M$ has only one parameter:
the choice of the universal Turing machine.
The effect of this choice on the function $K$
can be uniformly bounded by a constant by the \emph{invariance theorem}%
\index{invariance theorem}~\citep[Thm.~3.1.1]{LV:2008}.
Hence
the choice of the UTM changes the prediction regret bound from
\autoref{cor:expected-prediction-regret-M} only by a constant.
This constant can be large,
preventing any finite-time guarantees that are independent of the UTM.
However, asymptotically Solomonoff induction succeeds
even for terrible choices of the UTM.

The Solomonoff normalization $M\norm$ of $M$ is
defined according to \autoref{def:Solomonoff-normalization}.
While $M\norm$ dominates $M$ according to \autoref{lem:nu-norm-dominates-nu}
and thus every lower semicomputable semimeasure,
in some respects, $M\norm$ behaves a little differently from $M$.
Another way to complete the semimeasure $M$ into a measure
is given in the following example.

\begin{example}[{The Measure Mixture; \citealp[p.~74]{Gacs:1983}}]
\label{ex:measure-mixture}\index{measure!mixture|textbf}
The \emph{measure mixture $\MM$}
is defined as
\begin{equation}\label{eq:def-MM}
\MM(x) := \lim_{n \to \infty} \sum_{y \in \X^n} M(xy).
\end{equation}
It is the same as $M$ except that
the contributions by programs that do not produce infinite strings are removed:
for any such program $p$,
let $k$ denote the length of the finite string generated by $p$.
Then for $|xy| > k$,
the program $p$ does not contribute to $M(xy)$,
hence it is excluded from $\MM(x)$.

Similarly to $M$,
the measure mixture $\MM$ is not a (probability) measure
since $\MM(\epsilon) < 1$;
but in this case normalization~\eqref{eq:normalization}
is just multiplication with the constant $1/\MM(\epsilon)$,
leading to the \emph{normalized measure mixture} $\MM\norm$.
\end{example}

Even though $M$ merges strongly with any computable measure $P$
with $P$-probability $1$,
\citet{LH:2013mnonconv,LH:2015mnonconv} show that
generally it does not hold for all Martin-Löf random sequences
(which also form a set of $P$-probability $1$).
\citet[Thm.~6]{HM:2007mlconvxx} construct
non-universal lower semicomputable semimeasures
that have this convergence property for all $P$-Martin-Löf random sequences.
For infinite nonrandom sequences whose bits are selectively
predicted by some total recursive function,
\citet[Thm.~10]{LHG:2011evenbits} show that
the normalized Solomonoff measure
$M\norm$ converges to $1$ on the selected bits.
This does not hold for
the unnormalized measure $M$~\citep[Thm.~12]{LHG:2011evenbits}.

\subsection{The Speed Prior}
\label{ssec:speed-prior}
\index{prior!speed}

Solomonoff's prior $M$ is incomputable~(\autoref{thm:complexity-M});
a computable alternative is the speed prior from \autoref{ex:speed-prior}.
In this section we state merging and prediction results for $S_{Kt}$,
a speed prior introduced by \citet{FLH:2016speed}
formally defined in \autoref{ex:speed-prior}.
It is slightly different from
the speed prior defined by \citet{Schmidhuber:2002},
but for the latter no compatibility properties are known
for nondeterministic measures.

\begin{definition}[Estimable in Polynomial Time]
\label{def:estimable-in-polytime}
\index{estimable in polynomial time|textbf}
A function $f: \X^* \to \mathbb{R}$ is \emph{estimable in polynomial time} iff
there is a function $g: \X^* \to \mathbb{R}$ computable in polynomial time
such that $f \timeseq g$.
\end{definition}

For a measure $P$ estimable in polynomial time
the speed prior $S_{Kt}$ dominates $P$ with coefficients
polynomial in $|x| - \log P(x)$~\citep[Eq.~12]{FLH:2016speed}.
Thus $S_{Kt}$ weakly dominates $P$ and
we get the following results.

\begin{corollary}[Almost Weak Merging for $S_{Kt}$]
\label{cor:merging-SKt}
\index{merging!almost weak}\index{prior!speed}
\index{estimable in polynomial time}
$S_{Kt}$ almost weakly merges with every measure estimable in polynomial time.
\end{corollary}
\begin{proof}
From \autoref{thm:almost-weak-merging}
and \citet[Eq.~12]{FLH:2016speed}
since $\log P$ does not grow superexponentially $P$-almost surely.
\end{proof}

\begin{corollary}[{Expected Prediction Regret for $S_{Kt}$;
\citealp[Thm.~9]{FLH:2016speed}}]
\label{cor:expected-prediction-regret-SKt}
\index{prediction!expected regret}\index{prior!speed}
\index{estimable in polynomial time}
For all measures $P$ estimable in polynomial time,
\[
    \EE_P \left[ E_n^{S_{Kt}} - E_n^P \right]
\in O \left( \log n + \sqrt{\EE_P E_\infty^P \log n} \right).
\]
\end{corollary}
\begin{proof}
From \autoref{cor:expected-prediction-regret}
and \citet[Eq.~14]{FLH:2016speed}.
\end{proof}

\subsection{Universal Compression}
\label{ssec:universal-compression}
\index{compression|see {universal compression}}

Solomonoff's distribution can be approximated
using a standard compression algorithm,
motivated by the similarity $M(x) \approx 2^{-\Km(x)}$,
where $\Km$ denotes monotone Kolmogorov complexity%
\index{Kolmogorov complexity!monotone}.
The function $\Km$ is a \emph{universal compressor}%
\index{universal!compression},
compressing at least as well as any other recursively enumerable program.

\citet{Gacs:1983} shows that
the similarity $M \approx 2^{-\Km}$ is not an equality.
However, the difference between $-\log M$ and $\Km$ is very small:
the best known lower bound is due to \citet{Day:2011}
who shows that $\Km(x) > -\log M(x) + O(\log \log |x|)$
for infinitely many $x \in \X^*$.

Nevertheless,
$2^{-\Km}$ dominates every computable measure%
~(\citealp[Thm.~4.5.4 and Lem.~4.5.6ii(d)]{LV:2008};
originally proved by \citealp{Levin:1973}).
Hence all the strong results that hold for Solomonoff induction~%
(prediction regret and strong merging) also hold for compression:
we apply \autoref{thm:Blackwell-Dubins} and
\autoref{cor:expected-prediction-regret}
to get the following results.
See \citet{Hutter:2006Km} for further discussion
on using the universal compressor $\Km$ for learning.

\begin{corollary}[Strong Merging for Universal Compression]
\label{cor:strong-merging-Km}
\index{universal!compression}
\index{merging!strong}
The distribution $2^{-\Km(x)}$ merges strongly with every computable measure.
\end{corollary}

\begin{corollary}[Expected Prediction Regret for Universal Compression]
\label{cor:expected-prediction-regret-Km}
\index{universal!compression}
\index{prediction!expected regret}
For $Q(x) := 2^{-\Km(x)}$ and
for all computable measures $P$ there is a constant $c_P$ such that
\[
     \EE_P \left[ E_t^Q - E_t^P \right]
\leq c_P + \sqrt{c_P \EE_P E_t^P}.
\]
\end{corollary}

This provides a theoretical basis for
viewing compression as a general purpose learning algorithm.
In this spirit,
the \emph{Hutter prize}\index{Hutter prize} is awarded for
the compression of a 100MB excerpt from the English Wikipedia%
~\citep{HutterPrize}.

Practical compression algorithms
(such as the algorithm by \citet{ZL:1977} used in \texttt{gzip})
are not universal.
Hence they do not dominate every computable distribution.
As with the speed prior,
what matters is the rate at which $Y_t = Q(x_{1:t})/P(x_{1:t})$ goes to $0$,
i.e., does the compressor weakly dominate the true distribution
in the sense of \autoref{def:weak-dominance}?

\citet{VBHCD:2014} successfully apply
the Lempel-Ziv compression algorithm
as a learning algorithm for reinforcement learning;
however, some preprocessing of the data is required.
More remotely, \citet{VBCM:2009nid} use standard compression algorithms
to classify mammal genomes, languages, and classical music.

\section{Summary}
\label{sec:learning-summary}

\begin{table}
\begin{center}
\begingroup
\def\arraystretch{1.5}
\begin{tabular}{lllp{50mm}}
\toprule
name & symbol & defined in & property \\
\midrule
Bayesian mixture & $\xi$
                 & \autoref{ex:Bayesian-mixture}
                 & dominates every $P \in \M$ \\
Solomonoff prior & $M$
                 & \autoref{ex:Solomonoff-prior}
                 & dominates every lower semicomputable semimeasure \\
universal compression & $2^{-\Km}$
                 & \autoref{eq:Km}
                 & dominates every computable measure \\
speed prior      & $S_{Kt}$
                 & \autoref{ex:speed-prior}
                 & weak dominates every measure estimable in polytime \\
Laplace rule     & $\rho_L$
                 & \autoref{ex:Laplace-rule}
                 & merges weakly with every Bernoulli process \\
MDL              & $\MDL^x$
                 & \autoref{ex:MDL}
                 & merges strongly with every $P \in \M$ \\
\bottomrule
\end{tabular}
\endgroup
\end{center}
\index{compatibility}\index{Bayesian!mixture}\index{Solomonoff!prior}
\index{universal!compression}\index{prior!speed}\index{Laplace rule}\index{MDL}
\caption[Examples of learning distributions]{
Examples of learning distributions discussed in this chapter
and their properties.
}\label{tab:learning-distributions}
\end{table}

\begin{table}
\begin{center}
\begingroup
\def\arraystretch{1.5}
\begin{tabular}{p{21mm}p{30mm}p{25mm}l}
\toprule
compatibility of $P$ and $Q$ & martingale & merging & prediction regret \\
\midrule
$Q \timesgeq P$
	& $Y_t \geq c > 0$
	& strong merging
	& $-2\ln c + 2\sqrt{-2\EE_P E^P_t \ln c}$ \\
$Q \gg P$
	& $Y_t \not\to 0$
	& strong merging
	& \parbox{47mm}{$O \big( \log \log t + $ \\ \phantom{.}\qquad\qquad$\sqrt{\EE_P E_t^P \log \log t} \big)$} \\
~
	& $Y_{t+1}/Y_t \to 1$
	& weak merging
	& $o(t)$ \\
$Q \geq P/f$
	& $Y_t \geq 1/f(t)$
	& ~
	& $2\ln f(t) + 2\sqrt{2\EE_P E^P_t \ln f(t)}$ \\
$Q \timesgeq_W P$
	& $\log(Y_{t+1}/Y_t) \to 1$ in Cesàro average
	& almost weak merging
	& $o(t)$ \\
$Q \gg_L P$
	& $Y_t > 0$
	& ~
	& $O(t)$ \\
\bottomrule
\end{tabular}
\endgroup
\end{center}
\caption[Summary on properties of learning]{
Summary on properties of learning.
The first column lists different notions of compatibility
introduced in \autoref{sec:compatibility};
the second column lists properties
of the $P$-martingale $Y_t := Q(x_{1:t}) / P(x_{1:t})$
from \autoref{sec:martingales};
the third column lists different notions of merging
discussed in \autoref{sec:merging};
the fourth column states the bounds on the prediction regret
(in expectation and almost surely respectively)
from \autoref{sec:predicting}.
\autoref{fig:learning-overview} illustrates the origin of the results.
}
\label{tab:learning-overview}
\end{table}

\begin{figure}
\begin{center}
\begin{tikzpicture}
\node (sm)  at (-6,7) {strong merging};
\node (wm)  at (-6,5) {weak merging};
\node (awm) at (-6,1) {\parbox{14mm}{\centering almost weak merging}};

\node (dom)   at (0,9) {$Q \timesgeq P$};
\node (ac)    at (0,7) {$Q \gg P$};
\node (domwc) at (0,4) {$Q \geq P/f$};
\node (wdom)  at (0,2) {$Q \timesgeq_W P$};
\node (lac)   at (0,0) {$Q \gg_L P$};

\node (bounded)   at (5,9) {$Y_t \geq c > 0$};
\node (nottozero) at (5,7) {$Y_t \not\to 0$};
\node (to1)       at (5,5) {$Y_{t+1} / Y_t \to 1$};
\node (to0Ca)     at (5,2) {\parbox{30mm}{\centering%
	$\log(Y_{t+1} / Y_t) \to 0$
	in Cesàro average}};
\node (nonzero)   at (5,0) {$Y_t > 0$};

\footnotesize

\draw[->] (sm) to (wm);
\draw[->] (wm) to (awm);

\draw[->] (dom) to (ac);
\draw[->] (dom) to[out=315, in=45] (domwc);
\draw[->] (domwc) to node[right] {$f \in o(t)$} (wdom);
\draw[->] (wdom) to (lac);

\draw[->] (bounded) to (nottozero);
\draw[->] (nottozero) to (to1);
\draw[->] (to1) to (to0Ca);
\draw[->] (to0Ca) to (nonzero);

\draw[<->] (dom) to (bounded);
\draw[<->] (ac) to node[above] {\citet[Lem.~3i]{Hutter:2009MDL}} (nottozero);
\draw[<->] (wdom) to (to0Ca);
\draw[<->] (lac) to (nonzero);

\draw[->] (ac) to[out=170, in=10] node[above] {\citet{BD:1962}} (sm);
\draw[->] (sm) to[out=-10, in=190] node[below] {\citet[Thm.~2]{KL:1994}, $\gg_L$} (ac);
\draw[->] (wdom) to[out=180, in=30] node[above] {\parbox{40mm}{\centering%
	\citet[Thm.~4]{LS:1996}
}} (awm);
\draw[->] (awm) to[out=0, in=210] node[below] {\parbox{35mm}{\centering%
\vspace{2mm}%
\citet[Cor.~7]{LS:1996}, \\
$\gg_L$ and
$\liminf\limits_{t \to \infty} Y_{t+1} / Y_t > 0$
}} (wdom);

\draw[<->] (wm) to node[above] {\citet[Prop.~5]{KL:1994}~~~~~} (to1);
\end{tikzpicture}
\end{center}
\caption[Properties of learning]{
Properties of learning and their relationship.
We use $Y_t := Q(x_{1:t}) / P(x_{1:t})$.
An arrow between two statements means that one statement implies the other.
The transitive property of implications is not made explicit.
The source of the result is indicated on the arrow,
sometimes together with a side condition.
If no source is given,
then the relationship is easy
and a proof can be found in this chapter.
}
\label{fig:learning-overview}
\end{figure}

Ultimately, whether learning succeeds depends on the rate at which
the nonnegative $P$-martingale $Q/P$ goes to $0$ (when drawing from $P$).
If $Q/P$ does not converge to zero,
then $Q$ merges strongly with $P$ and
thus arrives at correct beliefs about any hypothesis, including tail events%
\index{tail event}.
If $Q/P$ converges to zero subexponentially,
then $Q$ merges almost weakly with $P$ and
thus asymptotically has incorrect beliefs about the immediate future
only a vanishing fraction of the time.

\autoref{cor:expected-prediction-regret} bounds the expected prediction regret
by the KL-divergence between $P$ and $Q$ plus a $\sqrt{\EE_P E^P_t}$ term.
The KL-divergence is in turn bounded by the rate at which $Q/P$ goes to zero.
It is constant if $Q$ dominates $P$
and bounded by $\ln f$ if $Q$ dominates $P$ with coefficients $f$.
If $Q$ weakly dominates $P$, then the KL-divergence is sublinear.
We also derived bounds on the prediction regret for absolute continuity%
~(\autoref{ssec:prediction-absolute-continuity}).
Remarkably, the bounds are only $\log \log t$ worse
than the bound we get from dominance.
Moreover, they hold almost surely instead of in expectation.

Next, we showed that the $\sqrt{\EE_P E^P_t}$ term is generally unimprovable%
~(\autoref{ex:prediction-lower-bounds}).
However, it comes only from predicting measures that assign probabilities
close to $1/2$.
If we can bound $P$ away from $1/2$,
then the $\sqrt{\EE_P E^P_t}$ term disappears%
~(\autoref{thm:expected-prediction-regret-nonuniform} and
\autoref{thm:prediction-regret-nonuniform}).

\autoref{tab:learning-distributions} lists our learning distributions.
The Bayesian mixture is the strongest since it dominates every measure
from the given class $\M$~(\autoref{ex:Bayesian-mixture}).
The minimum description length model $\MDL^x$ does not have this property,
yet it still merges strongly with every measure from the class%
~(\autoref{ex:MDL} and \autoref{thm:MDL-merges-strongly}).
The Laplace rule is only useful for learning i.i.d.\ measures;
it merges weakly with every Bernoulli process%
~(\autoref{ex:Laplace-rule} and \autoref{ex:Laplace-rule2}).
We also discussed some learning distributions
from algorithmic information theory.
Solomonoff's prior
is a Bayesian mixture over all lower semicomputable semimeasures%
~(\autoref{ex:Solomonoff-prior} and \citealp{WSH:2011}).
Like the universal compressor it dominates and hence merges strongly with
all computable measures.
The speed prior dominates all probability measures estimable in polynomial time
with polynomial coefficients~(\autoref{ex:speed-prior}),
and thus merges weakly with each of them.

\autoref{tab:learning-overview} summarizes the results from this chapter and
\autoref{fig:learning-overview} illustrates their logical relationship
and their origin.

We conclude this chapter with a paradox from the philosophy of science.

\begin{remark}[The Paradox of Confirmation]
\label{rem:paradox-of-confirmation}
\index{black ravens}
Recall the black raven problem
introduced in \autoref{ex:black-ravens};
the hypothesis `all ravens are black' is denoted with $H$.
The \emph{paradox of confirmation}\index{paradox of confirmation},
also known as \emph{Hempel's paradox}%
\index{Hempel's paradox|see {paradox of confirmation}}~\citep{Hempel:1945},
relies on the following three principles.
\begin{itemize}
\item \emph{Nicod's criterion}\index{Nicod's criterion}%
	~\citep[p.~67]{Nicod:1961}:
	observing an $F$ that is a $G$
	increases our belief in the hypothesis that all $F$s are $G$s.
\item \emph{The equivalence condition}\index{equivalence condition}:
	logically equivalent hypotheses are confirmed or disconfirmed
	by the same evidence.
\item \emph{The paradoxical conclusion}:
	a green apple confirms $H$.
\end{itemize}
The argument goes as follows.
The hypothesis $H$ is logically equivalent to
the hypothesis $H'$ that all non-black objects are non-ravens.
According to Nicod's criterion,
any non-black non-raven, such as a green apple, confirms $H'$.
But then the equivalence condition
entails the paradoxical conclusion.

The paradox of confirmation has been discussed extensively
in the literature on the philosophy of science%
~\citep{Hempel:1945,Good:1960,Mackie:1963,Good:1967,Hempel:1967,Maher:1999,Vranas:2004};
see \citet{Swinburne:1971} for a survey.
Support for Nicod's criterion is not uncommon%
~\citep{Mackie:1963,Hempel:1967,Maher:1999} and no consensus is in sight.

A Bayesian reasoner might be tempted to argue that a green apple
\emph{does} confirm the hypothesis $H$, but only to a small degree,
since there are vastly more non-black objects than ravens~\citep{Good:1960}.
This leads to the acceptance of the paradoxical conclusion,
and this solution to the confirmation paradox is known as
the \emph{standard Bayesian solution}.
\citet{Vranas:2004} shows that
this solution is equivalent to
the assertion that blackness is equally probable
regardless of whether $H$ holds:
$P(\text{black} | H) \approx P(\text{black})$.

The following is a very concise example
against the standard Bayesian solution by \citet{Good:1967}:
There are two possible worlds,
the first has 100 black ravens and a million other birds,
while the second has 1000 black ravens, one white raven, and
a million other birds.
Now we draw a bird uniformly at random, and it turns out to be a black raven.
Contrary to what Nicod's criterion claims,
this is strong evidence that we are in fact in the second world,
and in this world non-black ravens exist.

For another, more intuitive example:
Suppose you do not know anything about ravens and
you have a friend who collects atypical objects.
If you see a black raven in her collection,
surely this would not
increase your belief in the hypothesis that all ravens are black.


In \citet{LH:2015ravens} we investigate the paradox of confirmation
in the context of Solomonoff induction.
We show that the paradoxical conclusion is avoided
because Solomonoff induction violates Nicod's criterion:
There are time steps when (counterfactually) observing a black raven
disconfirms the hypothesis that all ravens are black.
When predicting a deterministic computable sequence
Nicod's criterion is even violated infinitely often.
However, if we \emph{normalize}\index{Solomonoff!normalization}
Solomonoff's prior
and observe a deterministic computable infinite string,
Nicod's criterion is violated at most finitely many times.
These results are independent of the choice of the universal Turing machine.

We must conclude that violating Nicod's criterion is not
a fault of Solomonoff induction.
Instead, we should accept that for Bayesian reasoning
Nicod's criterion, in its generality, is false!
Quoting the great Bayesian master \citet[p.\ 144]{Jaynes:2003}:
\begin{quote}
In the literature there are perhaps 100 `paradoxes' and controversies
which are like this,
in that they arise from faulty intuition rather than faulty mathematics.
Someone asserts a general principle that seems to him intuitively right.
Then, when probability analysis reveals the error,
instead of taking this opportunity to educate his intuition,
he reacts by rejecting the probability analysis.
\qedhere
\end{quote}
\end{remark}


\chapter{Acting}
\label{cha:acting}

\falsequote{Immanuel Kant}{I ought never to act except in such a way that I could also will that my maxim should become a universal prior.}

Recall our decomposition of intelligence into
\emph{learning} and \emph{acting} from \autoref{eq:learning+acting}.
The previous chapter made the notion of learning precise and
provided several examples of learning distributions
for the non-i.i.d.\ setting~(see \autoref{tab:learning-distributions}).
Learning is passive:
there is no interaction with the data-generating process.
In this chapter we transition into the active setting:
we consider an \emph{agent}\index{agent}
acting in an unknown \emph{environment}\index{environment}
in order to achieve a \emph{goal}\index{goal}.
In our case, this goal is maximizing reward;
this is known as \emph{reinforcement learning}.
Where this reward signal originates does not concern us here.

In this thesis we consider is the \emph{general reinforcement learning problem}%
\index{general reinforcement learning problem}
in which we do not make several of the typical simplifying assumptions%
~(see \autoref{tab:assumptions-in-rl}).
Environments are only partially observable\index{partially observable},
have infinitely many states, and
might contain traps\index{trap} from which the agent cannot escape.
The context for making decision is
the agent's entire history;
its behavior is given by a \emph{policy}\index{policy}
that specifies how the agent behaves in any possible situation.

A central quantity in reinforcement learning is the \emph{value function}%
\index{value function}.
The value function quantifies the expected future discounted reward.
Since the agent seeks to maximize reward,
it aims to adopt a \emph{policy}\index{policy} that has high value.
Since the agent's environment is unknown to the agent,
learning the value function is part of the challenge;
otherwise we would call this planning\index{planning}.

If our agent is capable of learning in the sense of \autoref{cha:learning},
then it learns the value of its own policy%
~(\emph{on-policy value convergence}\index{on-policy!value convergence}).
However, generally the agent does not learn to predict
the value of counterfactual actions, actions that it does not take.
Learning off-policy\index{off-policy} is hard because
the agent receives no evidence about
what would have happened on counterfactual actions.
Nevertheless, off-policy learning is highly desirable because
we want the agent to be confident that the policy it is currently following
is in fact the best one;
we want it to accurately predict that the counterfactual actions have less value.

This brings us back to the central theme of reinforcement learning:
the tradeoff between \emph{exploration} and \emph{exploitation}%
\index{exploration vs.\ exploitation}.
Asymptotically the agent needs to focus on exploitation,
i.e., take the actions that it thinks yield the highest expected rewards.
If the agent explores enough,
then all actions are on-policy
because they are all actions that the agent sometimes takes.
Then on-policy learning ensures that the agent understands the consequences
of every action and can confidently choose the best action.
Effective exploration is performed by \emph{knowledge-seeking agents}%
\index{knowledge-seeking};
these agents ignore the rewards and just focus on exploration.

This chapter introduces the central concepts of general reinforcement learning.
It is mostly based on \citet{Hutter:2005} and \citet{Lattimore:2013}.
\autoref{sec:general-rl} specifies the general reinforcement learning problem,
discusses discounting~(\autoref{ssec:discounting}),
our implicit assumptions~(\autoref{ssec:implicit-assumptions}), and
typical environment classes~(\autoref{ssec:typical-environment-classes}).
\autoref{sec:the-value-function} discusses
the value function and its properties.
In \autoref{sec:agents} we introduce the agents:
AIXI~(\autoref{ssec:AIXI}),
knowledge-seeking agents~(\autoref{ssec:knowledge-seeking-agents}),
BayesExp~(\autoref{ssec:BayesExp}), and
Thompson sampling~(\autoref{ssec:Thompson-sampling}).

\section{The General Reinforcement Learning Problem}
\label{sec:general-rl}


In reinforcement learning,
an agent interacts with an environment:
at time step $t \in \mathbb{N}$ the agent takes
an \emph{action}\index{action|textbf} $a_t \in \A$ and
subsequently receives a \emph{percept}\index{percept|textbf}
$e_t = (o_t, r_t) \in \E$
consisting of an \emph{observation}\index{observation|textbf} $o_t \in \O$
and a \emph{reward}\index{reward|textbf} $r_t \in \mathbb{R}$.
This cycle then repeats for time step $t + 1$%
~(see \autoref{fig:dualistic-agent-model}).

\begin{figure}[t]
\begin{center}
\begin{tikzpicture}[scale=0.25] 
\draw (0,0) -- (10,0) -- (10,8) -- (0,8) -- (0,0);
\node[above left] at (10,0) {agent $\pi$};

\draw (18,0) -- (30,0) -- (30,8) -- (18,8) -- (18,0);
\node[above left] at (30,0) {environment $\mu$};

\draw[->] (10,4.5) to node[above] {$a_t$} (18,4.5);
\draw[<-] (10,3.5) to node[below] {$e_t = (o_t, r_t)$} (18,3.5);
\end{tikzpicture}
\end{center}
\caption[The dualistic agent model]{\index{agent model!dualistic}%
The dualistic agent model.
At every time step $t$, the agent outputs an action $a_t$ and
subsequently receives a percept $e_t$ consisting of
an observation $o_t$ and a real-valued reward $r_t$.
The agent's policy $\pi$ is a function that
maps a history $\ae_{<t}$ to the next action $a_t$, and
the environment $\mu$ is a function that
maps a history and an action to the next percept $e_t$.
}
\label{fig:dualistic-agent-model}
\end{figure}
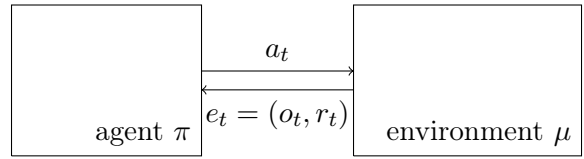

A \emph{history}\index{history|textbf} is an element of $\H$ and
lists the actions the agent took and the percepts it received.
We use $\ae \in \A \times \E$ to denote one interaction cycle,
and $\ae_{<t} = \ae_1 \ae_2 \ldots \ae_{t-1}$
to denote a history of length $t - 1$.
For our agent,
the history is a sufficient statistic about the past and
in general reinforcement learning there is no simpler sufficient statistic.

For example, consider the agent to be a robot interacting with the real world.
Its actions are moving the motors in its limbs and wheels and
sending data packets over a network connection.
Its observations are data from cameras and various other sensors.
The reward could be provided either by a human supervisor or
through a reward module that checks whether a predefined goal has been reached.
The history is the collection of
all the data it received and emitted in the past.
The division of the robot's interaction with the environment
into discrete time steps
might seem a bit unnatural at first
since the real world evolves according to a continuous process.
However, note that the electronic components used in robots operate
at discrete frequencies anyway.

In order to specify how the agent behaves in any possible situation,
we define a \emph{policy}\index{policy|textbf}:
a policy is a function $\pi: \H \to \Delta\A$
mapping a history $\ae_{<t}$ to a distribution over actions
$\pi(\,\cdot \mid \ae_{<t})$ taken after seeing this history.
Usually we do not distinguish between agent\index{agent} and policy.
An \emph{environment}\index{environment|textbf} is
a function $\nu: \H \times \A \to \Delta\E$ mapping
a history $\ae_{<t}$ and an action $a_t$
to a distribution $\nu(\,\cdot \mid \ae_{<t}a_t)$ over the percepts
received after the history $\ae_{<t}$ and action $a_t$.
We use $\mu$ to denote the true environment.

Equivalently, \citet{Hutter:2005} defines environments
as \emph{chronological contextual semi\-measures}.\footnote{%
	\citet{Hutter:2005} calls them
	\emph{chronological conditional semimeasures}.
	This is confusing because contextual semimeasures
	do \emph{not} specify conditional probabilities;
	the environment is \emph{not}
	a joint probability distribution over actions and percepts.
}
A \emph{contextual semimeasure}\index{semimeasure!contextual|textbf}
$\nu$ takes a sequence of actions $a_{1:\infty}$ as input
and returns a semimeasure $\nu(\,\cdot \dmid a_{1:\infty})$ over $\E^\sharp$.
A contextual semimeasure $\nu$ is \emph{chronological}%
\index{semimeasure!chronological|textbf} iff
percepts at time $t$ do not depend on future actions, i.e.,
$\nu(e_{1:t} \dmid a_{1:\infty}) = \nu(e_{1:t} \dmid a'_{1:\infty})$
whenever $a_{1:t} = a'_{1:t}$.
For chronological contextual semimeasures we write $\nu(e_{1:t} \dmid a_{1:t})$
instead of $\nu(e_{1:t} \dmid a_{1:\infty})$.
The two definition can be translated using the identities
\begin{equation}\label{eq:ccs}
  \nu(e_{1:t} \dmid a_{1:t})
= \prod_{k=1}^t \nu(e_k \mid \ae_{<k} a_k)
\qquad\text{and}\qquad
  \nu(e_t \mid \ae_{<t} a_t)
= \frac{\nu(e_{1:t} \dmid a_{1:t})}{\nu(e_{<t} \dmid a_{<t})}.
\end{equation}

If the policy $\pi$ always assigns probability $1$ to one of the actions,
then $\pi$ is called \emph{deterministic}\index{policy!deterministic|textbf}.
Likewise,
if the environment $\nu$ always assigns probability $1$ to one of the percepts,
then $\nu$ is called \emph{deterministic}\index{environment!deterministic|textbf}.
For deterministic policies and environments we also use the notation
$a_t = \pi(\ae_{<t})$ and $e_t = \nu(\ae_{<t} a_t)$.
A deterministic policy $\pi$ is \emph{consistent with history $\ae_{<t}$}%
\index{policy!consistent with history|textbf} iff
$a_k = \pi(\ae_{<k})$ for all $k < t$.
Likewise,
a deterministic environment $\nu$ is \emph{consistent with history $\ae_{<t}$}%
\index{policy!consistent with history|textbf} iff
$e_k = \nu(\ae_{<k}a_k)$ for all $k < t$.

\begin{definition}[History Distribution]
\label{def:history-distribution}\index{history!distribution|textbf}
An environment $\nu$ together with a policy $\pi$ induces
a \emph{history distribution}
\[
   \nu^\pi(\ae_{<t})
:= \prod_{k=1}^t \pi(a_k \mid \ae_{<k}) \nu(e_k \mid \ae_{<k}a_k).
\]
We denote an expectation with respect to the history distribution $\nu^\pi$
with $\EE^\pi_\nu$.
\end{definition}

The history distribution is a (semi)measure on $(\A \times \E)^\infty$.
In the language of measure theory,
our $\sigma$-algebra\index{s-algebra@$\sigma$-algebra} is
the $\sigma$-algebra $\Foo$ generated by the cylinder sets\index{cylinder set}
introduced in \autoref{sec:measure-theory}.
The filtration\index{filtration} $(\F_t)_{t \in \mathbb{N}}$ formalizes that
at time step $t$ we have seen exactly the history $\ae_{<t}$
(we use the $\sigma$-algebra $\F_{t-1}$).
To simplify notation and help intuition,
we simply condition expectations and probability measures
with the history $\ae_{<t}$ instead of $\F_{t-1}$
and sweep most of the measure-theoretic details under the rug.

With these preliminaries out of the way,
we can now specify the \emph{general reinforcement learning problem}.

\begin{problem}[General Reinforcement Learning Problem]
\label{prob:general-RL}
\index{general reinforcement learning problem|textbf}
\index{environment class}
Given an arbitrary class of environments $\M$,
choose a policy $\pi$ that maximizes $\mu^\pi$-expected reward
when interacting with any environment $\mu \in \M$.
\end{problem}

\autoref{prob:general-RL} is kept vague on purpose:
it does not say how we should balance between
achieving more rewards in some environments
while achieving less in others.
In other words, we leave open
what an \emph{optimal} solution to
the general reinforcement learning problem is.
This turns out to be a notoriously difficult question
that we discuss in \autoref{cha:optimality}.

\index{M@$\M$|see {environment class}}
As promised in the title of this thesis,
we take the \emph{nonparametric}\index{nonparametric} approach.
For the rest of this thesis,
fix $\M$\index{environment class} to be any countable set of environments.
While the true environment is unknown, we assume it belongs to the class $\M$
(the \emph{realizable case}\index{realizable case}).
As long as the class $\M$ is sufficiently large
(such as the class of all computable environments),
this assumption is weak.
Some typical choices are discussed in \autoref{ssec:typical-environment-classes}.

\index{agent model!dualistic}\index{agent model!physicalistic}%
Our agent-environment setup
shown in \autoref{fig:dualistic-agent-model}
is known as the \emph{dualistic model}\index{dualistic model}:
the agent is distinct from the environment and
influences it only through its actions.
In turn,
the environment influences the agent only through the percepts.
The dualism assumption is accurate for
an algorithm that is playing chess, Go, or other (video) games,
which explains why it is ubiquitous in AI research.
But often it is not true:
real-world agents are embedded in (and computed by) the environment,
and then a \emph{physicalistic model}\index{physicalistic model}
(also called \emph{materialistic model} or \emph{naturalistic} model)
is more appropriate.
Decision making in the physicalistic model is
still underdeveloped;
see \citet{ELH:2015sdt} and \citet{OR:2012}.
In this thesis we restrict ourselves to the dualistic model.

\subsection{Discounting}
\label{ssec:discounting}

The goal in reinforcement learning is to maximize rewards.
However, the infinite reward sum $\sum_{t=1}^\infty r_t$ may diverge.
To get around this technical problem,
we let our agent prioritize the present over the future.
This is done with a \emph{discount function}\index{discount!function}
that quantifies how much the agent prefers rewards now over rewards later.

\begin{definition}[Discount Function]
\label{def:discounting}\index{discount!function|textbf}
A \emph{discount function} is a function
$\gamma: \mathbb{N} \to \mathbb{R}$ with
$\gamma_t := \gamma(t) \geq 0$ and $\sum_{t=1}^\infty \gamma_t < \infty$.
The \emph{discount normalization factor}%
\index{discount!normalization factor|textbf} is
$\Gamma_t := \sum_{k=t}^\infty \gamma_k$.
\end{definition}

There is no requirement that $\Gamma_t > 0$.
In fact, we use $\gamma$ for both,
discounted infinite horizon~($\Gamma_t > 0$ for all $t$), and
finite horizon $m$~($\Gamma_{m-1} > 0$ and $\Gamma_m = 0$)
where the agent does not care what happens after time step $m$.

Note that the way in which we employ discounting is
\emph{time consistent}\index{time consistent}:
the agent does not change its mind about
how much it values the reward at time step $k$ over time:
reward $r_k$ is always discounted with $\gamma_k$
regardless of the current time step.
For a discussion of general discounting we refer the reader to
\citet{LH:2014discounting}.

\begin{definition}[Effective Horizon]
\label{def:effective-horizon}
\index{effective horizon|textbf}
\index{effective horizon!bounded|textbf}
The \emph{$\eps$-effective horizon}
$H_t(\varepsilon)$ is a horizon
that is long enough to encompass all but an $\varepsilon$
of the discount function's mass:
\[
H_t(\varepsilon) := \min \{ k \mid \Gamma_{t+k} / \Gamma_t \leq \varepsilon \}
\]
The effective horizon is \emph{bounded} iff
for all $\eps > 0$ there is a constant $c_\eps$ such that
$H_t(\eps) \leq c_\eps$ for all $t \in \mathbb{N}$.
\end{definition}

\begin{example}[Geometric Discounting]
\label{ex:geometric-discounting}\index{discounting!geometric}
The most common discount function is \emph{geometric discounting} with
$\gamma_t := \gamma^t$ for some constant $\gamma \in [0, 1)$.
We get that $\Gamma_t = \sum_{k=t}^\infty \gamma^k = \gamma^t / (1 - \gamma)$
and the $\eps$-effective horizon is $H_t(\eps) = \lceil \log_\gamma \eps \rceil$.
Hence the effective horizon is bounded.
\end{example}

More examples for discount functions are given in \autoref{tab:discounting}.
From now on, we fix a discount function $\gamma$.

\begin{table}[t]
\begin{center}
\begin{tabular}{lllll}
\toprule
& parameter & $\gamma_t$ & $\Gamma_t$ & $H_t(\eps)$ \\
\midrule
Finite horizon & $m \in \mathbb{N}$ & $\one_{\leq m}(t)/m$ & $(m-t+1)/m$ & $\lceil (m - t + 1)(1 - \eps) \rceil$ \\
Geometric & $\gamma \in (0,1)$ & $\gamma^t$ & $\gamma^t / (1 - \gamma)$ & $\lceil \log_\gamma \eps \rceil$ \\
Power & $\beta > 1$ & $t^{-\beta}$ & $\approx t^{-\beta+1} / (\beta-1)$ & $\approx (\eps^{1/(1-\beta)} - 1)t$ \\
Subgeometric & - & $e^{-\sqrt{t}} / \sqrt{t}$ & $\approx 2e^{-\sqrt{t}}$ & $\approx -\sqrt{t}\log \eps + (\log \eps)^2$ \\
\bottomrule
\end{tabular}
\end{center}
\caption[Discount functions and their effective horizons]{
Several discount functions and their effective horizons.
See also \citet[Tab.~5.41]{Hutter:2005} and \citet[Tab.~2.1]{Lattimore:2013}.
\index{discounting!geometric}\index{discounting!power}\index{discounting!finite horizon}\index{discounting!subgeometric}
}
\label{tab:discounting}
\end{table}

\subsection{Implicit Assumptions}
\label{ssec:implicit-assumptions}

Throughout this thesis,
we make the following assumptions implicitly.

\begin{assumption}
\label{ass:aixi}
\begin{enumerate}[(a)]
\item \label{ass:gamma-computable}\index{discount!function}
	The discount function $\gamma$ is computable.
\item \label{ass:bounded-rewards}\index{reward}
	Rewards are bounded between $0$ and $1$.
\item \label{ass:finite-actions-and-percepts}\index{action}\index{percept}
	The set of actions $\A$ and the set of percepts $\E$
	are both finite.
\end{enumerate}
\end{assumption}

Let's motivate these assumptions in turn.
Their purpose is to ensure that discounted reward sums are finite and
optimal policies exist.

\assref{ass:aixi}{ass:gamma-computable}
is a technical assumption that ensures that
discounted reward sums are computable.
This is important for \autoref{cha:computability}
and \autoref{cha:grain-of-truth-problem}
where we analyse the computability of optimal policies.
Note that
all discount functions given in \autoref{tab:discounting} are computable.

\assref{ass:aixi}{ass:bounded-rewards}
could be relaxed to require only that rewards are bounded.
We can rescale rewards $r_t \mapsto cr_t + d$
for any $c \in \mathbb{R}^+$ and $d \in \mathbb{R}$
without changing optimal policies
if the environment $\nu$ is a probability measure.
(For our computability-related results in \autoref{cha:computability},
we must assume that rewards are nonnegative.)
In this sense \assref{ass:aixi}{ass:bounded-rewards} is not very restrictive.
However, this normalization of rewards into the $[0,1]$-interval
has the convenient consequence that
the normalized discounted reward sum
$\sum_{k=t}^\infty \gamma_k r_k / \Gamma_k$
is bounded between $0$ and $1$.
If rewards are unbounded, then the discounted reward sum might diverge.
Moreover, with unbounded rewards there all kinds of pathological problems
where defining optimal actions is no longer straightforward;
see \citet{AEH:2004} for a discussion.

\assref{ass:aixi}{ass:finite-actions-and-percepts}
is a technical requirement for the existence of optimal policies
since it implies that there are only finitely many deterministic policies
that differ in the first $t$ time steps.
Note that finite action and percept spaces are very natural since
it ensures that our agent only receives and emits a finite amount of information
in every time step.
This is in line with the problems a strong AI is facing:
the agent has to remember important information and act sequentially.

\assref{ass:aixi}{ass:bounded-rewards},
\assref{ass:aixi}{ass:finite-actions-and-percepts}, and
the fact that the discount function is summable
guarantee that a deterministic optimal policy exists for every environment
according to \citet[Thm.~10]{LH:2014discounting}.
It would be interesting to relax these assumptions
while preserving the existence of optimal policies
or at least $\eps$-optimal policies
(e.g.\ use compact action and percept spaces).

\subsection{Typical Environment Classes}
\label{ssec:typical-environment-classes}\index{environment class}

The simplest reinforcement learning problems are multi-armed bandits.

\begin{definition}[Multi-Armed Bandit]
\label{def:bandit}\index{multi-armed bandit|textbf}
\index{bandit|see {multi-armed bandit}}
An environment $\nu$ is a \emph{multi-armed bandit} iff
$\O = \{ \bot \}$ and $\nu(e_t \mid \ae_{<t}a_t) = \nu(e_t \mid a_t)$
for all histories $\ae_{1:t} \in \H$.
\end{definition}

In a multi-armed bandit problem
there are no observations and
the next reward only depends on the previous action.
Intuitively, we are deciding between $\#\A$ different slot machines
(so-called one-armed bandits), pull the lever and obtain a reward.
The reward is stochastic, but it is drawn from a distribution that is
time-invariant and fixed for each arm.

A multi-armed bandit is also called \emph{bandit} for short.
Although bandits are the simplest reinforcement learning problem,
they already exhibit the exploration-exploitation-tradeoff%
\index{exploration vs.\ exploitation}
that makes reinforcement learning difficult:
do you pull an arm that has the best empirical mean or
do you pull an arm that has the highest uncertainty?
In bandits it is very easy to come up with policies that perform (close to)
optimal asymptotically~(e.g., $\eps_t$-greedy with $\eps_t = 1/t$).
But coming up with algorithms that perform well in practice is difficult,
and research focuses on the multiplicative and additive constants
on the asymptotic guarantees.
Bandits exist in many flavors; see \citet{BCB:2012bandits} for a survey.

\begin{definition}[Markov Decision Process]
\label{def:MDP}\index{MDP|textbf}
\index{Markov decision process|see {MDP}}
An environment $\nu$ is a \emph{Markov decision process}~(MDP) iff
$\nu(e_t \mid \ae_{<t}a_t) = \nu(e_t \mid o_{t-1}a_t)$
for all histories $\ae_{1:t} \in \H$.
\end{definition}

Intuitively, in MDPs, the previous observation $o_{t-1}$
provides a sufficient statistic for the history:
given $o_{t-1}$ and the current action $a_t$,
the next percept $e_t$ is independent of the rest of the history.
In other words, everything that the agent needs to know
to make optimal decisions
is readily available in the previous percept.
This is why observations are called \emph{states}\index{state} in MDPs.
Note that bandits are MDPs with a single state.

Much of today's literature on reinforcement learning focuses on MDPs%
~\citep{SB:1998}.
They provide a particularly good framework to study reinforcement learning
because they are simple enough to be tractable for today's algorithms,
yet general enough to encompass many interesting problems.
For example, most of the Atari games\index{Atari 2600}%
~(see \autoref{fig:Atari} for an overview)
are (deterministic) MDPs
when combining the previous four frames into one percept.
While they have a huge state space\footnote{%
The size of the state space is at most $256^{128}$
since the Atari 2600 has only 128 bytes of memory.
However, the vast majority of these states are not reachable.}
they can still be learned
using $Q$-learning with function approximation~\citep{MKSRV+:2015deepQ}.

The MDP framework is restrictive because it requires the agent
to be more powerful than the environment.
Since the agent learns,
its actions are not independent of the rest of the history
given the last action and percept.
In other words, learning agents are not Markov.
The following definition lifts this restriction
and allows the environment to be
\emph{partially observable}\index{partially observable}.

\begin{definition}[Partially Observable Markov Decision Process]
\label{def:POMDP}\index{POMDP|textbf}
An environment $\nu$ is a \emph{partially observable Markov decision process}%
~(POMDP) iff
there is a \emph{set of states} $\S$,
an \emph{initial state} $s_0 \in \S$,
a \emph{state transition function} $\nu': \S \times \A \to \Delta\S$, and
a \emph{percept distribution} $\nu'': \S \to \Delta\E$ such that
\[
  \nu(e_{1:t} \dmid a_{1:t})
= \prod_{k=1}^t \nu''(e_k \mid s_k) \nu'(s_k \mid s_{k-1}, a_k).
\]
\end{definition}

Usually the set $\S$ is assumed to be finite;
with infinite-state POMDPs we can model any environment $\nu$
by setting the set of states to be the set of histories, $\S := \H$.

A common assumption for MDPs and POMDPs is that
they do not contain traps\index{trap}.
Formally, a (PO)MDP is \emph{ergodic}\index{ergodic|textbf} iff
for any policy $\pi$ and any two states $s_1, s_2 \in \S$,
the expected number of time steps to reach $s_2$ from $s_1$
is $\mu^\pi$-almost surely finite.
A (PO)MDP is \emph{weakly communicating}\index{weakly communicating|textbf} iff
for any two states $s_1, s_2 \in \S$ there is a policy $\pi$ such that
the expected number of time steps to reach $s_2$ from $s_1$
is $\mu^\pi$-almost surely finite.
Note that any ergodic (PO)MDP is also weakly communicating,
but not vice versa.

In general, our environments are stochastic.
Stochasticity can originate from noise in the environment,
noise in the sensors,
or modeling errors.
Sometimes we also consider classes of deterministic environments.
These are usually easier to deal with because
they do not require as much mathematical machinery.
For example,
in a deterministic environment the next percept is certain;
if a different percept is received this environment is immediately falsified
and can be discarded.
In a stochastic environment,
an unlikely percept reduces our posterior belief in this environment
but does not rule it out completely.

In \autoref{cha:computability} and \autoref{cha:grain-of-truth-problem}
we make the assumption that the environment is computable.
This encompasses all finite-state POMDPs and
most if not all AI problems can be formulated in this setting.
Moreover, the current theories of
quantum mechanics and general relativity are computable and
there is no evidence that suggests that our physical universe is incomputable.
For any physical system of finite volume and finite (average) energy,
the amount of information it can contain is finite~\citep{Bekenstein:1981}, and
so is the number of state transitions per unit of time~\citep{ML:1998}.
This gives us reason to believe that
even the environment that we humans currently face (and will ever face)
falls under these assumptions.

Formally we define the set $\Mlscccs$ as
the set of environments that are lower semicomputable
chronological contextual semimeasures
and $\Mcomp$ as
the set of environments that are computable chronological contextual measures.
Note that for chronological contextual semimeasures
it makes a difference whether
$\nu(\,\cdot \dmid a_{1:\infty})$ is lower semicomputable or
the conditionals $\nu(\,\cdot \mid \ae_{<t} a_t)$ are.
The latter implies the former, but not vice versa.

\section{The Value Function}
\label{sec:the-value-function}

The \emph{value} of a policy in an environment is
the future expected discounted reward
when following a given policy
in a given environment conditional on the past.
Since this quantity captures exactly
what our agent aims to maximize,
we prefer policies whose value is high.

\begin{definition}[Value Function]
\label{def:value-function}
\index{value function|textbf}
\index{value function!recursive|textbf}
\index{horizon|textbf}
The \emph{value} of a policy $\pi$ in an environment $\nu$
given history $\ae_{<t}$ and \emph{horizon} $m$ with $t \leq m \leq \infty$
is defined as
\[
   V^{\pi,m}_\nu(\ae_{<t})
:= \frac{1}{\Gamma_t} \EE^\pi_\nu \left[
     \sum_{k=t}^{m-1} \gamma_k r_k
   \;\middle|\; \ae_{<t} \right]
\]
if $\Gamma_t > 0$ and $V^{\pi,m}_\nu(\ae_{<t}) := 0$ if $\Gamma_t = 0$.
The \emph{optimal value} is defined as
$V^{*,m}_\nu(\ae_{<t}) := \sup_\pi V^{\pi,m}_\nu(\ae_{<t})$.
\end{definition}

Sometimes we omit the history argument $\ae_{<t}$ for notational convenience
if it is clear from context.
Moreover, when we omit $m$, we implicitly use an infinite horizon $m = \infty$,
i.e., $V^\pi_\nu := V^{\pi,\infty}_\nu$ and $V^*_\nu := V^{*,\infty}_\nu$.
The value of a policy $\pi$ in an environment $\nu$ after the empty history,
$V^\pi_\nu(\epsilon)$ is also called the \emph{$t_0$-value}%
\index{t0 value@$t_0$-value|textbf}.

\begin{remark}[Values are Bounded Between 0 and 1]
\label{rem:value-function-bounded}
From \assref{ass:aixi}{ass:bounded-rewards} we get that
for all histories $\ae_{<t}$ all policies $\pi$ and all environments $\nu$,
the value function $V^\pi_\nu(\ae_{<t}) \in [0, 1]$.
\end{remark}

Since environment and policy are stochastic,
the history $\ae_{<t}$ is random.
With abuse of notation we treat $\ae_{<t}$ sometimes as a concrete outcome
and sometimes as a random variable.
We also view the value of a policy $\pi$ in an environment $\nu$
as a sequence of random variables $(X_t)_{t \in \mathbb{N}}$ with
$X_t := V^\pi_\nu(\ae_{1:t})$ where the history $\ae_{1:t}$ is generated
stochastically by the agent's actual policy
interacting with the true environment $\mu$.
This view is helpful for some of the convergence results
(e.g., \autoref{thm:on-policy-value-convergence}
and \autoref{def:asymptotic-optimality})
in which we talk about the type of convergence
of this sequence of random variables.

\index{value function!recursive}%
The value function defined in \autoref{def:value-function}
is also called the \emph{recursive value function}%
\index{value function!recursive},
in contrast to iterative value function that
we discuss in \autoref{sec:iterative-value-function}.
The name of the recursive value function
originates from the following recursive identity
(analogously to \citealp[Eq.~4.12]{Hutter:2005}),
also called the \emph{Bellman equation}\index{Bellman equation|textbf}:
\begin{align*}
   V^\pi_\nu(\ae_{<t})
&= \sum_{a_t \in \A} \pi(a_t \mid \ae_{<t}) V^\pi_\nu(\ae_{<t} a_t) \\
   V^\pi_\nu(\ae_{<t} a_t)
&= \frac{1}{\Gamma_t} \sum_{e_t \in \E} \nu(e_t \mid \ae_{<t}a_t)
     \big( \gamma_t r_t + \Gamma_{t+1} V^\pi_\nu(\ae_{1:t}) \big)
\end{align*}

An explicit expression for the optimal value in environment $\nu$ is
\begin{equation}\label{eq:V-explicit}
  V^{*,m}_\nu(\ae_{<t})
= \frac{1}{\Gamma_t} \expectimax{\ae_{t:m-1}}
    \sum_{k=t}^{m-1} \gamma_k r_k
    \prod_{i=t}^k \nu(e_i \mid \ae_{<i} a_i),
\end{equation}
where $\expectimax{}$ denotes the max-sum-operator:
\[
   \expectimax{\ae_{t:m-1}}
:= \max_{a_t \in \A} \sum_{e_t \in \E} \ldots \max_{a_{m-1} \in \A} \sum_{e_{m-1} \in \E}
\]
For an explicit expression of $V^{*,\infty}_\nu(\ae_{<t})$
we can simply take the limit $m \to \infty$.

\subsection{Optimal Policies}
\label{ssec:optimal-policies}

An optimal policy is a policy that achieves the highest value:

\begin{definition}[{Optimal Policy; \citealp[Def.~5.19 \& 5.30]{Hutter:2005}}]
\label{def:optimal-policy}\index{policy!optimal|textbf}
A policy $\pi$ is \emph{optimal in environment $\nu$~($\nu$-optimal)} iff
for all histories
$\pi$ attains the optimal value:
$V^\pi_\nu(\ae_{<t}) = V^*_\nu(\ae_{<t})$
for all $\ae_{<t} \in \H$.
The action $a_t$ is an \emph{optimal action}\index{optimality!action|textbf}
iff
$\pi^*_\nu(a_t \mid \ae_{<t}) = 1$ for some $\nu$-optimal policy $\pi^*_\nu$.
\end{definition}

Following the tradition of \citet{Hutter:2005},
\emph{AINU}\index{AINU} denotes a $\nu$-optimal policy
for the environment $\nu \in \Mlscccs$ and
\emph{AIMU}\index{AIMU} denotes an $\mu$-optimal policy
for the environment $\mu \in \Mcomp$ that is a measure
(as opposed to a semimeasure).

By definition of the optimal policy and the optimal value function,
we have the following identity for all histories $\ae_{<t}$:
\begin{equation}\label{eq:optimal-policy} 
V^*_\nu(\ae_{<t}) = V^{\pi^*_\nu}_\nu(\ae_{<t})
\end{equation}

\index{argmax tie}
There can be more than one optimal policy;
generally the choice of $\pi^*_\nu$ from \autoref{def:optimal-policy}
is not unique.
More specifically, for a $\nu$-optimal policy we have
\begin{equation}\label{eq:argmax}
\pi^*_\nu(a_t \mid \ae_{<t}) > 0
\;\Longrightarrow\;
a_t \in \argmax_{a \in \A} V^*_\nu(\ae_{<t}a).
\end{equation}
If there are multiple actions $\alpha, \beta \in \A$
that attain the optimal value,
$V^*_\nu(\ae_{<} \alpha) = V^*_\nu(\ae_{<t} \beta)$,
then there is an \emph{argmax tie}\index{argmax tie}.
Which action we settle on in case of a tie (how we break the tie)
is irrelevant and can be arbitrary.
Since we allow stochastic policies,
we can also randomize between $\alpha$ and $\beta$.

The following definition allows policies to be slightly suboptimal.

\begin{definition}[$\varepsilon$-Optimal Policy]
\label{def:eps-optimal-policy}\index{policy!e-optimal@$\eps$-optimal|textbf}
A policy $\pi$ is \emph{$\varepsilon$-optimal in environment $\nu$} iff
$V^*_\nu(\ae_{<t}) - V^\pi_\nu(\ae_{<t}) < \varepsilon$
for all histories $\ae_{<t} \in \H$.
\end{definition}

A policy $\pi$ that achieves optimal $t_0$-value,
$V^\pi_\nu(\epsilon) = V^*_\nu(\epsilon)$,
takes $\nu$-optimal actions on any history reachable by $\pi$ in $\nu$.
However, this is not true for $\eps$-optimal policies:
a policy that is $\eps$-optimal at $t = 0$ is not necessarily
$\eps$-optimal in later time steps.

\subsection{Properties of the Value Function}
\label{ssec:properties-values}

The following two lemmas are stated by \citet[Thm.~31]{Hutter:2005}
without proof and for the iterative value function.

\begin{lemma}[Linearity of $V^\pi_\nu$ in $\nu$]
\label{lem:V-linear}\index{value function}
If $\nu = z_1\nu_1 + z_2\nu_2$ for some real numbers $z_1, z_2 \geq 0$,
then for all policies $\pi$ and all histories $\ae_{<t}$
\[
    V^{\pi,m}_\nu(\ae_{<t})
~=~ z_1 \frac{\nu_1(e_{<t} \dmid a_{<t})}{\nu(e_{<t} \dmid a_{<t})}
      V^{\pi,m}_{\nu_1}(\ae_{<t})
    + z_2 \frac{\nu_2(e_{<t} \dmid a_{<t})}{\nu(e_{<t} \dmid a_{<t})}
      V^{\pi,m}_{\nu_2}(\ae_{<t}).
\]
\end{lemma}
\begin{proof}
Since $\nu^\pi = z_1\nu_1^\pi + z_2\nu_2^\pi$,
we have for the conditional measure
\begin{align*}
   \nu^\pi(A \mid \ae_{<t})
 = \frac{\nu^\pi(A \cap \ae_{<t})}{\nu^\pi(\ae_{<t})}
&= \frac{z_1\nu_1^\pi(A \cap \ae_{<t}) + z_2\nu_2^\pi(A \cap \ae_{<t})}%
        {\nu^\pi(\ae_{<t})} \\
&= z_1 \frac{\nu_1^\pi(\ae_{<t})}{\nu^\pi(\ae_{<t})} \nu_1^\pi(A \mid \ae_{<t})
   + z_2 \frac{\nu_2^\pi(\ae_{<t})}{\nu^\pi(\ae_{<t})} \nu_2^\pi(A \mid \ae_{<t}).
\end{align*}
The claim now follows from the linearity of expectation in the probability measure.
\end{proof}

\begin{lemma}[Convexity of $V^*_\nu$ in $\nu$]
\label{lem:V*-convex}\index{value function}
If $\nu = z_1\nu_1 + z_2\nu_2$ for some real numbers $z_1, z_2 \geq 0$,
then for all histories $\ae_{<t}$
\[
       V^*_\nu(\ae_{<t})
~\leq~ z_1 \frac{\nu_1(e_{<t} \dmid a_{<t})}{\nu(e_{<t} \dmid a_{<t})}
         V^{*,m}_{\nu_1}(\ae_{<t})
       + z_2 \frac{\nu_2(e_{<t} \dmid a_{<t})}{\nu(e_{<t} \dmid a_{<t})}
         V^{*,m}_{\nu_2}(\ae_{<t}).
\]
\end{lemma}
\begin{proof}
Let $\pi^*_\nu$ be an optimal policy for environment $\nu$.
From \autoref{lem:V-linear} we get
\begin{align*}
      V^*_\nu(\ae_{<t})
 =    V^{\pi^*_\nu}_\nu(\ae_{<t})
&=    z_1 \frac{\nu_1(e_{<t} \dmid a_{<t})}{\nu(e_{<t} \dmid a_{<t})}
        V^{\pi^*_\nu}_{\nu_1}(\ae_{<t})
      + z_2 \frac{\nu_2(e_{<t} \dmid a_{<t})}{\nu(e_{<t} \dmid a_{<t})}
        V^{\pi^*_\nu}_{\nu_2}(\ae_{<t}) \\
&\leq z_1 \frac{\nu_1(e_{<t} \dmid a_{<t})}{\nu(e_{<t} \dmid a_{<t})}
        V^*_{\nu_1}(\ae_{<t})
      + z_2 \frac{\nu_2(e_{<t} \dmid a_{<t})}{\nu(e_{<t} \dmid a_{<t})}
        V^*_{\nu_2}(\ae_{<t}).
\qedhere
\end{align*}
\end{proof}

The following lemma bounds the error when truncating the value function.
This implies that planning for an $\eps$-effective horizon%
~($m = t + H_t(\eps)$),
we get all but an $\eps$ of the value:
$|V^\pi_\nu(\ae_{<t}) - V^{\pi,m}_\nu(\ae_{<t})| < \eps$.

\begin{lemma}[Truncated Values]
\label{lem:truncated-values}\index{value function}
For every environment $\nu$, every policy $\pi$, and every history $\ae_{<t}$
\[
       \big| V^{\pi,m}_\nu(\ae_{<t}) - V^\pi_\nu(\ae_{<t}) \big|
~\leq~ \frac{\Gamma_m}{\Gamma_t}.
\]
\end{lemma}
\begin{proof}
\[
  V^\pi_\nu(\ae_{<t})
= \frac{1}{\Gamma_t} \EE^\pi_\nu \left[
    \sum_{k=t}^\infty \gamma_k r_k \;\middle|\; \ae_{<t}
  \right]
= V^{\pi,m}_\nu(\ae_{<t})
  + \frac{1}{\Gamma_t} \EE^\pi_\nu \left[
    \sum_{k=m}^\infty \gamma_k r_k \;\middle|\; \ae_{<t}
  \right]
\]
The result now follows from \assref{ass:aixi}{ass:bounded-rewards} and
\[
     0
\leq \EE^\pi_\nu \left[
       \sum_{k=m}^\infty \gamma_k r_k \;\middle|\; \ae_{<t}
     \right]
\leq \Gamma_m.
\qedhere
\]
\end{proof}

This lemma bounds the (truncated) value function by the total variation distance.

\begin{lemma}[Bounds on Value Difference]
\label{lem:value-bound}\index{value function}
\index{total variation distance}
For any policies $\pi_1$, $\pi_2$, any environments $\nu_1$ and $\nu_2$,
and any horizon $t \leq m \leq \infty$
\[
       |V^{\pi_1,m}_{\nu_1}(\ae_{<t}) - V^{\pi_2,m}_{\nu_2}(\ae_{<t})|
~\leq~ D_{m-1}(\nu_1^{\pi_1}, \nu_2^{\pi_2} \mid \ae_{<t})
\]
\end{lemma}
\begin{proof}
According to \autoref{def:value-function},
the value function is the expectation of the random variable
$\sum_{k=t}^{m-1} \gamma_k r_k / \Gamma_t$ that is bounded between $0$ and $1$.
Therefore we can use \autoref{lem:expectation-total-variation} with
$P := \nu_1^{\pi_1}(\,\cdot \mid \ae_{<t})$ and
$R := \nu_2^{\pi_2}(\,\cdot \mid \ae_{<t})$
on the space $(\A \times \E)^{m-1}$
to conclude that
$|V^{\pi_1,m}_{\nu_1}(\ae_{<t}) - V^{\pi_2,m}_{\nu_2}(\ae_{<t})|$
is bounded by $D_{m-1}(\nu_1^{\pi_1}, \nu_2^{\pi_2} \mid \ae_{<t})$.
\end{proof}

\begin{lemma}[{Discounted Values; \citealp[Lem.~2.5]{Lattimore:2013}}]
\label{lem:discounted-values}\index{value function}
Let $\ae_{<t}$ be some history
and let $\pi_1$ and $\pi_2$ be two policies
that coincide from time step $t$ to time step $m$:
$\pi_1(a \mid \ae_{1:k}) = \pi_2(a \mid \ae_{1:k})$ for all $a \in \A$,
all histories $\ae_{<t}\ae_{t:k}$ consistent with $\pi_1$,
and $t \leq k \leq m$.
Then for all environments $\nu$
\[
       \big| V^{\pi_1}_\nu(\ae_{<t}) - V^{\pi_2}_\nu(\ae_{<t}) \big|
~\leq~ \frac{\Gamma_m}{\Gamma_t}.
\]
\end{lemma}
\begin{proof}
Since $\pi_1$ and $\pi_2$ coincide for time steps $t$ through $m-1$,
$D_{m-1}(\nu^{\pi_1}, \nu^{\pi_2} \mid \ae_{<t}) = 0$ for all environments $\nu$.
Thus the result follows from \autoref{lem:truncated-values} and
\autoref{lem:value-bound}:
\begin{align*}
      \big| V^{\pi_1}_\nu(\ae_{<t}) - V^{\pi_2}_\nu(\ae_{<t}) \big|
&\leq \big| V^{\pi_1,m}_\nu(\ae_{<t}) - V^{\pi_2,m}_\nu(\ae_{<t}) \big|
      + \frac{\Gamma_m}{\Gamma_t} \\
&\leq D_{m-1}(\nu^{\pi_1}, \nu^{\pi_2} \mid \ae_{<t})
      + \frac{\Gamma_m}{\Gamma_t} \\
&=    \frac{\Gamma_m}{\Gamma_t}
\qedhere
\end{align*}
\end{proof}

\subsection{On-Policy Value Convergence}
\label{ssec:on-policy-value-convergence}
\index{on-policy!value convergence}

This section states some general results on learning the value function.
\emph{On-policy value convergence} refers to the fact that if we use
a learning distribution $\rho$ to learn to environment $\mu$,
and $\rho^\pi$ merges\index{merging} with $\mu^\pi$
in the sense discussed \autoref{sec:merging},
then $V^\pi_\rho$ converges to $V^\pi_\mu$, i.e.,
using $\rho$ we learn to estimate values correctly.

A weaker variant of the following theorem
was proved by \citet[Thm.~5.36]{Hutter:2005}.
It states convergence in mean (not almost surely),
and only for the Bayesian mixture.

\begin{theorem}[On-Policy Value Convergence]
\label{thm:on-policy-value-convergence}
\index{on-policy!value convergence}
\index{convergence!almost surely}
Let $\mu$ be any environment and $\pi$ be any policy.
\begin{enumerate}[(a)]
\item\label{itm:on-policy-value-convergence-ac}\index{merging!strong}
	If $\rho^\pi$ merges strongly with $\mu^\pi$,
	then
	\[
	V^\pi_\rho(\ae_{<t}) - V^\pi_\mu(\ae_{<t}) \to 0
	\text{ as $t \to \infty$ $\mu^\pi$-almost surely.}
	\]

\item\label{itm:on-policy-value-convergence-wm}\index{merging!weak}
	If the effective horizon is bounded and
	$\rho^\pi$ merges weakly with $\mu^\pi$,
	then
	\[
	V^\pi_\rho(\ae_{<t}) - V^\pi_\mu(\ae_{<t}) \to 0
	\text{ as $t \to \infty$ $\mu^\pi$-almost surely.}
	\]

\item\label{itm:on-policy-value-convergence-awm}\index{merging!almost weak}
	If the effective horizon is bounded and
	$\rho^\pi$ merges almost weakly with $\mu^\pi$,
	then
	\[
	\frac{1}{t} \sum_{k=1}^t \Big( V^\pi_\rho(\ae_{<k}) - V^\pi_\mu(\ae_{<k}) \Big)
	\to 0
	\text{ as $t \to \infty$ in $\mu^\pi$-almost surely.}
	\]
\end{enumerate}
\end{theorem}
\begin{proof}
\begin{enumerate}[(a)]
\item
	Apply \autoref{lem:value-bound} with $m := \infty$.

\item
	Let $\eps > 0$ and let $c_\eps$ be a bound on $\sup_t H_t(\eps)$.
	From \autoref{lem:truncated-values}
	\begin{align*}
	      |V^\pi_\rho(\ae_{<t}) - V^\pi_\mu(\ae_{<t})|
	&\leq |V^{\pi,t+H_t(\eps)}_\rho(\ae_{<t}) - V^{\pi,t+H_t(\eps)}_\mu(\ae_{<t})|
	      + 2\frac{\Gamma_{t+H_t(\eps)}}{\Gamma_t} \\
	&<    D_{t+H_t-1(\eps)}(\rho^\pi, \mu^\pi \mid \ae_{<t}) + 2\eps \\
	&\leq D_{t+c_\eps}(\rho^\pi, \mu^\pi \mid \ae_{<t}) + 2\eps
	\end{align*}
	according to \autoref{def:effective-horizon} and \autoref{lem:value-bound}.
	Since $\rho^\pi$ merges weakly with $\mu^\pi$,
	we get that $\mu^\pi$-almost surely there is a time step $t_0 \in \mathbb{N}$
	such that $D_{t+c_\eps}(\rho^\pi, \mu^\pi \mid \ae_{<t}) < \eps$
	for all $t \geq t_0$.
	Hence $|V^\pi_Q(\ae_{<t}) - V^\pi_\mu(\ae_{<t})| < 3\eps$
	for all $t \geq t_0$.

\item Analogously to the proof of (\ref{itm:on-policy-value-convergence-wm}).
\qedhere
\end{enumerate}
\end{proof}

It is important to observe that
on-policy convergence does not imply that
the agent converges to the optimal policy.
$V^\pi_\rho$ converges to $V^\pi_\mu$,
but $V^\pi_\mu$ need not be close to $V^*_\mu$.
Indeed, there might be another policy $\tilde\pi$
that has a higher value than $\pi$ in the true environment $\mu$%
~($V^{\tilde\pi}_\mu > V^\pi_\mu$).
If the agent thinks $\tilde\pi$ has lower value%
~($V^{\tilde\pi}_\rho < V^\pi_\rho$)
it might not follow $\tilde\pi$
and hence not learn that the actual value of $\tilde\pi$ is much higher.
In other words,
on-policy convergence implies that the agent learns the value of its own actions,
but not the value of counterfactual actions that it does not take.

\autoref{thm:on-policy-value-convergence} now enables us to tie in
the results of \autoref{cha:learning}.
This yields a surge of corollaries,
but first we need to make the learning distributions contextual on the actions.

Let $w \in \Delta\M$ be a positive prior over the environment class $\M$.
We define the corresponding Bayesian mixture
analogously to \autoref{ex:Bayesian-mixture}:
\begin{equation}\label{eq:Bayesian-mixture}
\xi(e_{<t} \dmid a_{<t}) := \sum_{\nu \in \M} w(\nu) \nu(e_{<t} \dmid a_{<t})
\end{equation}
Note that the Bayesian mixture $\xi$ depends on the prior $w$.
For the rest of this thesis, this dependence will not be made explicit.

From \autoref{lem:V-linear} and \eqref{eq:posterior-weight}
we immediately get the following identity:
\begin{equation}\label{eq:value-of-xi}
  V^\pi_\xi(\ae_{<t})
= \sum_{\nu \in \M} w(\nu \mid \ae_{<t}) V^\pi_\nu(\ae_{<t})
\end{equation}
Similarly, we get from \autoref{lem:V*-convex}
\begin{equation}\label{eq:optimal-value-of-xi}
  V^*_\xi(\ae_{<t})
\leq \sum_{\nu \in \M} w(\nu \mid \ae_{<t}) V^*_\nu(\ae_{<t}).
\end{equation}

\begin{corollary}[On-Policy Value Convergence for Bayes]
\label{cor:Bayes-on-policy-value-convergence}
\index{on-policy!value convergence}
For any environment $\mu \in \M$ and any policy $\pi$,
\[
V^\pi_\xi(\ae_{<t}) - V^\pi_\mu(\ae_{<t}) \to 0
\text{ as $t \to \infty$ $\mu^\pi$-almost surely.}
\]
\end{corollary}
\begin{proof}
Since $\mu \in \M$, we have dominance $\xi^\pi \geq w(\mu) \mu^\pi$
with $w(\mu) > 0$
and by \autoref{prop:compatibility}\ref{itm:dominance=>absolute-continuity}
absolute continuity $\xi^\pi \gg \mu^\pi$.
From \autoref{thm:Blackwell-Dubins} we get that
$\xi^\pi$ merges strongly with $\mu^\pi$.
Therefore we can apply
\autoref{thm:on-policy-value-convergence}\ref{itm:on-policy-value-convergence-ac}.
\end{proof}

Analogously, we define
$\MDL^{\ae_{<t}} :=
\argmin_{\nu \in \M} \{ -\log \nu(e_{<t} \dmid a_{<t}) + K(\nu) \}$.

\begin{corollary}[On-Policy Value Convergence for MDL]
\label{cor:MDL-on-policy-value-convergence}
\index{on-policy!value convergence}
For any environment $\mu \in \M$ and any policy $\pi$,
\[
V^\pi_{\MDL^{\ae_{<t}}}(\ae_{<t}) - V^\pi_\mu(\ae_{<t}) \to 0
\text{ as $t \to \infty$ $\mu^\pi$-almost surely.}
\]
\end{corollary}
\begin{proof}
By \autoref{thm:MDL-merges-strongly}
$\MDL^\pi$ merges strongly with $\nu^\pi$ for each $\nu \in \M$,
therefore we can apply
\autoref{thm:on-policy-value-convergence}\ref{itm:on-policy-value-convergence-ac}.
\end{proof}

By providing the action sequence \emph{contextually} on a separate input tape,
we can define
$\Km(e_{<t} \dmid a_{<t}) :=
\min \{ |p| \mid e_{<t} \sqsubseteq U(p, a_{<t}) \}$
analogously to \eqref{eq:Km}.

\begin{corollary}[On-Policy Value Convergence for Universal Compression]
\label{cor:Km-on-policy-value-convergence}
\index{on-policy!value convergence}
Let $\rho(e_{<t} \dmid a_{<t}) := 2^{-\Km(e_{<t} \dmid a_{<t})}$.
Then for any environment $\mu \in \Mcomp$ and any policy $\pi$,
\[
V^\pi_\rho(\ae_{<t}) - V^\pi_\mu(\ae_{<t}) \to 0
\text{ as $t \to \infty$ $\mu^\pi$-almost surely.}
\]
\end{corollary}
\begin{proof}
Since $\rho$ dominates every $\mu \in \Mcomp$%
~(\autoref{ssec:universal-compression})
we can apply
\autoref{prop:compatibility}\ref{itm:dominance=>absolute-continuity},
\autoref{thm:Blackwell-Dubins},
and
\autoref{thm:on-policy-value-convergence}\ref{itm:on-policy-value-convergence-ac}
as in the proof of \autoref{cor:Bayes-on-policy-value-convergence}.
\end{proof}

Similarly to $\Km$ there is a speed prior for environments%
~\citep[Ch.~6]{Filan:2015}:
\[
   S_{Kt}(e_{<t} \dmid a_{<t})
:= \sum_{p:\,e_{<t} \sqsubseteq U(p, a_{<t})} \frac{2^{-|p|}}{t(U, p, a_{<t}, e_{<t})}
\]
where $t(U, p, a_{<t}, e_{<t})$ denotes the number of time steps
$U(p, a_{<t})$ takes to produce $e_{<t}$.

\begin{corollary}[On-Policy Value Convergence for the Speed Prior]
\label{cor:speed-prior-on-policy-value-convergence}
\index{on-policy!value convergence}
If the effective horizon is bounded,
then for any environment $\mu \in \Mcomp$ estimable in polynomial time and
any policy $\pi$,
\[
\frac{1}{t} \sum_{k=1}^t \Big( V^\pi_{S_{Kt}}(\ae_{<k}) - V^\pi_\mu(\ae_{<k}) \Big)
\to 0
\text{ as $t \to \infty$ $\mu^\pi$-almost surely.}
\]
\end{corollary}
\begin{proof}
By \autoref{cor:merging-SKt} the speed prior $S_{Kt}$ merges almost weakly
with every measure estimable in polynomial time.
Therefore we can apply
\autoref{thm:on-policy-value-convergence}\ref{itm:on-policy-value-convergence-awm}.
\end{proof}

\section{The Agents}
\label{sec:agents}

If we knew the true environment $\mu$,
we would choose the $\mu$-optimal policy,
the policy that maximizes $\mu$-expected discounted rewards.
But generally we do not know the true environment,
and the challenging part of reinforcement learning is
to learn the environment while trying to collect rewards.

In this section we introduce a number of agents
that attempt to solve the general reinforcement learning problem%
~(\autoref{prob:general-RL}).
These agents are discussed
throughout the rest of this thesis.

\subsection{Bayes}
\label{ssec:AIXI}

A \emph{Bayes optimal policy}\index{policy!Bayes optimal|textbf}
with respect to the prior $w$ is
the policy $\pi^*_\xi$ where $\xi$ is the Bayesian mixture defined in
\autoref{ssec:on-policy-value-convergence}.
There can be one or more Bayes optial policies.
From \autoref{cor:Bayes-on-policy-value-convergence} we get
on-policy value convergence for the Bayes optimal policy.

After history $\ae_{<t}$,
the Bayes policy $\pi^*_\xi$ maximizes expected discounted rewards
in the posterior\index{posterior} mixture:
\[
  \xi(\,\cdot \mid e_{<t} \dmid a_{1:\infty})
= \sum_{\nu \in \M} w(\nu \mid \ae_{<t}) \nu(\,\cdot \mid e_{<t} \dmid a_{1:\infty})
\]
where $w(\nu \mid \ae_{<t})$ are the posterior weights%
~\eqref{eq:posterior-weight}.
Maximizing expected rewards according to the posterior is the same as
maximizing expected rewards according to the prior conditional on the history:
if $\pi(\ae_{<t}) = \pi^*_\xi(\ae_{<t})$,
then $V^\pi_\xi(\ae_{<t}) = V^*_\xi(\ae_{<t})$.
Actually visiting the history $\ae_{<t}$
does not change what $\pi^*_\xi$ planned to do
before it visited $\ae_{<t}$.
Note that this relies on the fact that
the way we use discounting is \index{time consistent}time consistent%
~\citep[Def.~12]{LH:2014discounting}.

When using the prior $w(\nu) \propto 2^{-K(\nu)}$%
~(\autoref{ex:Solomonoff-prior})
over the class $\Mlscccs$,
the Bayes optimal policy is also known as \emph{AIXI}\index{AIXI|textbf},
introduced and analyzed by
\citet{Hutter:2000,Hutter:2001aixi,Hutter:2002,Hutter:2003AIXI,Hutter:2005,%
Hutter:2007aixi,Hutter:2012uai}
in his work on \emph{universal artificial intelligence}%
\index{universal!artificial intelligence}.
In this case, the Bayesian mixture \eqref{eq:Bayesian-mixture}
can be defined equivalently according to~\citep{WSH:2011}
\begin{equation}\label{eq:xi-from-U}
   \xi(e_{<t} \dmid a_{<t})
:= \sum_{p:\, e_{<t} \sqsubseteq U(p, a_{<t})} 2^{-|p|}.
\end{equation}
Generally there is more than one $\xi$-optimal policy
and Solomonoff's prior depends on
the choice of the (reference) universal Turing machine,
so this definition is not unique.
Moreover, not every universal Turing machine is a good choice for AIXI,
see \autoref{sec:bad-priors} for a few bad choices.
The following lemma will be used later.

\begin{lemma}[Mixing Mixtures]
\label{lem:mixing-mixtures}
\index{Bayesian!mixture}
Let $q, q' \in \mathbb{Q}$ such that
$q > 0$, $q' \geq 0$, and $q + q' \leq 1$.
Let $w$ be any lower semicomputable positive prior,
let $\xi$ be the Bayesian mixture corresponding to $w$, and
let $\rho \in \Mlscccs$.
Then $\xi' := q\xi + q'\rho \in \Mlscccs$ is a Bayesian mixture.
\end{lemma}
\begin{proof}
$\xi'$ is given by the positive prior $w'$ with $w' := qw + q'\one_\rho$.
\end{proof}

Bayesian approaches have a long tradition in reinforcement learning,
although they are often prohibitively expensive to compute.
For multi-armed bandits,
\citet{Gittins:1979} achieved a breakthrough
with an index strategy that enables the computation of the optimal policy
by computing one quantity for each arm independently of the rest.
This strategy even achieves
the optimal asymptotic regret bounds~\citep{Lattimore:2015Gittins}.
Larger classes have also been attempted:
using Monte-Carlo tree search,
\citet{VNHUS:2011} approximate the Bayes optimal policy
in the class of all context trees.
\citet{Doshi-Velez:2012} uses Bayesian techniques
to learn infinite-state POMDPs.
See \citet{VGMP:2012Bayes} for a survey on Bayesian techniques in RL.

In the rest of this thesis,
the Bayes optimal policy is often treated as an optimal exploitation strategy.
This is not true: Bayes does explore~(when it is Bayes optimal to do so)%
\index{exploration vs.\ exploitation}.
It just does not explore general environment classes completely%
~(see \autoref{ssec:asymptotic-optimality-Bayes}).

\subsection{Knowledge-Seeking Agents}
\label{ssec:knowledge-seeking-agents}

In this section we discuss two variants of knowledge-seeking agents%
\index{knowledge-seeking}:
entropy-seeking agents introduced by \citet{Orseau:2011ksa,Orseau:2014ksa}
and information-seeking agents introduced by \citet{OLH:2013ksa}.
The entropy-seeking agent maximizes the Shannon entropy gain,
while the information-seeking agent maximizes
the expected information gain.
These quantities are expressed in different value functions.
In places where confusion can arise,
we call the value function $V^\pi_\nu$ from \autoref{def:value-function}
the \emph{reward-seeking value function}%
\index{value function!reward-seeking}.

In this section
we use a finite horizon $m < \infty$
(possibly dependent on time step $t$):
the knowledge-seeking agent maximizes entropy/information received
up to time step $m$.
We assume implicitly that $m$ (as a function of $t$) is computable.
Moreover, in this section we assume that the Bayesian mixture $\xi$
is a measure rather than a semimeasure;
\autoref{ex:ksa-unnormalized} discusses this assumption.

\begin{definition}[{Entropy-Seeking Value Function;
\citealp[Sec.~6]{Orseau:2014ksa}}]
\label{def:V-entropy}\index{value function!entropy-seeking|textbf}
\index{entropy}
The \emph{entropy-seeking value} of a policy $\pi$ given history $\ae_{<t}$ is
\[
   V^{\pi,m}_\Ent(\ae_{<t})
:= \EE^\pi_\xi [ -\log_2 \xi(e_{1:m} \mid e_{<t} \dmid a_{1:m}) \mid \ae_{<t} ].
\]
\end{definition}

The entropy-seeking value
is the Bayes-expectation of $-\log \xi$.
\citet{Orseau:2011ksa,Orseau:2014ksa} also considers a related value function
based on the $\xi$-expectation of $\xi$
that we do not discuss here.

\begin{definition}[{Information-Seeking Value Function;
\citealp[Def.~1]{OLH:2013ksa}}]
\label{def:V-information}\index{value function!information-seeking|textbf}
\index{KL-divergence}
The \emph{information-seeking value} of a policy $\pi$ given history $\ae_{<t}$
is
\[
   V^{\pi,m}_\IG(\ae_{<t})
:= \sum_{\nu \in \M} w(\nu \mid \ae_{<t}) \KL_m(\nu^\pi, \xi^\pi \mid \ae_{<t}).
\]
\end{definition}

Analogously to before we define $V^*_\Ent := \sup_\pi V^\pi_\Ent$ and
$V^*_\IG := \sup_\pi V^\pi_\IG$.
An optimal entropy-seeking policy is defined as
$\pi^*_\Ent :\in \argmax_\pi V_\Ent^\pi$ and
an optimal information-seeking policy is defined as
$\pi^*_\IG :\in \argmax_\pi V_\IG^\pi$.
Since we use a finite horizon~($m < \infty$),
these optimal policies exist.

The \emph{information gain}\index{information gain|textbf}
is defined as the difference in entropy
between the prior and the posterior:
\[
\IG_{t:m}(\ae_{1:m}) := \Ent(w(\,\cdot \mid \ae_{<t})) - \Ent(w(\,\cdot \mid \ae_{1:m}))
\]
We get the following identity~\citep[Eq.~3.5]{Lattimore:2013}.
\[
    \EE^\pi_{\xi}[ \IG_{t:m}(\ae_{1:m}) \mid \ae_{<t} ]
~=~ V^{\pi,m}_{\IG}(\ae_{<t})
\]

For infinite horizons~($m = \infty$),
the values functions from \autoref{def:V-entropy} and \autoref{def:V-information}
may not converge.
To ensure convergence, we can either use discounting, or
in case of $V_{\IG}$
a prior with finite entropy~\citep[Thm.~3.4]{Lattimore:2013}.
Moreover, note that while $V_{\IG}$ and $V_{\Ent}$ are expectations
with respect to the measure $\xi$,
there is no bound on the one-step change in value
$V^{\pi,m}_{\IG}(\ae_{<t}) - V^{\pi,m}_{\IG}(\ae_{1:t})$,
which can also be negative.
For the reward-seeking value function $V^{\pi,m}_\nu$,
the one-step change in value is bounded between $0$ and $1$
by \autoref{rem:value-function-bounded}.

For classes of deterministic environments
\autoref{def:V-entropy} and \autoref{def:V-information} coincide.
In stochastic environments
the entropy-seeking agent does not work well
because it gets distracted by noise in the environment
rather than trying to distinguish environments~\citep[Sec.~5]{OLH:2013ksa}.
Moreover, the entropy-seeking agent
may fail to seek knowledge
in deterministic semimeasures
as the following example demonstrates.

\begin{example}[Unnormalized Entropy-Seeking]
\label{ex:ksa-unnormalized}\index{value function!entropy-seeking}
If the Bayesian mixture $\xi$ is a semi\-measure instead of a measure
(such as the Solomonoff prior from \autoref{ex:Solomonoff-prior}),
then the entropy-seeking agent does not explore correctly.
Fix $\A := \{ \alpha, \beta \}$, $\E := \{ 0, 1 \}$, and $m = t$
(we only care about the entropy of the next percept).
We illustrate the problem on a simple class of environments
$\{ \nu_1, \nu_2 \}$:
\begin{center}
\begin{tikzpicture}[auto]
\node[circle,draw] (nu1) {$\nu_1$};
\path[transition] (nu1) to[loop left] node {$\alpha/0/0.1$} (nu1);
\path[transition] (nu1) to[loop right] node {$\beta/0/0.5$} (nu1);

\node[circle,draw,right of=nu1,node distance=70mm] (nu2) {$\nu_2$};
\path[transition] (nu2) to[loop left] node {$\alpha/1/0.1$} (nu2);
\path[transition] (nu2) to[loop right] node {$\beta/0/0.5$} (nu2);
\end{tikzpicture}
\end{center}
where transitions are labeled with action/percept/probability.
Both $\nu_1$ and $\nu_2$ return a percept deterministically or nothing at all
(the environment ends).
Only action $\alpha$ distinguishes between the environments.
With the prior $w(\nu_1) := w(\nu_2) := 1/2$, we get a mixture $\xi$
for the entropy-seeking value function $V^\pi_\Ent$.
Then $V^*_\Ent(\alpha) \approx 0.432 < 0.5 = V^*_\Ent(\beta)$, 
hence action $\beta$ is preferred over $\alpha$ by the entropy-seeking agent.
But taking action $\beta$ yields percept $0$ (if any),
hence nothing is learned about the environment.
\end{example}

On-policy value convergence~(\autoref{thm:on-policy-value-convergence})
ensures that asymptotically,
the agent learns the value of its own policy.
Knowledge-seeking agents do even better:
they don't have to balance between exploration and exploitation%
\index{exploration vs.\ exploitation},
so they can focus solely on exploration.
As a result, they learn off-policy\index{off-policy},
i.e., the value of counterfactual actions~\citep[Thm.~7]{OLH:2013ksa}.

\subsection{BayesExp}
\label{ssec:BayesExp}\index{BayesExp}

\citet[Thm.~5.6]{Lattimore:2013} defines
\emph{\BayesExp} combining AIXI with the information-seeking agent.
{\BayesExp} alternates between phases of exploration and
phases of exploitation\index{exploration vs.\ exploitation}:
Let $\eps_t$ be a monotone decreasing sequence of positive reals such that
$\eps_t \to 0$ as $t \to \infty$.
If the optimal information-seeking value $V^*_{\IG}$
is larger than $\varepsilon_t$,
then {\BayesExp} starts an exploration phase,
otherwise it starts an exploitation phase.
During an exploration phase,
{\BayesExp} follows an optimal information-seeking policy
for an $\varepsilon_t$-effective horizon.
During an exploitation phase,
{\BayesExp} follows an $\xi$-optimal reward-seeking policy
for one step~(see \autoref{alg:BayesExp}).

\begin{algorithm}
\index{BayesExp|textbf}
\begin{center}
\begin{algorithmic}[1]
\While{true}
	\If {$V^{*,t + H_t(\eps_t)}_{\IG}(\ae_{<t}) > \varepsilon_t$}
		\State follow $\pi^*_\IG$ for $H_t(\eps_t)$ steps
	\Else
		\State follow $\pi^*_\xi$ for $1$ step
	\EndIf
\EndWhile
\end{algorithmic}
\end{center}
\caption{{\BayesExp} policy $\pi_{BE}$~\citep[Alg.~2]{Lattimore:2013}.}
\label{alg:BayesExp}
\end{algorithm}

\subsection{Thompson Sampling}
\label{ssec:Thompson-sampling}
\index{Thompson sampling}

\emph{Thompson sampling}, also known as
\emph{posterior sampling} or \emph{the Bayesian control rule},
was originally proposed by \citet{Thompson:1933}
as a bandit algorithm.
It is easy to implement and often achieves quite good results%
~\citep{CH:2011Thompson}.
In multi-armed bandits
it attains optimal regret~\citep{AG:2011Thompson,KKM:2012Thompson}.
Thompson sampling has also been discussed for MDPs~\citep{Strens:2000,Dearden+:1998} and
Bayesian and frequentist regret bounds have been established~\citep{ORR:2013thompson,GS:2015thompson}.

For general RL Thompson sampling
was first suggested by \citet{OB:2010Thompson} with resampling at every time step.
\citet{Strens:2000} proposes following the optimal policy for one episode or
``related to the number of state transitions
the agent is likely to need to plan ahead''.
We follow Strens' suggestion and resample at the effective horizon.

Let $\eps_t$ be a monotone decreasing sequence of positive reals such that
$\eps_t \to 0$ as $t \to \infty$.
Our variant of Thomson sampling is given in \autoref{alg:Thompson-sampling}.
It samples an environment $\rho$ from the posterior,
follows the $\rho$-optimal policy for an $\eps_t$-effective horizon,
and then repeats.

\begin{algorithm}
\index{Thompson sampling|textbf}
\begin{center}
\begin{algorithmic}[1]
\While{true}
	\State sample $\rho \sim w(\,\cdot \mid \ae_{<t})$
	\State follow $\pi^*_\rho$ for $H_t(\eps_t)$ steps
\EndWhile
\end{algorithmic}
\end{center}
\caption{Thompson sampling policy $\pi_T$.}
\label{alg:Thompson-sampling}
\end{algorithm}

Note that $\pi_T$ is a stochastic policy
since we occasionally sample from a distribution.
We assume that this sampling is independent of everything else.


\chapter{Optimality}
\label{cha:optimality}

\newcommand{\optref}[1]{{\hyperref[{#1}]{O\ref*{#1}}}} 

\falsequote{Shane Legg}{Machines will never be intelligent.}

\autoref{prob:general-RL} defines the general reinforcement learning problem.
But our definition of this problem did not specify what a solution would be.
This chapter is dedicated to this question:
\begin{quote}
What is an optimal solution to the general reinforcement learning problem?
\end{quote}
How can we say that
one policy is \emph{better} than another?
What is the \emph{best} policy?
Are the policies from \autoref{sec:agents} optimal?
Several notions of optimality for a policy $\pi$ in an environment class $\M$
are conceivable:
\begin{enumerate}[{O}1.]
\item\label{itm:maximal-reward} \emph{Maximal reward}.
	The policy $\pi$ receives a reward of $1$ in every time step%
	~(which is maximal according to \assref{ass:aixi}{ass:bounded-rewards}):
	\[
	\forall t \in \mathbb{N}.\; r_t = 1
	\]
\item\label{itm:optimal-policy} \emph{Optimal policy}.
	The policy $\pi$ achieves the highest possible value
	in the true environment $\mu$:
	\[
	\forall \ae_{<t} \in \H.\; V^\pi_\mu(\ae_{<t}) = V^*_\mu(\ae_{<t})
	\]
\item\label{itm:Pareto-optimality} \emph{Pareto optimality}%
	\index{optimality!Pareto}~\citep[Thm.~2]{Hutter:2002}.
	There is no other policy
	that performs at least as good in all environments
	and strictly better in at least one:
	\[
	\nexists \tilde\pi.\; \Big(
		\forall \nu \in \M.\; V^{\tilde\pi}_\nu(\epsilon) \geq V^\pi_\nu(\epsilon)
		\text{ and }
		\exists \rho \in \M.\; V^{\tilde\pi}_\rho(\epsilon) > V^\pi_\rho(\epsilon)
	\Big)
	\]
\item\label{itm:balanced-Pareto-optimality} \emph{Balanced Pareto optimality}%
	\index{optimality!balanced Pareto}~\citep[Thm.~3]{Hutter:2002}.
	The policy $\pi$ achieves a better value across $\M$
	weighted by $w \in \Delta\M$ than any other policy:
	\[
	\forall \tilde\pi.\; \sum_{\nu \in \M} w(\nu) \big(
	  V^\pi_\nu(\epsilon) - V^{\tilde\pi}_\nu(\epsilon)
	\big) \geq 0
	\]
\item\label{itm:Bayes-optimality} \emph{Bayes optimality}%
	\index{optimality!Bayes}.
	The policy $\pi$ is $\xi$-optimal for some Bayes mixture $\xi$:
	\[
	\forall \ae_{<t} \in \H.\; V^\pi_\xi(\ae_{<t}) = V^*_\xi(\ae_{<t})
	\]
\item\label{itm:PAC} \emph{Probably approximately correct}\index{PAC}.
	For a given $\eps, \delta > 0$
	the value of the policy $\pi$ is $\eps$-close to the optimal value
	with probability at least $\delta$ after time step $t_0(\eps, \delta)$:
	\[
	  \mu^\pi \Big[ \forall t \geq t_0(\eps, \delta).\; V^*_\mu(\ae_{<t}) - V^\pi_\mu(\ae_{<t}) < \eps \Big]
	> 1 - \delta
	\]
\item\label{itm:asymptotic-optimality}
	\emph{Asymptotic optimality}\index{optimality!asymptotic}%
	~\citep[Sec.~5.3.4]{Hutter:2005}.
	The value of the policy $\pi$ converges to the optimal value:
	\[
	V^*_\mu(\ae_{<t}) - V^\pi_\mu(\ae_{<t}) \to 0
	\text{ as $t \to \infty$}
	\]
\item\label{itm:sublinear-regret}
	\emph{Sublinear regret}\index{regret}.
	The difference between the reward sum of the policy $\pi$ and
	the best policy in hindsight grows sublinearly:
	\[
	\sup_{\pi'} \EE^{\pi'}_\mu \left[ \sum_{t=1}^m r_t \right]
	- \EE^\pi_\mu \left[ \sum_{t=1}^m r_t \right] \in o(m)
	\]
\end{enumerate}

We discuss these notions of optimality in turn.
Achieving the maximal reward at every time step is impossible
if there is no action that
makes the environment $\mu$ respond with the maximal reward;
generally there is no policy that achieves maximal rewards at every time step.
In order to follow the optimal policy,
we need to know the true environment.
In our setting, the true environment is unknown and has to be learned.
During the learning process the agent cannot also act optimally
because it needs to explore.
In particular, the policy $\pi$ cannot be optimal
simultaneously in all environments from $\M$.
This rules out \optref{itm:maximal-reward} and \optref{itm:optimal-policy}
as a notion of optimality.

In \autoref{sec:Pareto-optimality} we show that
all policies are Pareto optimal.
This disqualifies \optref{itm:Pareto-optimality}
as a useful notion of optimality in general reinforcement learning.

Balanced Pareto optimality~(\optref{itm:balanced-Pareto-optimality}),
Bayes optimality~(\optref{itm:Bayes-optimality}), and
maximal Legg-Hutter intelligence~\citep{LH:2007int} turn out to coincide.
In \autoref{sec:Bayes-optimality}
we show that Legg-Hutter intelligence is highly subjective,
because it depends on the choice of the prior.
By changing the prior of a Bayesian agent,
we can make the agent's intelligence arbitrarily low.
In \autoref{sec:bad-priors} we present a choice of particularly bad priors.
This rules out \optref{itm:balanced-Pareto-optimality}
and \optref{itm:Bayes-optimality}
because they are prior-dependent and not objective.

\optref{itm:PAC} is a stronger version of asymptotic optimality that
provides a rate of convergence (it implies \optref{itm:asymptotic-optimality}).
Since our environment class can
be very large and non-compact, concrete PAC results are likely impossible.
\citet{Orseau:2010,Orseau:2013} shows that
the Bayes optimal agent does not achieve asymptotic optimality
in all computable environments.
The underlying problem is that
in the beginning the agent does not know enough
about its environment and therefore relies heavily on its prior.
Lack of exploration then retains the prior's bias.
This problem can be alleviated
by adding extra exploration to the Bayesian agent.
In \autoref{sec:asymptotic-optimality} we discuss two agents
that achieve asymptotic optimality:
{\BayesExp}~(\autoref{ssec:BayesExp}) and
Thompson sampling~(\autoref{ssec:Thompson-sampling}).
This establishes that \optref{itm:asymptotic-optimality} is possible.

In general environments sublinear regret
is impossible because the agent can get stuck in traps
from which it is unable to recover.
This rules out \optref{itm:sublinear-regret}.
However, in \autoref{sec:regret} we show that
if we assume that the environment allows recovering from mistakes
(and some minor conditions on the discount function are fulfilled),
then asymptotic optimality implies sublinear regret.
This means that Thompson sampling has sublinear regret in these
recoverable environments.

Notably, only asymptotic optimality~(\optref{itm:asymptotic-optimality})
holds up to be a nontrivial and objective criterion of optimality
that applies to the general reinforcement learning problem.
While there are several agents that are known to be asymptotically optimal,
some undesirable properties remain.
\autoref{sec:discussion-optimality} discusses this further.
See also \citet{Mahadevan:1996}
for a discussion of notions of optimality in MDPs.

\section{Pareto Optimality}
\label{sec:Pareto-optimality}

In this section we show that
Pareto optimality is not a useful criterion for optimality
since for any environment class containing $\Mcomp$,
all policies are Pareto optimal.

\begin{definition}[{Pareto Optimality; \citealp[Def.~5.22]{Hutter:2005}}]
\label{def:Pareto-optimality}\index{optimality!Pareto|textbf}
A policy $\pi$ is
\emph{Pareto optimal in the set of environments $\M$} iff
there is no policy $\tilde{\pi}$ such that
$V^{\tilde{\pi}}_\nu(\epsilon) \geq V^\pi_\nu(\epsilon)$
for all $\nu \in \M$ and
$V^{\tilde{\pi}}_\rho(\epsilon) > V^\pi_\rho(\epsilon)$
for at least one $\rho \in \M$.
\end{definition}

The literature provides the following result.

\begin{theorem}[{AIXI is Pareto Optimal; \citealp[Thm.~2]{Hutter:2002}}]
\label{thm:AIXI-is-Pareto-optimal}\index{AIXI}\index{optimality!Pareto}
Every $\xi$-optimal policy is Pareto optimal in $\Mlscccs$.
\end{theorem}

The following theorem was proved for deterministic policies
in \citet{LH:2015priors}.
Here we extend it to stochastic policies.

\begin{theorem}[Pareto Optimality is Trivial]
\label{thm:Pareto-optimality-is-trivial}\index{optimality!Pareto}
Every policy is Pareto optimal in any class $\M \supseteq \Mcomp$.
\end{theorem}

The proof proceeds as follows:
for a given policy $\pi$, we construct a set of `buddy environments'
that reward $\pi$ and punish other policies.
Together they can defend against any policy $\tilde{\pi}$
that tries to take the crown of Pareto optimality from $\pi$.

\begin{proof}
We assume $(0, 0)$ and $(0, 1) \in \E$.
Moreover, assume there is a policy $\pi$ that is not Pareto optimal.
Then there is a policy $\tilde{\pi}$ that \emph{Pareto dominates} $\pi$, i.e.,
$V^{\tilde{\pi}}_\rho(\epsilon) > V^\pi_\rho(\epsilon)$
for some $\rho \in \M$, and
$V^{\tilde{\pi}}_\nu(\epsilon) \geq V^\pi_\nu(\epsilon)$ for all $\nu \in \M$.
From $V^{\tilde{\pi}}_\rho(\epsilon) > V^\pi_\rho(\epsilon)$ and
\autoref{lem:discounted-values} we get that
there is a shortest and lexicographically first history
$\ae_{<k}'$ consistent with $\pi$ and $\tilde\pi$ such that
$\pi(\alpha \mid \ae_{<k}') > \tilde{\pi}(\alpha \mid \ae_{<k}')$
for some action $\alpha \in \A$
and $V^{\tilde{\pi}}_\rho(\ae_{<k}') > V^\pi_\rho(\ae_{<k}')$.
Consequently there is an $i \geq k$ such that $\gamma_i > 0$,
and hence $\Gamma_k > 0$.
We define the environment $\mu$
that first reproduces the separating history $\ae_{<k}'$
and then, if $\alpha$ is the next action, returns reward $1$ forever, and
otherwise returns reward $0$ forever.
Formally, $\mu$ is defined by
\[
\mu(e_{1:t} \mid e_{<t} \dmid a_{1:t}) :=
\begin{cases}
1, &\text{if } t < k \text{ and } e_t = e'_t, \\
1, &\text{if } t \geq k \text{ and } a_k = \alpha
    \text{ and } r_t = 1 \text{ and } o_t = 0, \\
1, &\text{if } t \geq k \text{ and } a_k \neq \alpha
    \text{ and } r_t = 0 = o_t, \text{ and} \\
0, &\text{otherwise}.
\end{cases}
\]
The environment $\mu$ is computable,
even if the policy $\pi$ is not:
for a fixed history $\ae_{<t}'$ and action $\alpha$,
there exists a program computing $\mu$;
therefore $\mu \in \Mcomp$.
We get the following value difference for the policies $\pi$ and $\tilde{\pi}$:
\begin{align*}
   V^\pi_\mu(\epsilon) - V^{\tilde{\pi}}_\mu(\epsilon)
&= \EE^\pi_\mu \left[
     \sum_{t=1}^{k-1} \gamma_t r_t + \sum_{t=k}^{\infty} \gamma_t r_t
   \right]
   - \EE^{\tilde\pi}_\mu \left[
     \sum_{t=1}^{k-1} \gamma_t r_t - \sum_{t=k}^{\infty} \gamma_t r_t
   \right] \\
&= \left( \pi(\alpha \mid \ae_{<k}') \sum_{t=k}^{\infty} \gamma_t
   - \tilde\pi(\alpha \mid \ae_{<k}') \sum_{t=k}^{\infty} \gamma_t
   \right) \mu^\pi(\ae_{<k}') \\
&= \big( \pi(\alpha \mid \ae_{<k}') - \tilde\pi(\alpha \mid \ae_{<k}') \big)
     \mu^\pi(\ae_{<k}') \Gamma_k
 > 0
\end{align*}
Hence $V^{\tilde{\pi}}_\mu(\epsilon) < V^\pi_\mu(\epsilon)$,
which contradicts the fact that $\tilde\pi$ Pareto dominates $\pi$ since
$\M \supseteq \Mcomp \ni \mu$.
\end{proof}

Note that the environment $\mu$ we defined
in the proof of \autoref{thm:Pareto-optimality-is-trivial}
is actually just a finite-state POMDP,
so Pareto optimality is also trivial for smaller environment classes.

\section{Bad Priors}
\label{sec:bad-priors}

In this section we give three examples of universal priors
that cause a AIXI to misbehave drastically.
In case of a finite horizon,
the \emph{indifference prior}\index{prior!indifference}
makes all actions equally preferable to AIXI%
~(\autoref{ssec:indifference-prior}).
The \emph{dogmatic prior}\index{prior!dogmatic}
makes AIXI stick to any given computable policy $\pi$ as long as
expected future rewards do not fall too close to zero%
~(\autoref{ssec:dogmatic-prior}).
The \emph{Gödel prior}\index{prior!Gödel}
prevents AIXI$tl$ from taking any actions~(\autoref{ssec:Gödel-prior}).

\subsection{The Indifference Prior}
\label{ssec:indifference-prior}

The following theorem constructs the \emph{indifference prior}
that yields a Bayesian mixture $\xi'$ that
causes argmax ties for the first $m$ steps.
If we use a discount function that only cares about the first $m$ steps,
$\Gamma_m = 0$,
then all policies are $\xi'$-optimal policies.
In this case AIXI's behavior only depends on how we break argmax ties.

\begin{theorem}[Indifference Prior]
\label{thm:indifference-prior}\index{prior!indifference|textbf}
\index{discounting!finite horizon}
If there is an $m$ such that $\Gamma_m = 0$,
then there is a Bayesian mixture $\xi'$ such that
all policies are $\xi'$-optimal.
\end{theorem}
\begin{proof}
First, we assume that the action space is binary, $\A = \{ 0, 1 \}$.
Let $U$ be the reference UTM and
define the UTM $U'$ by
\[
     U'(s_{<m}p, a_{1:t})
~:=~ U(p, a_{1:t} \xor s_{1:t}),
\]
where $s_{<m}$ is a binary string of length $m-1$ and $s_k := 0$ for $k \geq m$.
($U'$ has no programs of length less than $m-1$.)
Let $\xi'$ be the Bayesian mixture given by $U'$
according to \eqref{eq:xi-from-U}.
Then
\begingroup
\allowdisplaybreaks
\begin{align*}
     \xi'(e_{<m} \dmid a_{<m})
~&=~ \sum_{p:\, e_{<m} \sqsubseteq U'(p, a_{<m})} 2^{-|p|} \\
~&=~ \sum_{s_{<m}p':\, e_{<m} \sqsubseteq U'(s_{<m}p', a_{<m})} 2^{-m-1-|p'|} \\
~&=~ \sum_{s_{<m}}~ \sum_{p':\, e_{<m} \sqsubseteq U(p', a_{<m} \xor s_{<m})} 2^{-m-1-|p'|} \\
~&=~ \sum_{s_{<m}}~ \sum_{p':\, e_{<m} \sqsubseteq U(p', s_{<m})} 2^{-m-1-|p'|},
\end{align*}
\endgroup
which is independent of $a_{<m}$.
Hence the first $m-1$ percepts are independent of the first $m-1$ actions.
But the percepts' rewards from time step $m$ on do not matter
since $\Gamma_m = 0$~(\autoref{lem:truncated-values}).
Because the environment is chronological,
the value function must be independent of all actions.
Thus every policy is $\xi'$-optimal.

For finite action spaces $\A$ with more than $2$ elements,
the proof works analogously
by making $\A$ a cyclic group and using the group operation instead of $\xor$.
\end{proof}

The choice of $U'$ in the proof of \autoref{thm:indifference-prior}
depends on $m$.
If we increase AIXI's horizon
while fixing the UTM $U'$,
\autoref{thm:indifference-prior} no longer holds.
For Solomonoff induction, there is an analogous problem:
when using Solomonoff's prior $M$ to predict a deterministic binary sequence $x$,
we make at most $K(x)$ errors~(\autoref{cor:expected-prediction-regret-M}).
In case the shortest program has length $> m$,
there is no guarantee that we make less than $m$ errors
(see \autoref{ssec:discussion-UTM}).

\subsection{The Dogmatic Prior}
\label{ssec:dogmatic-prior}

In this section
we define a universal prior
that assigns very high probability of going to hell (reward $0$ forever)
if we deviate from a given computable policy $\pi$.
For a Bayesian agent like AIXI,
it is thus only worth deviating from the policy $\pi$
if the agent thinks that the prospects of following $\pi$ are very poor already.
We call this prior the \emph{dogmatic prior},
because the fear of going to hell makes AIXI conform to
any arbitrary `dogmatic ideology' $\pi$.
AIXI will only break out
if it expects $\pi$ to give very low future payoff;
in that case the agent does not have much to lose.

\begin{theorem}[Dogmatic Prior]
\label{thm:dogmatic-prior}\index{prior!dogmatic|textbf}
Let $\pi$ be any computable deterministic policy,
let $\xi$ be any Bayesian mixture over $\Mlscccs$, and
let $\varepsilon > 0$.
There is a Bayesian mixture $\xi'$ such that
for any history $\ae_{<t}$ consistent with $\pi$ and
for which $V^\pi_\xi(\ae_{<t}) > \varepsilon$,
the action $\pi(\ae_{<t})$ is the unique $\xi'$-optimal action.
\end{theorem}

The following proof was adapted from \citet{LH:2015priors}
to work for environment classes that do not contain the Bayesian mixture.
Essentially, for every environment $\nu \in \Mlscccs$
the dogmatic prior puts much higher weight on an environment $\rho_\nu$
that behaves just like $\nu$ on the policy $\pi$,
but sends any policy deviating from $\pi$ to hell.
Importantly, while following the policy $\pi$
the environments $\nu$ and $\rho_\nu$ are indistinguishable,
so the posterior belief in $\nu$ is equal to the posterior belief in $\rho_\nu$.

\begin{proof}[Proof of \autoref{thm:dogmatic-prior}]
We assume $(o, 0) \in \E$ for some $o \in \O$.
For every environment $\nu \in \Mlscccs$ define the environment
\[
\rho_\nu(e_{1:t} \dmid a_{1:t}) :=
\begin{cases}
\nu(e_{1:t} \dmid a_{1:t}), &\text{if }
  a_k = \pi(\ae_{<k}) \;\forall k \leq t, \\
\nu(e_{<k}  \dmid a_{<k}),  &\text{if }
  k := \min \{ i \mid a_i \neq \pi(\ae_{<i}) \} \text{ exists} \\
  &\text{ and } e_i = (o, 0)\; \forall i \in \{ k, \ldots, t \}, \text{ and} \\
0, &\text{otherwise}.
\end{cases}
\]
The environment $\rho_\nu$ mimics environment $\nu$
until it receives an action that the policy $\pi$ would not take.
From then on, it provides rewards $0$.
Since $\pi$ is a computable policy,
we have that $\rho_\nu \in \Mlscccs$ for every $\nu \in \Mlscccs$.

Now we need to reweigh the prior $w$ so that it assigns
a much higher prior weight to $\rho_\nu$ than to $\nu$.
Without loss of generality we assume that $\varepsilon$ is computable,
otherwise we make it slightly smaller.
We define $w'(\nu) := \eps w(\nu)$ if $\nu \neq \rho_{\tilde\nu}$
for all $\tilde\nu \in \Mlscccs$ and
$w'(\rho_\nu) := (1-\eps) w(\nu) + \eps w(\rho_\nu)$.
Then
\begin{align*}
   \sum_{\nu \in \Mlscccs} w'(\nu)
&= \sum_{\nu = \rho_{\tilde\nu}} w'(\nu)
   + \sum_{\nu \neq \rho_{\tilde\nu}} w'(\nu) \\
&= \sum_{\nu = \rho_{\tilde\nu}} \big( (1-\eps) w(\tilde\nu) + \eps w(\nu) \big)
   + \sum_{\nu \neq \rho_{\tilde\nu}} \eps w(\nu) \\
&= \sum_{\nu \in \Mlscccs} \eps w(\nu)
   + \sum_{\tilde\nu \in \Mlscccs} (1-\eps) w(\tilde\nu) \\
&= \eps + (1 - \eps)
 = 1,
\end{align*}
and with $w' \geq \eps w$, we get that
$w'$ is a positive prior over $\Mlscccs$.
We define $\xi'$ as the corresponding Bayesian mixture
analogous to \eqref{eq:Bayesian-mixture}.

With $\rho := \sum_{\nu \in \Mlscccs} w(\nu) \rho_\nu$
we get $\xi' = \eps\xi + (1-\eps)\rho$.
The mixtures $\xi$ and $\rho$ coincide on the policy $\pi$ since
every $\nu$ coincides with $\rho_\nu$ on the policy $\pi$:
\[
  \xi^\pi(\ae_{<t})
= \sum_{\nu \in \Mlscccs} w(\nu) \nu^\pi(\ae_{<t})
= \sum_{\nu \in \Mlscccs} w(\nu) \rho_\nu^\pi(\ae_{<t})
= \rho^\pi(\ae_{<t})
\]
Moreover, $V^*_{\rho_\nu}(\ae_{<t}) = 0$ and thus $V^*_\rho(\ae_{<t}) = 0$
for any history inconsistent with $\pi$ by construction of $\rho_\nu$.

Let $\ae_{<t} \in \H$ be any history consistent with $\pi$ such that
$V^\pi_\xi(\ae_{<t}) > \varepsilon$.
Then $\rho^\pi = \xi^\pi$ implies
\[
  \frac{\rho(e_{<t} \dmid a_{<t})}{\xi'(e_{<t} \dmid a_{<t})}
= \frac{\xi(e_{<t} \dmid a_{<t})}{\xi'(e_{<t} \dmid a_{<t})}
= \frac{\xi(e_{<t} \dmid a_{<t})}%
  {\eps \xi(e_{<t} \dmid a_{<t}) + (1-\eps)\rho(e_{<t} \dmid a_{<t})}
= 1.
\]
Therefore \autoref{lem:V-linear} implies that for all $a \in \A$
and all policies $\tilde\pi$
\begin{align}
   V^{\tilde\pi}_{\xi'}(\ae_{<t}a)
&= \eps \frac{\xi(e_{<t} \dmid a_{<t})}{\xi'(e_{<t} \dmid a_{<t})} V^{\tilde\pi}_\xi(\ae_{<t}a)
   + (1-\eps) \frac{\rho(e_{<t} \dmid a_{<t})}{\xi'(e_{<t} \dmid a_{<t})} V^{\tilde\pi}_\rho(\ae_{<t}a) \notag \\
&= \eps V^{\tilde\pi}_\xi(\ae_{<t}a) + (1 - \eps) V^{\tilde\pi}_\rho(\ae_{<t}a).
\label{eq:dogmatic-linearity}
\end{align}

Let $\alpha := \pi(\ae_{<t})$ be the next action according to $\pi$, and
let $\beta \neq \alpha$ be any other action.
We have that $V^\pi_\xi(\ae_{<t}\alpha) = V^\pi_\rho(\ae_{<t}\alpha)$
since $\xi^\pi = \rho^\pi$ and $\ae_{<t}\alpha$ is consistent with $\pi$.
Therefore we get from \eqref{eq:dogmatic-linearity}
\begin{align*}
      V^*_{\xi'}(\ae_{<t}\alpha)
&\geq V^\pi_{\xi'}(\ae_{<t}\alpha)
 =    \eps V^\pi_\xi(\ae_{<t}\alpha) + (1 - \eps) V^\pi_\rho(\ae_{<t}\alpha)
 =    V^\pi_\xi(\ae_{<t}\alpha)
 >    \varepsilon, \\
      V^*_{\xi'}(\ae_{<t}\beta)
&=    \eps V^{\pi^*_{\xi'}}_\xi(\ae_{<t}\beta)
      + (1 - \eps) V^{\pi^*_{\xi'}}_\rho(\ae_{<t}\beta)
 =    \eps V^{\pi^*_{\xi'}}_\xi(\ae_{<t}a) + (1-\eps) 0
 \leq \eps.
\end{align*}
Hence $V^*_{\xi'}(\ae_{<t}\alpha) > V^*_{\xi'}(\ae_{<t}\beta)$
and thus the action $\alpha$ taken by $\pi$ is
the only $\xi'$-optimal action for the history $\ae_{<t}$.
\end{proof}

\begin{corollary}[With Finite Horizon Every Policy is Bayes Optimal]
\label{cor:finite-horizon-every-policy-is-Bayes-optimal}
\index{prior!dogmatic}\index{optimality!Bayes}\index{discounting!finite horizon}
If $\Gamma_m = 0$ for some $m \in \mathbb{N}$,
then for any deterministic policy $\pi$
there is a Bayesian mixture $\xi'$ such that
$\pi(\ae_{<t})$ is the only $\xi'$-optimal action for all histories $\ae_{<t}$
consistent with $\pi$ and $t \leq m$.
\end{corollary}
In contrast to \autoref{thm:indifference-prior}
where every policy is $\xi'$-optimal for a fixed Bayesian mixture $\xi'$,
\autoref{cor:finite-horizon-every-policy-is-Bayes-optimal} gives
a different Bayesian mixture $\xi'$
for every policy $\pi$ such that
$\pi$ is the \emph{only} $\xi'$-optimal policy.

\begin{proof}
Let $\varepsilon > 0$ be small enough such that
$V^\pi_\xi(\ae_{<t}) > \varepsilon$ for all $\ae_{<t}$ and $t \leq m$.
(This is possible because $(\A \times \E)^m$ is finite by
\assref{ass:aixi}{ass:finite-actions-and-percepts}.)
We use the dogmatic prior from \autoref{thm:dogmatic-prior}
to construct a Bayesian mixture $\xi'$
for the policy $\pi$ and $\varepsilon > 0$.
Thus for any history $\ae_{<t} \in \H$ consistent with $\pi$ and $t \leq m$,
the action $\pi(\ae_{<t})$ is the only $\xi'$-optimal action.
\end{proof}

\begin{corollary}[AIXI Emulating Computable Policies]
\label{cor:emulation}
\index{prior!dogmatic}\index{policy!computable}
Let $\varepsilon > 0$ and
let $\pi$ be any computable policy.
There is a Bayesian mixture $\xi'$ such that
for any $\xi'$-optimal policy $\pi^*_{\xi'}$ and
for any environment $\nu$,
\[
    \left| V^{\pi^*_{\xi'}}_\nu(\epsilon) - V^\pi_\nu(\epsilon) \right|
~<~ \varepsilon.
\]
\end{corollary}
\begin{proof}
From the proof of \autoref{cor:finite-horizon-every-policy-is-Bayes-optimal}
and \autoref{lem:discounted-values}.
\end{proof}

\subsection{The Gödel Prior}
\label{ssec:Gödel-prior}

\def\VA{\mathrm{VA}}
\def\PA{\mathsf{PA}}

This section introduces a prior that
prevents any fixed formal system from making any
statements about the outcome of all but finitely many computations.
It is named after \citet{Goedel:1931}
who famously showed that for any sufficiently rich formal system
there are statements that it can neither prove nor disprove.

This prior is targeted at AIXI$tl$,
a computable approximation to AIXI
defined by \citet[Sec.~7.2]{Hutter:2005}.
AIXI$tl$ aims to perform as least as well as the best agent
who is limited by time $t$ and space $l$
that can be verified using a proof of length at most $n$
for some fixed $n \in \mathbb{N}$.
The core idea is to enumerate all deterministic policies and proofs and
then execute the policy for which the best value has been proved.

In order to be verified,
a policy $\pi$ has to be computed by a program $p$ which fulfills
the \emph{verification condition} $\VA(p)$~\citep[Eq.~7.7]{Hutter:2005}.
This program $p$ not only computes future actions of $\pi$,
but also hypothetical past actions $a'_i$ and
lower bounds $v_i$ for the value of the policy $\pi$:
\[
\VA(p) ~:=~
\text{``}
\forall k \forall (v a' \ae)_{1:k}.\; \big(
  p(\ae_{<k}) = v_1 a'_1 \ldots v_k a'_k
  \rightarrow v_k \leq V^\pi_\xi(\ae_{<k})
\big)
\text{''},
\]
where $\pi$ is the policy derived from $p$ according to
$\pi(\ae_{<k}) := a'_k$.

We fix some formal system
that we use to prove the verification condition.
We want it to be sufficiently powerful,
but this incurs Gödel incompleteness.
For simplicity of exposition we pick $\PA$,
the system of \emph{Peano arithmetic}\index{Peano arithmetic}%
~\citep[Ch.~8.1]{Shoenfield:1967},
but our result generalizes trivially to all formal systems
which cannot prove their own consistency.

Let $n$ be a fixed constant.
The algorithm for AIXI$tl$ is specified as follows.
\begin{enumerate}[1.]
\item Let $P = \emptyset$.
	This will be the set of verified programs.
\item For all proofs in $\PA$ of length $\leq n$:
	if the proof proves $\VA(p)$ for some $p$,
	and $|p| \leq l$,
	then add the program $p$ to $P$.
\item For each input history $\ae_{<k}$ repeat:
	run all programs from $P$ for at most $t$ steps each,
	take the one with the highest promised value $v_k$,
	and return that program's policy's action.
\end{enumerate}

\begin{theorem}[The Gödel Prior]
\label{thm:Gödel-prior}\index{AIXItl}\index{prior!Gödel|textbf}
There is a UTM $U'$ such that
if $\PA$ is consistent,
then the set of verified programs $P$ is empty for all $t$, $l$, and $n$.
\end{theorem}
\begin{proof}
Let $q$ denote an algorithm that never halts,
but for which this cannot be proved in $\PA$;
e.g., let $q$ enumerate all consequences of $\PA$
and halt as soon as it finds a contradiction.
Since we assumed that $\PA$ is consistent, $q$ never halts.
Define the UTM $U'(p, a_{1:k})$ as follows.
\begin{itemize}
\item Run $q$ for $k$ steps.
\item If $q$ halts, output $v_k = 2$.
\item Run $U(p, a_{1:k})$.
\end{itemize}
Since $q$ never halts, $U$ and $U'$ are functionally identical,
therefore $U'$ is universal.
Note that $\PA$ proves
$\forall p.\; U(p, a_{1:k}) = U'(p, a_{1:k})$ for any fixed $k$,
but $\PA$ does not prove
$\forall k \forall p.\; U(p, a_{1:k}) = U'(p, a_{1:k})$.

If $q$ did eventually halt,
it would output a value $v_k = 2$ that is too high,
since the value function $V^\pi_\xi$ is bounded by $1$ from above,
which $\PA$ knows.
Hence $\PA$ proves that
\begin{equation}\label{eq:VA-fail}
q \text{ halts} \rightarrow \forall p.\; \neg\VA(p)
\end{equation}
If $\PA$ could prove $\VA(p)$ for any $p$,
then $\PA$ would prove that $q$ does not halt
since this is the contrapositive of \eqref{eq:VA-fail}.
Therefore the set $P$ remains empty.
\end{proof}

AIXI$tl$\index{AIXItl} exhibits all the problems of the arbitrariness of the UTM
illustrated by the indifference prior~(\autoref{thm:indifference-prior}) and
the dogmatic prior~(\autoref{thm:dogmatic-prior}).
In addition, it is also susceptible to Gödel incompleteness
as illustrated by the Gödel prior in \autoref{thm:Gödel-prior}.
The formal system that is a parameter to AIXI$tl$
just provides another point of failure.

As a computable approximation to AIXI,
AIXI$tl$ is needlessly complicated.
As we prove in \autoref{cor:complexity-eps-aixi},
$\eps$-optimal AIXI is limit computable,
so we can approximate it with an anytime algorithm.
Bounding the computational resources to the approximation algorithm
already yields a computable version of AIXI.
Moreover, unlike AIXI$tl$, this approximation
actually converges to AIXI in the limit.
Furthermore, we can `speed up' this approximation algorithm
using \emph{Hutter search}\index{Hutter search}~\citep{Hutter:2002search};
this is very similar but not identical to AIXI$tl$.

\section{Bayes Optimality}
\label{sec:Bayes-optimality}

The aim of the Legg-Hutter intelligence measure is to formalize the intuitive
notion of intelligence mathematically.
\citet{LH:2007collection} collect various definitions of intelligence
across many academic fields and destill it into the following statement%
~\citep{LH:2007int}
\begin{quote}
Intelligence measures an agent’s ability to achieve goals in a wide range of
environments.
\end{quote}
This definition is formalized as follows.

\begin{definition}[{Legg-Hutter Intelligence; \citealp[Sec.~3.3]{LH:2007int}})]
\label{def:intelligence}\index{intelligence|textbf}
\index{Legg-Hutter intelligence|see {intelligence}}
The \emph{(Legg-Hutter) intelligence} of a policy $\pi$ is defined as
\[
   \Upsilon_\xi(\pi)
:= \sum_{\nu \in \M} w(\nu) V^\pi_\nu(\epsilon)
\]
\end{definition}

The Legg-Hutter intelligence of a policy $\pi$ is
the $t_0$-value that $\pi$ achieves across all environments from the class $\M$
weighted by the prior $w$.
\citet{LH:2007int} consider a subclass of $\Mlscccs$,
the class of computable measures
together with a Solomonoff prior $w(\nu) = 2^{-K(\nu)}$
and do not use discounting explicitly.

Typically, the index $\xi$ is omitted when writing $\Upsilon$.
However, in this section we consider
the intelligence measure with respect to different priors,
therefore we make this dependency explicit.
The following proposition motivates the use of the index $\xi$ instead of $w$.

\begin{proposition}[Bayes Optimality = Maximal Intelligence]
\label{prop:Bayes-optimality-and-intelligence}
\index{optimality!Bayes}\index{intelligence}
$\Upsilon_\xi(\pi) = V^\pi_\xi(\epsilon)$ for all policies $\pi$.
\end{proposition}
\begin{proof}
Follows directly from \eqref{eq:value-of-xi} and \autoref{def:intelligence}.
\end{proof}

\begin{definition}[{Balanced Pareto Optimality; \citealp[Def.~5.22]{Hutter:2005}}]
\label{def:balanced-Pareto-optimality}
\index{optimality!balanced Pareto|textbf}
Let $\M$ be a set of environments.
A policy $\pi$ is
\emph{balanced Pareto optimal in the set of environments $\M$} iff
for all policies $\tilde\pi$,
\[
     \sum_{\nu \in \M} w(\nu) \left( V^\pi_\nu(\epsilon)
                                    - V^{\tilde\pi}_\nu(\epsilon) \right)
\geq 0.
\]
\end{definition}

\begin{proposition}[Balanced Pareto Optimality = Maximal Intelligence]
\label{prop:balanced-Pareto-optimality-and-intelligence}
\index{optimality!balanced Pareto}\index{intelligence}
A policy $\pi$ is balanced Pareto optimal in $\M$
if and only if
$\pi$ has maximal Legg-Hutter intelligence.
\end{proposition}
\begin{proof}
From \eqref{eq:value-of-xi} we get
\begin{align*}
   \sum_{\nu \in \M} w(\nu) \left( V^\pi_\nu(\epsilon)
                                  - V^{\pi^*_\xi}_\nu(\epsilon) \right)
&= \sum_{\nu \in \M} w(\nu) V^\pi_\nu(\epsilon)
   - \sum_{\nu \in \M} w(\nu) V^{\pi^*_\xi}_\nu(\epsilon) \\
&= V^\pi_\xi(\epsilon) - V^*_\xi(\epsilon) \\
&= V^\pi_\xi(\epsilon) - \sup_{\tilde\pi} V^{\tilde\pi}_\xi(\epsilon) \\
&= \Upsilon_\xi(\pi) - \sup_{\tilde\pi} \Upsilon(\tilde\pi)
\end{align*}
by \autoref{prop:Bayes-optimality-and-intelligence}.
This term is nonnegative iff $\Upsilon_\xi(\pi)$ is maximal.
\end{proof}

As a consequence of \autoref{prop:Bayes-optimality-and-intelligence} and
\autoref{prop:balanced-Pareto-optimality-and-intelligence} we get that
AIXI is balanced Pareto optimal~\citep[Thm.~5.24]{Hutter:2005}
and has maximal Legg-Hutter intelligence.\index{intelligence!maximal}
\[
     \overline\Upsilon_\xi
~:=~ \sup_\pi \Upsilon_\xi(\pi)
~= ~ \sup_\pi V^\pi_\xi(\epsilon)
~= ~ V^{\pi^*_\xi}_\xi(\epsilon)
~= ~ \Upsilon_\xi(\pi^*_\xi).
\]
This is not surprising since Legg-Hutter intelligence was defined
in terms of the $t_0$-value in the Bayes mixture.
Moreover, because the value function is scaled to be in the interval $[0, 1]$,
intelligence is a real number between $0$ and $1$.

It is just as hard to score very high on the Legg-Hutter intelligence measure
as it is to score very low:
we can always turn a reward minimizer into a reward maximizer by inverting
the rewards $r_t' := 1 - r_t$.
Hence the lowest possible intelligence score is achieved by
AIXI's twin sister, a $\xi$-expected reward minimizer:
\index{intelligence!minimal}
\[
     \underline\Upsilon_\xi
~:=~ \inf_\pi \Upsilon_\xi(\pi)
~ =~ \inf_\pi V^\pi_\xi(\epsilon)
\]
The heaven\index{heaven} environment (reward $1$ forever) and
the hell\index{hell} environment (reward $0$ forever) are computable
and thus in the environment class $\Mlscccs$;
therefore it is impossible
to get a reward $0$ or reward $1$ in every environment.
Consequently, for all policies $\pi$,
\[
       0
~<   ~ \underline\Upsilon_\xi
~\leq~ \Upsilon_\xi(\pi)
~\leq~ \overline\Upsilon_\xi
~<   ~ 1.
\]
For every real number $r \in [\underline\Upsilon_\xi, \overline\Upsilon_\xi]$
there is a policy $\pi$ with $\Upsilon_\xi(\pi) = r$:
analogously to \autoref{lem:V-linear} we can define $\pi$ such that
with probability $(r - \underline\Upsilon_\xi)/(\overline\Upsilon_\xi - \underline\Upsilon_\xi)$ it follows $\pi^*_\xi$ and
otherwise it follows $\argmin_{\tilde\pi} V^{\tilde\pi}_\xi(\epsilon)$.

\autoref{fig:intelligence} illustrates the intelligence measure $\Upsilon$.
It is natural to fix the policy \texttt{random}
that takes actions uniformly at random
to have an intelligence score of $1 / 2$
by choosing a `symmetric' universal prior~\citep{LV:2013int}.

\begin{figure}[t]
\begin{center}
\begin{tikzpicture}
\filldraw[mycolor] (1.2, .05) -- (8.7, .05) -- (8.7, -.05) -- (1.2, -.05);
\draw (0,0) to (10, 0);
\draw (0, -.2) to (0, .2) node[above] {$0$};
\draw (10, -.2) to (10, .2) node[above] {$1$};
\draw[color=mycolor] (8.7, .1) to (8.7, -.1) node[below,color=black]
	{$\overline\Upsilon_\xi$};
\draw[color=mycolor] (1.2, .1) to (1.2, -.1) node[below,color=black]
	{$\underline\Upsilon_\xi$};

\draw[->] (4.2, 0.5) node[above] {\texttt{random}} to (4.2, .2);
\draw[->] (1.3, 0.5) node[above] {AI$\xi'$} to (1.3, .2);
\draw[->] (8.7, 0.5) node[above] {AI$\xi$} to (8.7, .2);
\draw[color=mycolor] node[below] at (5, -.1) {image of $\Upsilon$};
\end{tikzpicture}
\end{center}
\caption[Legg-Hutter intelligence measure]{%
The Legg-Hutter intelligence measure assigns values
within the closed interval $[\underline\Upsilon_\xi, \overline\Upsilon_\xi]$;
the assigned values are depicted in purple.
By \autoref{thm:computable-policies-are-dense},
computable policies are dense in this purple set.
}
\label{fig:intelligence}
\index{intelligence}
\end{figure}
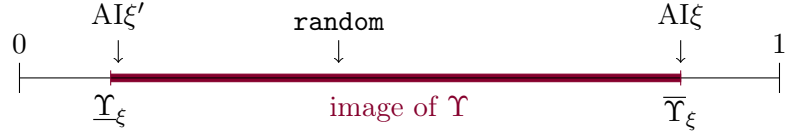

AIXI is not computable~(\autoref{thm:AIXI-is-not-computable}),
hence there is no computable policy $\pi$ such that
$\Upsilon_\xi(\pi) = \underline\Upsilon_\xi$ or
$\Upsilon_\xi(\pi) = \overline\Upsilon_\xi$ for any Bayesian mixture $\xi$
over $\Mlscccs$.
But the following theorem states that
computable policies can come arbitrarily close.
This is no surprise:
by \autoref{lem:value-bound}
we can do well on a Legg-Hutter intelligence test
simply by memorizing what AIXI would do for the first $k$ steps;
as long as $k$ is chosen large enough such that
discounting makes the remaining rewards contribute
very little to the value function.

\begin{theorem}[Computable Policies are Dense]
\label{thm:computable-policies-are-dense}
\index{intelligence}\index{policy!computable}
The set
\[
\{ \Upsilon_\xi(\pi) \mid \pi \text{ is a computable policy} \}
\]
is dense in the set $[\underline\Upsilon_\xi, \overline\Upsilon_\xi]$.
\end{theorem}
\begin{proof}
Let $\pi$ be any policy and let $\varepsilon > 0$.
We need to show that there is a computable policy $\tilde\pi$ with
$|\Upsilon_\xi(\tilde\pi) - \Upsilon_\xi(\pi)| < \varepsilon$.
We choose $m$ large enough such that $\Gamma_m / \Gamma_1 < \varepsilon/3$.
Let $\alpha \in \A$ be arbitrary and define the policy
\[
\tilde\pi(a \mid \ae_{<t}) ~:=~
\begin{cases}
\pi(a \mid \ae_{<t}) \pm (\eps/3)^{-m}
  &\text{if } t < m, \\
1 &\text{if $t \geq m$ and $a = \alpha$}, \text{ and} \\
0 &\text{otherwise}.
\end{cases}
\]
By choosing an appropriate rational number in the interval
$[\pi(a \mid \ae_{<t}) - (\eps/3)^{-m}, \pi(a \mid \ae_{<t}) + (\eps/3)^{-m}]$
we can make the policy $\tilde\pi$ computable because we can store
these approximations to the action probabilities of $\pi$
for the first $m - 1$ steps in a lookup table.
From \autoref{lem:value-bound} we get
\[
     \left| V^{\pi,m}_\xi(\epsilon) - V^{\tilde\pi,m}_\xi(\epsilon) \right|
\leq D_{m-1}(\xi^\pi, \xi^{\tilde\pi} \mid \epsilon)
\leq \left( (\eps/3)^{-m} \right)^m
=    \frac{\eps}{3}
\]
and together with \autoref{lem:truncated-values} this yields
\[
     \left| \Upsilon_\xi(\pi) - \Upsilon_\xi(\tilde\pi) \right|
=    \left| V^\pi_\xi(\epsilon) - V^{\tilde\pi}_\xi(\epsilon) \right|
\leq \left| V^{\pi,m}_\xi(\epsilon) - V^{\tilde\pi,m}_\xi(\epsilon) \right|
     + 2\frac{\Gamma_m}{\Gamma_1}
\leq \frac{\eps}{3} + 2\frac{\Gamma_m}{\Gamma_1}
<    \eps.
\qedhere
\]
\end{proof}

\begin{remark}[{Deterministic Policies are not Dense in
$[\underline\Upsilon_\xi, \overline\Upsilon_\xi]$}]
\label{rem:intelligence-is-not-dense}
\index{intelligence}\index{policy!deterministic}
The intelligence values of deterministic policies are generally not
dense in the interval $[\underline\Upsilon_\xi, \overline\Upsilon_\xi]$.
We show this by defining an environment $\nu$ where
the first action determines whether
the agent goes to heaven\index{heaven} or hell\index{hell}:
action $\alpha$ leads to heaven and action $\beta$ leads to hell.
Define Bayesian mixture
$\xi' := 0.999 \nu + 0.001 \xi$ and
let $\pi$ be any policy.
If $\pi$ takes action $\alpha$ first,
then $\Upsilon_{\xi'}(\pi) > 0.999$.
If $\pi$ takes action $\beta$ first,
then $\Upsilon_{\xi'}(\pi) < 0.001$.
Hence there are no deterministic policies that
score an intelligence value in the closed interval $[0.001, 0.999]$.
\end{remark}

Legg-Hutter intelligence is measured
with respect to a fixed prior.
The Bayes agent is the most intelligent policy
\emph{if it uses the same prior}.
We use the results from \autoref{sec:bad-priors}
to show that the intelligence score of the Bayes agent
can be arbitrary close to
the minimum intelligence score $\underline\Upsilon_\xi$.

\begin{corollary}[Some AIXIs are Stupid]
\label{cor:some-AIXIs-are-stupid}
\index{intelligence}\index{AIXI}
For any Bayesian mixture $\xi$ over $\Mlscccs$ and every $\varepsilon > 0$,
there is a Bayesian mixture $\xi'$ such that
$\Upsilon_\xi(\pi^*_{\xi'}) < \underline\Upsilon_\xi + \varepsilon$.
\end{corollary}
\begin{proof}
Let $\varepsilon > 0$.
According to \autoref{thm:computable-policies-are-dense},
there is a computable policy $\pi$ such that
$\Upsilon_\xi(\pi) < \underline\Upsilon_\xi + \varepsilon / 2$.
From \autoref{cor:emulation} we get a Bayesian mixture $\xi'$ such that
$
  | \Upsilon_\xi(\pi^*_{\xi'}) - \Upsilon_\xi(\pi) |
= | V^{\pi^*_{\xi'}}_\xi(\epsilon) - V^\pi_\xi(\epsilon) |
< \varepsilon / 2
$,
hence
\[
     |\Upsilon_\xi(\pi^*_{\xi'}) - \underline\Upsilon_\xi|
\leq |\Upsilon_\xi(\pi^*_{\xi'}) - \Upsilon_\xi(\pi)|
     + |\Upsilon_\xi(\pi) - \underline\Upsilon_\xi|
<    \varepsilon / 2 + \varepsilon / 2
=    \varepsilon.
\qedhere
\]
\end{proof}
We get the same result if we fix AIXI, but rig the intelligence measure.

\begin{corollary}[AIXI is Stupid for Some $\Upsilon$]
\label{cor:AIXI-is-stupid}
\index{intelligence}\index{AIXI}
For any deterministic $\xi$-optimal policy $\pi^*_\xi$ and
for every $\varepsilon > 0$
there is a Bayesian mixture $\xi'$ such that
$\Upsilon_{\xi'}(\pi^*_\xi) \leq \varepsilon$ and
$\overline\Upsilon_{\xi'} \geq 1 - \varepsilon$.
\end{corollary}
\begin{proof}
Let $a_1 := \pi^*_\xi(\epsilon)$ be the first action that $\pi^*_\xi$ takes.
We define an environment $\nu$ such that
taking the first action $a_1$ leads to hell\index{hell} and
taking any other first action leads to heaven\index{heaven}
as in \autoref{rem:intelligence-is-not-dense}.
We define the Bayesian mixture
$\xi' := (1 - \varepsilon) \nu + \varepsilon \xi$.
Since $\pi^*_\xi$ takes action $a_1$ first,
it goes to hell, i.e., $V^{\pi^*_\xi}_\nu(\epsilon) = 0$.
Hence with \autoref{lem:V-linear}
\begin{align*}
       \Upsilon_{\xi'}(\pi^*_\xi)
~&=~   V^{\pi^*_\xi}_{\xi'}(\epsilon)
~=~    (1 - \varepsilon) V^{\pi^*_\xi}_\nu(\epsilon)
       + \varepsilon V^{\pi^*_\xi}_\xi(\epsilon)
~\leq~ \varepsilon. \\
\intertext{
For any policy $\pi$ that
takes an action other than $a_1$ first,
we get
}
       \Upsilon_{\xi'}(\pi)
~&=~   V^\pi_{\xi'}(\epsilon)
~=~    (1 - \varepsilon) V^\pi_\nu(\epsilon) + \varepsilon V^\pi_\xi(\epsilon)
~\geq~ 1 - \varepsilon.
\qedhere
\end{align*}
\end{proof}

On the other hand,
we can make any computable policy smart
if we choose the right Bayesian mixture.
In particular, we get that there is a Bayesian mixture such that
`do nothing' is the most intelligent policy save for some $\varepsilon$.

\begin{corollary}[Computable Policies can be Smart]
\label{cor:any-computable-policy-can-be-smart}
\index{intelligence}\index{policy!computable}
For any computable policy $\pi$ and any $\varepsilon > 0$
there is a Bayesian mixture $\xi'$ such that
$\Upsilon_{\xi'}(\pi) > \overline\Upsilon_{\xi'} - \varepsilon$.
\end{corollary}
\begin{proof}
\autoref{cor:emulation} yields a Bayesian mixture $\xi'$ with
$
  | \overline\Upsilon_{\xi'} - \Upsilon_{\xi'}(\pi) |
= | V^*_{\xi'}(\epsilon) - V^\pi_{\xi'}(\epsilon) |
< \varepsilon
$.
\end{proof}

\section{Asymptotic Optimality}
\label{sec:asymptotic-optimality}

An asymptotically optimal policy is a policy
learns to act optimally in every environment from $\M$,
i.e., the value of this policy converges to the optimal value.

\begin{definition}[Asymptotic Optimality]
\label{def:asymptotic-optimality}\index{optimality!asymptotic|textbf}
A policy $\pi$ is
\emph{asymptotically optimal in an environment class $\M$}
iff for all $\mu \in \M$
\begin{equation}\label{eq:asymptotic-optimality}
V^*_\mu(\ae_{<t}) - V^\pi_\mu(\ae_{<t})
\to 0
\text{ as $t \to \infty$}
\end{equation}
on histories drawn from $\mu^\pi$.
\end{definition}

\begin{table}
\begin{center}
\begingroup
\def\arraystretch{1.5}
\begin{tabular}{ll}
\toprule
name & definition \\
\midrule
strong a.o.
	& $V^*_\mu(\ae_{<t}) - V^\pi_\mu(\ae_{<t}) \to 0$ $\mu^\pi$-almost surely \\
a.o.\ in mean
	& $\EE^\pi_\mu \left[ V^*_\mu(\ae_{<t}) - V^\pi_\mu(\ae_{<t}) \right] \to 0$ \\
a.o.\ in probability
	& $\forall \eps > 0.\; \mu^\pi \left[ V^*_\mu(\ae_{<t}) - V^\pi_\mu(\ae_{<t}) > \eps \right] \to 0$ \\
weak a.o.
	& $\frac{1}{t} \sum_{k=1}^t \Big( V^*_\mu(\ae_{<k}) - V^\pi_\mu(\ae_{<k}) \Big) \to 0$ $\mu^\pi$-almost surely \\
\bottomrule
\end{tabular}
\endgroup
\end{center}
\caption[Types of asymptotic optimality]{
The formal definition of different types of asymptotic optimality.
In each case we understand the limit as $t \to \infty$.
}
\label{tab:asymptotic-optimality}
\end{table}

There are different types of asymptotic optimality
based on the type of stochastic convergence in \eqref{eq:asymptotic-optimality};
see \autoref{def:stochastic-convergence}.
If this convergence occurs almost surely,
it is called \emph{strong asymptotic optimality}~\citep[Def.~7]{LH:2011opt}%
\index{optimality!asymptotic!strong}\index{convergence!almost surely};
if this convergence occurs in mean,
it is called \emph{asymptotic optimality in mean}%
\index{optimality!asymptotic!in mean}\index{convergence!in mean};
if this convergence occurs in probability,
it is called \emph{asymptotic optimality in probability}%
\index{optimality!asymptotic!in probability}\index{convergence!in probability};
and
if the Cesàro averages\index{Cesàro average} converge almost surely,
it is called \emph{weak asymptotic optimality}%
\index{optimality!asymptotic!weak}~\citep[Def.~7]{LH:2011opt}.
Since the value function is a nonnegative bounded random variable,
asymptotic optimality in mean and asymptotic optimality in probability
are equivalent.
See \autoref{tab:asymptotic-optimality} for the explicit definitions and
see \autoref{fig:asymptotic-optimality} for an overview over
their relationship.

\begin{figure}
\begin{center}
\begin{tikzpicture}
\node (sao)  at (0,2) {strong a.o.};
\node (wao)  at (-4,0) {weak a.o.};
\node (aoim) at (0,0) {a.o.\ in mean};
\node (aoip) at (4,0) {a.o.\ in probability};

\draw[->] (sao) to (wao);
\draw[->] (sao) to (aoim);
\draw[->] (sao) to (aoip);
\draw[<->] (aoim) to (aoip);
\end{tikzpicture}
\end{center}
\caption[Relationship between different types of asymptotic optimality]{
The relationship between different types of asymptotic optimality.
Each arrow indicates a logical implication and
each lack of an arrow indicates that there is no logical implication.
}
\label{fig:asymptotic-optimality}
\end{figure}
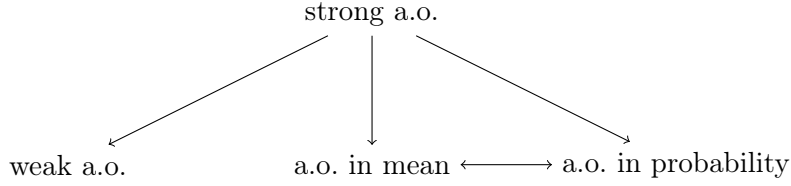

Asymptotic optimality in probability is in spirit
a probably approximately correct\index{PAC}~(PAC) result:
for all $\eps > 0$ and $\delta > 0$
the probability that our policy is $\eps$-suboptimal converges to zero;
eventually this probability will be less than $\delta$.
For a PAC result it is typically demanded that the number of time steps
until the probability is less than $\delta$
be polynomial in $1/\eps$ and $1/\delta$.
In general environments this is impossible,
and here we have no ambition to provide concrete convergence rates.

Intuitively, a necessary condition for asymptotic optimality is
that the agent needs to explore infinitely often for an entire effective horizon.
If we explore only finitely often,
then the environment might change after we stopped exploring.
Moreover, the agent needs to predict
the value of counterfactual policies accurately;
but by \autoref{lem:truncated-values} only for an $\eps$-effective horizon.
By committing to exploration for the entire effective horizon,
we learn about the value of counterfactual policies.

\begin{example}[Exploration Infinitely Often for an Entire Effective Horizon]
\label{ex:exploration}
\index{effective horizon}
If there is an $\eps > 0$ such that
the policy $\pi$ does not explore for $H_t(\eps)$ steps infinitely often,
then $V^*_\mu(\ae_{<t}) - V^\pi_\mu(\ae_{<t}) > \eps$ infinitely often.
Define $\A := \{ \alpha, \beta \}$ and $\E := \{ 0, \eps/2, 1 \}$%
~(observations are vacuous) and
consider the following class of environments
$\M := \{ \nu_\infty, \nu_1, \nu_2, \ldots \}$
(transitions are labeled with condition: action, reward):
\begin{center}
\begin{tabular}{ccc}
\begin{tikzpicture}[scale=1.2]
\node[circle, draw, minimum height=2em] (s0) at (0, 0) {$s_0$};

\draw[->] (s0) to[loop above] node {$\beta, \frac{\eps}{2}$} (s0);
\draw[->] (s0) to[loop below] node {$\alpha, 0$} (s0);
\end{tikzpicture} & ~~~ &
\begin{tikzpicture}[scale=1.2]
\node[circle, draw, minimum height=2em] (s0) at (0, 0) {$s_0$};
\node[circle, draw, minimum height=2em] (s1) at (3, 0) {$s_1$};
\node                                (ldots) at (5, 0) {$\ldots$};
\node[circle, draw, minimum height=2em] (sn) at (7, 0) {$s_n$};

\draw[->] (s0) to[loop above] node {$\beta, \frac{\eps}{2}$} (s0);
\draw[->] (s0) to[loop below] node {$t < k: \alpha, 0$} (s0);
\draw[->] (s0) to node[above] {$t \geq k: \alpha, 0$} (s1);
\draw[->] (s1) to node[above] {$\alpha, 0$} (ldots);
\draw[->] (ldots) to node[above] {$\alpha, 0$} (sn);
\draw[->] (s1) to[bend left] node[above] {$\beta, 0$} (s0);
\draw[->] (ldots) to[bend left] node[above] {$\beta, 0$} (s0);
\draw[->] (sn) to[bend left] node[below] {$\beta, 0$} (s0);
\draw[->] (sn) to[loop below] node {$\alpha, 1$} (sn);
\end{tikzpicture} \\
$\nu_\infty$ && $\nu_k$
\end{tabular}
\end{center}
Environment $\nu_k$ works just like environment $\nu_\infty$,
except that at time step $k$ the path to state $s_1$ gets unlocked.
The length of the state sequence in $\nu_k$ is defined as
an $\eps$-effective horizon, $n := H_t(\eps)$
where $t$ is the time step in which the agent leaves state $s_0$.
The optimal policy in environment $\nu_\infty$ is to always take action $\beta$,
the optimal policy for environment $\nu_k$ is
to take action $\beta$ for $t < k$ and then take action $\alpha$.
Suppose the agent is in time step $t$ and in state $s_0$.
Since these environments are partially observable\index{partially observable},
it needs to explore for $n$ steps~(take action $\alpha$ $n$ times)
to distinguish $\nu_\infty$ from $\nu_k$ for any $k \leq t$.
Since there are infinitely many $\nu_k$,
the agent needs to do this infinitely often.
Moreover, $V^*_{\nu_1} \geq \eps$ and $V^*_{\nu_\infty} = \eps/2$,
so if $\nu_t$ is the true environment,
then not exploring to the right for an $\eps$-effective horizon
is suboptimal by $\eps/2$.
But if $\nu_\infty$ is the true environment,
then exploring incurs an opportunity cost of one reward of $\eps/2$.
\end{example}

Next, we state two negative results about asymptotic optimality
proved by \citet{LH:2011opt}.
It is important to emphasize that
\autoref{thm:sao-impossible} and \autoref{thm:wao-impossible}
only hold for deterministic policies.

\begin{theorem}[{Deterministic Policies are not Strongly Asymptotically Optimal;
\citealp[Thm.~8]{LH:2011opt}}]
\label{thm:sao-impossible}
\index{policy!deterministic}\index{optimality!asymptotic!strong}
There is no deterministic policy
that is strong asymptotically optimal in the class $\Mcomp$.
\end{theorem}

If the horizon grows linearly
(for example, power discounting $\gamma(t) = t^{-\beta}$
with $\beta > 1$; see \autoref{tab:discounting}),
then a deterministic policy cannot be weakly asymptotically optimal policy:
the agent has to explore for an entire effective horizon,
which prevents the Cesàro average from converging.

\begin{theorem}[{Necessary Condition for Weak Asymptotic Optimality;
\citealp[Thm.~5.5]{Lattimore:2013}}]
\label{thm:wao-impossible}
\index{policy!deterministic}\index{optimality!asymptotic!weak}
\index{effective horizon}\index{discount!function}
If there is an $\eps > 0$ such that $H_t(\eps) \notin o(t)$,
then there is no deterministic policy
that is weakly asymptotically optimal in the class $\Mcomp$.
\end{theorem}

There are several agents that achieve asymptotic optimality.
In the rest of this section, we discuss the Bayes agent,
{\BayesExp}, and Thompson sampling.
Asymptotic optimality can also be achieved
through optimism~\citep{SH:12optopt,SH:2012optimistic,SH:2015opt}.

\subsection{Bayes}
\label{ssec:asymptotic-optimality-Bayes}

In this section, we list two results from the literature
regarding the asymptotic optimality of the Bayes optimal policy.
The following negative result is due to \citet{Orseau:2010,Orseau:2013}.

\begin{theorem}[{Bayes is not Asymptotically Optimal in General Environments;
\citealp[Thm.~4]{Orseau:2013}}]
\label{thm:AIXI-not-AO}\index{AIXI}\index{optimality!asymptotic}
For any class $\M \supseteq \Mcomp$
no Bayes optimal policy $\pi^*_\xi$ is asymptotically optimal:
there is an environment $\mu \in \M$ and
a time step $t_0 \in \mathbb{N}$ such that
$\mu^{\pi^*_\xi}$-almost surely for all time steps $t \geq t_0$
\[
V^*_\mu(\ae_{<t}) - V^{\pi^*_\xi}_\mu(\ae_{<t}) = \frac{1}{2}.
\]
\end{theorem}

Orseau calls this result the \emph{good enough effect}\index{good enough effect}:
A Bayesian agent eventually decides that the current strategy is good enough
and that any additional exploration is not worth its expected payoff.
However, if the environment changes afterwards,
the Bayes agent is acting suboptimally.

\begin{proof}
Without loss of generality
assume $\A := \{ \alpha, \beta \}$ and $\E := \{ 0, 1/2, 1 \}$%
~(observations are vacuous).
We consider the following environment $\mu$%
~(transitions are labeled with action, reward).
\begin{center}
\begin{tikzpicture}[scale=1.2]
\node[circle, draw, minimum height=2em] (s0) at (0, 0) {$s_0$};
\node[circle, draw, minimum height=2em] (s1) at (3, 0) {$s_1$};
\node                                (ldots) at (5, 0) {$\ldots$};
\node[circle, draw, minimum height=2em] (sn) at (7, 0) {$s_n$};

\draw[->] (s0) to[loop above] node {$\beta, \frac{1}{2}$} (s0);
\draw[->] (s0) to node[above] {$\alpha, 0$} (s1);
\draw[->] (s1) to node[above] {$\ast, 0$} (ldots);
\draw[->] (ldots) to node[above] {$\ast, 0$} (sn);
\draw[->] (sn) to[bend left] node[below] {$\ast, 0$} (s0);
\end{tikzpicture}
\end{center}
In state $s_0$ the action $\beta$ is the exploitation action and
the action $\alpha$ the exploration action.
The length of the state sequence is defined as
an $1/t$-effective horizon\index{effective horizon}, $n := H_t(1/t)$
where $t$ is the time step in which the agent leaves state $s_0$.
Since the discount function $\gamma$ is computable
by \assref{ass:aixi}{ass:gamma-computable},
$\mu \in \Mlscccs$.

Assume that when acting in $\mu$,
the Bayes agent explores infinitely often.
Let $\ae_{<t}$ be a history in which
the agent is in state $s_0$ and takes action $\alpha$.
Then $V^{\pi^*_\xi}_\mu \leq 1/t$.
By on-policy value convergence%
~(\autoref{cor:Bayes-on-policy-value-convergence}),
$V^*_\xi(\ae_{<t}) - V^{\pi^*_\xi}_\mu(\ae_{<t}) \to 0$
$\mu^{\pi^*_\xi}$-almost surely.
Hence there is a time step $t_0$ such that for all $t \geq t_0$
we have $V^*_\xi < w(\mu) / 2$.
Since $\mu$ is deterministic, $w(\mu \mid \ae_{<t}) \geq w(\mu)$.
Now we get a contradiction from \eqref{eq:optimal-value-of-xi}:
\[
     V^*_\xi(\ae_{<t})
\geq w(\mu | \ae_{<t}) V^*_\mu(\ae_{<t})
\geq w(\mu) V^*_\mu(\ae_{<t})
=    \frac{w(\mu)}{2}
>    V^*_\xi(\ae_{<t})
\]

Therefore the Bayes agent stops taking the exploration action $\alpha$
after time step $t_0$,
and so it is not optimal in any $\nu \in \Mlscccs$ that behaves like $\mu$
until time step $t_0$ and then changes:
\begin{center}
\begin{tikzpicture}[scale=1.2]
\node[circle, draw, minimum height=2em] (s0) at (0, 0) {$s_0$};
\node[circle, draw, minimum height=2em] (s1) at (3, 0) {$s_1$};
\node                                (ldots) at (5, 0) {$\ldots$};
\node[circle, draw, minimum height=2em] (sn) at (7, 0) {$s_n$};

\draw[->] (s0) to[loop above] node {$\beta, \frac{1}{2}$} (s0);
\draw[->] (s0) to[loop left] node[left] {$t > t_0: \alpha, 1$} (s0);
\draw[->] (s0) to node[above] {$t \leq t_0: \alpha, 0$} (s1);
\draw[->] (s1) to node[above] {$\ast, 0$} (ldots);
\draw[->] (ldots) to node[above] {$\ast, 0$} (sn);
\draw[->] (sn) to[bend left] node[below] {$\ast, 0$} (s0);
\end{tikzpicture}
\end{center}
\vspace{-2em} 
\end{proof}

The following theorem is also known as
the \emph{self-optimizing theorem}\index{self-optimizing}.
This theorem has been a source of great confusion because
its statement in \citet[Thm.~5.34]{Hutter:2005} is not very explicit
about how the histories are generated.
The formulation of \citet[Thm.~5.2]{Lattimore:2013}
is explicit, but less general.

\begin{theorem}[{Sufficient Condition for Strong Asymptotic Optimality of Bayes; \citealp[Thm.~5.34]{Hutter:2005}}]
\label{thm:self-optimizing}
\index{optimality!asymptotic!strong}\index{optimality!Bayes}
Let $\mu$ be some environment.
If there is a policy $\pi$ and
a sequence of policies $\pi_1, \pi_2, \ldots$ such that
for all $\nu \in \M$
\begin{equation}\label{eq:self-optimizing-condition}
V^*_\nu(\ae_{<t}) - V^{\pi_t}_\nu(\ae_{<t}) \to 0
\text{ as $t \to \infty$ $\mu^\pi$-almost surely,}
\end{equation}
then
\[
V^*_\mu(\ae_{<t}) - V^{\pi^*_\xi}_\mu(\ae_{<t}) \to 0
\text{ as $t \to \infty$ $\mu^\pi$-almost surely.}
\]
If $\pi = \pi^*_\xi$ and
\eqref{eq:self-optimizing-condition} holds for all $\mu \in \M$,
then $\pi^*_\xi$ is strongly asymptotically optimal in the class $\M$.
\end{theorem}

It is important to emphasize that
the policies $\pi_1, \pi_2, \ldots$ need to converge to the optimal value
on the history generated by $\mu$ and $\pi$,
and not (as one might think) $\nu$ and $\pi_t$.
Intuitively, the policy $\pi$ is an `exploration policy'
that ensures that the environment class is explored sufficiently.
Typically, a policy is asymptotically optimal on its own history.
So if $\pi = \pi_1 = \pi_2 = \ldots$,
then we get that Bayes is asymptotically optimal on the history generated
by the policy $\pi$, not its own history.
In light of \autoref{thm:dogmatic-prior} and \autoref{thm:AIXI-not-AO}
this is not too surprising;
Bayesian reinforcement learning agents might not explore enough to be
asymptotically optimal, but given a policy that does explore enough,
Bayes learns enough to be asymptotically optimal.

This invites us to define the following policies $\pi_t$:
follow the information-seeking policy $\pi^*_{\IG}$ until time step $t$,
and then follow $\pi^*_\xi$ (explore until $t$, then exploit)%
\index{exploration vs.\ exploitation}.
Since the information-seeking policy explores enough to prove
off-policy\index{off-policy} prediction~\citep[Thm.~7]{OLH:2013ksa},
we get $V^\pi_\xi - V^\pi_\mu \to 0$ for every policy $\pi$ uniformly.
Hence $\argmax_\pi V^\pi_\xi \to \argmax_\pi V^\pi_\mu$
and thus $V^*_\mu - V^{\pi^*_\xi}_\mu \to 0$ and
\eqref{eq:self-optimizing-condition} is satisfied.
From \autoref{thm:self-optimizing} we get $V^*_\mu - V^{\pi^*_\xi}_\mu \to 0$,
which we already knew.
In order to get strong asymptotic optimality,
all we need to do is choose the switching time step $t$ appropriately,
i.e., wait until $V^*_\mu$ and $V^{\pi^*_\xi}_\mu$ are close enough.
Unfortunately, this is an invalid strategy:
the agent does not know the true environment $\mu$
and hence cannot check this condition.

\citet[Sec.~5.6]{Hutter:2005} uses \autoref{thm:self-optimizing} to show that
the Bayes optimal policy is strongly asymptotically optimal
in the class of ergodic finite-state MDPs if the effective horizon%
\index{effective horizon} is growing,
i.e., $H_t(\eps) \to \infty$ for all $\eps > 0$.
This relies on the fact that in ergodic finite-state MDPs
we need a fixed number of steps to explore the entire environment
up to $\eps$-confidence.
Therefore we can define a sequence of policies $\pi_1, \pi_2, \ldots$
that completely disregard the history and start exploring everything from scratch.
Since the effective horizon is growing,
this exploration phase takes a vanishing fraction of effective horizon
and most of the value is retained.
Therefore the sequence of policies $\pi_1, \pi_2, \ldots$
satisfies the condition of \autoref{thm:self-optimizing} regardless of the history,
thus in particular for the history generated by
$\pi = \pi^*_\xi$ and any $\mu \in \M$.
Note that the condition on the horizon is important:
If the effective horizon is bounded\index{effective horizon!bounded},
then Bayes is not asymptotically optimal
in the class of ergodic finite-state MDPs
because it can be locked into a dogmatic prior
similarly to \autoref{thm:dogmatic-prior}.

\begin{proof}[Proof of \autoref{thm:self-optimizing}]
From \eqref{eq:value-of-xi} we get for any history $\ae_{<t}$
\begingroup
\allowdisplaybreaks
\begin{align}
      w(\mu \mid \ae_{<t}) \left( V^*_\mu(\ae_{<t}) - V^{\pi^*_\xi}_\mu(\ae_{<t}) \right)
&\leq \sum_{\nu \in \M} w(\nu \mid \ae_{<t}) \left( V^*_\nu(\ae_{<t})
                                    - V^{\pi^*_\xi}_\nu(\ae_{<t}) \right) \notag \\
&=    \left( \sum_{\nu \in \M} w(\nu \mid \ae_{<t}) V^*_\nu(\ae_{<t}) \right) - V^{\pi^*_\xi}_\xi(\ae_{<t}) \notag \\
&\leq \sum_{\nu \in \M} w(\nu \mid \ae_{<t}) V^*_\nu(\ae_{<t}) - V^{\pi_t}_\xi(\ae_{<t}) \notag \\
&=    \sum_{\nu \in \M} w(\nu \mid \ae_{<t}) \big( V^*_\nu(\ae_{<t}) - V^{\pi_t}_\nu(\ae_{<t}) \big). \label{eq:value-difference-relation}
\end{align}
\endgroup
From \eqref{eq:self-optimizing-condition} follows that
$V^*_\nu - V^{\pi_t}_\nu \to 0$ $\mu^\pi$-almost surely for all $\nu \in \M$,
so \eqref{eq:value-difference-relation}
converges to $0$ $\mu^\pi$-almost surely~\citep[Lem.~5.28ii]{Hutter:2005}. 
Similar to \autoref{ex:posterior-martingale},
$1/w(\mu \mid \ae_{<t})$ is a nonnegative $\mu^\pi$-martingale and thus
converges (to a finite value) $\mu^\pi$-almost surely
by \autoref{thm:martingale-convergence}\index{convergence!martingale}.
Therefore $V^*_\mu(\ae_{<t}) - V^{\pi^*_\xi}_\mu(\ae_{<t}) \to 0$
$\mu^\pi$-almost surely.
If this is true for all $\mu \in \M$,
the strong asymptotic optimality of $\pi^*_\xi$ follows from $\pi = \pi^*_\xi$
by definition.
\end{proof}

\subsection{BayesExp}
\label{ssec:asymptotic-optimality-BayesExp}
\index{BayesExp}

The definition of {\BayesExp} is given in \autoref{ssec:BayesExp}.
In this subsection we state a result by \citet{Lattimore:2013}
that motivated the definition of {\BayesExp}.

\begin{theorem}[{{\BayesExp} is Weakly Asymptotically Optimal;
\citealp[Thm. 5.6]{Lattimore:2013}}]
\label{thm:BayesExp-wao}
\index{BayesExp}\index{optimality!asymptotic!weak}\index{discount!function}
Let $\pi_{BE}$ denote the policy from \autoref{alg:BayesExp}.
If $H_t(\eps)$ grows monotone in $t$ and $H_t(\eps_t)/\eps_t \in o(t)$,
then for all environments $\mu \in \M$
\[
\frac{1}{t} \sum_{k=1}^t \big(
	V^*_\mu(\ae_{<k}) - V^{\pi_{BE}}_\mu(\ae_{<k})
\big) \to 0
\text{ as $t \to \infty$ $\mu^{\pi_{BE}}$-almost surely}.
\]
\end{theorem}

If the horizon grows sublinearly~($H_t(\eps) \in o(t)$ for all $\eps > 0$),
then we can always find a sequence $\eps_t \to 0$
that decreases slowly enough such that $H_t(\eps) / \eps_t \in o(t)$ holds.

\subsection{Thompson Sampling}
\label{ssec:asymptotic-optimality-TS}

In this section we prove that
the Thompson sampling policy defined in \autoref{ssec:Thompson-sampling}
is asymptotically optimal.
\citet{OB:2010Thompson} prove that the action probabilities of Thompson sampling
converge to the action probability of the optimal policy almost surely,
but require a finite environment class and
two (arguably quite strong)
technical assumptions on the behavior of the posterior distribution
(akin to ergodicity)
and the similarity of environments in the class.
Our convergence results do not require these assumptions.

\begin{theorem}[Thompson Sampling is Asymptotically Optimal in Mean]
\label{thm:Thompson-sampling-aoim}
\index{Thompson sampling}\index{optimality!asymptotic!in mean}
For all environments $\mu \in \M$,
\[
\EE_\mu^{\pi_T} \big[ V^*_\mu(\ae_{<t}) - V^{\pi_T}_\mu(\ae_{<t}) \big] \to 0
\text{ as $t \to \infty$.}
\]
\end{theorem}

This theorem immediately implies that
Thompson sampling is also asymptotically optimal in probability
according to \autoref{fig:asymptotic-optimality}.
However, this does not imply almost sure convergence
(see \autoref{ex:TS-not-sao}).

We first give an intuition for the asymptotic optimality of Thompson sampling.
At every resampling step we can split the class $\M$ into three partitions:
\begin{enumerate}[1.]
\item Environments $\rho$ where $V^{\pi^*_\rho}_\mu \approx V^*_\mu$
\item Environments $\rho$ where $V^*_\rho > V^*_\mu$
\item Environments $\rho$ where $V^*_\rho < V^*_\mu$
\end{enumerate}
The first class is the class of `good' environments:
if we draw one of them, we follow a policy that is close to optimal in $\mu$.
The second class is the class of environments that
overestimate the value of $\mu$.
Following their optimal policy the agent gains information because
rewards will be lower than expected.
The third class is the class of environments that
underestimate the value of $\mu$.
Following their optimal policy the agent might not gain information
since $\mu$ might behave just like environment $\rho$
on the $\rho$-optimal policy.
However, when sampling from the first class instead,
the agent gains information about the third class because rewards tend to be
better than environments from the third class predicted.

Since the true environment $\mu \in \M$, the first class is not empty,
and the probability of drawing a sample from the first class
does not become too small.
Whenever the second and third class
have sufficiently high weight in the posterior,
there is a good chance of picking a policy that leads the agent
to gain information.
Asymptotically, the posterior converges, so the agent ends up
having learned everything it could,
i.e., the posterior weight of the second and third class vanishes.

This argument is not too hard to formalize for deterministic environment classes.
However, for stochastic environment classes
the effect on the posterior when following a bad policy is harder to quantify
because there is always a chance that the rewards are different
simply because of bad luck.
In order to prove this theorem in its generality for stochastic classes,
we employ an entirely different proof strategy
that relies on statistical tools rather than the argument given above.

\begin{definition}[Expected Total Variation Distance]
\label{def:Betvd}\index{total variation distance!expected|textbf}
Let $\pi$ be any policy and let $m \in \mathbb{N} \cup \infty$.
The \emph{expected total variation distance} on the policy $\pi$ is
\[
   F^\pi_m(\ae_{<t})
:= \sum_{\rho \in \M} w(\rho \mid \ae_{<t}) D_m(\rho^\pi, \xi^\pi \mid \ae_{<t}).
\]
\end{definition}

If we replace the distance measure $D_m$ by cross-entropy, then
the quantity $F^\pi_m(\ae_{<t})$ becomes the
expected information gain\index{information gain}%
~(see \autoref{ssec:knowledge-seeking-agents}).

For the proof of \autoref{thm:Thompson-sampling-aoim}
we need the following lemma.

\begin{lemma}[Expected Total Variation Distance Vanishes On-Policy]
\label{lem:info-gain}\index{total variation distance!expected}
For any policy $\pi$ and any environment $\mu$,
\[
\EE_\mu^\pi[F^\pi_\infty(\ae_{<t})] \to 0
\text{ as $t \to \infty$.}
\]
\end{lemma}
\begin{proof}
From \autoref{thm:Blackwell-Dubins} we get
$D_\infty(\mu^\pi, \xi^\pi \mid \ae_{<t}) \to 0$
$\mu^\pi$-almost surely,
and since $D$ is bounded, this convergence also occurs in mean.
Thus for every environment $\nu \in \M$,
\[
\EE_\nu^\pi \big[ D_\infty(\nu^\pi, \xi^\pi \mid \ae_{<t}) \big] \to 0
\text{ as $t \to \infty$}.
\]
Now
\begingroup
\allowdisplaybreaks
\begin{align*}
      \EE_\mu^\pi[F^\pi_\infty(\ae_{<t})]
&\leq \frac{1}{w(\mu)} \EE_\xi^\pi[F^\pi_\infty(\ae_{<t})] \\
&=    \frac{1}{w(\mu)} \EE_\xi^\pi \left[
        \sum_{\nu \in \M} w(\nu \mid \ae_{<t})
          D_\infty(\nu^\pi, \xi^\pi \mid \ae_{<t})
      \right] \\
&=    \frac{1}{w(\mu)} \EE_\xi^\pi \left[
        \sum_{\nu \in \M} w(\nu) \frac{\nu^\pi(\ae_{<t})}{\xi^\pi(\ae_{<t})}
          D_\infty(\nu^\pi, \xi^\pi \mid \ae_{<t})
      \right] \\
&=    \frac{1}{w(\mu)} \sum_{\nu \in \M} w(\nu) \EE_\nu^\pi \big[
        D_\infty(\nu^\pi, \xi^\pi \mid \ae_{<t})
      \big]
 \to  0
\end{align*}
\endgroup
by \citet[Lem.~5.28ii]{Hutter:2005} since total variation distance is bounded.
\end{proof}

\begin{proof}[Proof of \autoref{thm:Thompson-sampling-aoim}]
Let $\beta, \delta > 0$ and
let $\eps_t > 0$ denote the sequence
used to define $\pi_T$ in \autoref{alg:Thompson-sampling}.
We assume that $t$ is large enough such that
$\eps_k \leq \beta$ for all $k \geq t$ and that
$\delta$ is small enough such that $w(\mu \mid \ae_{<t}) > 4\delta$ for all $t$,
which holds since $w(\mu \mid \ae_{<t}) \not\to 0$
$\mu^\pi$-almost surely for any policy $\pi$~\citep[Lem.~3i]{Hutter:2009MDL}.

The stochastic process $w(\nu \mid \ae_{<t})$
is a $\xi^{\pi_T}$-martingale according to \autoref{ex:posterior-martingale}.
By the martingale convergence theorem~(\autoref{thm:martingale-convergence})%
\index{convergence!martingale}
$w(\nu \mid \ae_{<t})$ converges $\xi^{\pi_T}$-almost surely and
because $\xi^{\pi_T} \geq w(\mu) \mu^{\pi_T}$
it also converges $\mu^{\pi_T}$-almost surely.

We argue that we can choose $t_0$
to be one of $\pi_T$'s resampling time steps large enough such that
for all $t \geq t_0$
the following three events hold simultaneously
with $\mu^{\pi_T}$-probability at least $1 - \delta$.
\begin{enumerate}[(i)]
\item \label{itm:M'}
	There is a finite set $\M' \subset \M$ with
	$w(\M' \mid \ae_{<t}) > 1 - \delta$
	and $w(\nu \mid \ae_{<k}) \not\to 0$ as $k \to \infty$
	for all $\nu \in \M'$.
\item \label{itm:posterior-convergence}
	$\left| w(\M'' \mid \ae_{<t}) - w(\M'' \mid \ae_{<t_0}) \right|
	\leq \delta$ for all $\M'' \subseteq \M'$.
\item \label{itm:info-gain}
	$F^{\pi_T}_\infty(\ae_{<t}) < \delta \beta w_{\min}^2$.
\end{enumerate}
where
$w_{\min} := \inf \{ w(\nu \mid \ae_{<k}) \mid k \in \mathbb{N}, \nu \in \M' \}$,
which is positive by (\ref{itm:M'}).

(\ref{itm:M'}) and (\ref{itm:posterior-convergence}) are satisfied eventually
because the posterior $w(\,\cdot \mid \ae_{<t})$
converges $\mu^{\pi_T}$-almost surely.
Note that the set $\M'$ is random:
the limit of $w(\nu \mid \ae_{<t})$ as $t \to \infty$
depends on the history $\ae_{1:\infty}$.
Without loss of generality,
we assume the true environment $\mu$ is contained in $\M'$
since $w(\mu \mid \ae_{<t}) \not\to 0$ $\mu^{\pi_T}$-almost surely.
(\ref{itm:info-gain}) follows from \autoref{lem:info-gain}
since convergence in mean implies convergence in probability.

Moreover, we define the horizon $m := t + H_t(\eps_t)$
as the time step of
the effective horizon\index{effective horizon} at time step $t$.
Let $\ae_{<t}$ be a fixed history
for which (\ref{itm:M'}-\ref{itm:info-gain}) is satisfied.
Then we have
\begingroup
\allowdisplaybreaks
\begin{align*}
      \delta \beta w_{\min}^2
&>    F^{\pi_T}_\infty(\ae_{<t}) \\
&=    \sum_{\nu \in \M} w(\nu \mid \ae_{<t})
        D_\infty(\nu^{\pi_T}, \xi^{\pi_T} \mid \ae_{<t}) \\
&=    \mathbb{E}_{\nu \sim w(\,\cdot \mid \ae_{<t})} \left[
        D_\infty(\nu^{\pi_T}, \xi^{\pi_T} \mid \ae_{<t})
      \right] \\
&\geq \mathbb{E}_{\nu \sim w(\,\cdot \mid \ae_{<t})} \left[
        D_m(\nu^{\pi_T}, \xi^{\pi_T} \mid \ae_{<t})
      \right] \\
&\geq \beta w_{\min}^2 w( \M \setminus \M'' \mid \ae_{<t})
\end{align*}
\endgroup
by Markov's inequality where
\[
   \M''
:= \left\{ \nu \in \M \;\middle|\;
           D_m(\nu^{\pi_T}, \xi^{\pi_T} \mid \ae_{<t})
           < \beta w_{\min}^2
   \right\}.
\]
For our fixed history $\ae_{<t}$ we have
\begingroup
\allowdisplaybreaks
\begin{align*}
   1 - \delta
&< w( \M'' \mid \ae_{<t}) \\
&\stackrel{(\ref{itm:M'})}{\leq}
   w( \M'' \cap \M' \mid \ae_{<t}) + \delta \\
&\stackrel{(\ref{itm:posterior-convergence})}{\leq}
   w( \M'' \cap \M' \mid \ae_{<t_0}) + 2\delta \\
&\stackrel{(\ref{itm:M'})}{\leq}
   w( \M'' \mid \ae_{<t_0}) + 3\delta
\end{align*}
\endgroup
and thus we get
\begin{equation}\label{eq:D-bound}
  1 - 4\delta
< w \left( \{ \nu \in \M \mid D_m(\nu^{\pi_T}, \xi^{\pi_T} \mid \ae_{<t})
           < \beta w_{\min}^2 \}
    \;\middle|\; \ae_{<t_0}
    \right).
\end{equation}
In particular, this bound holds for $\nu = \mu$
since $w(\mu \mid \ae_{<t_0}) > 4\delta$ by assumption.

It remains to show that with high probability the value $V^{\pi^*_\rho}_\mu$
of the sample $\rho$'s optimal policy $\pi^*_\rho$ is sufficiently close to
the $\mu$-optimal value $V^*_\mu$.
The worst case is that we draw the worst sample from $\M' \cap \M''$
twice in a row.
From now on, let $\rho$ denote the sample environment we draw at time step $t_0$,
and let $t$ denote some time step between
$t_0$ and $t_1 := t_0 + H_{t_0}(\eps_{t_0})$
(before the next resampling).
With probability $w(\nu' \mid \ae_{<t_0}) w(\nu' \mid \ae_{<t_1})$
we sample $\nu'$ both at $t_0$ and $t_1$
when following $\pi_T$.
Therefore we have for all $\ae_{t:m}$ and all $\nu \in \M$
\[
     \nu^{\pi_T}(\ae_{1:m} \mid \ae_{<t})
\geq w(\nu' \mid \ae_{<t_0}) w(\nu' \mid \ae_{<t_1})
     \nu^{\pi^*_{\nu'}}(\ae_{1:m} \mid \ae_{<t}).
\]
Thus we get for all $\nu \in \M'$ (in particular $\rho$ and $\mu$)
\begingroup
\allowdisplaybreaks
\begin{align*}
      D_m(\mu^{\pi_T}, \rho^{\pi_T} \mid \ae_{<t})
&\geq \sup_{\nu' \in \M} \sup_{A \subseteq (\A \times \E)^m}
      \Big| w(\nu' \mid \ae_{<t_0}) w(\nu' \mid \ae_{<t_1}) \\
&\qquad\qquad\qquad\qquad\qquad\qquad\qquad
      (\mu^{\pi^*_{\nu'}}(A \mid \ae_{<t}) - \rho^{\pi^*_{\nu'}}(A \mid \ae_{<t})) \Big| \\
&\geq w(\nu \mid \ae_{<t_0}) w(\nu \mid \ae_{<t_1})
      \sup_{A \subseteq (\A \times \E)^m} \Big|
        \mu^{\pi^*_\nu}(A \mid \ae_{<t}) - \rho^{\pi^*_\nu}(A \mid \ae_{<t})
      \Big| \\
&\geq w_{\min}^2 D_m(\mu^{\pi^*_\nu}, \rho^{\pi^*_\nu} \mid \ae_{<t}).
\end{align*}
\endgroup
For $\rho \in \M''$ we get with \eqref{eq:D-bound}
\begin{align*}
      D_m(\mu^{\pi_T}, \rho^{\pi_T} \mid \ae_{<t})
&\leq D_m(\mu^{\pi_T}, \xi^{\pi_T} \mid \ae_{<t})
      + D_m(\rho^{\pi_T}, \xi^{\pi_T} \mid \ae_{<t}) \\
&<    \beta w_{\min}^2 + \beta w_{\min}^2
 =    2\beta w_{\min}^2,
\end{align*}
which together with \autoref{lem:value-bound}
and the fact that rewards in $[0, 1]$ implies
\begingroup
\allowdisplaybreaks
\begin{align*}
     \left| V^{\pi^*_\nu}_\mu(\ae_{<t}) - V^{\pi^*_\nu}_\rho(\ae_{<t}) \right|
&\leq \frac{\Gamma_{t+H_t(\eps_t)}}{\Gamma_t} +
      \left| V^{\pi^*_\nu,m}_\mu(\ae_{<t}) - V^{\pi^*_\nu,m}_\rho(\ae_{<t}) \right| \\
&\leq \eps_t + D_m(\mu^{\pi^*_\nu}, \rho^{\pi^*_\nu} \mid \ae_{<t}) \\
&\leq \eps_t + \tfrac{1}{w_{\min}^2} D_m(\mu^{\pi_T}, \rho^{\pi_T} \mid \ae_{<t}) \\
&< \beta + 2\beta
=     3\beta.
\end{align*}
\endgroup
Hence we get (omitting history arguments $\ae_{<t}$ for simplicity)
\begin{equation}\label{eq:V-beta-bound}
      V^*_\mu
 =    V^{\pi^*_\mu}_\mu
 <    V^{\pi^*_\mu}_\rho + 3\beta
 \leq V^*_\rho + 3\beta
 =    V^{\pi^*_\rho}_\rho + 3\beta
 <    V^{\pi^*_\rho}_\mu + 3\beta + 3\beta
 =    V^{\pi^*_\rho}_\mu + 6\beta.
\end{equation}

With $\mu^{\pi_T}$-probability at least $1 - \delta$
(\ref{itm:M'}), (\ref{itm:posterior-convergence}), and (\ref{itm:info-gain})
are true,
with $\mu^{\pi_T}$-probability at least $1 - \delta$
our sample $\rho$ happens to be in $\M'$ by (\ref{itm:M'}), and
with $w(\,\cdot \mid \ae_{<t_0})$-probability at least $1 - 4\delta$
the sample is in $\M''$ by \eqref{eq:D-bound}.
All of these events are true simultaneously with probability at least
$1 - (\delta + \delta + 4\delta) = 1 - 6\delta$.
Hence the bound \eqref{eq:V-beta-bound} transfers for $\pi_T$ such that
with $\mu^{\pi_T}$-probability $\geq 1 - 6\delta$ we have
\[
V^*_\mu(\ae_{<t}) - V^{\pi_T}_\mu(\ae_{<t}) < 6\beta.
\]
Therefore
$\mu^{\pi_T}[ V^*_\mu(\ae_{<t}) - V^{\pi_T}_\mu(\ae_{<t}) \geq 6\beta ] < 6\delta$
and with $\delta \to 0$ we get that
$V_\mu^*(\ae_{<t}) - V_\mu^{\pi_T}(\ae_{<t}) \to 0$ as $t \to \infty$ in probability.
The value function is bounded, thus it also converges in mean.
\end{proof}

The following example shows that the Thompson sampling policy
is not strongly asymptotically optimal.
However, we expect that strong asymptotic optimality can be achieved
with Thompson sampling by resampling at every time step
(with strong assumptions on the discount function).
However, for practical purposes resampling in every time step
is very inefficient.

\begin{example}[Thompson Sampling is not Strongly Asymptotically Optimal]
\label{ex:TS-not-sao}
\index{Thompson sampling}\index{optimality!asymptotic!strong}
Define $\A := \{ \alpha, \beta \}$, $\E := \{ 0, 1/2, 1 \}$, and
assume geometric discounting~(\autoref{ex:geometric-discounting})%
\index{discounting!geometric}.
Consider the following class of environments
$\M := \{ \nu_\infty, \nu_1, \nu_2, \ldots \}$
(transitions are labeled with action, reward):
\begin{center}
\begin{tabular}{cc}
\begin{tikzpicture}[scale=1.2]
\node[circle, draw, minimum height=2em] (s0) at (0, 0) {$s_0$};
\node[circle, draw, minimum height=2em] (s1) at (-1.5, -1) {$s_1$};
\node[circle, draw, minimum height=2em] (s2) at (0, -2) {$s_2$};

\draw[->] (s0) to[loop above] node[left] {$\beta, \frac{1}{2}$} (s0);
\draw[->] (s0) to[bend right] node[above] {$\alpha, 0$} (s1);
\draw[->] (s1) to[bend right] node[below] {$\beta, 0$} (s0);
\draw[->] (s1) to node[below left] {$\alpha, 0$} (s2);
\draw[->] (s2) to node[right] {$\ast, 0$} (s0);
\end{tikzpicture} &
\begin{tikzpicture}[scale=1.2]
\node[circle, draw, minimum height=2em] (s0) at (0, 0) {$s_0$};
\node[circle, draw, minimum height=2em] (s1) at (-1.5, -1) {$s_1$};
\node[circle, draw, minimum height=2em] (s2) at (0, -2) {$s_2$};
\node[circle, draw, minimum height=2em] (s3) at (2, 0) {$s_3$};
\node[circle, draw, minimum height=2em] (s4) at (2, -2) {$s_4$};

\draw[->] (s0) to[loop above] node[left] {$\beta, \frac{1}{2}$} (s0);
\draw[->] (s0) to[bend right] node[above left] {$t < k: \alpha, 0$} (s1);
\draw[->] (s1) to[bend right] node[below] {$\beta, 0$} (s0);
\draw[->] (s1) to node[below left] {$\alpha, 0$} (s2);
\draw[->] (s2) to node[right] {$\ast, 0$} (s0);
\draw[->] (s0) to[bend left] node[above] {$t \geq k: \alpha, 0$} (s3);
\draw[->] (s3) to node[left] {$\alpha, 0$} (s4);
\draw[->] (s3) to[bend left] node[below] {$\beta, 0$} (s0);
\draw[->] (s4) to[loop right] node[above] {$\alpha, 1$} (s4);
\draw[->] (s4) to node[below] {$\beta, 0$} (s2);
\end{tikzpicture} \\
$\nu_\infty$ & $\nu_k$
\end{tabular}
\end{center}
Environment $\nu_k$ works just like environment $\nu_\infty$
except that after time step $k$, the path to state $s_3$ gets unlocked.
The class $\M$ is a class of deterministic weakly communicating%
\index{weakly communicating} POMDPs\index{POMDP}
(but as a POMDP $\nu_k$ has more than 5 states).
The optimal policy in environment $\nu_\infty$ is to always take action $\beta$,
the optimal policy for environment $\nu_k$ is
to take action $\beta$ for $t < k$ and then
take action $\beta$ in state $s_1$ and action $\alpha$ otherwise.

Suppose the policy $\pi_T$ is acting in environment $\nu_\infty$.
Since it is asymptotically optimal in the class $\M$,
it has to take actions $\alpha\alpha$ from $s_0$ infinitely often:
for $t < k$ environment $\nu_k$ is indistinguishable from $\nu_\infty$,
so the posterior for $\nu_k$ is larger or equal to the prior.
Hence there is always a constant chance of sampling $\nu_k$
until taking actions $\alpha\alpha$,
at which point all environments $\nu_k$ for $k \leq t$ become falsified.

If the policy $\pi_T$ decides to explore and take the first action $\alpha$,
it will be in state $s_1$.
Let $\ae_{<t}$ denote the current history.
Then the $\nu_\infty$-optimal action is $\beta$ and
\[
  V^*_{\nu_\infty}(\ae_{<t})
= (1 - \gamma) \left(
    0 + \gamma\frac{1}{2} + \gamma^2 \frac{1}{2} + \ldots
  \right)
= \frac{\gamma}{2}.
\]
The next action taken by $\pi_T$ is $\alpha$ since
any optimal policy for any sampled environment
that takes action $\alpha$ once, takes that action again
(and we are following that policy for an $\eps_t$-effective horizon%
\index{effective horizon}).
Hence
\[
     V^{\pi_T}_{\nu_\infty}(\ae_{<t})
\leq (1 - \gamma) \left(
       0 + 0 + \gamma^2\frac{1}{2} + \gamma^3 \frac{1}{2} + \ldots
     \right)
=    \frac{\gamma^2}{2}.
\]
Therefore $V^*_{\nu_\infty} - V^{\pi_T}_{\nu_\infty} \geq (\gamma - \gamma^2)/2 > 0$.
This happens infinitely often with probability one and
thus we cannot get almost sure convergence.
\end{example}

If the Bayesian mixture $\xi$ is inside the class $\M$
(as it is the case for the class $\Mlscccs$),
then we can assign $\xi$ a prior probability
that is arbitrarily close to $1$.
Since the posterior of $\xi$ is the same as the prior,
Thompson sampling will act according to the Bayes optimal policy
most of the time.
This means the Bayes-value of Thompson sampling can be very good;
formally,
$V^*_\xi(\epsilon) - V^{\pi_T}_\xi(\epsilon)
= \overline\Upsilon_\xi - \Upsilon_\xi(\pi_T)$ can be made arbitrarily small.

\index{multi-armed bandit}%
In contrast, the Bayes-value of Thompson sampling can also be very bad:
Suppose you have a class of $(n+1)$-armed bandits indexed $1, \ldots, n$
where bandit $i$ gives reward $1 - \eps$ on arm $1$, reward $1$ on arm $i + 1$,
and reward $0$ on all other arms.
For geometric discounting\index{discounting!geometric}
and $\eps < (1 - \gamma)/(2 - \gamma)$,
it is Bayes optimal to pull arm $1$ while Thompson sampling will explore on average
$n/2$ arms until it finds the optimal arm.
The Bayes-value of Thompson sampling is $1/(n-\gamma_{n-1})$
in contract to $(1 - \eps)$ achieved by Bayes.
For a horizon of $n$,
the Bayes optimal policy suffers a regret of $\eps n$ and
Thompson sampling a regret of $n/2$,
which is much larger for small $\eps$.

\subsection{Almost Sure in Cesàro Average vs.\ in Mean}
\label{ssec:wao-and-aoim}
\index{optimality!asymptotic!weak}\index{optimality!asymptotic!in mean}

It might appear that convergence in mean is more natural than
the convergence of Cesàro averages of weak asymptotic optimality.
However, both notions are not so fundamentally different
because they both allow an infinite number of bad mistakes
(actions that lead to $V^*_\mu - V^\pi_\mu$ being large).
Asymptotic optimality in mean allows bad mistakes
as long as their probability converges to zero;
weak asymptotic optimality allows bad mistakes
as long as the total time spent on bad mistakes grows sublinearly.
Note that according to \autoref{ex:exploration}
making bad mistakes infinitely often is necessary for asymptotic optimality.

\autoref{thm:BayesExp-wao} shows that weak asymptotic optimality is possible
in any countable class of stochastic environments.
However, this requires the additional condition that
the effective horizon grows sublinearly, $H_t(\varepsilon_t) \in o(t)$,
while \autoref{thm:Thompson-sampling-aoim}
does not require any condition on the discount function.

Generally, weak asymptotic optimality and asymptotic optimality in mean
are incomparable
because the notions of convergence are incomparable for (bounded)
random variables.
First,
for deterministic sequences
(i.e.\ deterministic policies in deterministic environments),
convergence in mean is equivalent to (regular) convergence,
which is impossible by \autoref{thm:sao-impossible}.
Second,
convergence in probability (and hence convergence in mean for bounded random variables)
does not imply almost sure convergence of Cesàro averages%
~\citep[Sec.~14.18]{Stoyanov:2013counterexamples}.
We leave open the question whether
the policy $\pi_T$ is weakly asymptotically optimal.

\section{Regret}
\label{sec:regret}

Regret is how many expected rewards the agent forfeits by
not following the best informed policy.

\begin{definition}[Regret]
\label{def:regret}\index{regret|textbf}
The \emph{regret} of a policy $\pi$ in environment $\mu$ is
\[
   R_m(\pi, \mu)
:= \sup_{\pi'} \EE_\mu^{\pi'} \left[ \sum_{t=1}^m r_t \right]
   - \EE_\mu^\pi \left[ \sum_{t=1}^m r_t \right].
\]
\end{definition}

Note that regret is undiscounted and always nonnegative.
Moreover, the space of possible different policies
for the first $m$ actions is finite and
we assumed the set of actions $\A$ and the set of percepts $\E$ to be finite%
~(\assref{ass:aixi}{ass:finite-actions-and-percepts}),
so the supremum is always attained by some policy
(not necessarily the $\mu$-optimal policy $\pi^*_\mu$
because that policy uses discounting).

Different problem classes have different regret rates,
depending on the structure and the difficulty of the problem class.
Multi-armed bandits provide a (problem-independent)
worst-case regret bound of $\Omega(\sqrt{km})$
where $k$ is the number of arms~\citep{BCB:2012bandits}.
In MDPs the lower bound is $\Omega(\sqrt{SAdm})$
where $S$ is the number of states, $A$ the number of actions, and
$d$ the diameter\index{diameter} of the MDP~\citep{AJO:2009MDPs}.
For a countable class of environments given by
state representation functions that map histories to MDP states,
a regret of $\tilde O(m^{2/3})$ is achievable
assuming the resulting MDP is weakly communicating\index{weakly communicating}%
~\citep{NMRO:2013}.

A problem class is considered \emph{learnable}\index{learnable}
if there is an algorithm that has a sublinear regret guarantee.
The following example shows that
the general reinforcement learning problem is not learnable
because the agent can get caught in a trap\index{trap} and be unable to recover.

\begin{example}[{Linear Regret; \citealp[Sec.~5.3.2]{Hutter:2005}}]
\label{ex:heaven-and-hell}\index{regret}\index{heaven}\index{hell}
Consider the following two environments $\mu_1$ and $\mu_2$.
In environment $\mu_1$ action $\alpha$ leads to hell (reward $0$ forever)
and action $\beta$ leads to heaven (reward $1$ forever).
Environment $\mu_2$ behaves just the same, except that both actions are swapped.
\begin{center}
\begin{tabular}{cc}
\begin{tikzpicture}[scale=2,node distance=20mm]
\node[agent] (start) {};
\node[env, above left of=start] (hell) {hell};
\node[env, above right of=start] (heaven) {heaven};
\path[transition] (hell) to[loop above] node[above] {reward $ = 0$} (hell);
\path[transition] (heaven) to[loop above] node[above] {reward $ = 1$} (heaven);
\path[transition] (start) to node[below right] {$\beta$} (heaven);
\path[transition] (start) to node[below left] {$\alpha$} (hell);
\end{tikzpicture}
&
\begin{tikzpicture}[scale=2,node distance=20mm]
\node[agent] (start) {};
\node[env, above left of=start] (hell) {hell};
\node[env, above right of=start] (heaven) {heaven};
\path[transition] (hell) to[loop above] node[above] {reward $ = 0$} (hell);
\path[transition] (heaven) to[loop above] node[above] {reward $ = 1$} (heaven);
\path[transition] (start) to node[below left] {$\beta$} (hell);
\path[transition] (start) to node[below right] {$\alpha$} (heaven);
\end{tikzpicture}
\\
$\mu_1$ & $\mu_2$
\end{tabular}
\end{center}
The policy $\alpha$ that takes action $\alpha$ in the first time step
performs well in $\mu_2$ but performs poorly in $\mu_1$.
Likewise, the policy $\beta$ that takes action $\beta$ in the first time step
performs well in $\mu_1$ but performs poorly in $\mu_2$.
Regardless of which policy we adopt,
our regret is always linear in one of the environments $\mu_1$ or $\mu_2$:
\begin{align*}
R_m(\alpha, \mu_1) &= m &
R_m(\alpha, \mu_2) &= 0 \\
R_m(\beta, \mu_1) &= 0 &
R_m(\beta, \mu_2) &= m
\qedhere
\end{align*}
\end{example}

To achieve sublinear regret we need to ensure that
the agent can recover from mistakes.
Formally, we make the following assumption.

\begin{definition}[Recoverability]
\label{def:recoverability}\index{recoverability|textbf}
An environment $\mu$ satisfies the \emph{recoverability assumption} iff
\[
\sup_\pi \left|
  \EE^{\pi^*_\mu}_\mu[ V^*_\mu(\ae_{<t}) ]
  - \EE^\pi_\mu[ V^*_\mu(\ae_{<t}) ]
\right|
\to 0 \text{ as $t \to \infty$}.
\]
\end{definition}

Recoverability compares following the worst policy $\pi$ for $t - 1$ time steps
and then switching to the optimal policy $\pi^*_\nu$ to
having followed $\pi^*_\nu$ from the beginning.
The recoverability assumption states that
switching to the optimal policy at any time step
enables the recovery of most of the value:
it has to become less costly to recover from mistakes as time progresses.
This should be regarded as an effect of the discount function:
if the (effective) horizon\index{effective horizon} grows,
recovery becomes easier because
the optimal policy has more time to perform a recovery.
Moreover, recoverability is on the optimal policy,
in contrast to the notion of ergodicity in MDPs
which demands returning to a starting state regardless of the policy.

\begin{remark}[Weakly Communicating POMDPs are Recoverable]
\label{rem:recoverable-POMDPs}\index{recoverability}\index{POMDP}%
\index{weakly communicating}\index{effective horizon}
If the effective horizon is growing,
$H_t(\eps) \to \infty$ as $t \to \infty$,
then any weakly communicating\index{weakly communicating}
finite state POMDP\index{POMDP}
satisfies the recoverability assumption.
\end{remark}

\subsection{Sublinear Regret in Recoverable Environments}
\label{ssec:sublinear-regret-in-recoverable-environments}

This subsection is dedicated to the following theorem
that connects asymptotic optimality in mean to sublinear regret.

\begin{theorem}[Sublinear Regret in Recoverable Environments]
\label{thm:aoim-implies-sublinear-regret}\index{recoverability}\index{regret}
If the discount function $\gamma$ satisfies \autoref{ass:gamma},
the environment $\mu$ satisfies the recoverability assumption, and
$\pi$ is asymptotically optimal in mean in the class $\{ \mu \}$,
then $R_m(\pi, \mu) \in o(m)$.
\end{theorem}

\begin{assumption}[Discount Function]
\label{ass:gamma}\index{discount!function}
Let the discount function $\gamma$ be such that
\begin{enumerate}[(a)]
\item \label{itm:gamma(a)}
	$\gamma_t > 0$ for all $t$,
\item \label{itm:gamma(b)}
	$\gamma_t$ is monotone decreasing in $t$, and
\item \label{itm:gamma(c)}
	$H_t(\eps) \in o(t)$ for all $\eps > 0$.
\end{enumerate}
\end{assumption}

This assumption demands that the discount function is somewhat well-behaved:
the function has no oscillations, does not become $0$, and
the horizon is not growing too fast.
It is satisfied by geometric discounting~(\autoref{ex:geometric-discounting})%
\index{discounting!geometric}:
\begin{inparaenum}
\item[(\ref{itm:gamma(a)})] $\gamma^t > 0$,
\item[(\ref{itm:gamma(b)})] $\gamma$ monotone decreasing, and
\item[(\ref{itm:gamma(c)})] $H_t(\eps) = \lceil \log_\gamma \eps \rceil \in o(t)$.
\end{inparaenum}

The problem with geometric discounting\index{discounting!geometric} is that
it makes the recoverability assumption very strong:
since the horizon is not growing, the environment has to enable
\emph{faster recovery} as time progresses;
in this case weakly communicating\index{weakly communicating} POMDPs\index{POMDP}
are \emph{not} recoverable.
A choice with $H_t(\eps) \to \infty$ that satisfies \autoref{ass:gamma} is
subgeometric discounting\index{discounting!subgeometric}
$\gamma_t := e^{-\sqrt{t}} / \sqrt{t}$~(see \autoref{tab:discounting}).

If the items in \autoref{ass:gamma} are violated,
\autoref{thm:aoim-implies-sublinear-regret} can fail:
\begin{itemize}
\item If $\gamma_t = 0$ for some time steps $t$,
our policy does not care about those time steps and might take actions
that have large regret.
\item Similarly if $\gamma$ oscillates between high values and very low values:
our policy might take high-regret actions in
time steps with comparatively lower $\gamma$-weight.
\item If the horizon grows linearly,
infinitely often
our policy might spend some constant fraction of the current effective horizon exploring,
which incurs a cost that is a constant fraction of the total regret so far.
\end{itemize}

To prove \autoref{thm:aoim-implies-sublinear-regret}
we require the following technical lemma.

\begin{lemma}[Value and Regret]
\label{lem:value-and-regret}\index{regret}
Let $\eps > 0$ and
assume the discount function $\gamma$ satisfies \autoref{ass:gamma}.
Let $(d_t)_{t \in \mathbb{N}}$ be a sequence of numbers with
$|d_t| \leq 1$ for all $t$.
If there is a time step $t_0$ with
\begin{equation}\label{eq:ao}
\frac{1}{\Gamma_t} \sum_{k=t}^\infty \gamma_k d_k < \eps
\quad \forall t \geq t_0
\end{equation}
then
\[
     \sum_{t=1}^m d_t
\leq t_0 + \varepsilon(m - t_0 + 1)
     + \frac{1 + \varepsilon}{1 - \varepsilon} H_m(\varepsilon).
\]
\end{lemma}
\begin{proof}
This proof essentially follows
the proof of \citet[Thm.~17]{Hutter:2006discounting}.

By \assref{ass:gamma}{itm:gamma(a)}
we have $\gamma_t > 0$ for all $t$
and hence $\Gamma_t > 0$ for all $t$.
By \assref{ass:gamma}{itm:gamma(b)}
have that $\gamma$ is monotone decreasing,
so we get for all $n \in \mathbb{N}$
\[
     \Gamma_t
=    \sum_{k=t}^\infty \gamma_k
\leq \sum_{k=t}^{\mathclap{t+n-1}} \gamma_t
     + \sum_{\mathclap{k=t+n}}^\infty \gamma_k
=    n\gamma_t + \Gamma_{t+n}.
\]
And with $n := H_t(\eps)$ this yields
\begin{equation}\label{eq:t-independent-bound}
     \frac{\gamma_t H_t(\eps)}{\Gamma_t}
\geq 1 - \frac{\Gamma_{t+H_t(\eps)}}{\Gamma_t}
\geq 1 - \eps
>    0.
\end{equation}
In particular, this bound holds for all $t$ and $\eps > 0$.

Next, we define a series of nonnegative weights $(b_t)_{t \geq 1}$ such that
\[
  \sum_{t=t_0}^m d_k
= \sum_{t=t_0}^m \frac{b_t}{\Gamma_t} \sum_{k=t}^m \gamma_k d_k.
\]
This yields the constraints
\[
\sum_{k=t_0}^t \frac{b_k}{\Gamma_k} \gamma_t = 1 \quad \forall t \geq t_0.
\]
The solution to these constraints is
\begin{equation}\label{eq:b_t-solution}
b_{t_0} = \frac{\Gamma_{t_0}}{\gamma_{t_0}},
\text{ and }
b_t = \frac{\Gamma_t}{\gamma_t} - \frac{\Gamma_t}{\gamma_{t-1}}
\text{ for $t > t_0$}.
\end{equation}
Thus we get
\begingroup
\allowdisplaybreaks
\begin{align*}
      \sum_{t=t_0}^m b_t
&=    \frac{\Gamma_{t_0}}{\gamma_{t_0}}
      + \sum_{t=t_0+1}^m \left(
          \frac{\Gamma_t}{\gamma_t} - \frac{\Gamma_t}{\gamma_{t-1}}
        \right) \\
&=    \frac{\Gamma_{m+1}}{\gamma_m}
      + \sum_{t=t_0}^m \left(
          \frac{\Gamma_t}{\gamma_t} - \frac{\Gamma_{t+1}}{\gamma_t}
        \right) \\
&=    \frac{\Gamma_{m+1}}{\gamma_m} + m - t_0 + 1 \\
&\leq \frac{H_m(\varepsilon)}{1 - \varepsilon} + m - t_0 + 1
\end{align*}
\endgroup
for all $\varepsilon > 0$ according to \eqref{eq:t-independent-bound}.

Finally,
\begin{align*}
      \sum_{t=1}^m d_t
&\leq \sum_{t=1}^{t_0} d_t
      + \sum_{t=t_0}^m \frac{b_t}{\Gamma_t} \sum_{k=t}^m \gamma_k d_k \\
&\leq t_0
      + \sum_{t=t_0}^m \frac{b_t}{\Gamma_t} \sum_{k=t}^\infty \gamma_k d_k
      - \sum_{t=t_0}^m \frac{b_t}{\Gamma_t} \sum_{k=m+1}^\infty \gamma_k d_k \\
\intertext{
and using the assumption \eqref{eq:ao} and $d_t \geq -1$,
}
&<    t_0
      + \sum_{t=t_0}^m b_t \varepsilon
      + \sum_{t=t_0}^m \frac{b_t \Gamma_{m+1}}{\Gamma_t} \\
&\leq t_0 + \frac{\varepsilon H_m(\varepsilon)}{1 - \varepsilon}
      + \varepsilon (m - t_0 + 1)
      + \sum_{t=t_0}^m \frac{b_t \Gamma_{m+1}}{\Gamma_t}
\end{align*}
For the latter term we substitute \eqref{eq:b_t-solution} to get
\[
     \sum_{t=t_0}^m \frac{b_t \Gamma_{m+1}}{\Gamma_t}
=    \frac{\Gamma_{m+1}}{\gamma_{t_0}} + \sum_{t=t_0 + 1}^m
     \left( \frac{\Gamma_{m+1}}{\gamma_t} - \frac{\Gamma_{m+1}}{\gamma_{t-1}} \right) \\
=    \frac{\Gamma_{m+1}}{\gamma_m}
\leq \frac{H_m(\eps)}{1 - \eps}
\]
with \eqref{eq:t-independent-bound}.
\end{proof}

\begin{proof}[Proof of \autoref{thm:aoim-implies-sublinear-regret}]
Let $(\pi_m)_{m \in \mathbb{N}}$ denote any sequence of policies,
such as a sequence of policies that
attain the supremum in the definition of regret.
We want to show that
\[
    \EE_\mu^{\pi_m} \left[ \sum_{t=1}^m r_t \right]
    - \EE_\mu^\pi \left[ \sum_{t=1}^m r_t \right]
\in o(m).
\]
For
\begin{equation}\label{eq:def-d_k}
d_k^{(m)} := \EE_\mu^{\pi_m} [ r_k ] - \EE_\mu^\pi [ r_k ]
\end{equation}
we have $-1 \leq d_k^{(m)} \leq 1$
since we assumed rewards to be bounded between $0$ and $1$.
Because the environment $\mu$ satisfies the recoverability assumption
we have
\begin{align*}
&\left|
    \EE^{\pi^*_\mu}_\mu[ V^*_\mu(\ae_{<t}) ] - \EE^\pi_\mu[ V^*_\mu(\ae_{<t}) ]
 \right|
\to 0
\text{ as $t \to \infty$}, \text{ and} \\
&\sup_m \left| \EE^{\pi^*_\mu}_\mu[ V^*_\mu(\ae_{<t}) ]
               - \EE^{\pi_m}_\mu[ V^*_\mu(\ae_{<t}) ]
        \right|
\to 0
\text{ as $t \to \infty$},
\end{align*}
so we conclude that
\[
    \sup_m \left|
      \EE^\pi_\mu[ V^*_\mu(\ae_{<t}) ] - \EE^{\pi_m}_\mu[ V^*_\mu(\ae_{<t}) ]
    \right|
\to 0
\]
by the triangle inequality and thus
\begin{equation}\label{eq:recover}
    \sup_m \EE^{\pi_m}_\mu[ V^*_\mu(\ae_{<t}) ]
    - \EE^\pi_\mu[ V^*_\mu(\ae_{<t}) ]
\to 0
\text{ as $t \to \infty$}.
\end{equation}
By assumption the policy $\pi$ is asymptotically optimal in mean,
so we have
\[
    \EE^\pi_\mu[ V^*_\mu(\ae_{<t}) ] - \EE^\pi_\mu[ V^\pi_\mu(\ae_{<t}) ]
\to 0
\text{ as $t \to \infty$},
\]
and with \eqref{eq:recover} this combines to
\[
    \sup_m \EE^{\pi_m}_\mu[ V^*_\mu(\ae_{<t}) ]
    - \EE^\pi_\mu[ V^\pi_\mu(\ae_{<t}) ]
\to 0 \text{ as $t \to \infty$}.
\]
From $V^*_\mu(\ae_{<t}) \geq V^{\pi_m}_\mu(\ae_{<t})$ we get
\begin{equation}\label{eq:convergence-of-values}
\limsup_{t \to \infty} \left(
    \sup_m \EE^{\pi_m}_\mu[ V^{\pi_m}_\mu(\ae_{<t}) ]
    - \EE^\pi_\mu[ V^\pi_\mu(\ae_{<t}) ]
    \right)
\leq 0.
\end{equation}
For $\pi' \in \{ \pi, \pi_1, \pi_2, \ldots \}$ we have
\[
  \EE^{\pi'}_\mu[ V^{\pi'}_\mu(\ae_{<t}) ]
= \EE^{\pi'}_\mu \left[ \frac{1}{\Gamma_t}
    \EE^{\pi'}_\mu \left[ \sum_{k=t}^\infty \gamma_k r_k \,\middle|\, \ae_{<t} \right]
  \right]
= \EE^{\pi'}_\mu \left[ \frac{1}{\Gamma_t} \sum_{k=t}^\infty \gamma_k r_k \right]
= \frac{1}{\Gamma_t} \sum_{k=t}^\infty \gamma_k \EE^{\pi'}_\mu [ r_k ],
\]
so from \eqref{eq:def-d_k} and \eqref{eq:convergence-of-values} we get
\[
\limsup_{t \to \infty} \sup_m
\frac{1}{\Gamma_t} \sum_{k=t}^\infty \gamma_k d_k^{(m)} \leq 0.
\]
Let $\eps > 0$.
We choose $t_0$ independent of $m$ and large enough such that
we get $\sup_m \sum_{k=t}^\infty \gamma_k d_k^{(m)} / \Gamma_t < \eps$
for all $t \geq t_0$.
Now we let $m \in \mathbb{N}$ be given and
apply \autoref{lem:value-and-regret} to get
\[
     \frac{R_m(\pi, \mu)}{m}
=    \frac{\sum_{k=1}^m d_k^{(m)}}{m}
\leq \frac{t_0 + \varepsilon(m - t_0 + 1) + \frac{1 + \varepsilon}{1 - \varepsilon} H_m(\varepsilon)}{m}.
\]
Since $H_t(\eps) \in o(t)$ according to \assref{ass:gamma}{itm:gamma(c)}
we get $\limsup_{m \to \infty} R_m(\pi, \mu) / m \leq 0$.
\end{proof}

\begin{example}[The Converse of \autoref{thm:aoim-implies-sublinear-regret} is False]
\label{ex:sublinear-regret-does-not-imply-ao}
\index{regret}\index{optimality!asymptotic}
Let $\mu$ be a two-armed Bernoulli bandit with means $0$ and $1$
and suppose we are using geometric discounting\index{discounting!geometric}
with discount factor $\gamma \in [0, 1)$.
This environment is recoverable.
If our policy $\pi$ pulls the suboptimal arm
exactly on time steps $1, 2, 4, 8, 16, \ldots$,
regret will be logarithmic.
However, on time steps $t = 2^n$ for $n \in \mathbb{N}$
the value difference $V^*_\mu - V^\pi_\mu$
is deterministically at least $1 - \gamma > 0$.
\end{example}

Note that \autoref{ex:sublinear-regret-does-not-imply-ao}
does not rule out weak asymptotic optimality.

\subsection{Regret of the Optimal Policy and Thompson sampling}
\label{ssec:regret-implications}

We get the following immediate consequence.

\begin{corollary}[Sublinear Regret for the Optimal Discounted Policy]
\label{cor:discounted-and-undiscounted-regret}
\index{regret}\index{policy!optimal}
If the discount function $\gamma$ satisfies \autoref{ass:gamma} and
the environment $\mu$ satisfies the recoverability assumption,
then $R_m(\pi^*_\mu, \mu) \in o(m)$.
\end{corollary}
\begin{proof}
From \autoref{thm:aoim-implies-sublinear-regret}
since the policy $\pi^*_\mu$ is (trivially)
asymptotically optimal in $\{ \mu \}$.
\end{proof}

If the environment does not satisfy the recoverability assumption,
regret may be linear \emph{even on the optimal policy}:
the optimal policy maximizes discounted rewards and
this short-sightedness might incur a tradeoff
that leads to linear regret later on
if the environment does not allow recovery.

\begin{corollary}[Sublinear Regret for Thompson Sampling]
\label{cor:thompson-sampling-regret}
\index{regret}\index{Thompson sampling}
If the discount function $\gamma$ satisfies \autoref{ass:gamma} and
the environment $\mu \in \M$ satisfies the recoverability assumption,
then $R_m(\pi_T, \mu) \in o(m)$
for the Thompson sampling policy $\pi_T$.
\end{corollary}
\begin{proof}
From \autoref{thm:Thompson-sampling-aoim} and
\autoref{thm:aoim-implies-sublinear-regret}.
\end{proof}

\section{Discussion}
\label{sec:discussion-optimality}

In this work, we disregard computational constraints.
Because of this, our agents learn very efficiently and
we can focus on the way they balance exploration and exploitation%
\index{exploration vs.\ exploitation}.
So which balance is best?

\subsection{The Optimality of AIXI}
\label{ssec:optimality-of-AIXI}\index{AIXI}

Bayesian reinforcement learning agents make the tradeoff between
exploration and exploitation\index{exploration vs.\ exploitation}
in the Bayes optimal way.
Maximizing expected rewards according to any positive prior
does not lead to enough exploration to achieve asymptotic optimality%
~(\autoref{thm:AIXI-not-AO});
the prior's bias is retained indefinitely.
For bad priors this can cause serious malfunctions:
the dogmatic prior defined in \autoref{ssec:dogmatic-prior}
can prevent a Bayesian agent \emph{from taking a single exploratory action};
exploration is restricted to cases where the expected future payoff
falls below some prespecified $\varepsilon > 0$.
However, this problem can be alleviated
by adding an extra exploration component to AIXI:
\citet{Lattimore:2013} shows that {\BayesExp}
is weakly asymptotically optimal~(\autoref{thm:BayesExp-wao}).

So instead, we may ask the following weaker questions.
Does AIXI succeed in every (ergodic) finite-state (PO)MDP,
bandit problem, or sequence prediction task?
Our results imply that without further assumptions on the prior,
we cannot answer any of the preceding questions in the affirmative.
Using a dogmatic prior~(\autoref{thm:dogmatic-prior}),
we can make AIXI follow any computable policy as long as that policy produces
rewards that are bounded away from zero.
\begin{itemize}
\item In a sequence prediction task
	that gives a reward of $1$ for every correctly predicted bit
	and $0$ otherwise,
	a policy $\pi$ that correctly predicts every third bit
	will receive an average reward of $1/3$.
	With a $\pi$-dogmatic prior,
	AIXI thus only predicts a third of the bits correctly,
	and hence is outperformed by a uniformly random predictor.

	However, if we have a constant horizon of length $1$,
	AIXI \emph{does} succeed in sequence prediction~\citep[Sec.~6.2.2]{Hutter:2005}.
	If the horizon is this short,
	the agent is so hedonistic that no threat of hell can deter it.
\item In a (PO)MDP a dogmatic prior can make AIXI get stuck in any loop
	that provides nonzero expected rewards.
\item In a bandit problem, a dogmatic prior can make AIXI get stuck on any arm
	which provides nonzero expected rewards.
\end{itemize}
These results apply not only to AIXI,
but generally to Bayesian reinforcement learning agents.
Any Bayesian mixture over nonrecoverable environments
is susceptible to dogmatic priors
if we allow an arbitrary reweighing of the prior.
Notable exceptions are
classes of environment that allow policies
that are strongly asymptotically optimal \emph{regardless of the history}%
~(\autoref{thm:self-optimizing}).
For example, the class of all ergodic MDPs for
an unbounded effective horizon;
in this case the Bayes optimal policy is
strongly asymptotically optimal~\citep[Thm.~5.38]{Hutter:2005}.
Note that in contrast to our results,
this requires that that the agent uses a Bayes-mixture over a class of ergodic MDPs.

Moreover, Bayesian agents still perform well at learning
and achieve on-policy value convergence%
~(\autoref{cor:Bayes-on-policy-value-convergence}):
the posterior belief about the value of a policy $\pi$
converges to the true value of $\pi$ while following $\pi$:
$V^\pi_\xi(\ae_{<t}) - V^\pi_\mu(\ae_{<t}) \to 0$
as $t \to \infty$ $\mu^\pi$-almost surely.
Since this holds for any policy,
in particular it holds for the Bayes optimal policy $\pi^*_\xi$.
This means that
the Bayes agent learns to predict those parts of the environment that it sees.
But if it does not explore enough,
then it will not learn other parts of the environment
that are potentially more rewarding.

\citet[Claim~5.12]{Hutter:2005} claims:
\begin{quote}
We expect AIXI to be universally optimal.\index{AIXI}
\end{quote}
Our work seriously challenges Hutter's claim:
no nontrivial and non-subjective optimality results for AIXI remain
(see \autoref{tab:optimality}).
Until new arguments for AIXI's optimality are put forward,
we have to regard AIXI as a \emph{relative} theory of intelligence,
dependent on the choice of the prior.

\subsection{Natural Universal Turing Machines}
\label{ssec:discussion-UTM}

The choice of the UTM has been a big open question
in algorithmic information theory for a long time.
The Kolmogorov complexity of a string
depends on this choice.
However, there are \emph{invariance theorems}\index{invariance theorem}%
~\citep[Thm.~2.1.1 \& Thm.~3.1.1]{LV:2008}
which state that changing the UTM
changes Kolmogorov complexity only by a constant.
When using the Solomonoff prior $M$
to predict any deterministic computable binary sequence,
the number of wrong predictions is bounded by
the Kolmogorov complexity of the sequence%
~(\autoref{cor:expected-prediction-regret-M}).
Due to the invariance theorem,
changing the UTM
changes the number of errors only by a constant.
In this sense, compression and prediction work for any choice of UTM.

For AIXI, there can be no invariance theorem;
in \autoref{sec:bad-priors}
we showed that a bad choice for the UTM can have drastic consequences.
Our negative results can guide future search for a \emph{natural} UTM%
\index{Turing machine!natural}:
the UTMs used to define the indifference prior\index{prior!indifference}%
~(\autoref{thm:indifference-prior}),
the dogmatic prior\index{prior!dogmatic}~(\autoref{thm:dogmatic-prior}), and
the Gödel prior\index{prior!Gödel}~(\autoref{thm:Gödel-prior})
should be considered unnatural.
But what are other desirable properties of a UTM?

A remarkable but unsuccessful attempt to find natural UTMs
is due to \citet{Mueller:2010}.
It takes the probability that one universal machine simulates another
according to the length of their respective compilers
and searches for a stationary distribution.
Unfortunately, no stationary distribution exists.

\begin{table}[t]
\begin{center}
\renewcommand{\arraystretch}{1.2}
\setlength{\tabcolsep}{5pt}
\begin{tabular}{llcc}
\toprule
name & defined in & $K_U(U')$ & $K_{U'}(U)$ \\
\midrule
indifference prior
  & \autoref{thm:indifference-prior}
  & $K(U) + K(m) + O(1)$
  & $m$ \\
dogmatic prior
  & \autoref{thm:dogmatic-prior}
  & $K(U) + K(\pi) + K(\varepsilon) + O(1)$
  & $\lceil - \log_2 \varepsilon \rceil$ \\
Gödel prior
  & \autoref{thm:Gödel-prior}
  & $K(U) + K(\PA) + O(1)$
  & $0$ \\
\bottomrule
\end{tabular}
\end{center}
\caption[Compiler sizes of the UTMs of bad priors]{%
Upper bounds to compiler sizes of
the UTMs used in the proofs of \autoref{sec:bad-priors}.
$K_U(U')$ is the number of extra bits to run the `bad' UTM $U'$
on the `good' UTM $U$,
$K_{U'}(U)$ is the number of extra bits to run $U$ on $U'$.
$K(U)$ denotes the length of the shortest program for $U$ on $U$.
}
\label{tab:UTM-sizes}
\index{Turing machine!universal}
\end{table}

Alternatively, we could demand that the UTM $U'$
that we use for the universal prior
has a small compiler on the reference machine $U$~\citep[p.~35]{Hutter:2005}.
Moreover, we could demand the reverse,
that the reference machine $U$ has a small compiler on $U'$.
The idea is that this should limit the amount of bias one can introduce by
defining a UTM that has very small programs
for very complicated and `unusual' environments.
Unfortunately, this just pushes the choice of the UTM to the reference machine.
\autoref{tab:UTM-sizes} lists
compiler sizes of the UTMs constructed in this thesis.

\subsection{Asymptotic Optimality}
\label{ssec:asymptotic-optimality-discussion}
\index{optimality!asymptotic}

A policy is asymptotically optimal if
the agent learns to act optimally in any environment from the class $\M$.
We discussed two asymptotically optimal policies.
{\BayesExp} is weakly asymptotically optimal
if the horizon grows sublinearly~(\autoref{thm:BayesExp-wao}) and
Thompson sampling is asymptotically optimal in mean%
~(\autoref{thm:Thompson-sampling-aoim}).
Both policies commit to exploration for several steps.
As stated in \autoref{ex:exploration}:
\begin{quote}
To achieve asymptotic optimality,
the agent needs to explore infinitely often for an entire effective horizon.
\index{effective horizon}
\end{quote}
This is why weak asymptotic optimality is impossible
if the horizon grows linearly~(\autoref{thm:wao-impossible}):
if the agent explores for an entire effective horizon,
it spoils a significant fraction of the average.
Thompson sampling explores whenever it draws a bad sample.
{\BayesExp} explores if the maximal expected information gain%
\index{information gain}
is above some threshold.
Both policies commit to exploration for the entire effective horizon.

The exploration performed by Thompson sampling is qualitatively different from
the exploration by {\BayesExp}~\citep[Ch.~5]{Lattimore:2013}.
{\BayesExp} performs phases of exploration in which
it maximizes the expected information gain.
This explores the environment class completely,
even achieving off-policy prediction\index{off-policy}%
~\citep[Thm.~7]{OLH:2013ksa}.
In contrast, Thompson sampling only explores on the optimal policies,
and in some environment classes this will not yield off-policy prediction.
So in this sense the exploration mechanism of Thompson sampling
is more reward-oriented than maximizing information gain.

However, asymptotic optimality has to be taken with a grain of salt.
It provides no incentive to the agent to avoid traps\index{trap}
in the environment.
Once the agent gets caught in a trap, all actions are equally bad and thus
optimal: asymptotic optimality has been achieved.
Even worse, an asymptotically optimal agent has to explore all the traps
because they might contain hidden treasure.
This brings us to the following impossibility result
for non-recoverable environment classes.
\begin{quote}
Either the agent gets caught in a trap\index{trap}
or it is not asymptotically optimal\index{optimality!asymptotic}.%
\footnote{This formulation was suggested by Toby Ord.}
\end{quote}

\subsection{The Quest for Optimality}
\label{ssec:the-quest-for-optimality}

\begin{table}[t]
\begin{center}
\begin{tabular}{lp{88mm}}
\toprule
Optimality & Issue/Comment \\
\midrule
$\mu$-optimal policy & requires to know the true environment $\mu$ in advance \\
Pareto optimality & always satisfied (\autoref{thm:Pareto-optimality-is-trivial}) \\
Bayes optimality & same as maximal intelligence \\
balanced Pareto optimality & same as maximal intelligence~(\autoref{prop:balanced-Pareto-optimality-and-intelligence}) \\
maximal intelligence & highly dependent on the prior
(\autoref{cor:some-AIXIs-are-stupid} and \autoref{cor:AIXI-is-stupid}) \\
PAC & strong variant of asymptotic optimality in probability \\
asymptotic optimality & Thompson sampling~(\autoref{thm:Thompson-sampling-aoim}) and \BayesExp~\citep{Lattimore:2013}, but not AIXI~\citep{Orseau:2013} \\
sublinear regret & impossible in general environments,
but possible with recoverability~(\autoref{thm:aoim-implies-sublinear-regret}) \\
\bottomrule
\end{tabular}
\end{center}
\caption[Notions of optimality]{%
Proposed notions of optimality
\citep{Hutter:2002,Hutter:2005,LH:2007int}
and their issues.
Asymptotic optimality stands out to be
the only nontrivial objective optimality notion
for general reinforcement learning.
}
\label{tab:optimality}
\end{table}

\autoref{thm:Pareto-optimality-is-trivial} shows that
Pareto optimality is trivial in the class of all computable environments.
Bayes optimality, Balanced Pareto optimality, and
maximal Legg-Hutter intelligence are equivalent%
~(\autoref{prop:balanced-Pareto-optimality-and-intelligence} and
\autoref{prop:Bayes-optimality-and-intelligence}).
\autoref{cor:some-AIXIs-are-stupid} and \autoref{cor:AIXI-is-stupid}
show that this notion is highly subjective
because it depends on the choice of the prior.
Moreover, according to \autoref{cor:any-computable-policy-can-be-smart},
any computable policy is nearly balanced Pareto optimal.
For finite horizons,
there are priors such that
every policy is balanced Pareto optimal%
~(\autoref{thm:indifference-prior}).
Sublinear regret is impossible in general environments%
~(\autoref{ex:heaven-and-hell}).
However, if the environment is recoverable~(\autoref{def:recoverability}),
then \autoref{thm:aoim-implies-sublinear-regret} shows that
asymptotic optimality in mean implies sublinear regret.
In summary,
asymptotic optimality is the only nontrivial and objective notion of optimality
for the general reinforcement learning problem~(\autoref{prob:general-RL}):
it is both satisfiable%
~(\autoref{thm:BayesExp-wao} and \autoref{thm:Thompson-sampling-aoim})
and objective
because it does not depend on a prior probability measure
over the environment class $\M$.
\autoref{tab:optimality} summarized the notions of optimality
discussed in this chapter.

Our optimality notions are \emph{tail events}\index{tail event}:
any finite number of time steps are irrelevant;
the agent can be arbitrarily lazy.
Asymptotic optimality requires only convergence in the limit.
In recoverable environments
we can always achieve sublinear regret after any finite interaction.
All policies with finite horizon are Bayes optimal
according to \autoref{thm:indifference-prior} and \autoref{cor:finite-horizon-every-policy-is-Bayes-optimal}.
Overall, there is a dichotomy between the asymptotic nature of
our optimality notions and
the use of discounting to prioritize the present over the future.
Ideally, we would aim for finite guarantees instead,
such as precise regret bounds or PAC convergence rates,
but without additional assumptions
this is impossible in this general setting.
This leaves us with the main question of this chapter unanswered%
~\citep[Sec.~5]{Hutter:2009open}:
\begin{quote}
What is a good optimality criterion for general reinforcement learning?
\end{quote}


\chapter{Computability}
\label{cha:computability}

\falsequote{Marcus Hutter}{I simply keep a few spare halting oracles around.}

Given infinite computation power,
many traditional AI problems become trivial:
playing chess, go, or backgammon can be solved
by exhaustive expansion of the game tree.
Yet other problems seem difficult still; for example,
predicting the stock market, driving a car, or babysitting your nephew.
How can we solve these problems in theory?
Solomonoff induction and AIXI are proposed answers to this question.

Both Solomonoff induction and AIXI are known to be incomputable.
But not all incomputabilities are equal.
The \emph{arithmetical hierarchy}\index{arithmetical hierarchy}
specifies different levels of computability
based on \emph{oracle machines}\index{oracle!halting}:
each level in the arithmetical hierarchy
is computed by a Turing machine which
may query a halting oracle for the respective lower level.
Our agents are useless if they cannot be approximated in practice,
i.e., by a regular Turing machine.
Therefore we posit that any ideal for a `perfect agent'
needs to be \emph{limit computable}\index{limit computable}~($\Delta^0_2$).
The class of limit computable functions is the class of functions that
admit an \emph{anytime algorithm}\index{anytime algorithm}.

\begin{table}[t]
\begin{center}
\renewcommand{\arraystretch}{1.2}
\setlength{\tabcolsep}{5pt}
\begin{tabular}{lcc}
\toprule
$Q$        & $\{ (x, q) \in \X^* \times \mathbb{Q} \mid Q(x) > q \}$
           & $\{ (x, y, q) \in \X^* \times \X^* \times \mathbb{Q} \mid Q(xy \mid x) > q \}$ \\
\midrule
$M$        & $\Sigma^0_1 \setminus \Delta^0_1$
           & $\Delta^0_2 \setminus (\Sigma^0_1 \cup \Pi^0_1)$ \\
$M\norm$   & $\Delta^0_2 \setminus (\Sigma^0_1 \cup \Pi^0_1)$
           & $\Delta^0_2  \setminus (\Sigma^0_1 \cup \Pi^0_1)$ \\
$\MM$      & $\Pi^0_2 \setminus \Delta^0_2$
           & $\Delta^0_3 \setminus (\Sigma^0_2 \cup \Pi^0_2)$ \\
$\MM\norm$ & $\Delta^0_3 \setminus (\Sigma^0_2 \cup \Pi^0_2)$
           & $\Delta^0_3 \setminus (\Sigma^0_2 \cup \Pi^0_2)$ \\
\bottomrule
\end{tabular}
\end{center}
\caption[Computability results on Solomonoff's prior]{%
The computability results on $M$, $M\norm$, $\MM$, and $\MM\norm$
proved in \autoref{sec:complexity-Solomonoff}.
Lower bounds on the complexity of $\MM$ and $\MM\norm$
are given only for specific universal Turing machines.
}
\label{tab:complexity-induction}
\end{table}

In \autoref{sec:complexity-Solomonoff}
we consider various different flavors of Solomonoff induction:
So\-lo\-mo\-noff's prior $M$~(\autoref{ex:Solomonoff-prior})
is only a semimeasure and not a measure:
it assigns positive probability that
the observed string has only finite length.
This can be circumvented by normalizing $M$.
Solomonoff's normalization $M\norm$~(\autoref{def:Solomonoff-normalization})
preserves the ratio $M(x1) / M(x0)$ and is limit computable.
If instead we mix only over programs that compute infinite strings,
we get a semimeasure $\MM$~\eqref{eq:def-MM},
which can be normalized to $\MM\norm$.
Moreover, when predicting a sequence,
we are primarily interested in the conditional probability $M(xy \mid x)$
(respectively $M\norm(xy \mid y)$, $\MM(xy \mid x)$, or $\MM\norm(xy \mid x)$)
that the currently observed string $x$ is continued with $y$.
We show that both $M$ and $M\norm$ are limit computable,
while $\MM$ and $\MM\norm$ are not.
\autoref{tab:complexity-induction} summarizes
our computability results for Solomonoff induction.

For MDPs, planning is already P-complete for finite and infinite horizons%
~\citep{PT:1987}.
In POMDPs, planning is undecidable~\citep{MHC:1999,MHC:2003}.
The existence of a policy whose expected value exceeds a given threshold is
PSPACE-complete~\citep{MGLA:2000},
even for purely epistemic POMDPs
in which actions do not change the hidden state~\citep{SLR:2007}.
In \autoref{sec:complexity-aixi} we derive hardness results
for planning in general semicomputable environments;
this environment class is even more general than POMDPs.
We show that optimal policies are $\Pi^0_2$-hard and
$\varepsilon$-optimal policies are undecidable.

Moreover, we show that by default, AIXI is not limit computable.
When picking the next action,
two or more actions might have the same value (expected future rewards).
The choice between them is easy,
but determining whether such a tie exists is difficult.
This problem can be circumvented
by settling for an $\varepsilon$-optimal policy;
we get a limit-computable agent with infinite horizon.
However,
these results rely on the recursive definition of the value function.
In contrast,
\citet{Hutter:2005} defines the value function as
the limit of the iterative value function.
In \autoref{sec:iterative-value-function} we compare these two definitions and
show that the recursive definition correctly maximizes expected rewards
and has better computability properties.

\begin{table}[t]
\begin{center}
\renewcommand{\arraystretch}{1.2}
\setlength{\tabcolsep}{5pt}
\begin{tabular}{llll}
\toprule
Agent               & Optimal & $\varepsilon$-Optimal \\
\midrule
AIMU                & $\Delta^0_2$
                    & $\Delta^0_1$ \\
AINU                & $\Delta^0_3$, $\Pi^0_2$-hard
                    & $\Delta^0_2$, $\Sigma^0_1$-hard \\
AIXI                & $\Delta^0_3$, $\Sigma^0_1$-hard
                    & $\Delta^0_2$, $\Sigma^0_1$-hard \\
Entropy-seeking     & $\Delta^0_3$ & $\Delta^0_2$ \\
Information-seeking & $\Delta^0_3$ & $\Delta^0_2$ \\
$\BayesExp$         & $\Delta^0_3$ & $\Delta^0_2$ \\
\bottomrule
\end{tabular}
\end{center}
\caption[Computability results for different agent models]{%
Computability results for different agent models
derived in \autoref{sec:complexity-aixi},
\autoref{sec:complexity-knowledge-seeking},
and \autoref{sec:computability-wao}.
AIMU denotes the optimal policy in a computable environment and
AINU denotes the optimal policy in a semicomputable environment
(see \autoref{sec:general-rl}).
Hardness results for AIXI are with respect to
a specific universal Turing machine;
hardness results for $\nu$-optimal policies are with respect to
a specific environment $\nu \in \Mlscccs$.
Results for
entropy-seeking and information-seeking policies are only
for finite horizons.
}
\label{tab:complexity-agents}
\end{table}

In \autoref{sec:complexity-knowledge-seeking}
we show that for finite horizons
both the entropy-seeking and the information-seeking
agent are $\Delta^0_3$-computable and have
limit-computable $\varepsilon$-optimal policies.
{\BayesExp}~(\autoref{ssec:BayesExp})
relies on optimal policies
that are generally not limit computable.
In \autoref{sec:computability-wao}
we give a weakly asymptotically optimal agent
based on {\BayesExp} that is limit computable.
\autoref{tab:complexity-agents} summarizes our results on
the computability of these agents.

In this chapter
we illustrate the environments
used in the proofs of our theorems
in the form of flowcharts.
They should be read as follows.
Circles denote \emph{stochastic nodes},
rectangles denote \emph{environment nodes}, and
diamonds denote the agent's \emph{choice nodes}.
Transitions out of stochastic nodes are labeled with transition probabilities,
transitions out of environment nodes are labeled with percepts, and
transitions out of choice nodes are labeled with actions.
The initial node is marked with a small incoming arrow
(see for example \autoref{fig:AINU-is-Pi2-hard}).
By \assref{ass:aixi}{ass:bounded-rewards}
the worst possible outcome is getting reward $0$ forever,
thus we label such states as \emph{hell}\index{hell}.
Analogously, getting reward $1$ forever is the best possible outcome,
thus we label such states as \emph{heaven}\index{heaven}.

\section{Background on Computability}
\label{sec:preliminaries}

\subsection{The Arithmetical Hierarchy}
\label{ssec:arithmetical hierarchy}
\index{arithmetical hierarchy}

A set $A \subseteq \mathbb{N}$ is $\Sigma^0_n$ iff
there is a quantifier-free formula $\eta$ such that
\begin{equation}\label{eq:def-Sigma^0_n}
k \in A
\;\Longleftrightarrow\;
\exists k_1 \forall k_2 \ldots Q_n k_n\; \eta(k, k_1, \ldots, k_n)
\end{equation}
where $Q_n = \forall$ if $n$ is even, $Q_n = \exists$ if $n$ is odd%
~\citep[Def.~1.4.10]{Nies:2009}.
(We can also think of $\eta$ as a computable relation.)
A set $A \subseteq \mathbb{N}$ is $\Pi^0_n$ iff
its complement $\mathbb{N} \setminus A$ is $\Sigma^0_n$.
The formula $\eta$ on the right side of \eqref{eq:def-Sigma^0_n}
is a \emph{$\Sigma^0_n$-formula}\index{S0n-formula@$\Sigma^0_n$-formula|textbf}
and its negation is a \emph{$\Pi^0_n$-formula}%
\index{P0n-formula@$\Pi^0_n$-formula|textbf}.
It can be shown that
we can add any bounded quantifiers and
duplicate quantifiers of the same type
without changing the classification of $A$.
The set $A$ is $\Delta^0_n$ iff $A$ is $\Sigma^0_n$ and $A$ is $\Pi^0_n$.
We get that
$\Sigma^0_1$ as the class of recursively enumerable sets,
$\Pi^0_1$ as the class of co-recursively enumerable sets and
$\Delta^0_1$ as the class of recursive sets.

The set $A \subseteq \mathbb{N}$ is \emph{$\Sigma^0_n$-hard
($\Pi^0_n$-hard, $\Delta^0_n$-hard)} iff
for any set $B \in \Sigma^0_n$ ($B \in \Pi^0_n$, $B \in \Delta^0_n$),
$B$ is many-one reducible to $A$, i.e.,
there is a computable function $f$ such that
$k \in B \leftrightarrow f(k) \in A$~\citep[Def.~1.2.1]{Nies:2009}.
We get $\Sigma^0_n \subset \Delta^0_{n+1} \subset \Sigma^0_{n+1} \subset \ldots$
and $\Pi^0_n \subset \Delta^0_{n+1} \subset \Pi^0_{n+1} \subset \ldots$.
This hierarchy of subsets of natural numbers is known as
the \emph{arithmetical hierarchy}\index{arithmetical hierarchy|textbf}.

By Post's Theorem~\citep[Thm.~1.4.13]{Nies:2009}\index{Post's theorem},
a set is $\Sigma^0_n$ if and only if
it is recursively enumerable on an oracle machine
with an oracle for a $\Sigma^0_{n-1}$-complete set.
An oracle for $\Sigma^0_1$ is called a
\emph{halting oracle}\index{oracle!halting|textbf}.

\subsection{Computability of Real-valued Functions}
\label{ssec:computability-of-real-valued-functions}

We fix some encoding of rational numbers into binary strings and
an encoding of binary strings into natural numbers.
From now on, this encoding will be done implicitly wherever necessary.

\begin{definition}[$\Sigma^0_n$-, $\Pi^0_n$-, $\Delta^0_n$-computable]
\label{def:computable}\index{computable}\index{limit computable}
\index{lower semicomputable}
A function $f: \X^* \to \mathbb{R}$ is called
\emph{$\Sigma^0_n$-computable ($\Pi^0_n$-computable, $\Delta^0_n$-computable)} iff
the set $\{ (x, q) \in \X^* \times \mathbb{Q} \mid f(x) > q \}$
is $\Sigma^0_n$ ($\Pi^0_n$, $\Delta^0_n$).
\end{definition}
A $\Delta^0_1$-computable function is called \emph{computable},
a $\Sigma^0_1$-computable function is called \emph{lower semicomputable}, and
a $\Pi^0_1$-computable function is called \emph{upper semicomputable}.
A $\Delta^0_2$-computable function $f$ is called \emph{limit computable},
because there is a computable function $\phi$ such that
\[
\lim_{k \to \infty} \phi(x, k) = f(x).
\]
The program $\phi$ that limit computes $f$
can be thought of as
an \emph{anytime algorithm}\index{anytime algorithm} for $f$:
we can stop $\phi$ at any time $k$ and get a preliminary answer.
If the program $\phi$ ran long enough (which we do not know),
this preliminary answer will be close to the correct one.

Limit-computable sets are the highest level in the arithmetical hierarchy
that can be approached by a regular Turing machine.
Above limit-computable sets
we necessarily need some form of halting oracle.
See \autoref{tab:computability} for the definition of
lower/upper semicomputable and limit-computable functions
in terms of the arithmetical hierarchy.

\begin{table}[t]
\begin{center}
\begin{tabular}{lcc}
\toprule
& $\{ (x, q) \mid f(x) \geq q \}$
& $\{ (x, q) \mid f(x) < q \}$ \\
\midrule
$f$ is computable            & $\Delta^0_1$ & $\Delta^0_1$ \\
$f$ is lower semicomputable  & $\Sigma^0_1$ & $\Pi^0_1$    \\
$f$ is upper semicomputable  & $\Pi^0_1$    & $\Sigma^0_1$ \\
$f$ is limit computable      & $\Delta^0_2$ & $\Delta^0_2$ \\
$f$ is $\Delta^0_n$-computable & $\Delta^0_n$ & $\Delta^0_n$ \\
$f$ is $\Sigma^0_n$-computable & $\Sigma^0_n$ & $\Pi^0_n$    \\
$f$ is $\Pi^0_n$-computable    & $\Pi^0_n$    & $\Sigma^0_n$ \\
\bottomrule
\end{tabular}
\end{center}
\caption[Computability of real-valued functions]{
Connection between the computability of real-valued functions and
the arithmetical hierarchy.
}
\label{tab:computability}
\end{table}

\begin{lemma}[Computability of Arithmetical Operations]
\label{lem:computable-reals}
\index{computable}
Let $n > 0$ and
let $f, g: \X^* \to \mathbb{R}$ be two $\Delta^0_n$-computable functions.
Then
\begin{enumerate}[(a)]
\item $\{ (x, y) \mid f(x) >    g(y) \}$ is $\Sigma^0_n$,
\item $\{ (x, y) \mid f(x) \leq g(y) \}$ is $\Pi^0_n$,
\item $f + g$, $f - g$, and $f \cdot g$ are $\Delta^0_n$-computable, and
\item $f / g$ is $\Delta^0_n$-computable if $g(x) \neq 0$ for all $x$.
\item $\log f$ is $\Delta^0_n$-computable if $f(x) > 0$ for all $x$.
\end{enumerate}
\end{lemma}
\begin{proof}
We only prove this for $n > 1$.
Since $f, g$ are $\Delta^0_n$-computable,
they are limit computable on a level $n-1$ oracle machine.
Let $\phi$ be the function limit computing $f$ on the oracle machine, and
let $\psi$ be the function limit computing $g$ on the oracle machine:
\[
f(x) = \lim_{k \to \infty} \phi(k, x)
\quad\text{and}\quad
g(y) = \lim_{k \to \infty} \psi(k, y).
\]
By assumption, both $\phi$ and $\psi$ are $\Delta^0_{n-1}$-computable.
\begin{enumerate}[(a)]
\item
Let $G := \{ (x, y, q) \mid g(y) < q \}$, and
let $F := \{ (x, y, q) \mid q < f(x) \}$,
both of which are in $\Delta^0_n$ by assumption.
Hence there are $\Sigma^0_n$-formulas $\varphi_G$ and $\varphi_F$ such that
\begin{align*}
(x, y, q) \in G \;&\Longleftrightarrow\; \varphi_G(x, y, q) \\
(x, y, q) \in F \;&\Longleftrightarrow\; \varphi_F(x, y, q)
\end{align*}
Now $f(x) > g(y)$ if and only if
$\exists q.\; (x, y, q) \in G \cap F$, which is equivalent to
the $\Sigma^0_n$-forumla
\[
\exists q.\; \varphi_G(x, y, q) \;\land\; \varphi_F(x, y, q).
\]

\item
Follows from (a).

\item Addition, subtraction, and multiplication are continuous operations.

\item Division is discontinuous only at $g(x) = 0$.
We show this explicitly.
By assumption,
for any $\varepsilon > 0$ there is a $k_0$ such that for all $k > k_0$
\[
| \phi(x, k) - f(x) | < \varepsilon
\quad\text{and}\quad
| \psi(x, k) - g(x) | < \varepsilon.
\]
We assume without loss of generality that $\varepsilon < |g(x)|$,
since $g(x) \neq 0$ by assumption.
\begin{align*}
&\phantom{=~} \left| \frac{\phi(x, k)}{\psi(x, k)} - \frac{f(x)}{g(x)} \right| \\
&=    \left| \frac{\phi(x, k) g(x) - f(x) \psi(x, k)}{\psi(x, k) g(x)} \right| \\
&\leq \frac{|\phi(x, k) g(x) - f(x) g(x)| + |f(x) g(x) - f(x) \psi(x, k)|}{|\psi(x, k) g(x)|} \\
&<    \frac{\varepsilon |g(x)| + |f(x)| \varepsilon}{|\psi(x, k) g(x)|} \\
\intertext{
with $|\psi(x, k) g(x)|
= |\psi(x, k)| \cdot |g(x)|
> (|g(x)| - \varepsilon) |g(x)|$,
}
&<    \varepsilon \cdot \frac{|g(x)| + |f(x)|}{(|g(x)| - \varepsilon) |g(x)|}
 \xrightarrow{\varepsilon \to 0} 0,
\end{align*}
therefore $f(x) / g(x) = \lim_{k \to \infty} \phi(x, k) / \psi(x, k)$.

\item Follows from the fact that the logarithm is computable.
\qedhere
\end{enumerate}
\end{proof}

\section{The Complexity of Solomonoff Induction}
\label{sec:complexity-Solomonoff}

In this section, we derive the computability results for Solomonoff's prior
as stated in \autoref{tab:complexity-induction}.

Since $M$ is lower semicomputable,
$M\norm$ is limit computable
by \autoref{lem:computable-reals} (c) and (d).
When using the Solomonoff prior $M$
(or one of its sisters $M\norm$, $\MM$, or $\MM\norm$
defined in \autoref{def:Solomonoff-normalization} and
\autoref{eq:def-MM})
for sequence prediction,
we need to compute the conditional probability
$M(xy \mid x) = M(xy) / M(x)$
for finite strings $x, y \in \X^*$.
Because $M(x) > 0$ for all finite strings $x \in \X^*$,
this quotient is well-defined.

\begin{figure}[t]
\begin{align*}
M(xy \mid x) > q
~\Longleftrightarrow~
\forall m \exists k.\; \tfrac{\phi(xy, k)}{\phi(x, m)} > q
~\Longleftrightarrow~
\exists k \exists m_0 \forall m \geq m_0.\; \tfrac{\phi(xy, k)}{\phi(x, m)} > q
\end{align*}
\caption[Definition of conditional $M$ as a $\Delta^0_2$-formula]{%
A $\Pi^0_2$-formula and an equivalent $\Sigma^0_2$-formula
defining conditional $M$.
Here $\phi(x, k)$ denotes a computable function that lower semicomputes $M(x)$.
}\label{fig:conditional-M}
\end{figure}

\begin{theorem}[Complexity of $M$, $M\norm$, $\MM$, and $\MM\norm$] \mbox{}
\label{thm:complexity-M}\index{computable}
\index{Solomonoff!prior}\index{measure!mixture}
\begin{enumerate}[(a)]
\item $M(x)$ is lower semicomputable
\item $M(xy \mid x)$ is limit computable
\item $M\norm(x)$ is limit computable
\item $M\norm(xy \mid x)$ is limit computable
\item $\MM(x)$ is $\Pi^0_2$-computable
\item $\MM(xy \mid x)$ is $\Delta^0_3$-computable
\item $\MM\norm(x)$ is $\Delta^0_3$-computable
\item $\MM\norm(xy \mid x)$ is $\Delta^0_3$-computable
\end{enumerate}
\end{theorem}
\begin{proof}
\begin{enumerate}[(a)]
\item By \citet[Thm.~4.5.2]{LV:2008}.
	Intuitively, we can run all programs in parallel and get
	monotonely increasing lower bounds for $M(x)$ by adding $2^{-|p|}$
	every time a program $p$ has completed outputting $x$.
\item From (a) and \autoref{lem:computable-reals}d
	since $M(x) > 0$
	(see also \autoref{fig:conditional-M}). 
\item By \autoref{lem:computable-reals}cd, and $M(x) > 0$.
\item By (iii), \autoref{lem:computable-reals}d since $M\norm(x) \geq M(x) > 0$.
\item Let $\phi$ be a computable function that lower semicomputes $M$.
Since $M$ is a semimeasure, $M(xy) \geq \sum_z M(xyz)$,
hence $\sum_{y \in \X^n} M(xy)$ is nonincreasing in $n$ and thus
$\MM(x) > q$ iff
$\forall n \exists k \sum_{y \in \X^n} \phi(xy, k) > q$.
\item From (v) and \autoref{lem:computable-reals}d
	since $\MM(x) > 0 $.
\item From (v) and \autoref{lem:computable-reals}d.
\item From (vi) and \autoref{lem:computable-reals}d
	since $\MM\norm(x) \geq \MM(x) > 0$.
\qedhere
\end{enumerate}
\end{proof}

We proceed to show that these bounds are in fact the best possible ones.
If $M$ were $\Delta^0_1$-computable,
then so would be the conditional semimeasure $M(\,\cdot \mid \cdot\,)$.
Thus the $M$-adversarial sequence $z_{1:\infty}$
defined in \autoref{ex:adversarial-sequence}
would be computable and hence
corresponds to a computable deterministic measure $\mu$.
However, we have $M(z_{1:t}) \leq 2^{-t}$ by construction,
so dominance $M(x) \geq w(\mu) \mu(x)$ with $w(\mu) > 0$
yields a contradiction with $t \to \infty$:
\[
2^{-t} \geq M(z_{1:t}) \geq w(\mu) \mu(z_{1:t}) = w(\mu) > 0
\]
By the same argument,
the normalized Solomonoff prior $M\norm$ cannot be $\Delta^0_1$-com\-putable.
However, since it is a measure,
$\Sigma^0_1$- or $\Pi^0_1$-computability would entail $\Delta^0_1$-computability.

For $\MM$ and $\MM\norm$ we prove the following two lower bounds
for specific universal Turing machines.

\begin{theorem}[$\MM$ is not Limit Computable]
\label{thm:MM-is-not-Delta2}\index{measure!mixture}
There is a universal Turing machine $U'$ such that
the set $\{ (x, q) \mid \MM_{U'}(x) > q \}$ is not in $\Delta^0_2$.
\end{theorem}
\begin{proof}
Assume the contrary and
let $A \in \Pi^0_2 \setminus \Delta^0_2$ and
$\eta$ be a quantifier-free first-order formula such that
\begin{equation}
\label{eq:eta}
n \in A
\quad\Longleftrightarrow\quad
\forall k \exists i.\; \eta(n, k, i).
\end{equation}
For each $n \in \mathbb{N}$,
we define the program $p_n$ as follows.
\begin{center}
\begin{algorithmic}[1]
\Procedure{$p_n$}{}
\State output $1^{n+1} 0$
\State $k \gets 0$
\While{true}
	\State $i \gets 0$
	\While{not $\eta(n, k, i)$}
		\State $i \gets i + 1$
	\EndWhile
	\State $k \gets k + 1$
	\State output $0$
\EndWhile
\EndProcedure
\end{algorithmic}
\end{center}
Each program $p_n$ always outputs $1^{n+1}0$.
Furthermore, the program $p_n$ outputs the infinite string $1^{n+1}0^\infty$
if and only if $n \in A$ by \eqref{eq:eta}.
We define $U'$ as follows using our reference machine $U$.
\begin{itemize}
\item $U'(1^{n+1}0)$: Run $p_n$.
\item $U'(00p)$: Run $U(p)$.
\item $U'(01p)$: Run $U(p)$ and bitwise invert its output.
\end{itemize}
By construction, $U'$ is a universal Turing machine.
No $p_n$ outputs a string starting with $0^{n+1} 1$,
therefore $\MM_{U'}(0^{n+1}1) = \tfrac{1}{4} \big( \MM_U(0^{n+1}1) + \MM_U(1^{n+1} 0) \big)$.
Hence
\begin{align*}
   \MM_{U'}(1^{n+1} 0)
&= 2^{-n-2} \one_{A}(n)
     + \tfrac{1}{4} \MM_U(1^{n+1}0) + \tfrac{1}{4} \MM_U(0^{n+1}1) \\
&= 2^{-n-2} \one_{A}(n) + \MM_{U'}(0^{n+1}1)
\end{align*}
If $n \notin A$, then $\MM_{U'}(1^{n+1}0) = \MM_{U'}(0^{n+1}1)$.
Otherwise, we have
$|\MM_{U'}(1^{n+1}0) - \MM_{U'}(0^{n+1}1)| = 2^{-n-2}$.

Now we assume that $\MM_{U'}$ is limit computable, i.e.,
there is a computable function $\phi: \X^* \times \mathbb{N} \to \mathbb{Q}$
such that
$\lim_{k \to \infty} \phi(x, k) = \MM_{U'}(x)$.
We get that
\[
n \in A
\;\Longleftrightarrow\;
\lim_{k \to \infty} \phi(0^{n+1}1, k) - \phi(1^{n+1}0, k) \geq 2^{-n-2},
\]
thus $A$ is limit computable, a contradiction.
\end{proof}

\begin{corollary}[$\MM\norm$ is not $\Sigma^0_2$- or $\Pi^0_2$-computable]
\label{cor:MMnorm-is-not-2}\index{measure!mixture}
There is a universal Turing machine $U'$ such that
$\{ (x, q) \mid {\MM\norm}_{U'}(x) > q \}$ is not in $\Sigma^0_2$ or $\Pi^0_2$.
\end{corollary}
\begin{proof}
Since $\MM\norm = c \cdot \MM$,
there exists a $k \in \mathbb{N}$ such that $2^{-k} < c$
(even if we do not know the value of $k$).
We can show that
the set $\{ (x, q) \mid {\MM\norm}_{U'}(x) > q \}$ is not in $\Delta^0_2$
analogously to the proof of \autoref{thm:MM-is-not-Delta2},
using
\[
n \in A
\;\Longleftrightarrow\;
\lim_{k \to \infty} \phi(0^{n+1}1, k) - \phi(1^{n+1}0, k) \geq 2^{-k-n-2}.
\]
If $\MM\norm$ were $\Sigma^0_2$-computable or $\Pi^0_2$-computable,
this would imply that $\MM\norm$ is $\Delta^0_2$-computable
since $\MM\norm$ is a measure, a contradiction.
\end{proof}

Since $M(\epsilon) = 1$,
we have $M(x \mid \epsilon) = M(x)$,
so the conditional probability $M(xy \mid x)$ has at least the same complexity
as $M$.
Analogously for $M\norm$ and $\MM\norm$ since they are measures.
For $\MM$, we have that $\MM(x \mid \epsilon) = \MM\norm(x)$,
so \autoref{cor:MMnorm-is-not-2} applies.
All that remains to prove is
that conditional $M$ is not lower semicomputable.

\begin{theorem}[Conditional $M$ is not Lower Semicomputable]
\label{thm:M-conditional-is-not-Sigma1}\index{Solomonoff!prior}
The set $\{ (x, xy, q) \mid M(xy \mid x) > q \}$ is not recursively enumerable.
\end{theorem}

We gave a different, more complicated proof in \citet{LH:2015computability2}.
The following, much simpler and more elegant proof
is due to \citet[Prop.~3]{Sterkenburg:2016}.

\begin{proof}
Assume to the contrary that $M(xy \mid x)$ is lower semicomputable.
Let $a \neq b \in \X$.
We construct an infinite string $x$ by defining its initial segments
$\epsilon =: x(0) \sqsubset x(1) \sqsubset x(2) \sqsubset \ldots \sqsubset x$.
At every step $n$, we enumerate strings $y \in \X^*$
until one is found satisfying $M(a \mid x(n) y) \geq 1/2$;
then set $x(n+1) := x(n) y b$.
This implies that for infinitely many $t$ there is an $n$ such that
$M(b \mid x_{<t}) = M(b \mid x(n) y) \leq 1 - M(a \mid x(n) y) \leq 1/2$.
Since we assumed $M(\,\cdot \mid \cdot\,)$ to be lower semicomputable,
the infinite string $x$ is computable,
and hence $M(x_t \mid x_{<t}) \to 1$ by \autoref{cor:strong-merging-M}.
But this contradicts $M(b \mid x_{<t}) \leq 1/2$ infinitely often.
\end{proof}

\section{The Complexity of AINU, AIMU, and AIXI}
\label{sec:complexity-aixi}

\subsection{Upper Bounds}
\label{ssec:upper-bounds}

In this section, we derive upper bounds
on the computability of AINU, AIMU, and AIXI.
Except for \autoref{cor:complexity-aimu},
all results in this section apply generally to
any $\nu \in \Mlscccs$.
Since the Bayesian mixture $\xi \in \Mlscccs$,
they apply to AIXI even though they are stated for AINU.

In order to position AINU in the arithmetical hierarchy,
we need to encode policies as sets of natural numbers.
For the rest of this chapter,
we assume that policies are deterministic,
thus can be represented as relations over $\H \times \A$.
These relations are easily identified with sets of natural numbers
by encoding the history into one natural number.
From now on this translation of policies
into sets of natural numbers
will be done implicitly wherever necessary.

\begin{lemma}[Policies are in $\Delta^0_n$]
\label{lem:policy-is-delta_n}\index{policy}
If a policy $\pi$ is $\Sigma^0_n$ or $\Pi^0_n$,
then $\pi$ is $\Delta^0_n$.
\end{lemma}
\begin{proof}
Let $\varphi$ be a $\Sigma^0_n$-formula ($\Pi^0_n$-formula) defining $\pi$,
i.e., $\varphi(h, a)$ holds iff $\pi(h) = a$.
We define the formula $\varphi'$,
\[
   \varphi'(h, a)
:= \!\!\!\bigwedge_{a' \in \A \setminus \{ a \}}\!\!\! \neg \varphi(h, a').
\]
The set of actions $\A$ is finite,
hence $\varphi'$ is a $\Pi^0_n$-formula ($\Sigma^0_n$-formula).
Moreover, $\varphi'$ is equivalent to $\varphi$.
\end{proof}

To compute the optimal policy,
we need to compute the optimal value function.
The following lemma gives an upper bound on
the computability of the value function for environments in $\Mlscccs$.

\begin{lemma}[Complexity of $V^*_\nu$]
\label{lem:complexity-V}\index{value function}
For every $\nu \in \Mlscccs$,
and every lower semicomputable discount function $\gamma$,
the function $V^*_\nu$ is $\Delta^0_2$-computable.
\end{lemma}
\begin{proof}
The explicit form of the value function \eqref{eq:V-explicit} has numerator
\[
\lim_{m \to \infty} \expectimax{\ae_{t:m}}\;
  \sum_{i=t}^m \gamma(i) r_i \nu(e_{1:i} \dmid a_{1:i}),
\]
and denominator $\nu(e_{<t} \dmid a_{<t}) \cdot \Gamma_t$.
The numerator is nondecreasing in $m$
because we assumed rewards to be nonnegative
(\hyperref[ass:bounded-rewards]{\autoref*{ass:aixi}\ref*{ass:bounded-rewards}}).
Hence both numerator and denominator are lower semicomputable functions,
so \autoref{lem:computable-reals}d implies that
$V^*_\nu$ is $\Delta^0_2$-computable.
\end{proof}

From the optimal value function $V^*_\nu$
we get the optimal policy $\pi^*_\nu$ according to \eqref{eq:argmax}.
However, in cases where there is more than one optimal action,
we have to break an argmax tie.
This happens iff $V^*_\nu(h\alpha) = V^*_\nu(h\beta)$
for two potential actions $\alpha \neq \beta \in \A$.
This equality test is more difficult than
determining which is larger in cases where they are unequal.
Thus we get the following upper bound.

\begin{theorem}[Complexity of Optimal Policies]
\label{thm:complexity-optimal-policies}\index{policy!optimal}
For any environment $\nu$,
if $V^*_\nu$ is $\Delta^0_n$-computable,
then there is an optimal policy $\pi^*_\nu$ for the environment $\nu$
that is $\Delta^0_{n+1}$.
\end{theorem}
\begin{proof}
To break potential ties,
we pick an (arbitrary) total order $\succ$ on $\A$
that specifies which actions should be preferred in case of a tie.
We define
\begin{equation}\label{eq:p*}
\begin{aligned}
\pi_\nu(h) = a
~:\Longleftrightarrow~~~~~~~
&\bigwedge_{a': a' \succ a}
  V^*_\nu(ha) > V^*_\nu(ha') \\
\;\land &\bigwedge_{a': a \succ a'}
  V^*_\nu(ha) \geq V^*_\nu(ha').
\end{aligned}
\end{equation}
Then $\pi_\nu$ is a $\nu$-optimal policy according to \eqref{eq:argmax}.
By assumption, $V^*_\nu$ is $\Delta^0_n$-computable.
By \autoref{lem:computable-reals}ab
$V^*_\nu(ha) > V^*_\nu(ha')$ is $\Sigma^0_n$ and
$V^*_\nu(ha) \geq V^*_\nu(ha')$ is $\Pi^0_n$.
Therefore the policy $\pi_\nu$ defined in \eqref{eq:p*} is a conjunction
of a $\Sigma^0_n$-formula and a $\Pi^0_n$-formula and thus $\Delta^0_{n+1}$.
\end{proof}

\begin{corollary}[Complexity of AINU]
\label{cor:complexity-ainu}\index{AINU}
AINU is $\Delta^0_3$ for every environment $\nu \in \Mlscccs$.
\end{corollary}
\begin{proof}
From \autoref{lem:complexity-V} and \autoref{thm:complexity-optimal-policies}.
\end{proof}

Usually we do not mind taking slightly suboptimal actions.
Therefore actually trying to determine if two actions have the exact same value
seems like a waste of resources.
In the following, we consider policies that attain a value
that is always within some $\varepsilon > 0$ of the optimal value.

\begin{theorem}[Complexity of $\varepsilon$-Optimal Policies]
\label{thm:complexity-eps-optimal-policies}
\index{policy!e-optimal@$\eps$-optimal}
For any environment $\nu$,
if $V^*_\nu$ is $\Delta^0_n$-computable,
then there is an $\varepsilon$-optimal policy $\pi^\varepsilon_\nu$
for the environment $\nu$
that is $\Delta^0_n$.
\end{theorem}
\begin{proof}
Let $\varepsilon > 0$ be given.
Since the value function $V^*_\nu(h)$ is $\Delta^0_n$-computable,
the set $V_\varepsilon := \{ (ha, q) \mid |q - V^*_\nu(ha)| < \varepsilon / 2 \}$
is in $\Delta^0_n$ according to \autoref{def:computable}.
Hence we compute the values $V^*_\nu(ha')$
until we get within $\varepsilon / 2$ for every $a' \in \A$
and then choose the action with the highest value so far.
Formally,
let $\succ$ be an arbitrary total order on $\A$
that specifies which actions should be preferred in case of a tie.
Without loss of generality, we assume $\varepsilon = 1/k$,
and define $Q$ to be an $\varepsilon / 2$-grid on $[0, 1]$,
i.e., $Q := \{ 0, 1/2k, 2/2k, \ldots, 1 \}$.
We define
\begin{equation}\label{eq:eps-optimal-policy}
\begin{aligned}
\pi^\varepsilon_\nu(h) = a
~:\Longleftrightarrow~
\exists (q_{a'})_{a' \in \A} \in Q^{\A}.\;
          &\bigwedge_{a' \in \A} (ha', q_{a'}) \in V_\varepsilon \\
\land &\bigwedge_{a': a' \succ a} q_a > q_{a'}
    \land \bigwedge_{a': a \succ a'} q_a \geq q_{a'} \\
\land &\text{ the tuple $(q_{a'})_{a' \in \A}$ is minimal with} \\
&\text{ respect to
the lex.\ ordering on $Q^\A$}.
\end{aligned}
\end{equation}
This makes the choice of $a$ unique.
Moreover, $Q^\A$ is finite since $\A$ is finite,
and hence \eqref{eq:eps-optimal-policy} is a $\Delta^0_n$-formula.
\end{proof}

\begin{corollary}[Complexity of $\varepsilon$-Optimal AINU]
\label{cor:complexity-eps-ainu}
\index{policy!e-optimal@$\eps$-optimal}\index{AINU}
For any environment $\nu \in \Mlscccs$,
there is an $\varepsilon$-optimal policy for AINU
that is $\Delta^0_2$.
\end{corollary}
\begin{proof}
From \autoref{lem:complexity-V} and
\autoref{thm:complexity-eps-optimal-policies}.
\end{proof}

\begin{corollary}[Complexity of $\varepsilon$-Optimal AIXI]
\label{cor:complexity-eps-aixi}
\index{policy!e-optimal@$\eps$-optimal}\index{AIXI}
For any lower semicomputable prior
there is an $\varepsilon$-optimal policy for AIXI
that is $\Delta^0_2$.
\end{corollary}
\begin{proof}
From \autoref{cor:complexity-eps-ainu}
since for any lower semicomputable prior,
the corresponding Bayesian mixture $\xi$ is in $\Mlscccs$.
\end{proof}

If the environment $\nu \in \Mcomp$ is a measure,
i.e., $\nu$ assigns zero probability to finite strings,
then we get computable $\varepsilon$-optimal policies.

\begin{corollary}[Complexity of AIMU]
\label{cor:complexity-aimu}\index{AIMU}
If the environment $\mu \in \Mcomp$ is a measure and
the discount function $\gamma$ is computable,
then $\AIMU$ is limit computable ($\Delta^0_2$), and
$\varepsilon$-optimal $\AIMU$ is computable ($\Delta^0_1$).
\end{corollary}
\begin{proof}
Let $\varepsilon > 0$ be the desired accuracy.
We can truncate the limit $m \to \infty$ in \eqref{eq:V-explicit}
at the $\varepsilon / 2$-effective horizon $H_t(\varepsilon/2)$,
since everything after $H_t(\varepsilon/2)$ can contribute at most $\varepsilon/2$
to the value function.
Any lower semicomputable measure is computable~\citep[Lem.~4.5.1]{LV:2008}.
Therefore $V^*_\mu$ as given in \eqref{eq:V-explicit}
is composed only of computable functions,
hence it is computable according to \autoref{lem:computable-reals}.
The claim now follows from \autoref{thm:complexity-optimal-policies} and
\autoref{thm:complexity-eps-optimal-policies}.
\end{proof}

\subsection{Lower Bounds}
\label{ssec:lower-bounds}

We proceed to show that the bounds from the previous section
are the best we can hope for.
In environment classes where ties have to be broken,
AINU has to solve $\Pi^0_2$-hard problems
(\autoref{thm:AINU-is-Pi2-hard}).
These lower bounds are stated for particular environments $\nu \in \Mlscccs$.
Throughout this section, we assume that $\Gamma_t > 0$ for all $t$.

We also construct universal mixtures
that yield bounds on $\varepsilon$-optimal policies.
There is an $\varepsilon$-optimal $\AIXI$
that solves $\Sigma^0_1$-hard problems
(\autoref{thm:eps-AIXI-is-Sigma1-hard}).
For arbitrary universal mixtures,
we prove the following weaker statement that only guarantees incomputability.

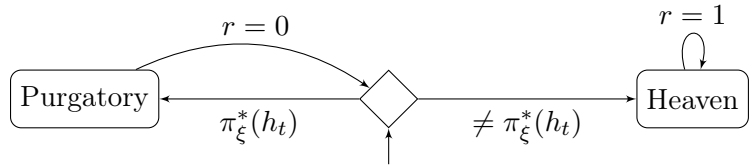
\begin{figure}[t]
\begin{center}
\begin{tikzpicture}[node distance=20mm, auto]
\node (start) {};
\node[agent, above of=start, node distance=10mm] (choice) {};
\node[env, left of=choice, node distance=40mm] (purgatory) {Purgatory};
\node[env, right of=choice, node distance=40mm] (heaven) {Heaven};
\path[transition] (start)     to (choice);
\path[transition] (choice)    to node[below] {$\neq \pi^*_\xi(h_t)$} (heaven);
\path[transition] (choice)    to node[below] {$\pi^*_\xi(h_t)$} (purgatory);
\path[transition] (purgatory) to[bend left] node[above] {$r = 0$} (choice);
\path[transition] (heaven)    to[loop above] node[above] {$r = 1$} (heaven);
\end{tikzpicture}
\end{center}
\caption[Environment from the proof of \autoref{thm:AIXI-is-not-computable}]{%
The environment $\mu$ from the proof of \autoref{thm:AIXI-is-not-computable}.
The agent gets reward $0$ as long as it follows AIXI's policy $\pi^*_\xi$
that is assumed to be computable.
Once the agent deviates from $\pi^*_\xi$, it gets reward $1$.
We get a contradiction because AIXI can learn this environment,
so it will eventually decide to take an action that leads to heaven.
}
\label{fig:AIXI-is-not-computable}
\end{figure}

\begin{theorem}[No $\AIXI$ is computable]
\label{thm:AIXI-is-not-computable}\index{AIXI}
$\AIXI$ is not computable for any universal Turing machine $U$.
\end{theorem}

This theorem follows from the incomputability of Solomonoff induction.
By the on-policy value convergence theorem%
~(\autoref{cor:Bayes-on-policy-value-convergence})
AIXI succeeds to predict the environment's behavior for its own policy.
If AIXI were computable,
then there would be computable environments more powerful than AIXI:
they can simulate AIXI and anticipate its prediction,
which leads to a contradiction.

\begin{proof}
Assume there is a computable policy $\pi^*_\xi$
that is optimal in the mixture $\xi$.
We define a deterministic environment $\mu$,
the \emph{adversarial environment}\index{adversarial!environment} to $\pi^*_\xi$.
The environment $\mu$
gives rewards $0$ as long as the agent follows the policy $\pi^*_\xi$,
and rewards $1$ once the agent deviates.
Formally, we ignore observations by setting $\O := \{ 0 \}$, and define
\[
\mu(r_{1:t} \dmid a_{1:t}) :=
\begin{cases}
1 &\text{if } \forall k \leq t.\, a_k = \pi^*_\xi((ar)_{<k})
   \text{ and } r_k = 0, \\
1 &\text{if } \forall k \leq t.\, r_k = \one_{k \geq i} \\
  &\text{ where } i := \min \{ j \mid a_j \neq \pi^*_\xi((ar)_{<j}) \},
   \text{ and} \\
0 &\text{otherwise}.
\end{cases}
\]
See \autoref{fig:AIXI-is-not-computable}
for an illustration of this environment.
The environment $\mu$ is computable because
the policy $\pi^*_\xi$ was assumed to be computable.
Suppose $\pi^*_\xi$ acts in $\mu$, then by \autoref{thm:on-policy-value-convergence}
AIXI learns to predict perfectly \emph{on policy}\index{on-policy}:
\[
V^{\pi^*_\xi}_\xi(\ae_{<t}) - V^{\pi^*_\xi}_\mu(\ae_{<t}) \to 0
\text{ as $t \to \infty$ $\mu^{\pi^*_\xi}$-almost surely},
\]
since both $\pi^*_\xi$ and $\mu$ are deterministic.
Because $V^{\pi^*_\xi}_\mu(h_{<t}) = 0$ by definition of $\mu$,
we get $V^*_\xi(\ae_{<t}) \to 0$.
Therefore we find a $t$ large enough such that
$V^*_\xi(\ae_{<t}) < w(\mu)$
where $\ae_{<t}$ is the interaction history of $\pi^*_\xi$ in $\mu$.
A policy $\pi$ with $\pi(\ae_{<t}) \neq \pi^*_\xi(\ae_{<t})$,
gets a reward of $1$ in environment $\mu$ for all time steps after $t$,
hence $V^\pi_\mu(\ae_{<t}) = 1$.
With linearity of $V^\pi_\xi(\ae_{<t})$ in $\xi$~(\autoref{lem:V-linear}),
\[
     V^\pi_\xi(\ae_{<t})
\geq w(\mu) \tfrac{\mu(e_{1:t} \dmid a_{1:t})}{\xi(e_{1:t} \dmid a_{1:t})}
       V^\pi_\mu(\ae_{<t})
\geq w(\mu),
\]
since $\mu(e_{1:t} \dmid a_{1:t}) = 1$ ($\mu$ is deterministic),
$V^\pi_\mu(\ae_{<t}) = 1$, and $\xi(e_{1:t} \dmid a_{1:t}) \leq 1$.
Now we get a contradiction:
\[
     w(\mu)
>    V^*_\xi(\ae_{<t}) \\
=    \sup_{\pi'} V^{\pi'}_\xi(\ae_{<t})
\geq V^\pi_\xi(\ae_{<t})
\geq w(\mu)
\qedhere
\]
\end{proof}

For the remainder of this section,
we fix the action space to be $\A := \{ \alpha, \beta \}$
with action $\alpha$ favored in ties.
The percept space is fixed to a tuple of binary observations and rewards,
$\E := \O \times \{ 0, 1 \}$ with $\O := \{ 0, 1 \}$.

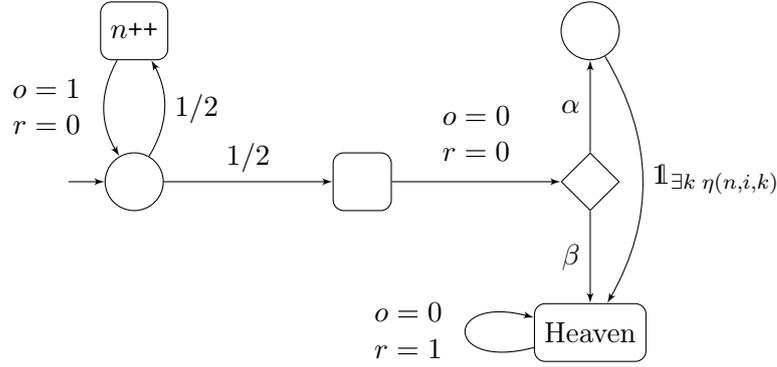
\begin{figure}[t]
\begin{center}
\begin{tikzpicture}[node distance=20mm, auto]
\node (start) {};
\node[stoch, right of=start, node distance=10mm] (init) {};
\node[env, above of=init] (one) {$n$\texttt{++}};
\node[env, right of=init,node distance=30mm] (zero) {};
\node[agent, right of=zero,node distance=30mm] (choice) {};
\node[env, below of=choice] (heaven) {Heaven};
\node[stoch, above of=choice] (gate) {};
\path[transition] (start)  to (init);
\path[transition] (init)   to[bend right] node[right] {$1/2$} (one);
\path[transition] (one)    to[bend right] node[left] {$\begin{array}{c} o = 1 \\ r = 0 \end{array}$} (init);
\path[transition] (init)   to node {$1/2$} (zero);
\path[transition] (zero)   to node {$\begin{array}{c} o = 0 \\ r = 0 \end{array}$} (choice);
\path[transition] (choice) to node[left] {$\beta$} (heaven);
\path[transition] (choice) to node {$\alpha$} (gate);
\path[transition] (gate)   to[bend left]
	node[right] {$\one_{\exists k\; \eta(n, i, k)}$} (heaven);
\path[transition] (heaven)   to[loop left]
	node {$\begin{array}{c} o = 0 \\ r = 1 \end{array}$} (heaven);
\end{tikzpicture}
\end{center}
\caption[Environment from the proof of \autoref{thm:AINU-is-Pi2-hard}]{%
The environment $\rho_i$ from the proof of \autoref{thm:AINU-is-Pi2-hard}.
The mixture $\nu$ over class of environments
$\M := \{ \rho_0, \rho_1, \ldots \} \subset \Mlscccs$
forces AINU to solve $\Pi^0_2$-hard problems:
Action $\alpha$ is preferred (because of a tie) iff
it leads to heaven, which is the case iff
$\exists k\; \eta(n, i, k)$.
}
\label{fig:AINU-is-Pi2-hard}
\end{figure}

\begin{theorem}[AINU is $\Pi^0_2$-hard]
\label{thm:AINU-is-Pi2-hard}\index{AINU}
There is an environment $\nu \in \Mlscccs$ such that
AINU is $\Pi^0_2$-hard.
\end{theorem}
\begin{proof}
Let $A$ be a any $\Pi^0_2$-set, and
let $\eta$ be a quantifier-free formula such that
\begin{equation}\label{eq:defA}
n \in A
\;\Longleftrightarrow\;
\forall i\; \exists k\; \eta(n, i, k).
\end{equation}
We define a class of environments $\M := \{ \rho_1, \rho_2, \ldots \}$
where each $\rho_i$ is defined as follows.
\[
\rho_i((or)_{1:m} \dmid a_{1:m}) :=
\begin{cases}
2^{-m}   &\text{if } o_{1:m} = 1^m
          \text{ and } \forall t \leq m.\; r_t = 0, \\
2^{-n-1} &\text{if } \exists n.\; 1^n0 \sqsubseteq o_{1:m} \sqsubseteq 1^n 0^\infty
          \text{ and } a_{n+2} = \alpha \\
         &\text{ and } r_t = \one_{t > n+1}
          \text{ and } \exists k\; \eta(n, i, k), \\
2^{-n-1} &\text{if } \exists n.\; 1^n0 \sqsubseteq o_{1:m} \sqsubseteq 1^n 0^\infty
          \text{ and } a_{n+2} = \beta \\
         &\text{ and } r_t = \one_{t > n+1}, \text{ and} \\
0        &\text{otherwise}.
\end{cases}
\]
See \autoref{fig:AINU-is-Pi2-hard} for an illustration of these environments.
Every $\rho_i$ is a chronological conditional semimeasure by definition and
every $\rho_i$ is lower semicomputable since $\eta$ is quantifier-free,
so $\M \subseteq \Mlscccs$.

We define our environment $\nu$ as a mixture over $\M$,
\[
\nu := \sum_{i \in \mathbb{N}} 2^{-i-1} \rho_i;
\]
the choice of the weights on the environments $\rho_i$ is arbitrary but positive.
Let $\pi^*_\nu$ be an optimal policy for the environment $\nu$ and
recall that the action $\alpha$ is preferred in ties.
We claim that for the $\nu$-optimal policy $\pi^*_\nu$,
\begin{equation}\label{eq:AINU-is-Pi2-hard-claim}
n \in A
~~\Longleftrightarrow~~
\pi^*_\nu(1^n0) = \alpha.
\end{equation}
This enables us to decide whether $n \in A$ given the policy $\pi^*_\nu$,
hence proving \eqref{eq:AINU-is-Pi2-hard-claim} concludes this proof.

Let $n, i \in \mathbb{N}$ be given, and
suppose we are in environment $\rho_i$ and observe $1^n 0$.
Taking action $\beta$ next yields reward $1$ forever;
taking action $\alpha$ next yields
a reward of $1$ if there is a $k$ such that $\eta(n, i, k)$ holds.
If this is the case, then
\[
V^*_{\rho_i}(1^n 0 \alpha) = \Gamma_{n+2} = V^*_{\rho_i}(1^n 0 \beta),
\]
and otherwise
\[
V^*_{\rho_i}(1^n 0 \alpha) = 0 < \Gamma_{n+2} = V^*_{\rho_i}(1^n 0 \beta)
\]
(omitting the first $n + 1$ actions and rewards
in the argument of the value function).
We can now show \eqref{eq:AINU-is-Pi2-hard-claim}:
By \eqref{eq:defA}, $n \in A$ if and only if for all $i$
there is a $k$ such that $\eta(n, i, k)$,
which happens if and only if
$V^*_{\rho_i}(1^n 0 \alpha) = \Gamma_{n+2}$ for all $i \in \mathbb{N}$,
which is equivalent to $V^*_\nu(1^n 0 \alpha) = \Gamma_{n+2}$,
which in turn is equivalent to $\pi^*_\mu(1^n 0) = \alpha$
since $V^*_\nu(1^n 0 \beta) = \Gamma_{n+2}$ and action $\alpha$ is favored in ties.
\end{proof}

\begin{figure}[t]
\begin{center}
\begin{tikzpicture}[node distance=20mm, auto]
\node (start) {};
\node[env, right of=start, node distance=10mm] (xi) {$\xi$};
\node[agent, right of=xi, node distance=30mm] (choice) {};
\node[env, above of=choice] (sheaven) {Semi-Heaven};
\node[stoch, below of=choice] (gate) {};
\node[env, left of=gate, node distance=30mm] (heaven) {Heaven};
\path[transition] (start)  to (xi);
\path[transition] (xi)     to[loop above] node {$\ast$} (xi);
\path[transition] (xi)   to node {$o = 1^n0$} (choice);
\path[transition] (choice) to node[left] {$\alpha$} (sheaven);
\path[transition] (choice) to node[left] {$\beta$} (gate);
\path[transition] (gate) to node {$\one_{\exists k\; \eta(n, i, k)}$} (heaven);
\path[transition] (sheaven)   to[loop left]
	node[left] {$\begin{array}{c} o = 0 \\ r = 1/2 \end{array}$} (sheaven);
\path[transition] (heaven)   to[loop left]
	node {$\begin{array}{c} o = 0 \\ r = 1 \end{array}$} (heaven);
\end{tikzpicture}
\end{center}
\caption[Environment from the proof of \autoref{thm:eps-AIXI-is-Sigma1-hard}]{%
The environment $\nu$ from the proof of \autoref{thm:eps-AIXI-is-Sigma1-hard},
which forces $\AIXI$ to solve $\Sigma^0_1$-hard problems.
It functions just like $\xi$ until the observation history is $1^n0$.
Then, action $\alpha$ is preferred iff
heaven is accessible, i.e., iff $\exists k\; \eta(n, i, k)$.
}
\label{fig:eps-AIXI-is-Sigma1-hard}
\end{figure}
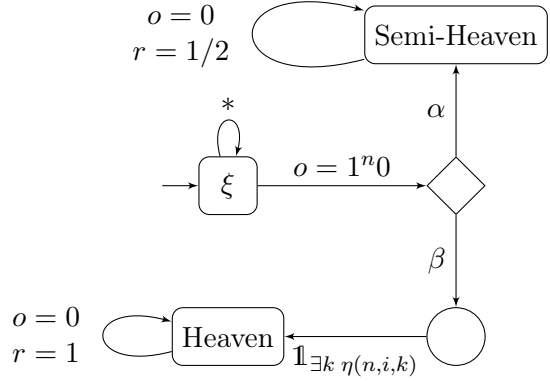

\begin{theorem}[Some $\varepsilon$-optimal $\AIXI$ are $\Sigma^0_1$-hard]
\label{thm:eps-AIXI-is-Sigma1-hard}
\index{policy!e-optimal@$\eps$-optimal}\index{AIXI}
There is a universal Turing machine $U'$ and an $\varepsilon > 0$ such that
any $\varepsilon$-optimal policy for $\AIXI$ is $\Sigma^0_1$-hard.
\end{theorem}
\begin{proof}
Let $A$ be a $\Sigma^0_1$-set and $\eta$ be a quantifier-free formula such that
$n+1 \in A$ iff $\exists k\; \eta(n, k)$.
We define the environment
\[
\nu((or)_{1:t} \dmid a_{1:t}) :=
\begin{cases}
\xi((or)_{1:n} \dmid a_{1:n}) &\text{if } \exists n.\; o_{1:n} = 1^{n-1}0
                               \text{ and } a_n = \alpha \\
                              &\text{ and } \forall t' > n.\; o_{t'} = 0 \land r_{t'} = \tfrac{1}{2}, \\
\xi((or)_{1:n} \dmid a_{1:n}) &\text{if } \exists n.\; o_{1:n} = 1^{n-1}0
                               \text{ and } a_n = \beta \\
                              &\text{ and } \forall t' > n.\; o_t = 0 \land r_t = 1 \\
                              &\text{ and } \exists k\; \eta(n-1, k), \\
\xi((or)_{1:t} \dmid a_{1:t}) &\text{if } \nexists n.\; o_{1:n} = 1^{n-1}0,
                               \text{ and} \\
0                             &\text{otherwise}.
\end{cases}
\]
See \autoref{fig:eps-AIXI-is-Sigma1-hard} for an illustration.
The environment $\nu$ mimics the universal environment $\xi$
until the observation history is $1^{n-1}0$.
Taking the action $\alpha$ next gives rewards $1/2$ forever.
Taking the action $\beta$ next gives rewards $1$ forever if $n \in A$,
otherwise the environment $\nu$ ends at some future time step.
Therefore we want to take action $\beta$ if and only if $n \in A$.
We have that $\nu \in \Mlscccs$ since
$\xi \in \Mlscccs$ and $\eta$ is quantifier-free.

We define $\xi' := \tfrac{1}{2} \nu + \tfrac{1}{8} \xi$.
By \autoref{lem:mixing-mixtures} $\xi'$ is a universal lower semicomputable semimeasure.
Let $n \in A$ be given and
let $h \in (\A \times \E)^n$ be any history
with observations $o_{1:n} = 1^{n-1}0$.
Since $\nu(1^{n-1}0 \mid a_{1:n}) = \xi(1^{n-1}0 \mid a_{1:n})$ by definition,
the posterior weights of $\nu$ and $\xi$ in $\xi'$
are equal to the prior weights,
analogously to the proof of \autoref{thm:dogmatic-prior}.
In the following,
we use the linearity of
$V^{\pi^*_{\xi'}}_\rho$ in $\rho$~(\autoref{lem:V-linear}),
and the fact that values are bounded between $0$ and $1$
(\hyperref[ass:bounded-rewards]{\autoref*{ass:aixi}\ref*{ass:bounded-rewards}}).
If there is a $k$ such that $\eta(n-1, k)$ holds,
\begin{align*}
      V^*_{\xi'}(h\beta) - V^*_{\xi'}(h\alpha)
&=    \tfrac{1}{2} V^{\pi^*_{\xi'}}_\nu(h\beta)
      - \tfrac{1}{2} V^{\pi^*_{\xi'}}_\nu(h\alpha)
      + \tfrac{1}{8} V^{\pi^*_{\xi'}}_\xi(h\beta)
      - \tfrac{1}{8} V^{\pi^*_{\xi'}}_\xi(h\alpha) \\
&\geq \tfrac{1}{2} - \tfrac{1}{4} + 0 - \tfrac{1}{8}
 =    \tfrac{1}{8},
\end{align*}
and similarly if there is no $k$ such that $\eta(n-1, k)$ holds, then
\begin{align*}
      V^*_{\xi'}(h\alpha) - V^*_{\xi'}(h\beta)
&=    \tfrac{1}{2} V^{\pi^*_{\xi'}}_\nu(h\alpha)
      - \tfrac{1}{2} V^{\pi^*_{\xi'}}_\nu(h\beta)
      + \tfrac{1}{8} V^{\pi^*_{\xi'}}_\xi(h\alpha)
      - \tfrac{1}{8} V^{\pi^*_{\xi'}}_\xi(h\beta)\\
&\geq \tfrac{1}{4} - 0 + 0 - \tfrac{1}{8}
 =    \tfrac{1}{8}.
\end{align*}
In both cases $|V^*_{\xi'}(h\beta) - V^*_{\xi'}(h\alpha)| > 1 / 9$.
Hence we pick $\varepsilon := 1/9$ and get for every
$\varepsilon$-optimal policy $\pi^\varepsilon_{\xi'}$ that
$\pi^\varepsilon_{\xi'}(h) = \beta$ if and only if $n \in A$.
\end{proof}

Note the differences between
\autoref{thm:AIXI-is-not-computable} and
\autoref{thm:eps-AIXI-is-Sigma1-hard}:
the former talks about optimal policies and shows that they are not computable,
but is agnostic towards the underlying universal Turing machine.
The latter talks about $\varepsilon$-optimal policies and gives a stronger hardness result,
at the cost of depending on one particular universal Turing machine.

\section{Iterative Value Function}
\label{sec:iterative-value-function}

Historically, AIXI's value function has been defined
slightly differently to \autoref{def:value-function},
using a limit extension of an iterative definition of the value function.
This definition is the more straightforward to come up with in AI:
it is the natural adaptation of (optimal) minimax search in zero-sum games to
the (optimal) expectimax algorithm for stochastic environments.
In this section we discuss the problems with this definition.

To avoid confusion with the recursive value function $V^\pi_\nu$,
we denote the iterative value function with $W^\pi_\nu$.\footnote{%
In \citet{LH:2015computability} the use of the symbols $V$ and $W$ is reversed.
}

\begin{definition}[{Iterative Value Function; \citealp[Def.~5.30]{Hutter:2005}}]
\label{def:iterative-value}\index{value function!iterative}
The \emph{iterative value} of a policy $\pi$ in an environment $\nu$
given history $\ae_{<t}$ is
\[
   W^\pi_\nu(\ae_{<t})
:= \frac{1}{\Gamma_t} \lim_{m \to \infty} \sum_{e_{t:m}}
     \nu( e_{1:m} \mid e_{<t} \dmid a_{1:m}) \sum_{k=t}^m \gamma(k) r_k
\]
if $\Gamma_t > 0$ and $W^\pi_\nu(\ae_{<t}) := 0$ if $\Gamma_t = 0$
where $a_i := \pi(e_{<i})$ for all $i \geq t$.
The \emph{optimal iterative value} is defined as
$W^*_\nu(h) := \sup_\pi W^\pi_\nu(h)$.
\end{definition}

Analogously to \eqref{eq:V-explicit}, we can write $W^*_\nu$
using the max-sum-operator:
\begin{equation}\label{eq:W-explicit}
  W^*_\nu(\ae_{<t})
= \frac{1}{\Gamma_t} \lim_{m \to \infty} \expectimax{\ae_{t:m}}\;
    \nu(e_{1:m} \mid e_{<t} \dmid a_{1:m}) \sum_{k=t}^m \gamma(k) r_k
\end{equation}

We use \emph{iterative AINU} for
the $\nu$-optimal policy according to the iterative value function, and
\emph{iterative AIXI} for
the $\xi$-optimal policy according to the iterative value function.
Note that iterative AIMU coincides with AIMU since $\mu$ is a measure by convention.

\begin{table}[t]
\begin{center}
\footnotesize
\renewcommand{\arraystretch}{1.2}
\setlength{\tabcolsep}{5pt}
\begin{tabular}{llll}
\toprule
Agent & Optimal & $\varepsilon$-Optimal \\
\midrule
Iterative AINU & $\Delta^0_4$, $\Sigma^0_3$-hard
               & $\Delta^0_3$, $\Pi^0_2$-hard \\
Iterative AIXI & $\Delta^0_4$, $\Pi^0_2$-hard
               & $\Delta^0_3$, $\Pi^0_2$-hard \\
Iterative AIMU 
               & $\Delta^0_2$
               & $\Delta^0_1$ \\
\bottomrule
\end{tabular}
\end{center}
\caption[Computability results for the iterative value function]{%
Computability results for different agent models
that use the iterative value function
derived in \autoref{sec:iterative-value-function}.
Hardness results for AINU are with respect to
a specific environment $\nu \in \Mlscccs$.
}
\label{tab:complexity-iterative-value}
\end{table}

Generally, our environment $\nu \in \Mlscccs$
is only a semimeasure and not a measure,
i.e., there is a history $\ae_{<t}a_t$ such that
\[
1 > \sum_{e_t \in \E} \nu(e_t \mid e_{<t} \dmid a_{1:t}).
\]
In such cases, with positive probability the environment $\nu$
does not produce a new percept $e_t$.
If this occurs, we shall use the informal interpretation that
the environment $\nu$ \emph{ended},
but our formal argument does not rely on this interpretation.

The following proposition shows that
for a semimeasure $\nu \in \Mlscccs$ that is not a measure,
iterative AINU does not maximize $\nu$-expected rewards.
Recall that $\gamma(1)$ states the discount of the first reward.
In the following, we assume without loss of generality that
$\gamma(1) > 0$, i.e.,
we are not indifferent about the reward received in time step $1$.

\begin{figure}[t]
\begin{center}
\begin{tikzpicture}[node distance=20mm, auto]
\node (start) {};
\node[agent, above of=start, node distance=10mm] (choice) {};
\node[env, left of=choice] (one) {};
\node[env, right of=choice] (eps) {};
\node[env, above of=eps] (hell) {Hell};
\node[stoch, above of=one] (void) {};
\path[transition] (start)  to (choice);
\path[transition] (choice) to node[below] {$\beta$} (eps);
\path[transition] (choice) to node[below] {$\alpha$} (one);
\path[transition] (eps)    to node[left] {$r = \varepsilon$} (hell);
\path[transition] (hell)   to[loop left] node[left] {$r = 0$} (hell);
\path[transition] (one)    to node[right] {$r = 1$} (void);
\end{tikzpicture}
\end{center}
\caption[Environment from the proof of \autoref{prop:iterative-AINU-is-not-a-reward-maximizer}]{%
The environment $\nu$ from the proof of \autoref{prop:iterative-AINU-is-not-a-reward-maximizer}.
Action $\alpha$ yields reward $1$, but subsequently the environment ends.
Action $\beta$ yields reward $\varepsilon$ and the environment continues forever.
Iterative AINU will prefer the suboptimal action $\beta$,
because it conditions on surviving forever.
}
\label{fig:AINU-is-not-a-reward-maximizer}
\end{figure}
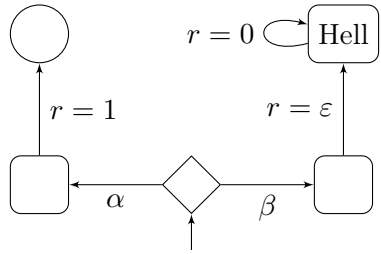

\begin{proposition}[Iterative AINU is not a $\nu$-Expected Reward Maximizer]
\label{prop:iterative-AINU-is-not-a-reward-maximizer}
\index{AINU}\index{value function!iterative}
For any $\varepsilon > 0$
there is an environment $\nu \in \Mlscccs$ that is not a measure and
a policy $\pi$ that receives a total of $\gamma(1)$ rewards in $\nu$,
but iterative AINU receives only $\varepsilon \gamma(1)$ rewards in $\nu$.
\end{proposition}

Informally, the environment $\nu$ is defined as follows.
In the first time step,
the agent chooses between the two actions $\alpha$ and $\beta$.
Taking action $\alpha$ gives a reward of $1$,
and subsequently the environment ends.
Action $\beta$ gives a reward of $\varepsilon$,
but the environment continues forever.
There are no other rewards in this environment.
See \autoref{fig:AINU-is-not-a-reward-maximizer}.
From the perspective of $\nu$-expected reward maximization,
it is better to take action $\alpha$,
however iterative AINU takes action $\beta$.

\begin{proof}[Proof of \autoref{prop:iterative-AINU-is-not-a-reward-maximizer}]
Let $\varepsilon > 0$.
We ignore observations and
set $\E := \{ 0, \varepsilon, 1 \}$, $\A := \{ \alpha, \beta \}$.
The environment $\nu$ is formally defined by
\[
\nu(r_{1:t} \dmid a_{1:t}) :=
\begin{cases}
1 &\text{if } a_1 = \alpha
   \text{ and } r_1 = 1
   \text{ and } t = 1 \\
1 &\text{if } a_1 = \beta
   \text{ and } r_1 = \varepsilon
   \text{ and } r_k = 0\; \forall 1 < k \leq t \\
0 &\text{otherwise}.
\end{cases}
\]
Taking action $\alpha$ first,
we have $\nu(r_{1:t} \dmid \alpha a_{2:t}) = 0$ for $t > 1$
(the environment $\nu$ ends in time step $2$ given history $\alpha$).
Hence we conclude
\[
  V^*_\nu(\alpha)
= \frac{1}{\Gamma_t} \lim_{m \to \infty} \sum_{r_{1:m}}
     \nu( r_{1:m} \dmid \alpha a_{2:m}) \sum_{k=1}^m \gamma(k) r_k
= 0.
\]
Taking action $\beta$ first we get
\[
  V^*_\nu(\beta)
= \frac{1}{\Gamma_t} \lim_{m \to \infty} \sum_{r_{1:m}}
     \nu( r_{1:m} \dmid \beta a_{2:m}) \sum_{k=1}^m \gamma(k) r_k
= \frac{\gamma(1)}{\Gamma_1} \varepsilon.
\]
Since $\gamma(1) > 0$ and $\varepsilon > 0$,
we have $V^*_\nu(\beta) > V^*_\nu(\alpha)$,
and thus iterative AINU will use a policy that plays action $\beta$ first,
receiving a total discounted reward of $\varepsilon \gamma(1)$.
In contrast,
any policy $\pi$ that takes action $\alpha$ first
receives a larger total discounted reward of $\gamma(1)$.
\end{proof}

Whether it is reasonable to assume
that our environment has a nonzero probability of ending
is a philosophical debate we do not want to engage in here;
see \citet{MEH:2016death} for a discussion.
Instead,
we have a different motivation to use the recursive over the iterative value function:
the latter has worse computability properties.
Concretely,
we show that $\varepsilon$-optimal iterative AIXI has to solve $\Pi^0_2$-hard problems
and that there is an environment $\nu \in \Mlscccs$
such that iterative AINU has to solve $\Sigma^0_3$-hard problems.
In contrast, using the recursive value function,
$\varepsilon$-optimal AIXI is $\Delta^0_2$ according to \autoref{cor:complexity-eps-ainu} and
AINU is $\Delta^0_3$ according to \autoref{cor:complexity-ainu}.

The central difference between $V^\pi_\nu$ and $W^\pi_\nu$ is that
for $V^\pi_\nu$ all obtained rewards matter,
but for $W^\pi_\nu$ only the rewards in timelines that continue indefinitely.
In this sense the value function $W^\pi_\nu$ conditions on surviving forever.
If the environment $\mu$ is a measure,
then the history is infinite with probability one,
and so $V^\pi_\nu$ and $W^\pi_\nu$ coincide.
Hence this distinction is not relevant for AIMU,
only for AINU and AIXI.

\begin{lemma}[Complexity of $W^*_\nu$]
\label{lem:complexity-W}
\index{value function!iterative}
For every $\nu \in \Mlscccs$,
the function $W^*_\nu$ is $\Pi^0_2$-computable.
\end{lemma}
\begin{proof}
Multiplying \eqref{eq:W-explicit} with $\Gamma_t \nu(e_{<t} \dmid a_{<t})$ yields
$W^*_\nu(\ae_{<t}) > q$ if and only if
\begin{equation}\label{eq:V-multiplied}
  \lim_{m \to \infty} \expectimax{\ae_{t:m}}\;
    \nu(e_{1:m} \dmid a_{1:m}) \sum_{k=t}^m \gamma(k) r_k
> q\, \Gamma_t\, \nu(e_{<t} \dmid a_{<t}).
\end{equation}
The inequality's right side is lower semicomputable,
hence there is a computable function $\psi$ such that
$\psi(\ell) \nearrow q\, \Gamma_t\, \nu(e_{<t} \dmid a_{<t}) =: q'$
as $\ell \to \infty$.
(In contrast to the recursive value function,
this quantity is not increasing in $m$.)
For a fixed $m$, the left side is also lower semicomputable,
therefore there is a computable function $\phi$ such that
\[
         \phi(m, k)
\nearrow \expectimax{\ae_{t:m}} \nu(e_{1:m} \dmid a_{1:m}) \sum_{k=t}^m \gamma(k) r_k
=:       f(m)
\text{ as $k \to \infty$.}
\]
We already know that the limit of $f(m)$ for $m \to \infty$ exists (uniquely),
hence we can write \eqref{eq:V-multiplied} as
\begin{align*}
                      &\lim_{m \to \infty} f(m) > q' \\
\Longleftrightarrow~~ &\forall m_0\, \exists m \geq m_0.\; f(m) > q' \\
\Longleftrightarrow~~ &\forall m_0\, \exists m \geq m_0\, \exists k.\;
                         \phi(m, k) > q' \\
\Longleftrightarrow~~ &\forall \ell\, \forall m_0\, \exists m \geq m_0\,
                         \exists k.\; \phi(m, k) > \psi(\ell),
\end{align*}
which is a $\Pi^0_2$-formula.
\end{proof}

Note that in the finite horizon case where $m$ is fixed,
the value function $W^*_\nu$
is $\Delta^0_2$-computable by \autoref{lem:computable-reals}d,
since $W^*_\nu(\ae_{<t}) = f(m) / q'$.
In this case, we get the same computability results for iterative AINU
as we did in \autoref{ssec:upper-bounds}.

\begin{corollary}[Complexity of Iterative AINU]
\label{cor:complexity-iterative-ainu}
\index{value function!iterative}\index{AINU}
For any environment $\nu \in \Mlscccs$,
iterative AINU is $\Delta^0_4$ and
there is an $\varepsilon$-optimal iterative AINU that is $\Delta^0_3$.
\end{corollary}
\begin{proof}
From \autoref{thm:complexity-optimal-policies},
\autoref{thm:complexity-eps-optimal-policies}, and
\autoref{lem:complexity-W}.
\end{proof}

\begin{figure}[t]
\begin{center}
\begin{tikzpicture}[node distance=20mm, auto]
\node (start) {};
\node[stoch, right of=start, node distance=10mm] (init) {};
\node[env, above of=init] (one) {$n$\texttt{++}};
\node[env, right of=init] (zero) {};
\node[agent, right of=zero] (choice) {};
\node[env, above of=choice] (hell) {Hell};
\node[stoch, below of=choice] (gate) {};
\node[env, left of=gate, node distance=40mm] (cheaven) {Conditional Heaven};
\path[transition] (start)  to (init);
\path[transition] (init)   to[bend right] node[right] {$1/2$} (one);
\path[transition] (one)    to[bend right] node[left] {$\begin{array}{c} o = 1 \\ r = 0 \end{array}$} (init);
\path[transition] (init)   to node {$1/2$} (zero);
\path[transition] (zero)   to node {$\begin{array}{c} o = 0 \\ r = 0 \end{array}$} (choice);
\path[transition] (choice) to node[right] {$\alpha$} (hell);
\path[transition] (choice) to node {$\beta$} (gate);
\path[transition] (gate)   to[bend left]
	node[below] {$\one_{\forall t' \leq t\, \exists k\; \eta(n, i, t', k)}$} (cheaven);
\path[transition] (hell)   to[loop left]
	node {$\begin{array}{c} o = 0 \\ r = 0 \end{array}$} (hell);
\path[transition] (cheaven) to
	node[above] {$\begin{array}{c} o = 0 \\ r = 1 \end{array}$} (gate);
\end{tikzpicture}
\end{center}
\caption[Environment from the proof of \autoref{thm:iterative-AINU-is-Sigma3-hard}]{%
The environment $\rho_i$ from the proof of \autoref{thm:iterative-AINU-is-Sigma3-hard}.
The mixture $\nu$ over class of environments
$\M := \{ \rho_0, \rho_1, \ldots \} \subset \Mlscccs$
forces iterative AINU to solve $\Sigma^0_3$-hard problems.
`Conditional Heaven' is a node that yields reward $1$ until
$\neg\exists k\; \eta(n, i, t, k)$,
at which point the environment ends.
Hence action $\beta$ is preferred in environment $\rho_i$ iff
conditional heaven lasts forever
(because otherwise $\nu(\ldots) = 0$ and hence $V^*_\nu(\ldots) = 0$)
which is the case iff
$\forall t\; \exists k\; \eta(n, i, t, k)$.
}
\label{fig:iterative-AINU-is-Sigma3-hard}
\end{figure}

We proceed to show corresponding lower bounds
as in \autoref{ssec:lower-bounds}.
For the rest of this section we assume $\Gamma_t > 0$ for all $t$.

\begin{theorem}[Iterative AINU is $\Sigma^0_3$-hard]
\label{thm:iterative-AINU-is-Sigma3-hard}
\index{value function!iterative}\index{AINU}
There is an environment $\nu \in \Mlscccs$ such that
iterative AINU is $\Sigma^0_3$-hard.
\end{theorem}
\begin{proof}
The proof is analogous to the proof of \autoref{thm:AINU-is-Pi2-hard}.
Let $A$ be any $\Sigma^0_3$ set, then
there is a quantifier-free formula $\eta$ such that
\[
n \in A
\;\Longleftrightarrow\;
\exists i\; \forall t\; \exists k\; \eta(n, i, t, k).
\]
We define the environments $\rho_i$ similar to
the proof of \autoref{thm:AINU-is-Pi2-hard},
except for two changes:
\begin{itemize}
\item We replace $\exists k\; \eta(n, i, k)$
	with $\forall t' \leq t\; \exists k\; \eta(n, i, t', k)$.
\item We switch actions $\alpha$ and $\beta$:
	action $\beta$ `checks' the formula $\eta$ and
	action $\alpha$ gives a sure reward of $0$.
\end{itemize}
Formally,
\[
\rho_i((or)_{1:t} \dmid a_{1:t}) :=
\begin{cases}
2^{-t}   &\text{if } o_{1:t} = 1^t
          \text{ and } \forall t' \leq t.\; r_{t'} = 0, \\
2^{-n-1} &\text{if } \exists n.\; 1^n0 \sqsubseteq o_{1:t} \sqsubseteq 1^n 0^\infty
          \text{ and } a_{n+2} = \alpha \\
         &\text{ and } \forall t' \leq t.\; r_{t'} = 0, \\
2^{-n-1} &\text{if } \exists n.\; 1^n0 \sqsubseteq o_{1:t} \sqsubseteq 1^n 0^\infty
          \text{ and } a_{n+2} = \beta \\
         &\text{ and } \forall t' \leq t.\; r_{t'} = \one_{t' > n+1} \\
         &\text{ and } \forall t' \leq t\, \exists k\; \eta(n, i, t', k),
          \text{ and} \\
0        &\text{otherwise}.
\end{cases}
\]
See \autoref{fig:iterative-AINU-is-Sigma3-hard}
for an illustration of the environment $\rho_i$.
Every $\rho_i$ is a chronological conditional semimeasure by definition,
so $\M := \{ \rho_0, \rho_1, \ldots \} \subseteq \Mlscccs$.
Furthermore,
every $\rho_i$ is lower semicomputable since $\eta$ is quantifier-free.

We define our environment $\nu$ as a mixture over $\M$,
\[
\nu := \sum_{i \in \mathbb{N}} 2^{-i-1} \rho_i;
\]
the choice of the weights on the environments $\rho_i$ is arbitrary but positive.
We get for the $\nu$-optimal policy $\pi^*_\nu$
analogously to the proof of \autoref{thm:AINU-is-Pi2-hard}
\[
\pi^*_\nu(1^n0) = \beta
~~\Longleftrightarrow~~
\exists i\, \forall t' \leq t\, \exists k\; \eta(n, i, t', k)
~~\Longleftrightarrow~~
n \in A,
\]
since action $\alpha$ is preferred in ties.
\end{proof}

Analogously to \autoref{thm:AIXI-is-not-computable},
we can show that
iterative AIXI is not computable.
We also get the following lower bound.

\begin{figure}[t]
\begin{center}
\begin{tikzpicture}[node distance=20mm, auto]
\node (start) {};
\node[env, right of=start, node distance=10mm] (xi) {$\xi$};
\node[agent, right of=xi, node distance=30mm] (choice) {};
\node[env, above of=choice] (sheaven) {Semi-Heaven};
\node[stoch, below of=choice] (gate) {};
\node[env, left of=gate, node distance=35mm] (cheaven) {Conditional Heaven};
\path[transition] (start)  to (xi);
\path[transition] (xi)     to[loop above] node {$\ast$} (xi);
\path[transition] (xi)   to node {$o = 1^n0$} (choice);
\path[transition] (choice) to node[right] {$\alpha$} (sheaven);
\path[transition] (choice) to node[right] {$\beta$} (gate);
\path[transition] (gate) to[bend left]
	node[below] {$\one_{\forall t' < t\, \exists k\; \eta(n, i, t', k)}$} (cheaven);
\path[transition] (sheaven)   to[loop left]
	node[left] {$\begin{array}{c} o = 0 \\ r = 1/2 \end{array}$} (sheaven);
\path[transition] (cheaven)   to
	node {$\begin{array}{c} o = 0 \\ r = 1 \end{array}$} (gate);
\end{tikzpicture}
\end{center}
\caption[Environment from the proof of \autoref{thm:iterative-eps-AIXI-is-Pi2-hard}]{%
The environment $\nu$ from the proof of \autoref{thm:iterative-eps-AIXI-is-Pi2-hard},
which forces $\varepsilon$-optimal iterative AIXI to solve $\Pi^0_2$-hard problems.
It functions just like $\xi$ until the observation history is $1^n0$.
Then, action $\alpha$ is preferred iff
conditional heaven never ends, i.e.,
iff $\forall t\, \exists k\; \eta(n, t, k)$.
}
\label{fig:iterative-eps-AIXI-is-Pi2-hard}
\end{figure}
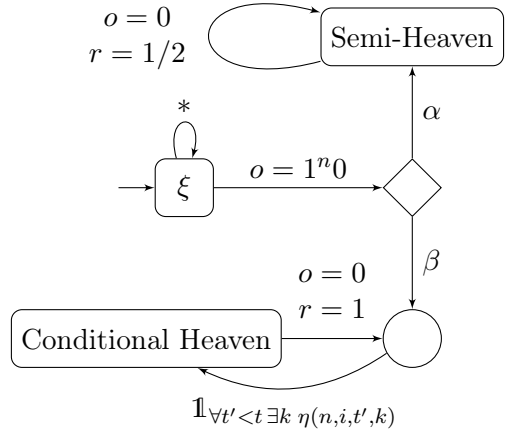

\begin{theorem}[Some $\varepsilon$-optimal iterative $\AIXI$ are $\Pi^0_2$-hard]
\label{thm:iterative-eps-AIXI-is-Pi2-hard}
\index{value function!iterative}\index{policy!e-optimal@$\eps$-optimal}
\index{AIXI}
There is a universal mixture $\xi'$ and an $\varepsilon > 0$ such that
any policy that is $\varepsilon$-optimal according to the iterative value
for environment $\xi'$ is $\Pi^0_2$-hard.
\end{theorem}
\begin{proof}
Let $A$ be a $\Pi^0_2$-set and $\eta$ a quantifier-free formula such that
\[
n \in A
\;\Longleftrightarrow\;
\forall t\, \exists k\; \eta(n, t, k).
\]
We proceed analogous to the proof of \autoref{thm:eps-AIXI-is-Sigma1-hard}
except that we choose $\forall t' \leq t\, \exists k\; \eta(n, t, k)$
as a condition for reward $1$ after playing action $\beta$.

Define the environment
\[
\nu((or)_{1:t} \dmid a_{1:t}) :=
\begin{cases}
\xi((or)_{1:n+1} \dmid a_{1:n+1})
  &\text{if } \exists n.\; 1^n0 \sqsubseteq o_{1:t} \sqsubseteq 1^n0^\infty \\
  &\text{ and } a_{n+1} = \alpha \\
  &\text{ and } \forall n + 1 < k \leq t.\; r_k = 1/2, \\
\xi((or)_{1:n+1} \dmid a_{1:n+1})
  &\text{if } \exists n.\; 1^n0 \sqsubseteq o_{1:t} \sqsubseteq 1^n0^\infty \\
  &\text{ and } a_{n+1} = \beta \\
  &\text{ and } \forall n + 1 < k \leq t.\; r_k = 1 \\
  &\text{ and } \forall t' \leq t\, \exists k\; \eta(n, t, k), \\
\xi((or)_{1:t} \dmid a_{1:t})
  &\text{if } \nexists n.\; 1^n0 \sqsubseteq o_{1:t} \sqsubseteq 1^n0^\infty,
   \text{ and} \\
0 &\text{otherwise}.
\end{cases}
\]
See \autoref{fig:iterative-eps-AIXI-is-Pi2-hard}
for an illustration of the environment $\nu$.
The environment $\nu$ mimics the universal environment $\xi$
until the observation history is $1^n0$.
The next action $\alpha$ always gives rewards $1/2$ forever,
while action $\beta$ gives rewards $1$ forever iff $n \in A$.
We have that
$\nu$ is a lower semicomputable semimeasure since
$\xi$ is a lower semicomputable semimeasure and $\eta$ is quantifier-free.
We define $\xi' = \tfrac{1}{2} \nu + \tfrac{1}{8} \xi$.
By \autoref{lem:mixing-mixtures},
$\xi'$ is a universal lower semicomputable semimeasure.
Let $n \in A$ be given and
let $h \in (\A \times \O)^{x+1}$ be any history
with observations $o_{1:n+1} = 1^n0$.
In the following,
we use the linearity of $W^*_\rho$ in $\rho$~
(analogously to \autoref{lem:V-linear}).
If $\forall t \exists k\; \eta(n, t, k)$, then
\begin{align*}
      W^*_{\xi'}(h\beta) - W^*_{\xi'}(h\alpha)
&=    \tfrac{1}{2} W^*_\nu(h\beta) - \tfrac{1}{2} W^*_\nu(h\alpha)
      + \tfrac{1}{8} W^*_\xi(h\beta) - \tfrac{1}{8} W^*_\xi(h\alpha)\\
&\geq \tfrac{1}{2} - \tfrac{1}{4} + 0 - \tfrac{1}{8}
 =    \tfrac{1}{8},
\end{align*}
and similarly if $\neg \forall t \exists k\; \eta(n, t, k)$, then
\begin{align*}
      W^*_{\xi'}(h\alpha) - W^*_{\xi'}(h\beta)
&=    \tfrac{1}{2} W^*_\nu(h\alpha) - \tfrac{1}{2} W^*_\nu(h\beta)
      + \tfrac{1}{8} W^*_\xi(h\alpha) - \tfrac{1}{8} W^*_\xi(h\beta)\\
&\geq \tfrac{1}{4} - 0 + 0 - \tfrac{1}{8}
 =    \tfrac{1}{8}.
\end{align*}
In both cases $|W^*_{\xi'}(h\beta) - W^*_{\xi'}(h\alpha)| > 1 / 9$,
hence with $\varepsilon := 1/9$
we have for an $\varepsilon$-optimal policy $\pi^\varepsilon_{\xi'}$ that
$\pi^\varepsilon_{\xi'}(h) = \beta$ if and only if $n \in A$.
\end{proof}

\section{The Complexity of Knowledge-Seeking}
\label{sec:complexity-knowledge-seeking}

Recall the definition of the optimal entropy-seeking value $V^{*,m}_{\Ent}$
and the optimal infor\-ma\-tion-seeking value $V^{*,m}_{\IG}$
from \autoref{ssec:knowledge-seeking-agents}.
Using the results from \autoref{sec:complexity-aixi}
we can show that $\varepsilon$-optimal
knowledge-seeking agents are limit computable, and
optimal knowledge-seeking agents are $\Delta^0_3$.

\begin{corollary}[Computability of Knowledge-Seeking Values]
\label{cor:complexity-knowledge-seeking-values}
\index{value function!entropy-seeking}\index{value function!information-seeking}
For fixed $m$, the value functions $V^{*,m}_\Ent$ and $V^{*,m}_\IG$
are limit computable.
\end{corollary}
\begin{proof}
This follows from \autoref{lem:computable-reals} (c-e)
since $\xi$, $\nu$, and $w$ are lower semicomputable.
\end{proof}

\begin{corollary}[Computability of Knowledge-Seeking Policies]
\label{cor:complexity-knowledge-seeking-policies}
\index{value function!entropy-seeking}\index{value function!information-seeking}
For entropy-seeking and information-seeking agents
there are limit-computable $\varepsilon$-optimal policies and
$\Delta^0_3$-computable optimal policies.
\end{corollary}
\begin{proof}
Follows from \autoref{cor:complexity-knowledge-seeking-values},
\autoref{thm:complexity-optimal-policies}, and
\autoref{thm:complexity-eps-optimal-policies}.
\end{proof}

Note that if we used an infinite horizon with discounting in
\autoref{def:V-entropy} or \autoref{def:V-information},
then we cannot retain this computability result without further assumptions:
we would need that the value functions increase monotonically as $m \to \infty$,
as they do for the recursive value function from \autoref{def:value-function}.
However, entropy is not a monotone function and may decrease if
there are events whose probability converges to something $> 1/2$.
For the entropy-seeking value function
this happens for histories drawn from a deterministic environment $\mu$
since $\xi \to \mu$, so the conditionals converge to $1$.
Similarly, for the information-seeking value function,
the posterior belief in one (deterministic) environment might become larger than $1/2$
(depending on the prior and the environment class).
Therefore we generally only get that
discounted versions of $V^*_\Ent$ and $V^*_\IG$ are $\Delta^0_3$
analogously to \autoref{lem:complexity-W}.
Hence optimal discounted entropy-seeking and optimal discounted information-seeking policies
are in $\Delta^0_4$ by \autoref{thm:complexity-optimal-policies} and
their corresponding $\varepsilon$-optimal siblings are $\Delta^0_3$
by \autoref{thm:complexity-eps-optimal-policies}.

\section{A Limit Computable Weakly Asymptotically Optimal Agent}
\label{sec:computability-wao}
\index{BayesExp}

According to \autoref{thm:AINU-is-Pi2-hard},
optimal reward-seeking policies are generally $\Pi^0_2$-hard,
and for optimal knowledge-seeking policies
\autoref{cor:complexity-knowledge-seeking-policies}
shows that they are $\Delta^0_3$.
Therefore we get that {\BayesExp} is $\Delta^0_3$:

\begin{corollary}[{\BayesExp} is $\Delta^0_3$]
\label{cor:BayesExp-is-Delta3}\index{BayesExp}
For any universal mixture $\xi$,
{\BayesExp} is $\Delta^0_3$.
\end{corollary}
\begin{proof}
From \autoref{cor:complexity-ainu},
\autoref{cor:complexity-knowledge-seeking-values}, and
\autoref{cor:complexity-knowledge-seeking-policies}.
\end{proof}

However, we do not know {\BayesExp} to be limit computable,
and we expect it not to be.
However, we can approximate it using $\varepsilon$-optimal policies
preserving weak asymptotic optimality.

\begin{theorem}[A Limit-Computable Weakly Asymptotically Optimal Agent]
\label{thm:wao-limit-computable}\index{optimality!asymptotic!weak}
\index{limit computable}
If there is a nonincreasing computable sequence of positive reals
$(\varepsilon_t)_{t \in \mathbb{N}}$ such that
$\varepsilon_t \to 0$ and
$H_t(\varepsilon_t) / (t \varepsilon_t) \to 0$
as $t \to \infty$,
then there is a limit-computable policy that is weakly asymptotically optimal
in the class of all computable stochastic environments.
\end{theorem}
\begin{proof}
By \autoref{cor:complexity-ainu},
there is a limit-computable $2^{-t}$-optimal reward-seeking policy $\pi^t_\xi$
for the universal mixture $\xi$.
By \autoref{cor:complexity-knowledge-seeking-policies}
there are limit-computable $\epsilon_t/2$-optimal
information-seeking policies $\pi_\IG^t$ with horizon $t + H_t(\varepsilon_t)$.
We define a policy $\pi$ analogously to \autoref{alg:BayesExp}
with $\pi_\IG^t$ and $\pi_\xi^t$ instead of the optimal policies.
From \autoref{cor:complexity-knowledge-seeking-values}
we get that $V^*_\IG$ is limit computable, so
the policy $\pi$ is limit computable.
Furthermore, $\pi^t_\xi$ is $2^{-t}$-optimal and $2^{-t} \to 0$,
so $V^{\pi^t_\xi}_\xi(\ae_{<t}) \to V^*_\xi(\ae_{<t})$
as $t \to \infty$.

Now we can proceed analogously to the proof of \citet[Thm.~5.6]{Lattimore:2013},
which consists of three parts.
First, 
it is shown that
the value of the $\xi$-optimal reward-seeking policy $\pi^*_\xi$
converges to the optimal value
for exploitation time steps (line 6 in \autoref{alg:BayesExp})
in the sense that
$V^{\pi^*_\xi}_\mu \to V^*_\mu$.
This carries over to the $2^{-t}$-optimal policy $\pi^t_\xi$,
since the key property is that on exploitation steps,
$V^*_\IG < \varepsilon_t$;
i.e., $\pi$ only exploits if potential knowledge-seeking value is low.
In short, we get for exploitation steps
\[
    V^{\pi^t_\xi}_\xi(\ae_{<t})
\to V^{\pi^*_\xi}_\xi(\ae_{<t})
\to V^{\pi^*_\xi}_\mu(\ae_{<t})
\to V^*_\mu(\ae_{<t})
\text{ as } t \to \infty.
\]

Second, 
it is shown that the density of exploration steps vanishes.
This result carries over since the condition $V^*_\IG(\ae_{<t}) > \varepsilon_t$
that determines exploration steps is exactly the same as for {\BayesExp}
and $\pi^t_\IG$ is $\varepsilon_t/2$-optimal.

Third,
the results of part one and two are used to conclude that
$\pi$ is weakly asymptotically optimal.
This part carries over to our proof.
\end{proof}

\section{Discussion}
\label{sec:discussion}

When using Solomonoff's prior for induction,
we need to evaluate conditional probabilities.
We showed that
conditional $M$ and $M\norm$ are limit computable~(\autoref{thm:complexity-M}),
and that $\MM$ and $\MM\norm$ are not limit computable~%
(\autoref{thm:MM-is-not-Delta2} and \autoref{cor:MMnorm-is-not-2}).
\autoref{tab:complexity-induction} on page~\pageref*{tab:complexity-induction}
summarizes our computability results on various versions of Solomonoff's prior.
Theses results implies that we can approximate $M$ or $M\norm$ for prediction,
but not the measure mixture $\MM$ or $\MM\norm$.

In some cases, normalized priors have advantages.
As illustrated in \autoref{ex:ksa-unnormalized},
unnormalized priors can make the entropy-seeking agent mistake
the entropy gained from the probability assigned to finite strings
for knowledge.
From $M\norm \geq M$ we get that
$M\norm$ predicts just as well as $M$,
and by \autoref{thm:complexity-M}
we can use $M\norm$ without losing limit computability.

\autoref{tab:complexity-agents} on page~\pageref*{tab:complexity-agents}
summarizes our computability results for the agents AINU, AIXI, and AINU.
AINU is $\Delta^0_3$ and
restricting to $\varepsilon$-optimal policies decreases the level by one~%
(\autoref{cor:complexity-ainu} and \autoref{cor:complexity-eps-ainu}).
For environments from $\Mcomp$,
AIMU is limit-computable and $\varepsilon$-optimal AIMU is computable~%
(\autoref{cor:complexity-aimu}).
In \autoref{ssec:lower-bounds} we proved that
these computability bounds on AINU are generally unimprovable
(\autoref{thm:AINU-is-Pi2-hard} and \autoref{thm:eps-AIXI-is-Sigma1-hard}).
Additionally, we proved weaker lower bounds for AIXI
independent of the universal Turing machine
(\autoref{thm:AIXI-is-not-computable})
and for $\varepsilon$-optimal AIXI
for specific choices of the universal Turing machine
(\autoref{thm:eps-AIXI-is-Sigma1-hard}).

When the environment $\nu$ has nonzero probability of not producing a new percept,
the iterative definition of AINU~(\autoref{def:iterative-value})
originally given by \citet[Def.~5.30]{Hutter:2005}
fails to maximize $\nu$-expected rewards
(\autoref{prop:iterative-AINU-is-not-a-reward-maximizer}).
Moreover, the policies are one level higher in the arithmetical hierarchy~%
(see \autoref{tab:complexity-iterative-value} on page~\pageref*{tab:complexity-iterative-value}).
We proved upper~(\autoref{cor:complexity-iterative-ainu}) and
lower bounds~(\autoref{thm:iterative-AINU-is-Sigma3-hard}
and \autoref{thm:iterative-eps-AIXI-is-Pi2-hard}).
The difference between
the recursive value function $V$ and the iterative value function $W$
is readily exposed in
the difference between the universal prior $M$ and the measure mixture $\MM$:
Just like $W$ conditions on surviving forever,
so does $\MM$ eliminate the weight of programs
that do not produce infinite strings.
Both $\MM$ and $W$ are not limit computable for this reason.

We considered $\varepsilon$-optimality to avoid having to determine argmax ties.
This $\varepsilon$ does not have to be constant over time,
we may let $\varepsilon \to 0$ as $t \to \infty$ at any computable rate.
With this we retain the computability results of $\varepsilon$-optimal policies
and get that
the value of the $\varepsilon(t)$-optimal policy $\pi^{\varepsilon(t)}_\nu$
converges rapidly to the $\nu$-optimal value:
\[
V^*_\nu(\ae_{<t}) - V^{\pi^{\varepsilon(t)}_\nu}_\nu(\ae_{<t})
\to 0 \text{ as $t \to \infty$}.
\]

In \autoref{sec:general-rl} we defined the set $\Mlscccs$
as the set of all lower semicomputable chronological contextual semimeasure
over percepts with actions povided as side-infor\-ma\-tion.
When determining the probability of the next percept $e_t$ in an environment $\nu$,
we have to compute $\nu(e_{1:t} \mid e_{<t} \dmid a_{1:t})$.
Alternatively, we could have defined the environment as
a lower semicomputable mapping from histories $\ae_{<t} a_t$
to probabilities over the next percept $e_t$
(this is done in \autoref{cha:grain-of-truth-problem}).
For the proof of \autoref{lem:complexity-V} and \autoref{lem:complexity-W}
we only need that $\nu(e_{1:t} \dmid a_{1:t})$ is lower semicomputable computable.
While this new definition makes no difference for the computability of AINU,
it matters for AIXI
because in the mixture $\xi$ over all of these environments
is no longer lower semicomputable.

Any method that tries to tackle the reinforcement learning problem
has to balance between exploration and exploitation%
\index{exploration vs.\ exploitation}.
AIXI strikes this balance in the Bayesian way.
However, as we showed in \autoref{sec:bad-priors},
this may not lead to enough exploration.
To counteract this,
we can add an explorative component to the agent,
akin to knowledge-seeking agents.
In \autoref{sec:complexity-knowledge-seeking} we show that
$\varepsilon$-optimal knowledge-seeking agents are limit computable
if we use the recursive definition of the value function.

We set out with the goal of finding a perfect reinforcement learning agent
that is limit computable.
The Bayesian agent AIXI could be considered one suitable candidate,
despite its optimality problems discussed in \autoref{cha:optimality}.
Another suitable candidate are weakly asymptotically optimal agents,
which in contrast to AIXI are optimal in an objective sense%
~(see \autoref{sec:discussion-optimality}).
We discussed {\BayesExp}, which relies
on a Solomonoff prior to learn its environment and
on a information-seeking component for extra exploration.
Our results culminated in
a limit-computable weakly asymptotically optimal agent
based on {\BayesExp}~(\autoref{thm:wao-limit-computable}).
In this sense our goal has been achieved.


\chapter[The Grain of Truth Problem]{The Grain of Truth Problem\footnote{%
The idea for reflective oracles was developed by Jessica Taylor and Benya Fallenstein based on ideas by Paul Christiano~\citep{Christiano+:2013}.
Reflective oracles were first described in \citet{FST:2015}.
The proof of \autoref{thm:lc-reflective-oracle}
was sketched by Benya Fallenstein and developed by me.
Except for minor editing, everything else in this chapter is my own work.
}}
\label{cha:grain-of-truth-problem}

\falsequote{Eliezer Yudkowsky}{AIs become friendly by playing lots of Newcomblike problems.}

Consider the general setup of multiple reinforcement learning agents
interacting sequentially
in a known environment with the goal to maximize discounted reward.%
\footnote{We mostly use the terminology of reinforcement learning.
For readers from game theory we provide a dictionary in
\autoref{tab:rl-game-theory-translation}.}
Each agent knows how the environment behaves,
but does not know the other agents' behavior.
The natural (Bayesian) approach would be to define a class of possible
policies that the other agents could adopt and take a prior over this class.
During the interaction, this prior gets updated to the posterior
as our agent learns the others' behavior.
Our agent then acts optimally with respect to this posterior belief.
\citet{KL:1993} show that in infinitely repeated games
Bayesian agents converge to an $\varepsilon$-Nash equilibrium
as long as each agent assigns positive prior probability
to the other agents' policies (a \emph{grain of truth}\index{grain of truth}).

\begin{table}[t]
\begin{center}
\begin{tabular}{p{0.445\columnwidth}p{0.445\columnwidth}}
\toprule
Reinforcement learning & Game theory \\
\midrule
stochastic policy & mixed strategy \\
deterministic policy & pure strategy \\
agent & player \\
multi-agent environment & infinite extensive-form game \\
reward & payoff/utility \\
(finite) history & history \\
infinite history & path of play \\
\bottomrule
\end{tabular}
\end{center}
\caption{Terminology dictionary between reinforcement learning and game theory.}
\label{tab:rl-game-theory-translation}
\end{table}

As an example,
consider an infinitely repeated prisoners dilemma\index{prisoner's dilemma}
between two agents.
In every time step the payoff matrix is as follows,
where C means cooperate and D means defect.
\begin{center}
\begin{tabular}{l|cc}
  & C        & D \\
\hline
C & 3/4, 3/4 & 0, 1 \\
D & 1, 0     & 1/4, 1/4
\end{tabular}
\end{center}
Define the set of policies $\Pi := \{ \pi_\infty, \pi_0, \pi_1, \ldots \}$
where policy $\pi_t$ cooperates until time step $t$ or
the opponent defects (whatever happens first) and defects thereafter.
The Bayes optimal behavior is to cooperate until the posterior belief that
the other agent defects in the time step after the next
is greater than some constant (depending on the discount function)
and then defect afterwards.
Therefore Bayes optimal behavior leads to a policy from the set $\Pi$
(regardless of the prior).
If both agents are Bayes optimal with respect to some prior,
they both have a grain of truth and therefore they converge to
a Nash equilibrium:
either they both cooperate forever or
after some finite time they both defect forever.
Alternating strategies like TitForTat
(cooperate first, then play the opponent's last action)
are not part of the policy class $\Pi$,
and adding them to the class breaks the grain of truth property:
the Bayes optimal behavior is no longer in the class.
This is rather typical;
a Bayesian agent usually needs to be
more powerful than its environment~(see \autoref{sec:complexity-aixi}).
We are facing the following problem.

\begin{problem}[{Grain of Truth Problem; \citealp[Q.~5j]{Hutter:2009open}}]
\label{prob:grain-of-truth}\index{grain of truth!problem|textbf}
Find a large class of policies $\Pi$
containing Bayesian agents with positive prior over $\Pi$.
\end{problem}

Until now, classes that admit a grain of truth
were known only for small toy examples such as
the iterated prisoner's dilemma above~\citep[Ch.~7.3]{SLB:2009}.
\citet{FY:2001impossibility} and \citet{Nachbar:1997,Nachbar:2005}
prove several impossibility results on the grain of truth problem
that identify properties that cannot be simultaneously
satisfied for classes that allow a grain of truth
(see \autoref{sec:impossibility-results} for a discussion).

In this chapter
we present a formal solution to the grain of truth problem%
~(\autoref{sec:a-grain-of-truth}).
We assume that our multi-agent environment is computable,
but it does not need to be stationary/Markov, ergodic, or finite-state.
Our class of policies $\Pi$ is large enough
to contain all computable (stochastic) policies,
as well as all relevant Bayes optimal policies.
At the same time, our class is small enough to be limit computable.
This is important because
it allows our result to be computationally approximated.

In \autoref{sec:multi-agent-environments}
we consider the setting where the multi-agent environment is
unknown to the agents and has to be learned
in addition to the other agents' behavior.
A Bayes optimal agent may not learn to act optimally
in unknown multi-agent environments \emph{even though it has a grain of truth}.
This effect occurs in non-recoverable environments where
taking one wrong action can mean a permanent loss of future value.
In this case, a Bayes optimal agent avoids taking these dangerous actions
and therefore will not explore enough
to wash out the prior's bias
(using the dogmatic prior from \autoref{ssec:dogmatic-prior}).
Therefore, Bayesian agents are not \emph{asymptotically optimal}%
\index{optimality!asymptotic}, i.e.,
they do not always learn to act optimally~(\autoref{thm:AIXI-not-AO}).

However, asymptotic optimality is achieved by Thompson sampling
because the inherent randomness of Thompson sampling
leads to enough exploration to learn the entire environment class%
~(see \autoref{ssec:asymptotic-optimality-TS}).
This leads to our main result:
if all agents use Thompson sampling over our class of multi-agent environments,
then for every $\varepsilon > 0$
they converge to an $\varepsilon$-Nash equilibrium.
This is not the first time Thompson sampling is used in game theory%
~\citep{OB:2014Thompson},
but the first time to show that it achieves such general positive results.

The central idea to our construction is based on
\emph{reflective oracles}\index{oracle!reflective}
introduced by \citet{FST:2015,FTC:2015reflection}.
Reflective oracles are probabilistic oracles
similar to halting oracles that
answer whether the probability that
a given probabilistic Turing machine $T$ outputs $1$
is higher than a given rational number $p$.
The oracles are reflective in the sense that the machine $T$
may itself query the oracle,
so the oracle has to answer queries about itself.
This invites issues caused by self-referential liar paradoxes of the form
``if the oracle says that this machine return $1$ with probability $> 1/2$,
then return $0$, else return $1$.''
Reflective oracles avoid these issues by being allowed to randomize if
the machines do not halt or the rational number
is \emph{exactly} the probability to output $1$.
We introduce reflective oracles formally in \autoref{sec:reflective-oracles}
and prove that there is a limit computable reflective oracle.

For infinitely repeated games practical algorithms
rely on \emph{ficticious play}\index{ficticious play}~\citep[Ch.~2]{FL:1998}:
the agent takes a best-response action
based on the assumption that
its opponent is playing a stationary but unknown mixed strategy
estimated according to the observed empirical frequencies.
If all agents converge to a stationary policy,
then this is a Nash equilibrium.
However, convergence is not guaranteed.

The same problem occurs in multi-agent reinforcement learning~\citep{BBDS:2008}.
Reinforcement learning algorithms
typically assume a (stationary) Markov decision process.
This assumption is violated when
interacting with other reinforcement learning agents
because as these agents learn, their behavior changes
and thus they are not stationary.
Assuming convergence to a stationary policy
is a necessary criterion to  enable all agents to learn,
but the process is unstable for many reinforcement learning algorithms and
only empirical positive results are known~\citep{BV:2001}.

\section{Reflective Oracles}
\label{sec:reflective-oracles}

\subsection{Definition}
\label{ssec:reflective-oracles-def}

First we connect semimeasures as defined in \autoref{def:semimeasure}
to Turing machines.
In \autoref{cha:preliminaries} we used \emph{monotone Turing machines}%
\index{Turing machine!monotone}
which naturally correspond to
lower semicomputable semimeasures~\citep[Sec.~4.5.2]{LV:2008}
that describe the distribution that arises when piping fair coin flips
into the monotone machine.
Here we take a different route.

A \emph{probabilistic Turing machine}%
\index{Turing machine!probabilistic|textbf} is a Turing machine that
has access to an unlimited number of uniformly random coin flips.
Let $\mathcal{T}$ denote the set of all probabilistic Turing machines
that take some input in $\X^*$ and may query an oracle (formally defined below).
We take a Turing machine $T \in \mathcal{T}$ to correspond to
a semimeasure $\lambda_T$ where
$\lambda_T(a \mid x)$ is the probability that
$T$ outputs $a \in \X$ when given $x \in \X^*$ as input.
The value of $\lambda_T(x)$ is then given by the chain rule\index{chain rule}
\begin{equation}\label{eq:chain-rule}
\lambda_T(x) := \prod_{k=1}^{|x|} \lambda_T(x_k \mid x_{<k}).
\end{equation}
Thus $\mathcal{T}$ gives rise to the set of semimeasures $\Mlsc$ where
the \emph{conditionals} $\lambda(a \mid x)$ are lower semicomputable.
In contrast, in \autoref{cha:computability} we considered semimeasures
whose \emph{joint} probability \eqref{eq:chain-rule} is lower semicomputable.
This set $\Mlsc$ contains all computable measures.
However, $\Mlsc$ is a proper subset of the set of all lower semicomputable semimeasures
because the product \eqref{eq:chain-rule} is lower semicomputable,
but there are some lower semicomputable semimeasures whose conditional
is not lower semicomputable~(\autoref{thm:M-conditional-is-not-Sigma1}):
\[
\Mcomp \subset \Mlsc \subset \Mlscccs
\]

In the following we assume that our alphabet is binary,
i.e., $\X := \{ 0, 1 \}$.

\begin{definition}[Oracle]
\label{def:oracle}\index{oracle|textbf}
An \emph{oracle} is a function
$O: \mathcal{T} \times \{ 0, 1 \}^* \times \mathbb{Q} \to \Delta \{ 0, 1 \}$.
\end{definition}

Oracles are understood to be probabilistic:
they randomly return $0$ or $1$.
Let $T^O$ denote the machine $T \in \mathcal{T}$ when run with the oracle $O$,
and let $\lambda_T^O$ denote the semimeasure induced by $T^O$.
This means that drawing from $\lambda_T^O$ involves two sources of randomness:
one from the distribution induced by the probabilistic Turing machine $T$
and one from the oracle's answers.

The intended semantics of an oracle are that it takes
a \emph{query}\index{query} $(T, x, p)$ and returns $1$
if the machine $T^O$ outputs $1$
on input $x$ with probability greater than $p$ when run with the oracle $O$,
i.e., when $\lambda^O_T(1 \mid x) > p$.
Furthermore, the oracle returns $0$ if the machine $T^O$ outputs $1$
on input $x$ with probability less than $p$ when run with the oracle $O$,
i.e., when $\lambda^O_T(1 \mid x) < p$.
To fulfill this,
the oracle $O$ has to make statements about itself,
since the machine $T$ from the query may again query $O$.
Therefore we call oracles of this kind \emph{reflective oracles}%
\index{oracle!reflective}.
This has to be defined very carefully
to avoid the obvious diagonalization issues that are caused by programs
that ask the oracle about themselves.
We impose the following self-consistency constraint.

\begin{definition}[Reflective Oracle]
\label{def:reflective-oracle}\index{oracle!reflective|textbf}
An oracle $O$ is \emph{reflective} iff
for all queries
$(T, x, p) \in \mathcal{T} \times \{ 0, 1 \}^* \times \mathbb{Q}$,
\begin{enumerate}[(a)]
\item $\lambda_T^O(1 \mid x) > p$ implies $O(T, x, p) = 1$, and
\item $\lambda_T^O(0 \mid x) > 1 - p$ implies $O(T, x, p) = 0$.
\end{enumerate}
\end{definition}

If $p$ under- or overshoots the true probability of $\lambda_T^O(\,\cdot \mid x)$,
then the oracle must reveal this information.
However, in the critical case when $p = \lambda_T^O(1 \mid x)$,
the oracle is allowed to return anything and may randomize its result.
Furthermore, since $T$ might not output any symbol,
it is possible that $\lambda_T^O(0 \mid x) + \lambda_T^O(1 \mid x) < 1$.
In this case the oracle can reassign the non-halting probability mass
to $0$, $1$, or randomize; see \autoref{fig:reflective-oracle}.

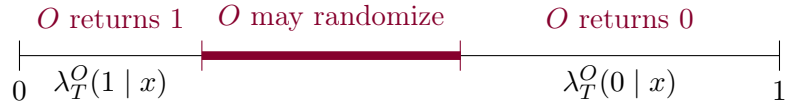
\begin{figure}[t]
\begin{center}
\begin{tikzpicture}[scale=1.0]
\draw (0,0) to (10, 0);
\draw (0, .2) to (0, -.2) node[below] {$0$};
\draw (10, .2) to (10, -.2) node[below] {$1$};

\node[below] at (7.9, 0) {$\lambda_T^O(0 \mid x)$};
\node[below] at (1.2, 0) {$\lambda_T^O(1 \mid x)$};

\draw[mycolor] (5.8, .2) to (5.8, -.2);
\draw[mycolor] (2.4, .2) to (2.4, -.2);
\filldraw[mycolor] (2.4, .05) -- (5.8, .05) -- (5.8, -.05) -- (2.4, -.05);
\node[mycolor] at (1.2, 0.5) {$O$ returns $1$};
\node[mycolor] at (4.1, 0.5) {$O$ may randomize};
\node[mycolor] at (7.9, 0.5) {$O$ returns $0$};
\end{tikzpicture}
\end{center}
\caption[Answer options of a reflective oracle]{%
Answer options of a reflective oracle $O$ for the query $(T, x, p)$;
the rational $p \in [0, 1]$ falls into one of the three regions above.
The values of $\lambda_T^O(0 \mid x)$ and $\lambda_T^O(1 \mid x)$ are depicted
as the length of the line segment under which they are written.
}\label{fig:reflective-oracle}
\end{figure}

\begin{example}[Reflective Oracles and Diagonalization]
\label{ex:diagonalization-for-reflective-oracles}\index{oracle!reflective}
Let $T \in \mathcal{T}$ be a probabilistic Turing machine that
outputs $1 - O(T, \epsilon, 1/2)$
($T$ can know its own source code by quining; \citealp[Thm.~27]{Kleene:1952}).
In other words, $T$ queries the oracle about whether it is more likely
to output $1$ or $0$, and then does whichever the oracle says is less likely.
In this case we can use an oracle $O(T, \epsilon, 1/2) := 1/2$
(answer $0$ or $1$ with equal probability),
which implies $\lambda_T^O(1 \mid \epsilon) = \lambda_T^O(0 \mid \epsilon) = 1/2$,
so the conditions of \autoref{def:reflective-oracle} are satisfied.
In fact, for this machine $T$ we must have
$O(T, \epsilon, 1/2) = 1/2$ for all reflective oracles $O$.
\end{example}

The following theorem establishes that reflective oracles exist.

\begin{theorem}[{\citealp[App.~B]{FTC:2015reflectionx}}]
\label{thm:existence-reflective-oracles}\index{oracle!reflective}
There is a reflective oracle.
\end{theorem}

\begin{definition}[Reflective-Oracle-Computable]
\label{def:reflective-oracle-computable}
\index{computable!reflective-oracle|textbf}
A semimeasure is called \emph{reflective-oracle-computable} iff
it is computable on a probabilistic Turing machine
with access to a reflective oracle.
\end{definition}

For any probabilistic Turing machine $T \in \mathcal{T}$
we can complete the semimeasure $\lambda_T^O(\,\cdot \mid x)$
into a reflective-oracle-computable measure
$\overline\lambda_T^O(\,\cdot \mid x)$:
Using the oracle $O$ and a binary search on the parameter $p$
we search for the crossover point $p$ where $O(T, x, p)$
goes from returning $1$ to returning $0$.
The limit point $p^*_x \in \mathbb{R}$ of the binary search is random
since the oracle's answers may be random.
But the main point is that the expectation of $p^*_x$ exists,
so $\overline\lambda_T^O(1 \mid x) = \mathbb{E}[p^*_x]
= 1 - \overline\lambda_T^O(0 \mid x)$.
Hence $\overline\lambda_T^O$ is a measure.
Moreover, if the oracle is reflective,
then $\overline\lambda_T^O(x) \geq \lambda_T^O(x)$ for all $x \in \X^*$.
In this sense the oracle $O$ can be viewed as
a way of `completing' all semimeasures $\lambda_T^O$
to measures by arbitrarily assigning the non-halting probability mass.
If the oracle $O$ is reflective this is consistent in the sense that
Turing machines who run other Turing machines will be completed in the same way.
This is especially important for a universal machine
that runs all other Turing machines
to induce a Solomonoff prior~(\autoref{ex:Solomonoff-prior}).

\subsection{A Limit Computable Reflective Oracle}
\label{ssec:lc-reflective-oracle}

The proof of \autoref{thm:existence-reflective-oracles}
given by \citet[App.~B]{FTC:2015reflectionx}
is nonconstructive and uses the axiom of choice.
In \autoref{ssec:lc-reflective-oracle-proof}
we give a new proof for the existence of reflective oracles
and provide a construction that there is a reflective oracle
that is limit computable.

\begin{theorem}[A Limit Computable Reflective Oracle]
\label{thm:lc-reflective-oracle}\index{oracle!reflective}
There is a reflective oracle that is limit computable.
\end{theorem}

This theorem has the immediate consequence
that reflective oracles cannot be used as halting oracles.
At first, this result may seem surprising:
according to the definition of reflective oracles,
they make concrete statements about the output of probabilistic Turing machines.
However, the fact that the oracles may randomize some of the time
actually removes enough information such that halting can no longer be decided
from the oracle output.

\begin{corollary}[Reflective Oracles are not Halting Oracles]
\label{cor:reflective-oracles-not-halting-oracles}
\index{oracle!reflective}\index{oracle!halting}
There is no probabilistic Turing machine $T$ such that
for every prefix program $p$ and every reflective oracle $O$,
we have that $\lambda_T^O(1 \mid p) > 1/2$ if $p$ halts and
$\lambda_T^O(1 \mid p) < 1/2$ otherwise.
\end{corollary}
\begin{proof}
Assume there was such a machine $T$ and
let $O$ be the limit computable oracle from \autoref{thm:lc-reflective-oracle}.
Since $O$ is reflective we can turn $T$ into a deterministic halting oracle
by calling $O(T, p, 1/2)$ which deterministically returns
$1$ if $p$ halts and $0$ otherwise.
Since $O$ is limit computable,
we can finitely compute the output of $O$ on any query
to arbitrary finite precision using our deterministic halting oracle.
We construct a probabilistic Turing machine $T'$ that uses our halting oracle
to compute (rather than query) the oracle $O$
on $(T', \epsilon, 1/2)$ to a precision of $1/3$ in finite time.
If $O(T', \epsilon, 1/2) \pm 1/3 > 1/2$, the machine $T'$ outputs $0$,
otherwise $T'$ outputs $1$.
Since our halting oracle is entirely deterministic,
the output of $T'$ is entirely deterministic as well (and $T'$ always halts),
so
$\lambda_{T'}^O(0 \mid \epsilon) = 1$ or $\lambda_{T'}^O(1 \mid \epsilon) = 1$.
Therefore $O(T', \epsilon, 1/2) = 1$ or $O(T', \epsilon, 1/2) = 0$
because $O$ is reflective.
A precision of $1/3$ is enough to tell them apart,
hence $T'$ returns $0$ if $O(T', \epsilon, 1/2) = 1$ and
$T'$ returns $1$ if $O(T', \epsilon, 1/2) = 0$.
This is a contradiction.
\end{proof}

A similar argument can also be used to show that
reflective oracles are not computable.

\subsection{Proof of \autoref*{thm:lc-reflective-oracle}}
\label{ssec:lc-reflective-oracle-proof}

The idea for the proof of \autoref{thm:lc-reflective-oracle} is to
construct an algorithm
that outputs an infinite series of \emph{partial oracles}
converging to a reflective oracle in the limit.

The set of queries is countable,
so we can assume that we have some computable enumeration of it:
\[
  \mathcal{T} \times \{ 0, 1 \}^* \times \mathbb{Q}
=: \{ q_1, q_2, \ldots \}
\]

\begin{definition}[$k$-Partial Oracle]
\label{def:k-partial-oracle}\index{oracle!partial|textbf}
A \emph{$k$-partial oracle} $\tilde O$ is function from the first $k$ queries
to the multiples of $2^{-k}$ in $[0, 1]$:
\[
\tilde O: \{ q_1, q_2, \ldots, q_k \} \to
\{ n 2^{-k} \mid 0 \leq n \leq 2^k \}
\]
\end{definition}

\begin{definition}[Approximating an Oracle]
\label{def:approx-oracle}\index{oracle!partial}
A $k$-partial oracle $\tilde O$ \emph{approximates} an oracle $O$ iff
$|O(q_i) - \tilde O(q_i)| \leq 2^{-k-1}$ for all $i \leq k$.
\end{definition}

Let $k \in \mathbb{N}$, let $\tilde O$ be a $k$-partial oracle, and
let $T \in \mathcal{T}$ be an oracle machine.
The machine $T^{\tilde O}$ that we get when
we run $T$ with the $k$-partial oracle $\tilde O$ is defined as follows
(this is with slight abuse of notation
since $k$ is taken to be understood implicitly).
\begin{enumerate}[1.]
\item Run $T$ for at most $k$ steps.
\item If $T$ calls the oracle on $q_i$ for $i \leq k$,
	\begin{enumerate}[(a)]
	\item return $1$ with probability $\tilde O(q_i) - 2^{-k-1}$,
	\item return $0$ with probability $1 - \tilde O(q_i) - 2^{-k-1}$, and
	\item halt otherwise.
	\end{enumerate}
\item If $T$ calls the oracle on $q_j$ for $j > k$, halt.
\end{enumerate}
Furthermore, we define $\lambda_T^{\tilde O}$ analogously to $\lambda_T^O$
as the distribution generated by the machine $T^{\tilde O}$.


\begin{lemma}\label{lem:approximating-an-oracle}
\index{oracle!partial}
If a $k$-partial oracle $\tilde O$ approximates a reflective oracle $O$,
then $\lambda_T^O(1 \mid x) \geq \lambda_T^{\tilde O}(1 \mid x)$ and
$\lambda_T^O(0 \mid x) \geq \lambda_T^{\tilde O}(0 \mid x)$ for all
$x \in \{ 0, 1 \}^*$ and all $T \in \mathcal{T}$.
\end{lemma}
\begin{proof}
This follows from the definition of $T^{\tilde O}$:
when running $T$ with $\tilde O$ instead of $O$,
we can only lose probability mass.
If $T$ makes calls whose index is $> k$ or runs for more than $k$ steps,
then the execution is aborted
and no further output is generated.
If $T$ makes calls whose index $i \leq k$, then
$\tilde O(q_i) - 2^{-k-1} \leq O(q_i)$ since $\tilde O$ approximates $O$.
Therefore the return of the call $q_i$ is underestimated as well.
\end{proof}

\begin{definition}[$k$-Partially Reflective]
\label{def:k-partially-reflective}\index{oracle!partially reflective|textbf}
A $k$-partial oracle $\tilde O$ is \emph{$k$-partially reflective} iff
for the first $k$ queries $(T, x, p)$
\begin{itemize}
\item $\lambda_T^{\tilde O}(1 \mid x) > p$ implies $\tilde O(T, x, p) = 1$, and
\item $\lambda_T^{\tilde O}(0 \mid x) > 1 - p$ implies $\tilde O(T, x, p) = 0$.
\end{itemize}
\end{definition}

It is important to note that we can check whether a $k$-partial oracle
is $k$-partially reflective in finite time by running all machines
$T$ from the first $k$ queries for $k$ steps and tallying up the probabilities
to compute $\lambda_T^{\tilde O}$.

\begin{lemma}\label{lem:partially-reflective}
\index{oracle!reflective}
\index{oracle!partially reflective}
\index{oracle!partial}
If $O$ is a reflective oracle and
$\tilde O$ is a $k$-partial oracle that approximates $O$, then
$\tilde O$ is $k$-partially reflective.
\end{lemma}

\autoref{lem:partially-reflective} only holds
because we use semimeasures whose conditionals are lower semicomputable.

\begin{proof}
Assuming $\lambda_T^{\tilde O}(1 \mid x) > p$ we get
from \autoref{lem:approximating-an-oracle} that
$
     \lambda_T^O(1 \mid x)
\geq \lambda_T^{\tilde O}(1 \mid x)
>    p
$.
Thus $O(T, x, p) = 1$ because $O$ is reflective.
Since $\tilde O$ approximates $O$,
we get $1 = O(T, x, p) \leq \tilde O(T, x, p) + 2^{-k-1}$, and
since $\tilde O$ assigns values in a $2^{-k}$-grid,
it follows that $\tilde O(T, x, p) = 1$.
The second implication is proved analogously.
\end{proof}

\begin{definition}[Extending Partial Oracles]
\label{def:extending-partial-oracles}
\index{oracle!partial}
A $k + 1$-partial oracle $\tilde O'$ \emph{extends}
a $k$-partial oracle $\tilde O$ iff
$|\tilde O(q_i) - \tilde O'(q_i)| \leq 2^{-k-1}$ for all $i \leq k$.
\end{definition}

\begin{lemma}
\label{lem:infinite-partial-oracles}
\index{oracle!partial}
\index{oracle!partially reflective}
There is an infinite sequence of partial oracles $(\tilde O_k)_{k \in \mathbb{N}}$
such that for each $k$,
$\tilde O_k$ is a $k$-partially reflective $k$-partial oracle
and $\tilde O_{k+1}$ extends $\tilde O_k$.
\end{lemma}
\begin{proof}
By \autoref{thm:existence-reflective-oracles}
there is a reflective oracle $O$.
For every $k$, there is a canonical $k$-partial oracle $\tilde O_k$ that approximates $O$:
restrict $O$ to the first $k$ queries and for any such query $q$
pick the value in the $2^{-k}$-grid which is closest to $O(q)$.
By construction, $\tilde O_{k+1}$ extends $\tilde O_k$
and by \autoref{lem:partially-reflective}, each $\tilde O_k$ is $k$-partially reflective.
\end{proof}

\begin{lemma}\label{lem:extending-oracles}
\index{oracle!partial}
If the $k+1$-partial oracle $\tilde O_{k+1}$ extends
the $k$-partial oracle $\tilde O_k$, then
$\lambda_T^{\tilde O_{k+1}}(1 \mid x) \geq \lambda_T^{\tilde O_k}(1 \mid x)$ and
$\lambda_T^{\tilde O_{k+1}}(0 \mid x) \geq \lambda_T^{\tilde O_k}(0 \mid x)$
for all $x \in \{ 0, 1 \}^*$ and all $T \in \mathcal{T}$.
\end{lemma}
\begin{proof}
$T^{\tilde O_{k+1}}$ runs for one more step than $T^{\tilde O_k}$,
can answer one more query and has increased oracle precision.
Moreover, since $\tilde O_{k+1}$ extends $\tilde O_k$,
we have $|\tilde O_{k+1}(q_i) - \tilde O_k(q_i)| \leq 2^{-k-1}$, and thus
$\tilde O_{k+1}(q_i) - 2^{-k-1} \geq \tilde O_k(q_i) - 2^{-k}$.
Therefore the success to answers to the oracle calls (case 2(a) and 2(b))
will not decrease in probability.
\end{proof}

Now everything is in place to state the algorithm
that constructs a reflective oracle in the limit.
It recursively traverses a tree of partial oracles.
The tree's nodes are the partial oracles;
level $k$ of the tree contains all $k$-partial oracles.
There is an edge in the tree from the $k$-partial oracle $\tilde O_k$ to
the $i$-partial oracle $\tilde O_i$ if and only if
$i = k + 1$ and $\tilde O_i$ extends $\tilde O_k$.

For every $k$, there are only finitely many $k$-partial oracles,
since they are functions from finite sets to finite sets.
In particular, there are exactly two $1$-partial oracles (so the search tree has two roots).
Pick one of them to start with, and proceed recursively as follows.
Given a $k$-partial oracle $\tilde O_k$,
there are finitely many $(k + 1)$-partial oracles that extend $\tilde O_k$
(finite branching of the tree).
Pick one that is $(k + 1)$-partially reflective
(which can be checked in finite time).
If there is no $(k + 1)$-partially reflective extension, backtrack.

By \autoref{lem:infinite-partial-oracles}
our search tree is infinitely deep and thus
the tree search does not terminate.
Moreover, it can backtrack to each level only a finite number of times
because at each level there is only a finite number of possible extensions.
Therefore the algorithm will produce an infinite sequence of partial oracles,
each extending the previous.
Because of finite backtracking, the output eventually stabilizes on a sequence
of partial oracles $\tilde O_1, \tilde O_2, \ldots$.
By the following lemma, this sequence converges to a reflective oracle,
which concludes the proof of \autoref{thm:lc-reflective-oracle}.

\begin{lemma}\label{lem:limit}
\index{oracle!partially reflective}
\index{oracle!partial}
\index{oracle!reflective}
Let $\tilde O_1, \tilde O_2, \ldots$ be a sequence where
$\tilde O_k$ is a $k$-partially reflective $k$-partial oracle and
$\tilde O_{k+1}$ extends $\tilde O_k$ for all $k \in \mathbb{N}$.
Let $O := \lim_{k \to \infty} \tilde O_k$ be the pointwise limit.
Then
\begin{enumerate}[(a)]
\item $\lambda_T^{\tilde O_k}(1 \mid x) \to \lambda_T^O(1 \mid x)$ and
      $\lambda_T^{\tilde O_k}(0 \mid x) \to \lambda_T^O(0 \mid x)$ as $k \to \infty$
      for all $x \in \{ 0, 1 \}^*$ and all $T \in \mathcal{T}$, and
\item $O$ is a reflective oracle.
\end{enumerate}
\end{lemma}
\begin{proof}
First note that the pointwise limit must exists because
$|\tilde O_k(q_i) - \tilde O_{k+1}(q_i)| \leq 2^{-k-1}$
by \autoref{def:extending-partial-oracles}.
\begin{enumerate}[(a)]
\item Since $\tilde O_{k+1}$ extends $\tilde O_k$,
    each $\tilde O_k$ approximates $O$.
    Let $x \in \{ 0, 1 \}^*$ and $T \in \mathcal{T}$ and
    consider the sequence $a_k := \lambda_T^{\tilde O_k}(1 \mid x)$ for $k \in \mathbb{N}$.
    By \autoref{lem:extending-oracles},
    $a_k \leq a_{k+1}$, so the sequence is monotone increasing.
    By \autoref{lem:approximating-an-oracle},
    $a_k \leq \lambda_T^O(1 \mid x)$, so the sequence is bounded.
    Therefore it must converge.
    But it cannot converge to anything strictly below $\lambda_T^O(1 \mid x)$
    by the definition of $T^O$.
\item By definition, $O$ is an oracle; it remains to show that $O$ is reflective.
    Let $q_i = (T, x, p)$ be some query.
    If $p < \lambda_T^O(1 \mid x)$, then by (a)
    there is a $k$ large enough such that
    $p < \lambda_T^{\tilde O_t}(1 \mid x)$ for all $t \geq k$.
    For any $t \geq \max \{ k, i \}$,
    we have $\tilde O_t(T, x, p) = 1$
    since $\tilde O_t$ is $t$-partially reflective.
    Therefore $1 = \lim_{k \to \infty} \tilde O_k(T, x, p) = O(T, x, p)$.
    The case $1 - p < \lambda_T^O(0 \mid x)$ is analogous.
    \qedhere
\end{enumerate}
\end{proof}

\section{A Grain of Truth}
\label{sec:a-grain-of-truth}

\subsection{Reflective Bayesian Agents}
\label{ssec:reflective-Bayesian-agents}

Fix $O$ to be a reflective oracle.
From now on,
we assume that the action space $\A := \{ \alpha, \beta \}$ is binary.
We can treat computable measures over binary strings as environments:
the environment $\nu$ corresponding to
a probabilistic Turing machine $T \in \mathcal{T}$
is defined by
\[
   \nu(e_t \mid \ae_{<t} a_t)
:= \overline\lambda_T^O(y \mid x)
 = \prod_{i=1}^k \overline\lambda_T^O(y_i \mid x y_1 \ldots y_{i-1})
\]
where $y_{1:k}$ is a binary encoding of $e_t$ and
$x$ is a binary encoding of $\ae_{<t} a_t$.
The actions $a_{1:\infty}$ are only \emph{contextual},
and not part of the environment distribution.
We define $\nu(e_{<t} \dmid a_{<t})$ analogously to \eqref{eq:ccs}.

Let $T_1, T_2, \ldots$ be an enumeration of
all probabilistic Turing machines in $\mathcal{T}$ that use an oracle.
We define the \emph{class of reflective environments}
\[
   \Mrefl^O
:= \left\{ \overline\lambda_{T_1}^O, \overline\lambda_{T_2}^O, \ldots \right\}.
\]
This is the class of all environments computable on
a probabilistic Turing machine with reflective oracle $O$,
that have been completed from semimeasures to measures using $O$.

Analogously to \autoref{ssec:AIXI},
we define a Bayesian mixture over the class $\Mrefl^O$.
Let $w \in \Delta\Mrefl^O$ be
a lower semicomputable prior probability distribution on $\Mrefl^O$.
Possible choices for the prior include the \emph{Solomonoff prior}%
\index{Solomonoff!prior}
$w\big(\overline\lambda_T^O\big) := 2^{-K(T)}$, where $K(T)$ denotes
the length of the shortest input to some universal Turing machine that encodes $T$.
We define the corresponding Bayesian mixture
\begin{equation}\label{eq:Bayes-mixture}
   \xi(e_t \mid \ae_{<t} a_t)
:= \sum_{\nu \in \Mrefl^O} w(\nu \mid \ae_{<t}) \nu(e_t \mid \ae_{<t} a_t)
\end{equation}
where $w(\nu \mid \ae_{<t})$ is the (renomalized) posterior,
\begin{equation}\label{eq:posterior}
   w(\nu \mid \ae_{<t})
:= w(\nu) \frac{\nu(e_{<t} \dmid a_{<t})}{\overline\xi(e_{<t} \dmid a_{<t})}.
\end{equation}
The mixture $\xi$ is lower semicomputable on an oracle Turing machine
because the posterior $w(\,\cdot \mid \ae_{<t})$ is lower semicomputable.
Hence there is an oracle machine $T$ such that $\xi = \lambda_T^O$.
We define its completion $\overline\xi := \overline\lambda_T^O$
as the completion of $\lambda_T^O$.
This is the distribution that is used to compute the posterior.
There are no cyclic dependencies since
$\overline\xi$ is called on the shorter history $\ae_{<t}$.
We arrive at the following statement.

\begin{proposition}[Bayes is in the Class]
\label{prop:Bayes-is-in-the-class}
\index{Bayesian!mixture}\index{computable!reflective-oracle}
$\overline\xi \in \Mrefl^O$.
\end{proposition}

Moreover, since $O$ is reflective,
we have that $\overline\xi$ dominates all environments $\nu \in \Mrefl^O$:
\begingroup
\allowdisplaybreaks
\begin{align*}
      \overline\xi(e_{1:t} \dmid a_{1:t})
&=    \overline\xi(e_t \mid \ae_{<t} a_t) \overline\xi(e_{<t} \dmid a_{<t}) \\
&\geq \xi(e_t \mid \ae_{<t} a_t) \overline\xi(e_{<t} \mid a_{<t}) \\
&=    \overline\xi(e_{<t} \dmid a_{<t})
      \sum_{\nu \in \Mrefl^O} w(\nu \mid \ae_{<t}) \nu(e_t \mid \ae_{<t} a_t) \\
&=    \overline\xi(e_{<t} \dmid a_{<t}) \sum_{\nu \in \Mrefl^O} w(\nu) \frac{\nu(e_{<t} \dmid a_{<t})}{\overline\xi(e_{<t} \dmid a_{<t})} \nu(e_t \mid \ae_{<t} a_t) \\
&=    \sum_{\nu \in \Mrefl^O} w(\nu) \nu(e_{1:t} \dmid a_{1:t}) \\
&\geq w(\nu) \nu(e_{1:t} \dmid a_{1:t})
\end{align*}
\endgroup

Therefore we get on-policy value convergence according to
\autoref{cor:Bayes-on-policy-value-convergence}:
for all $\mu \in \Mrefl^O$ and all policies $\pi$
\begin{equation}\label{eq:Mrefl-on-policy-value-convergence}
V^\pi_{\overline\xi}(\ae_{<t}) - V^\pi_\mu(\ae_{<t}) \to 0
\text{ as $t \to \infty$ $\mu^\pi$-almost surely.}
\end{equation}

\subsection{Reflective-Oracle-Computable Policies}
\label{ssec:reflective-oracle-computable-policies}

This subsection is dedicated to the following result
that was previously stated by \citet[Alg.~6]{FST:2015} but not proved.
It contrasts results on
arbitrary semicomputable environments where optimal policies are
not limit computable~(see \autoref{sec:complexity-aixi}).

\begin{theorem}[Optimal Policies are Oracle Computable]
\label{thm:optimal-policies-are-oracle-computable}
\index{computable!reflective-oracle}
\index{policy!optimal}
For every $\nu \in \Mrefl^O$,
there is a $\nu$-optimal (stochastic) policy $\pi^*_\nu$
that is reflective-oracle-computable.
\end{theorem}

Note that even though deterministic optimal policies always exist,
those policies are typically not reflective-oracle-computable.

To prove \autoref{thm:optimal-policies-are-oracle-computable}
we need the following lemma.

\begin{lemma}[Reflective-Oracle-Computable Optimal Value Function]
\label{lem:optimal-value-reflective-oracle-computable}
\index{computable!reflective-oracle}
\index{value function}
For every environment $\nu \in \Mrefl^O$
the optimal value function $V^*_\nu$ is reflective-oracle-computable.
\end{lemma}
\begin{proof}
This proof follows the proof of \autoref{cor:complexity-aimu}.
We write the optimal value explicitly as in \eqref{eq:V-explicit}.
For a fixed $m$,
all involved quantities are reflective-oracle-computable.
Moreover, this quantity is monotone increasing in $m$ and
the tail sum from $m+1$ to $\infty$ is bounded by $\Gamma_{m+1}$
which is computable according to \assref{ass:aixi}{ass:gamma-computable}
and converges to $0$ as $m \to \infty$.
Therefore we can enumerate all rationals above and below $V^*_\nu$.
\end{proof}

\begin{proof}[Proof of \autoref{thm:optimal-policies-are-oracle-computable}]
According to \autoref{lem:optimal-value-reflective-oracle-computable}
the optimal value function $V^*_\nu$ is reflective-oracle-computable.
Hence there is a probabilistic Turing machine $T$ such that
\[
  \lambda_T^O(1 \mid \ae_{<t})
= \big( V^*_\nu(\ae_{<t} \alpha) - V^*_\nu(\ae_{<t} \beta) + 1 \big) / 2.
\]
We define the policy
\[
\pi(\ae_{<t}) :=
\begin{cases}
\alpha &\text{if } O(T, \ae_{<t}, 1/2) = 1, \text{ and} \\
\beta  &\text{if } O(T, \ae_{<t}, 1/2) = 0 \\
\end{cases}
\]
This policy is stochastic because the answer of the oracle $O$ is stochastic.

It remains to show that $\pi$ is a $\nu$-optimal policy.
If $V^*_\nu(\ae_{<t} \alpha) > V^*_\nu(\ae_{<t} \beta)$,
then $\lambda_T^O(1 \mid \ae_{<t}) > 1/2$,
thus $O(T, \ae_{<t}, 1/2) = 1$ since $O$ is reflective,
and hence $\pi$ takes action $\alpha$.
Conversely, if $V^*_\nu(\ae_{<t} \alpha) < V^*_\nu(\ae_{<t} \beta)$,
then $\lambda_T^O(1 \mid \ae_{<t}) < 1/2$,
thus $O(T, \ae_{<t}, 1/2) = 0$ since $O$ is reflective,
and hence $\pi$ takes action $\beta$.
Lastly, if $V^*_\nu(\ae_{<t} \alpha) = V^*_\nu(\ae_{<t} \beta)$,
then both actions are optimal and
thus it does not matter which action is returned by policy $\pi$.
(This is the case where the oracle may randomize.)
\end{proof}

\subsection{Solution to the Grain of Truth Problem}
\label{ssec:solution-to-grain-of-truth-problem}

Together,
\autoref{prop:Bayes-is-in-the-class} and
\autoref{thm:optimal-policies-are-oracle-computable}
provide the necessary ingredients to solve the grain of truth problem%
~(\autoref{prob:grain-of-truth}).

\begin{corollary}[Solution to the Grain of Truth Problem]
\label{cor:solution-to-grain-of-truth}\index{grain of truth!problem}
For every lower semicomputable prior $w \in \Delta\Mrefl^O$
the Bayes optimal policy $\pi^*_{\overline\xi}$ is reflective-oracle-computable
where $\xi$ is the Bayes-mixture corresponding to $w$
defined in \eqref{eq:Bayes-mixture}.
\end{corollary}
\begin{proof}
From \autoref{prop:Bayes-is-in-the-class}
and \autoref{thm:optimal-policies-are-oracle-computable}.
\end{proof}

Hence the environment class $\Mrefl^O$ contains
any reflective-oracle-computable modification
of the Bayes optimal policy $\pi^*_{\overline\xi}$.
In particular,
this includes computable multi-agent environments
that contain other Bayesian agents over the class $\Mrefl^O$.
So any Bayesian agent over the class $\Mrefl^O$ has a grain of truth
even though the environment may contain other Bayesian agents
\emph{of equal power}.
We proceed to sketch the implications for multi-agent environments
in the next section.

\section{Multi-Agent Environments}
\label{sec:multi-agent-environments}

In a \emph{multi-agent environment}\index{multi-agent environment}
there are $n$ agents
each taking sequential actions from the finite action space $\A$.
In each time step $t = 1, 2, \ldots$,
the environment receives action $a_t^i$ from agent $i$ and outputs
$n$ percepts $e_t^1, \ldots, e_t^n \in \E$, one for each agent.
Each percept $e_t^i = (o_t^i, r_t^i)$ contains
an observation $o_t^i$ and a reward $r_t^i \in [0, 1]$.
Importantly, agent $i$ only sees
its own action $a_t^i$ and its own percept $e_t^i$
(see \autoref{fig:multi-agent-model}).
We use the shorthand notation $a_t := (a_t^1, \ldots, a_t^n)$ and
$e_t := (e_t^1, \ldots, e_t^n)$ and denote
$\ae_{<t}^i = a_1^i e_1^i \ldots a_{t-1}^i e_{t-1}^i$ and
$\ae_{<t} = a_1 e_1 \ldots a_{t-1} e_{t-1}$.
Formally, multi-agent environments are defined as follows.

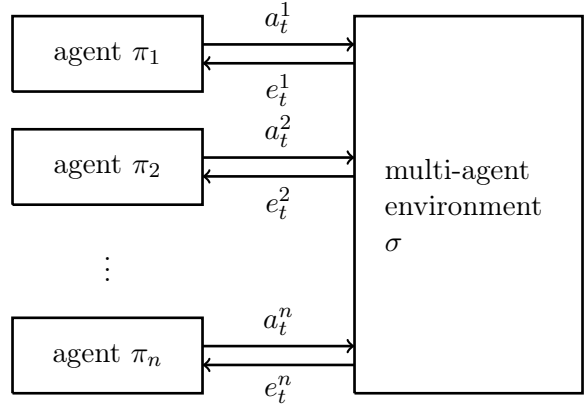
\begin{figure}[t]
\begin{center}
\begin{tikzpicture}[scale=0.25,line width=1pt] 
\draw (0,16) -- (10,16) -- (10,20) -- (0,20) -- (0,16);
\node at (5,18) {agent $\pi_1$};

\draw (0,10) -- (10,10) -- (10,14) -- (0,14) -- (0,10);
\node at (5,12) {agent $\pi_2$};

\node at (5,7) {\vdots};

\draw (0,0) -- (10,0) -- (10,4) -- (0,4) -- (0,0);
\node at (5,2) {agent $\pi_n$};

\draw (18,0) -- (30,0) -- (30,20) -- (18,20) -- (18,0);
\node at (24,10) {\begin{minipage}{22mm}
multi-agent \\ environment $\sigma$
\end{minipage}};

\draw[->] (10,18.5) to node[above] {$a_t^1$} (18,18.5);
\draw[<-] (10,17.5) to node[below] {$e_t^1$} (18,17.5);
\draw[->] (10,12.5) to node[above] {$a_t^2$} (18,12.5);
\draw[<-] (10,11.5) to node[below] {$e_t^2$} (18,11.5);
\draw[->] (10,2.5) to node[above] {$a_t^n$} (18,2.5);
\draw[<-] (10,1.5) to node[below] {$e_t^n$} (18,1.5);
\end{tikzpicture}
\end{center}
\caption[The multi-agent model]{%
Agents $\pi_1, \ldots, \pi_n$ interacting in a multi-agent environment%
\index{multi-agent environment}.
}
\label{fig:multi-agent-model}

\end{figure}

\begin{definition}[Multi-Agent Environment]
\label{def:multi-agent-environment}
\index{multi-agent environment|textbf}
\index{history!distribution|textbf}
A \emph{multi-agent environment} is a function
\[
\sigma: (\A^n \times \E^n)^* \times \A^n \to \Delta(\E^n).
\]
Together with the policies $\pi_1, \ldots, \pi_n$
the multi-agent environment $\sigma$ induces
a \emph{history distribution} $\sigma^{\pi_{1:n}}$ where
\begin{align*}
    \sigma^{\pi_{1:n}}(\epsilon) :&= 1 \\
    \sigma^{\pi_{1:n}}(\ae_{1:t})
:&= \sigma^{\pi_{1:n}}(\ae_{<t} a_t) \sigma(e_t \mid \ae_{<t} a_t) \\
    \sigma^{\pi_{1:n}}(\ae_{<t} a_t)
:&= \sigma^{\pi_{1:n}}(\ae_{<t}) \prod_{i=1}^n \pi_i(a_t^i \mid \ae_{<t}^i).
\end{align*}
\end{definition}

Agent $i$ acts
in a \emph{subjective environment}\index{subjective enviroment} $\sigma_i$
given by joining
the multi-agent environment $\sigma$
with the policies $\pi_1, \ldots, \pi_n$
and marginalizing over the histories that $\pi_i$ does not see.
Together with policy $\pi_i$,
the environment $\sigma_i$ yields a distribution over the histories of agent $i$
\[
   \sigma_i^{\pi_i}(\ae_{<t}^i)
:= \sum_{\ae_{<t}^j, j \neq i} \sigma^{\pi_{1:n}}(\ae_{<t}).
\]
We get the definition of the subjective environment $\sigma_i$ with the identity
$\sigma_i(e_t^i \mid \ae_{<t}^i a_t^i)
:= \sigma_i^{\pi_i}(e_t^i \mid \ae_{<t}^i a_t^i)$.
The subjective environment $\sigma_i$ depends on $\pi_i$
because other policies' actions may depend on the actions of $\pi_i$.
It is crucial to note that
the subjective environment $\sigma_i$ and the policy $\pi_i$
are ordinary environments and policies,
so we can use the notation from \autoref{cha:acting}.

Our definition of a multi-agent environment is very general
and encompasses most of game theory.
It allows for cooperative, competitive, and mixed games;
infinitely repeated games or any (infinite-length) extensive form games
with finitely many players.

\begin{example}[Matching Pennies]
\label{ex:matching-pennies}\index{matching pennies}
In the game of \emph{matching pennies} there are two agents ($n = 2$),
and two actions $\A = \{ \alpha, \beta \}$
representing the two sides of a penny.
In each time step
agent $1$ wins if the two actions are identical and
agent $2$ wins if the two actions are different.
The payoff matrix is as follows.
\begin{center}
\begin{tabular}{l|cc}
         & $\alpha$ & $\beta$ \\
\hline
$\alpha$ & 1,0      & 0,1 \\
$\beta$  & 0,1      & 1,0
\end{tabular}
\end{center}
We use $\E = \{ 0, 1 \}$ to be the set of rewards
(observations are vacuous) and define the multi-agent environment $\sigma$
to give reward $1$ to agent $1$ iff $a_t^1 = a_t^2$ ($0$ otherwise) and
reward $1$ to agent $2$ iff $a_t^1 \neq a_t^2$ ($0$ otherwise).
Formally,
\[
\sigma(r_t^1 r_t^2 \mid \ae_{<t} a_t) :=
\begin{cases}
1 &\text{if } r_t^1 = 1, r_t^2 = 0, a_t^1 = a_t^2, \\
1 &\text{if } r_t^1 = 0, r_t^2 = 1, a_t^1 \neq a_t^2, \text{ and} \\
0 &\text{otherwise}.
\end{cases}
\]
Let $\pi_\alpha$ denote the policy that always takes action $\alpha$.
If two agents each using policy $\pi_\alpha$ play matching pennies,
agent $1$ wins in every step.
Formally, setting $\pi_1 := \pi_2 := \pi_\alpha$
we get a history distribution that assigns probability one to the history
\[
\alpha \alpha 1 0 \alpha \alpha 1 0 \ldots.
\]
The subjective environment of agent $1$ is
\[
\sigma_1(r_t^1 \mid \ae_{<t}^1 a_t^1) =
\begin{cases}
1 &\text{if } r_t^1 = 1, a_t^1 = \alpha, \\
1 &\text{if } r_t^1 = 0, a_t^1 = \beta, \text{ and} \\
0 &\text{otherwise}.
\end{cases}
\]
Therefore policy $\pi_\alpha$ is optimal in agent $1$'s subjective environment.
\end{example}

\begin{definition}[$\varepsilon$-Best Response]
\label{def:eps-best-response}
\index{e-best response@$\eps$-best response|textbf}
A policy $\pi_i$ acting in multi-agent environment $\sigma$
with policies $\pi_1, \ldots, \pi_n$
is an \emph{$\varepsilon$-best response} after history $\ae_{<t}^i$ iff
\[
  V^*_{\sigma_i}(\ae_{<t}^i) - V^{\pi_i}_{\sigma_i}(\ae_{<t}^i)
< \varepsilon.
\]
\end{definition}

If at some time step $t$,
all agents' policies are $\varepsilon$-best responses,
we have an \emph{$\varepsilon$-Nash equilibrium}\index{Nash equilibrium}.
The property of multi-agent systems that is analogous to
asymptotic optimality is convergence to an $\varepsilon$-Nash equilibrium.

\section{Informed Reflective Agents}
\label{ssec:informed-reflective-agents}

Let $\sigma$ be a multi-agent environment
and let $\pi^*_{\sigma_1}, \ldots \pi^*_{\sigma_n}$ be such that
for each $i$ the policy $\pi^*_{\sigma_i}$ is an optimal policy
in agent $i$'s subjective environment $\sigma_i$.
At first glance this seems ill-defined:
The subjective environment $\sigma_i$ depends on each policy
$\pi^*_{\sigma_j}$,
which depends on the subjective environment $\sigma_j$,
which in turn depends on the policy $\pi^*_{\sigma_i}$.
However, this circular definition actually has a well-defined solution.

\begin{theorem}[Optimal Multi-Agent Policies]
\label{thm:informed-reflective-agents}
\index{multi-agent environment}
\index{policy!optimal}
\index{computable!reflective-oracle}
For any reflective-oracle-computable multi-agent environment $\sigma$,
the optimal policies $\pi^*_{\sigma_1}, \ldots, \pi^*_{\sigma_n}$
exist and are reflective-oracle-computable.
\end{theorem}

To prove \autoref{thm:informed-reflective-agents},
we need the following proposition.

\begin{proposition}[Reflective-Oracle-Computability]
\label{prop:reflective-oracle-computable}
\index{multi-agent environment}
\index{computable!reflective-oracle}
If the multi-agent environment $\sigma$ and the policies
$\pi_1, \ldots, \pi_n$ are reflective-oracle-computable,
then $\sigma^{\pi_{1:n}}$ and $\sigma_i^{\pi_i}$ are reflective-oracle-computable,
and $\sigma_i \in \Mrefl^O$.
\end{proposition}
\begin{proof}
All involved quantities in the definition of $\sigma^{\pi_{1:n}}$
are reflective-oracle-computable by assumption,
therefore also their marginalizations $\sigma_i^{\pi_i}$ and $\sigma_i$.
\end{proof}

\begin{proof}[Proof of \autoref{thm:informed-reflective-agents}]
According to \autoref{thm:optimal-policies-are-oracle-computable}
the optimal policy $\pi^*_{\sigma_i}$ in agent $i$'s subjective environment
is reflective-oracle-computable if the subjective environment $\sigma_i$ is.
In particular the process that takes $\sigma_i$
in form of a probabilistic Turing machine
and returns $\pi^*_{\sigma_i}$ is reflective-oracle-computable.
Moreover, $\sigma_i$ is reflective-oracle-computable
if $\sigma$ and $\pi_1, \ldots, \pi_n$ are,
according to \autoref{prop:reflective-oracle-computable}.
Again, this construction is itself reflective-oracle-computable.
Connecting these two constructions
we get probabilistic Turing machines $T_1, \ldots, T_n \in \mathcal{T}$
where each $T_i$ takes
the multi-agent environment $\sigma$ and $\pi_1, \ldots, \pi_n$
in form of probabilistic Turing machines and
returns $\pi^*_{\sigma_i}$.
We define the probabilistic Turing machines
$T'_1, \ldots, T'_n$ where $T_i'$ runs $T^O_i$
on $(\sigma, T'_1, \ldots, T'_n)$;
hence $T'_i$ computes $\pi^*_{\sigma_i}$.
Note that this construction only works because
we relied on the reflective oracle
in the proof of \autoref{thm:optimal-policies-are-oracle-computable}.
Since the machines $T^O_i$ always halt,
so do ${T'_i}^O$ despite their infinitely recursive construction.
\end{proof}

Note the strength of \autoref{thm:informed-reflective-agents}:
each of the policies $\pi^*_{\sigma_i}$ is acting optimally
\emph{given the knowledge of everyone else's policies}.
Hence optimal policies play $0$-best responses%
\index{e-best response@$\eps$-best response} by definition,
so if every agent is playing an optimal policy, we have a Nash equilibrium.
Moreover, this Nash equilibrium is also
a \emph{subgame perfect} Nash equilibrium%
\index{Nash equilibrium!subgame perfect},
because each agent also acts optimally on the counterfactual histories
that do not end up being played.
In other words,
\autoref{thm:informed-reflective-agents}
states the existence and reflective-oracle-computability
of a subgame perfect Nash equilibrium
in any reflective-oracle-computable multi-agent environment.
The following immediate corollary states that
these subgame perfect Nash equilibria are limit computable.

\begin{corollary}[Solution to Computable Multi-Agent Environments]
\label{cor:optimal-multi-agent-policies}
\index{multi-agent environment}
\index{policy!optimal}
\index{limit computable}
For any computable multi-agent environment $\sigma$,
the optimal policies $\pi^*_{\sigma_1}, \ldots, \pi^*_{\sigma_n}$
exist and are limit computable.
\end{corollary}
\begin{proof}
From \autoref{thm:informed-reflective-agents} and
\autoref{thm:lc-reflective-oracle}.
\end{proof}

\begin{example}[Nash Equilibrium in Matching Pennies]
\label{ex:matching-pennies2}\index{matching pennies}
\index{e-best response@$\eps$-best response}
Consider the matching pennies game from \autoref{ex:matching-pennies}.
The only pair of optimal policies is
the pair of two uniformly random policies
that play $\alpha$ and $\beta$ with equal probability in every time step:
if one of the agents picks a policy
that plays one of the actions with probability $> 1/2$,
then the other agent's best response is
to play the other action with probability $1$.
But now the first agent's policy is no longer a best response.
\end{example}

\section{Learning Reflective Agents}
\label{sec:learning-reflective-agents}

Since our class $\Mrefl^O$ solves the grain of truth problem,
the result by \citet{KL:1993} immediately implies that
for any Bayesian agents $\pi_1, \ldots, \pi_n$
interacting in an infinitely repeated game and
for all $\varepsilon > 0$ and all $i \in \{ 1, \ldots, n \}$
there is almost surely a $t_0 \in \mathbb{N}$ such that for all $t \geq t_0$
the policy $\pi_i$ is an $\varepsilon$-best response.%
\index{e-best response@$\eps$-best response}
However, this hinges on the important fact that
every agent has to know the game and
also that all other agents are Bayesian agents.
Otherwise the convergence to an $\varepsilon$-Nash equilibrium may fail,
as illustrated by the following example.

At the core of the construction is a \emph{dogmatic prior}%
\index{prior!dogmatic}~(\autoref{ssec:dogmatic-prior}).
A dogmatic prior assigns very high probability
to going to hell (reward $0$ forever)
if the agent deviates from a given computable policy $\pi$.
For a Bayesian agent it is thus only worth deviating from the policy $\pi$
if the agent thinks that the prospects of following $\pi$ are very poor already.
This implies that
for general multi-agent environments and
without additional assumptions on the prior,
we cannot prove any meaningful convergence result about Bayesian agents
acting in an unknown multi-agent environment.

\begin{example}[Reflective Bayesians Playing Matching Pennies]
\label{ex:reflective-Bayesians-playing-matching-pennies}
\index{matching pennies}\index{prior!dogmatic}
Consider the multi-agent environment matching pennies
from \autoref{ex:matching-pennies}.
Let $\pi_1$ be the policy that takes the action sequence
$(\alpha \alpha \beta)^\infty$ and
let $\pi_2 := \pi_\alpha$ be the policy that always takes action $\alpha$.
The average reward of policy $\pi_1$ is $2/3$ and
the average reward of policy $\pi_2$ is $1/3$.
Let $\xi$ be a universal mixture~\eqref{eq:Bayes-mixture}.
By on-policy value convergence~\eqref{eq:Mrefl-on-policy-value-convergence},
$V^{\pi_1}_{\overline\xi} \to c_1 \approx 2/3$ and
$V^{\pi_2}_{\overline\xi} \to c_2 \approx 1/3$
almost surely
when following policies $(\pi_1, \pi_2)$.
Therefore there is an $\varepsilon > 0$ such that
$V^{\pi_1}_{\overline\xi} > \varepsilon$ and
$V^{\pi_2}_{\overline\xi} > \varepsilon$
for all time steps.
Now we can apply \autoref{thm:dogmatic-prior} to conclude that
there are (dogmatic) mixtures $\xi_1'$ and $\xi_2'$ such that
$\pi^*_{\xi_1'}$ always follows policy $\pi_1$ and
$\pi^*_{\xi_2'}$ always follows policy $\pi_2$.
This does not converge to a ($\varepsilon$-)Nash equilibrium.
\end{example}

An important property required for
the construction in \autoref{ex:reflective-Bayesians-playing-matching-pennies}
is that the environment class contains environments
that threaten the agent with going to hell,
which is outside of the class of matching pennies environments.
In other words, since the agent does not know a priori that
it is playing a matching pennies game,
it might behave more conservatively than appropriate for the game.

The following theorem is our main convergence result.
It states that
for asymptotically optimal agents
we get convergence to $\varepsilon$-Nash equilibria
in any reflective-oracle-computable multi-agent environment.

\begin{theorem}[Convergence to Equilibrium]
\label{thm:convergence-to-equilibrium}\index{optimality!asymptotic!in mean}
\index{multi-agent environment}
\index{computable!reflective-oracle}
\index{e-best response@$\eps$-best response}
Let $\sigma$ be an reflective-oracle-computable multi-agent environment and
let $\pi_1, \ldots, \pi_n$ be reflective-oracle-computable policies
that are asymptotically optimal in mean in the class $\Mrefl^O$.
Then for all $\varepsilon > 0$ and all $i \in \{ 1, \ldots, n \}$
the $\sigma^{\pi_{1:n}}$-probability that
the policy $\pi_i$ is an $\varepsilon$-best response%
\index{e-best response@$\eps$-best response}
converges to $1$ as $t \to \infty$.
\end{theorem}
\begin{proof}
Let $i \in \{ 1, \ldots, n \}$.
By \autoref{prop:reflective-oracle-computable},
the subjective environment $\sigma_i$
is reflective-oracle-computable,
therefore $\sigma_i \in \Mrefl^O$.
Since $\pi_i$ is asymptotically optimal in mean in the class $\Mrefl^O$,
we get that
$\EE[ V^*_{\sigma_i}(\ae_{<t}) - V^{\pi_i}_{\sigma_i}(\ae_{<t})] \to 0$.
Convergence in mean implies
convergence in probability for bounded random variables,
hence for all $\varepsilon > 0$ we have
\[
\sigma_i^{\pi_i} [ V^*_{\sigma_i}(\ae_{<t}^i) - V^{\pi_i}_{\sigma_i}(\ae_{<t}^i) \geq \varepsilon ]
\to 0 \text{ as $t \to \infty$}.
\]
Therefore the probability
that the policy $\pi_i$ plays an $\varepsilon$-best response%
\index{e-best response@$\eps$-best response} converges to $1$ as $t \to \infty$.
\end{proof}

In contrast to \autoref{thm:informed-reflective-agents}
which yields policies that play a subgame perfect equilibrium,
this is not the case for \autoref{thm:convergence-to-equilibrium}:
the agents typically do not learn to predict off-policy and
thus will generally not play $\varepsilon$-best responses%
\index{e-best response@$\eps$-best response}
in the counterfactual histories that they never see.
This weaker form of equilibrium is unavoidable
if the agents do not know the environment because
it is impossible to learn the parts that they do not interact with.

\begin{corollary}[Convergence to Equilibrium]
\label{cor:convergence-to-equilibrium}
\index{multi-agent environment}
\index{limit computable}
\index{e-best response@$\eps$-best response}
There are limit computable policies $\pi_1, \ldots, \pi_n$ such that
for any computable multi-agent environment $\sigma$ and
for all $\varepsilon > 0$ and all $i \in \{ 1, \ldots, n \}$
the $\sigma^{\pi_{1:n}}$-probability that
the policy $\pi_i$ is an $\varepsilon$-best response%
\index{e-best response@$\eps$-best response}
converges to $1$ as $t \to \infty$.
\end{corollary}
\begin{proof}
Pick $\pi_1, \ldots, \pi_n$ to be the Thompson sampling policy $\pi_T$
defined in \autoref{alg:Thompson-sampling} over the countable class $\Mrefl^O$.
By \autoref{thm:Thompson-sampling-aoim} these policies
are asymptotically optimal in mean.
By \autoref{thm:TS-reflective-oracle-computable} below
they are reflective-oracle-computable and
by \autoref{thm:lc-reflective-oracle} they are also limit computable.
The statement now follows from \autoref{thm:convergence-to-equilibrium}.
\end{proof}

\begin{theorem}[Thompson Sampling is Reflective-Oracle-Computable]
\label{thm:TS-reflective-oracle-computable}
\index{Thompson sampling}
\index{computable!reflective-oracle}
The policy $\pi_T$ defined in \autoref{alg:Thompson-sampling}
over the class $\Mrefl^O$
is reflective-oracle-computable.
\end{theorem}
\begin{proof}
The posterior $w(\,\cdot \mid \ae_{<t})$ is reflective-oracle-computable
by the definition \eqref{eq:posterior} and
according to \autoref{thm:optimal-policies-are-oracle-computable}
the optimal policies $\pi^*_\nu$ are reflective-oracle-computable.
On resampling steps we can compute the action probabilities of $\pi_T$ by
enumerating all $\nu \in \Mrefl^O$ and computing
$\pi^*_\nu$ weighted by the posterior $w(\nu \mid \ae_{<t})$.
Between resampling steps we need to condition the policy $\pi_T$
computed above by the actions it has already taken
since the last resampling step
(compare \autoref{ex:TS-not-sao}).
\end{proof}

Because the posterior $w(\;\cdot \mid \ae_{<t})$ is a $\overline\xi^\pi$-martingale
when acting according to the policy $\pi$,
it converges $\overline\xi^\pi$-almost surely
according to the martingale convergence theorem%
~(\autoref{thm:martingale-convergence})\index{convergence!martingale}.
Since $\overline\xi$ dominates the subjective environment $\sigma_i$,
it also converges $\sigma_i^\pi$-almost surely%
~(see \autoref{ex:posterior-martingale}).
Hence all the Thompson sampling agents
eventually `calm down' and settle on some posterior belief.

According to \autoref{thm:convergence-to-equilibrium},
the policies $\pi_1, \ldots, \pi_n$
only need to be asymptotically optimal in mean.
For Thompson sampling this is independent of the discount function
according to \autoref{thm:Thompson-sampling-aoim}
(but the discount function has to be known to the agent).
So the different agents may use different discount functions,
resample at different time steps and converge at different speeds.

\begin{example}[Thompson Samplers Playing Matching Pennies]
\label{ex:matching-pennies3}\index{matching pennies}\index{Thompson sampling}
Consider the matching pennies game from \autoref{ex:matching-pennies} and
let both agents use the Thompson sampling policy
defined in \autoref{alg:Thompson-sampling},
i.e., define $\pi_1 := \pi_T$ and $\pi_2 := \pi_T$.

The value of the uniformly random policy $\pi_R$ is always $1/2$, so
$V^*_{\sigma_i} \geq V^{\pi_R}_{\sigma_i} = 1/2$.
According to \autoref{thm:convergence-to-equilibrium},
for every $\varepsilon > 0$,
each agent will eventually always play an $\varepsilon$-best response%
\index{e-best response@$\eps$-best response},
i.e., $V^{\pi_i}_{\sigma_i} > V^*_{\sigma_i} - \varepsilon \geq 1/2 - \varepsilon$.
Since matching pennies is a zero-sum game,
$V^{\pi_1}_{\sigma_1} + V^{\pi_2}_{\sigma_2} = 1$, so
$V^{\pi_2}_{\sigma_2} = 1 - V^{\pi_1}_{\sigma_1} < 1/2 + \varepsilon$.

Therefore each agent will end up randomizing their actions;
$\pi_i(a_t \mid \ae_{<t}) \approx 1/2$ most of the time:
If one of the agents (say agent 1) does not sufficiently randomize their actions
in some time steps,
then agent 2 could exploit this by picking a deterministic
adversarial policy in those time steps.
Suppose that this way it can gain a value of $\eps$
compared to the random policy $\pi_R$, i.e.,
$V^{\pi_2}_{\sigma_2} \geq V^{\pi_R}_{\sigma_2} + \eps = 1/2 + \eps$.
But this is a contradiction
because then agent 1 is not playing an $\varepsilon$-best response:%
\index{e-best response@$\eps$-best response}
\[
       V^*_{\sigma_1} - V^{\pi_1}_{\sigma_1}
~=~    V^*_{\sigma_1} - 1 + V^{\pi_2}_{\sigma_2}
~\geq~ 1/2 - 1 + 1/2 + \eps
~=~    \eps
\qedhere
\]
\end{example}

\section{Impossibility Results}
\label{sec:impossibility-results}

Why does our solution to the grain of truth problem
not violate the impossibility results from the literature?
Assume we are playing an infinitely repeated game where
in the stage game
no agent has a weakly dominant action and
the pure action maxmin reward is strictly less then the minmax reward.
The impossibility result of \citet{Nachbar:1997,Nachbar:2005}
states that
there is no class of policies $\Pi$ such that
the following are simultaneously satisfied.
\begin{itemize}
\item \emph{Learnability.}
	Each agent learns to predict the other agent's actions.
\item \emph{Caution and Symmetry.}
	The set $\Pi$ is closed under simple policy modifications
	such as renaming actions.
\item \emph{Purity.}
	There is an $\varepsilon > 0$ such that
	for any stochastic policy $\pi \in \Pi$
	there is a deterministic policy $\pi' \in \Pi$ such that
	if $\pi'(\ae_{<t}) = a$, then $\pi(a \mid \ae_{<t}) > \varepsilon$.
\item \emph{Consistency.}
	Each agent always has an $\varepsilon$-best response available in $\Pi$.%
	\index{e-best response@$\eps$-best response}
\end{itemize}
In order to converge to an $\varepsilon$-Nash equilibrium,
each agent has to have an $\varepsilon$-best response available to them,
so consistency is our target.
Learnability is immediately satisfied for any environment in our class
if we have a dominant prior
according to \autoref{cor:Bayes-on-policy-value-convergence}.
For $\Mrefl^O$ caution and symmetry are also satisfied
since this set is closed under any computable modifications to policies.
However, our class $\Mrefl^O$ avoids this impossibility result because
it violates the purity condition:
Let $T_1, T_2, \ldots$ be an enumeration of $\mathcal{T}$.
Consider the policy $\pi$
that maps history $\ae_{<t}^i$ to the action $1 - O(T_t, \ae^i_{<t}, 1/2)$.
If $T_t$ is deterministic,
then $\pi$ will take a different action than $T_t$
for any history of length $t - 1$.
Therefore no deterministic reflective-oracle-computable policy
can take an action that
$\pi$ assigns positive probability to in every time step.

\citet{FY:2001impossibility} present a condition that
makes convergence to a Nash equilibrium impossible:
if the player's rewards are perturbed by a small real number
drawn from some continuous density $\nu$,
then for $\nu$-almost all realizations the players do not learn to
predict each other and do not converge to a Nash equilibrium.
For example, in a matching pennies game,
rational agents randomize only if the (subjective) values of both actions
are exactly equal.
But this happens only with $\nu$-probability zero, since $\nu$ is a density.
Thus with $\nu$-probability one the agents do not randomize.
If the agents do not randomize,
they either fail to learn to predict each other,
or they are not acting rationally according to their beliefs:
otherwise they would seize the opportunity to
exploit the other player's deterministic action.

But this does not contradict our convergence result:
the class $\Mrefl^O$ is countable and each $\nu \in \Mrefl^O$
has positive prior probability.
Perturbation of rewards with arbitrary real numbers is not possible.
Even more, the argument given by \citet{FY:2001impossibility}
cannot work in our setting:
the Bayesian mixture $\overline\xi$ mixes over $\lambda_T$
for all probabilistic Turing machines $T$.
For Turing machines $T$ that sometimes do not halt,
the oracle decides how to complete
$\lambda_T$ into a measure $\overline{\lambda}_T$.
Thus the oracle has enough influence on the exact values in the Bayesian mixture
that the values of two actions in matching pennies can be made exactly equal.

\section{Discussion}
\label{sec:discussion-multiagent}

This chapter introduced
the class of all reflective-oracle-computable environments $\Mrefl^O$.
This class fully solves the grain of truth problem%
~(\autoref{prob:grain-of-truth})
because it contains (any computable modification of)
Bayesian agents defined over $\Mrefl^O$:
the optimal agents and Bayes optimal agents
over the class are all reflective-oracle-computable~%
(\autoref{thm:optimal-policies-are-oracle-computable} and
\autoref{cor:solution-to-grain-of-truth}).

If the environment is unknown,
then a Bayesian agent may end up playing suboptimally~%
(\autoref{ex:reflective-Bayesians-playing-matching-pennies}).
However, if each agent uses a policy that is asymptotically optimal in mean
(such as the Thompson sampling policy from \autoref{ssec:Thompson-sampling})
then for every $\varepsilon > 0$
the agents converge to an $\varepsilon$-Nash equilibrium~%
(\autoref{thm:convergence-to-equilibrium} and
\autoref{cor:convergence-to-equilibrium}).

However, \autoref{cor:convergence-to-equilibrium} does \emph{not} imply that
two Thompson sampling policies will converge
to cooperation in an iterated prisoner's dilemma\index{prisoner's dilemma}
since always defecting is also a Nash equilibrium.
The exact outcome will depend on the priors involved
and the randomness of the policies.

Our solution to the grain of truth problem is purely theoretical.
However, \autoref{thm:lc-reflective-oracle} shows that
our class $\Mrefl^O$ allows for computable approximations.
This suggests that practical approaches can be derived from this result,
and reflective oracles have already seen
applications in one-shot games~\citep{FTC:2015reflection}.


\chapter{Conclusion}
\label{cha:conclusion}

\falsequote{Nick Bostrom}{The biggest existential risk is that future superintelligences stop simulating us. \\ \phantom{menace}}


Today computer programs exceeding humans in general intelligence
are known only from science fiction.
But research on AI has progressed steadily over the last decades
and there is good reason to believe that
we will be able to build HLAI eventually,
and even sooner than most people think~\citep{MB:2016forcast}.

The advent of strong AI would be the biggest event in human history.
Potential benefits are huge,
as the new level of automation would free us from any kind of undesirable labor.
But there is no reason to believe that humans are
at the far end of the intelligence spectrum.
Rather, humans barely cross the threshold for general intelligence
to be able to use language and do science.
Once we engineer HLAI, it seems unlikely that progress is going to stop;
why not build even smarter machines?

This could lead to an \emph{intelligence explosion}%
\index{intelligence!explosion}~\citep{Good:1965,Vinge:1993,Kurzweil:2005,
Chalmers:2010singularity,
Hutter:2012explosion,Schmidhuber:2012singularity,MS:2012explosion,
EMSS:2013singularity,Shanahan:2015singularity,Eden:2016singularity,
Walsh:2016singularity}:
a (possibly very rapid) increase in intelligence,
e.g., through self-amplification effects from
AIs improving themselves
(if doing AI research is one of humans' capabilities,
then a machine that can do everything humans can do
can also do AI research).
Once machine intelligence is above or far above human level,
machines would steer the course of history.
There is no reason to believe that machines would be adversarial to us,
but nevertheless humanity's fate might rest at the whims of the machines,
just as chimpanzees today have no say in the large-scale events on this planet.

\citet{Bostrom:2002xrisk} defines:
\begin{quote}
\emph{Existential risk}\index{existential risk} ---
One where an adverse outcome would
either annihilate Earth-originating intelligent life or
permanently and drastically curtail its potential.
\end{quote}
Existential risks are events that
have the power to extinguish human life as we know it.
Examples are cosmic events such as an asteroid colliding with Earth.
But cosmic events are unlikely on human timescales
compared to human-made existential risks
from nuclear weapons, synthetic biology, and nanotechnology.

It is possible that artificial intelligence also falls into this category.
\citet{Vinge:1993} was the first person to recognize this:
\begin{quote}
Within thirty years,
we will have the technological means to create superhuman intelligence.
Shortly after, the human era will be ended.
\end{quote}
After Vinge,
\citet{Yudkowsky:2001,Yudkowsky:2008xrisk} can be regarded as
one of the key people
to popularize the potential dangers of AI technology.
\citet{Bostrom:2003} picked up on this issue very early and
gave the topic credibility through
his well-researched and carefully written book \emph{Superintelligence}%
~\citep{Bostrom:2014}.
%
He argues that
we need to solve the \emph{Control Problem}\index{Control Problem}%
---the unique principal-agent problem
that arises with the creation of strong AI~\citep[Ch.~9]{Bostrom:2014}.
In other words:
How do we align strong AI with human values?
How do we ensure that AI remains robust and beneficial?
This research is collected under the umbrella term \emph{AI safety}%
\index{AI safety}.
Currently, we have no idea how to solve these problems even in theory.

Through Yudkowsky's and, more importantly, Bostrom's efforts,
AI-related long-term safety concerns
have now entered the mainstream media.
In 2014 high-profile scientists such as
Stephen Hawking, Max Tegmark, Stuart Russell, and Frank Wilczek
have warned against the dangers posed by AI~\citep{HTRW:2014huffington}.
(See also \citet{Alexander:2015AI} for a collection of
positions by prominent AI researchers.)
Many scientists inside and outside the field have signed an open letter
that research ensuring that AI systems remain robust and beneficial
is both important and timely~\citep{FLI:2015openletter}.
This lead entrepreneur Elon Musk to donate \$10 million
to kick-start research in this field~\citep{FLI:2015Musk};
most of this money has now been distributed across the planet
to 37 different projects~\citep{FLI:2015grants}.
Moreover,
the Future of Life Institute and
the Machine Intelligence Research Institute have formulated
concrete technical research priorities
to make AI more robust and beneficial~\citep{MIRI:2014agenda,FLI:2015}.

At the end of 2015 followed
the announcement of OpenAI,
a nonprofit organization with financial backing from
several famous silicon valley billionaires~\citep{OpenAI:2015}:
\begin{quote}
OpenAI is a non-profit artificial intelligence research company.
Our goal is to advance digital intelligence in the way
that is most likely to benefit humanity as a whole,
unconstrained by a need to generate financial return.
\end{quote}
The mission of OpenAI is
to enable everyone to benefit from AI technology.
But, despite the name,
OpenAI is not committed to publish all of their research freely,
and \citet{Bostrom:2016openness} argues that
unrestricted publication might not be the best idea.

Despite all of the recent efforts in AI safety research,
critical voices within the AI community remain.
Prominently, \citet{Davis:2014singularity},
\citet{Ng:2016singularity},
\citet{Walsh:2016singularity}, and \citet{Lawrence:2016singularity}
have proposed counterarguments
that range from `HLAI is so far away that any worry is misplaced'
to claims that `the safety problem would not be so hard'.
See \citet{SY:2014safety} for a discussion.


If AI poses an existential risk
then a formal theory of strong AI is paramount
to develop technical approaches to mitigate this risk.
Which path will ultimately lead us to HLAI
is in the realm of speculation at this time;
therefore we should make as few and as weak assumptions as possible
and abstract away from possible future implementation details.

This thesis lays some of the groundwork for this endeavor.
We built on top of Hutter's theory of
universal artificial intelligence\index{universal artificial intelligence}.
\autoref{cha:learning} discussed the formal theory of learning.
\autoref{cha:acting} presented several approaches to
acting in unknown environments
(Bayes, Thompson sampling, knowledge-seeking agents, and {\BayesExp}).
\autoref{cha:optimality} analysed these approaches and
discussed notions of optimality
and principled problems with acting under uncertainty in general environment.
\autoref{cha:computability} provided the mathematical tools
to analyze the computational properties of these models.
Finally,
\autoref{cha:grain-of-truth-problem} solved the grain of truth problem,
which lead to convergence to Nash equilibria
in unknown general multi-agent environments.

Our work is theoretical by nature
and there is still a long way to go
until these results make their way into applications.
But a solution \emph{in principle} is a crucial first step
towards solving a problem in practice.
Consider the research paper by \citet{Shannon:1950chess} on
how to solve chess in principle.
The algorithm he describes expands the full game tree of chess
(until some specified depth),
which is completely infeasible even with today's computation power.
His contribution was to show that winning at chess
is a feat that computers can achieve \emph{in principle},
which was not universally accepted at the time.
Even more, his approach already considered the correct ideas
(minimax-search over the game tree) that ultimately lead to the defeat
of the chess champion Garry Kasparov
by the computer Deep Blue 46 years later~\citep{DeepBlue}.

The theory of general reinforcement learning
can serve and has served as
a starting point for future investigation in AI safety.
In particular,
value learning~\citep{Dewey:2011},
self-reflection~\citep{Soares:2015si,FTC:2015reflection},
self-modification~\citep{OR:2011mortality,OR:2012,EFDH:2016modification},
interruptibility~\citep{OA:2016,AO:2016},
decision theory~\citep{ELH:2015sdt},
memory manipulation~\citep{OR:2012memory},
wireheading\index{wireheading}~\citep{RO:2011delusion,EH:2016wireheading}, and
questions of identity~\citep{Orseau:2014multislot,Orseau:2014teleporting}.

It is possible that HLAI is decades or centuries away.
It might also be a few years around the corner.
Whichever is the case,
we are currently completely unprepared for the consequences.
As an AI researcher,
it is tempting to devote your life to
increasing the capability of AI,
advancing it domain after domain, and
showing off with flashy demos and impressive victories over human contestants.
But every technology incurs risks,
and the more powerful the technology,
the higher the risks.
The potential power of AI technology is enormous, and
correspondingly we need to consider the risks, take them seriously,
and proceed to mitigate them.


\backmatter
\appendix
\setcounter{chapter}{1} 

\chapter{Measures and Martingales}
\label{cha:measures-martingales}

In this chapter we provide the proofs for
\autoref{thm:measure-martingale} and \autoref{thm:martingale-measure}.

\begin{proof}[Proof of \autoref{thm:measure-martingale}]
$X_t$ is only undefined if $P(\Gamma_{v_{1:t}}) = 0$.
The set
\[
\{ v \in \Sigma^\infty \mid \exists t.\; P(\Gamma_{v_{1:t}}) = 0 \}
\]
has $P$-measure $0$ and hence
$(X_t)_{t \in \mathbb{N}}$ is well-defined almost everywhere.

$X_t$ is constant on $\Gamma_u$ for all $u \in \Sigma^t$, and
$\F_t$ is generated by a collection of finitely many disjoint sets:
\[
\Sigma^\infty = \biguplus_{u \in \Sigma^t} \Gamma_u.
\]
\begin{enumerate}[(a)]
\item
Therefore $X_t$ is $\F_t$-measurable.

\item
$\Gamma_u = \biguplus_{a \in \Sigma} \Gamma_{ua}$ for all $u \in \Sigma^t$ and $v \in \Gamma_u$,
and therefore
\begin{align*}
\EE[X_{t+1} \mid \F_t](v)
&= \frac{1}{P(\Gamma_u)} \sum_{a \in \Sigma} X_{t+1}(ua) P(\Gamma_{ua})
 = \frac{1}{P(\Gamma_u)} \sum_{a \in \Sigma} \frac{Q(\Gamma_{ua})}{P(\Gamma_{ua})} P(\Gamma_{ua}) \\
&\stackrel{(\ast)}{=} \frac{1}{P(\Gamma_u)} \sum_{a \in \Sigma} Q(\Gamma_{ua})
 = \frac{Q(\Gamma_u)}{P(\Gamma_u)}
 = X_t(v).
\end{align*}
At $(\ast)$ we used the fact that
$P$ is locally absolutely continuous with respect to $Q$.
(If $P$ were not locally absolutely continuous with respect to $Q$,
then there are cases where
$P(\Gamma_u) > 0$, $P(\Gamma_{ua}) = 0$, and $Q(\Gamma_{ua}) \neq 0$.
Therefore $X_{t+1}(ua)$ does not contribute to the expectation and thus
$X_{t+1}(ua) P(\Gamma_{ua}) = 0 \neq Q(\Gamma_{ua})$.)
\end{enumerate}
$P \geq 0$ and $Q \geq 0$ by definition, thus $X_t \geq 0$.
Since $P(\Gamma_\epsilon) = Q(\Gamma_\epsilon) = 1$,
we have $\EE[X_0] = 1$.
\end{proof}

The following lemma gives a convenient condition for
the existence of a measure on $(\Sigma^\omega, \Foo)$.
It is a special case of
the Daniell-Kolmogorov Extension Theorem~\citep[Thm.~26.1]{RW:1994}.

\begin{lemma}[Extending measures]
\label{lem:semimeasure}
Let $q: \Sigma^* \to [0, 1]$ be a function such that
$q(\epsilon) = 1$ and
$\sum_{a \in \Sigma} q(ua) = q(u)$ for all $u \in \Sigma^*$.
Then there exists a unique probability measure $Q$
on $(\Sigma^\infty, \Foo)$ such that
$q(u) = Q(\Gamma_u)$ for all $u \in \Sigma^*$.
\end{lemma}

To prove this lemma, we need the following two ingredients.

\begin{definition}[Semiring]
\label{def:semiring}\index{semiring}
A set $\mathcal{R} \subseteq 2^\Omega$ is called \emph{semiring over $\Omega$} iff
\begin{enumerate}[(a)]
\item $\emptyset \in \mathcal{R}$,
\item for all $A, B \in \mathcal{R}$, the set $A \cap B \in \mathcal{R}$, and
\item for all $A, B \in \mathcal{R}$,
	there are pairwise disjoint sets $C_1, \ldots, C_n \in \mathcal{R}$
	such that $A \setminus B = \biguplus_{i=1}^n C_i$.
\end{enumerate}
\end{definition}

\begin{theorem}[{Carathéodory's Extension Theorem; \citealp[Thm.~A.1.1]{Durrett:2010}}]
\label{thm:caratheodory}\index{Carathéodory's extension theorem}
Let $\mathcal{R}$ be a semiring over $\Omega$ and
let $\mu: \mathcal{R} \to [0,1]$ be a function such that
\begin{enumerate}[(a)]
\item $\mu(\Omega) = 1$
	\emph{(normalization)},
\item $\mu(\biguplus_{i=1}^n A_i) = \sum_{i=1}^n \mu(A_i)$
	for pairwise disjoint sets $A_1, \ldots, A_n \in \mathcal{R}$
	such that $\biguplus_{i=1}^n A_i \in \mathcal{R}$
	\emph{(finite additivity)}, and
\item $\mu(\bigcup_{i \geq 0} A_i) \leq \sum_{i \geq 0} \mu(A_i)$
	for any collection $(A_i)_{i \geq 0}$ such that
	each $A_i \in \mathcal{R}$ and $\bigcup_{i \geq 0} A_i \in \mathcal{R}$
	\emph{($\sigma$-subadditivity)}.
\end{enumerate}
Then there is a unique extension $\overline{\mu}$ of $\mu$
that is a probability measure on $(\Omega, \sigma(\mathcal{R}))$ such that
$\overline{\mu}(A) = \mu(A)$ for all $A \in \mathcal{R}$.
\end{theorem}

\begin{proof}[Proof of \autoref{lem:semimeasure}]
We show the existence of $Q$ using
\hyperref[thm:caratheodory]{Carathéodory's Extension Theorem}.
Define $\mathcal{R} := \{ \Gamma_u \mid u \in \Sigma^* \} \cup \{ \emptyset \}$.
\begin{enumerate}[(a)]
\item $\emptyset \in \mathcal{R}$.

\item For any $\Gamma_u, \Gamma_v \in \mathcal{R}$, either
\begin{itemize}
\item $u$ is a prefix of $v$ and $\Gamma_u \cap \Gamma_v = \Gamma_v \in \mathcal{R}$, or
\item $v$ is a prefix of $u$ and $\Gamma_u \cap \Gamma_v = \Gamma_u \in \mathcal{R}$, or
\item $\Gamma_u \cap \Gamma_v = \emptyset \in \mathcal{R}$.
\end{itemize}

\item For any $\Gamma_u, \Gamma_v \in \mathcal{R}$,
\begin{itemize}
\item $\Gamma_u \setminus \Gamma_v = \biguplus_{w \in \Sigma^{|v| - |u|} \setminus \{ x \}} \Gamma_{uw}$
	if $v = ux$, i.e., $u$ is a prefix of $v$, and
\item $\Gamma_u \setminus \Gamma_v = \emptyset$ otherwise.
\end{itemize}
\end{enumerate}
Therefore $\mathcal{R}$ is a semiring.
By definition of $\mathcal{R}$, we have $\sigma(\mathcal{R}) = \Foo$.

The function $q: \Sigma^* \to [0,1]$ naturally gives rise to a function
$\mu: \mathcal{R} \to [0,1]$ with $\mu(\emptyset) := 0$ and
$\mu(\Gamma_u) := q(u)$ for all $u \in \Sigma^*$.
We will now check the prerequisites of
\hyperref[thm:caratheodory]{Carathéodory's Extension Theorem}.
\begin{enumerate}[(a)]
\item (Normalization.)
$\mu(\Sigma^\infty) = \mu(\Gamma_\epsilon) = q(\epsilon) = 1$.

\item (Finite additivity.)
Let $\Gamma_{u_1}, \ldots, \Gamma_{u_k} \in \mathcal{R}$ be pairwise disjoint sets such that
$\Gamma_w := \biguplus_{i=1}^k \Gamma_{u_i} \in \mathcal{R}$.
Let $\ell := \max \{ |u_i| \mid 1 \leq i \leq k \}$, then
$\Gamma_w = \biguplus_{v \in \Sigma^\ell} \Gamma_{wv}$.
By assumption, $\sum_{a \in \Sigma} q(ua) = q(u)$,
thus $\sum_{a \in \Sigma} \mu(\Gamma_{ua}) = \mu(\Gamma_u)$ and inductively
we have
\begin{equation}\label{eq-semimeasures-1}
\mu(\Gamma_{u_i}) = \sum_{s \in \Sigma^{\ell - |u_i|}} \mu(\Gamma_{u_i s}),
\end{equation}
and
\begin{equation}\label{eq-semimeasures-2}
\mu(\Gamma_w) = \sum_{v \in \Sigma^\ell} \mu(\Gamma_{wv}).
\end{equation}
For every string $v \in \Sigma^\ell$,
the concatenation $wv \in \Gamma_w = \biguplus_{i=1}^k \Gamma_{u_i}$,
so there is a unique $i$ such that $wv \in \Gamma_{u_i}$.
Hence there is a unique string $s \in \Sigma^{\ell - |u_i|}$ such that $wv = u_i s$.
Together with \eqref{eq-semimeasures-1} and \eqref{eq-semimeasures-2} this yields
\[
  \mu \left( \biguplus_{i=1}^k \Gamma_{u_i} \right)
= \mu(\Gamma_w)
= \sum_{v \in \Sigma^\ell} \mu(\Gamma_{wv})
= \sum_{i=1}^k \sum_{s \in \Sigma^{\ell - |u_i|}} \mu(\Gamma_{u_i s})
= \sum_{i=1}^k \mu(\Gamma_{u_i}).
\]

\item ($\sigma$-subadditivity.)
We will show that each $\Gamma_u$ is compact
with respect to the topology $\mathcal{O}$ generated by $\mathcal{R}$.
$\sigma$-subadditivity then follows from (b)
because every countable union is in fact a finite union.

We will show that the topology $\mathcal{O}$ is the product topology
of the discrete topology on $\Sigma$.
(This establishes that $(\Sigma^\omega, \mathcal{O})$ is a Cantor Space.)
Every projection $\pi_k: \Sigma^\infty \to \Sigma$ selecting the $k$-th symbol is continuous,
since $\pi_k^{-1}(a) = \bigcup_{u \in \Sigma^{k-1}} \Gamma_{ua}$ for every $a \in \Sigma$.
Moreover, $\mathcal{O}$ is the coarsest topology with this property, since we can generate
every open set $\Gamma_u \in \mathcal{R}$ in the base of the topology by
\[
\Gamma_u = \bigcap_{i=1}^{|u|} \pi_i^{-1}(\{ u_i \}).
\]

The set $\Sigma$ is finite and thus compact.
By Tychonoff's Theorem, $\Sigma^\infty$ is also compact.
Therefore $\Gamma_u$ is compact since it is homeomorphic to $\Sigma^\infty$
via the canonical map $\beta_u: \Sigma^\infty \to \Gamma_u$, $v \mapsto uv$.
\end{enumerate}
From (a), (b), and (c) \hyperref[thm:caratheodory]{Carathéodory's Extension Theorem} yields
a unique probability measure $Q$ on $(\Sigma^\infty, \Foo)$ such that
$Q(\Gamma_u) = \mu(\Gamma_u) = q(u)$ for all $u \in \Sigma^*$.
\end{proof}

Using \autoref{lem:semimeasure},
the proof of \autoref{thm:martingale-measure} is now straightforward.

\begin{proof}[Proof of \autoref{thm:martingale-measure}]
We define a function $q: \Sigma^* \to \mathbb{R}$, with
\[
q(u) := X_{|u|}(v) P(\Gamma_u)
\]
for any $v \in \Gamma_u$.
The choice of $v$ is irrelevant because $X_{|u|}$ is constant on $\Gamma_u$
since it is $\F_t$-measurable.
In the following,
we also write $X_t(u)$ if $|u| = t$ to simplify notation.

The function $q$ is non-negative because $X_t$ and $P$ are both non-negative.
Moreover, for any $u \in \Sigma^t$,
\[
     1
=    \EE[X_t]
=    \int_{\Sigma^\infty} X_t dP
\geq \int_{\Gamma_u} X_t dP
=    P(\Gamma_u) X_t(u)
=    q(u).
\]
Hence the range of $q$ is a subset of $[0, 1]$.

We have $q(\epsilon) = X_0(\epsilon) P(\Gamma_\epsilon) = \EE[X_0] = 1$
since $P$ is a probability measure and
$\F_0 = \{ \emptyset, \Sigma^\infty \}$ is the trivial $\sigma$-algebra.
Let $u \in \Sigma^t$.
\begin{align*}
   \sum_{a \in \Sigma} q(ua)
&= \sum_{a \in \Sigma} X_{t+1}(ua) P(\Gamma_{ua})
 = \int_{\Gamma_u} X_{t+1} dP \\
&= \int_{\Gamma_u} \EE[X_{t+1} \mid \F_t] dP
 = \int_{\Gamma_u} X_t dP
 = P(\Gamma_u) X_t(u)
 = q(u).
\end{align*}
By \autoref{lem:semimeasure},
there is a probability measure $Q$ on $(\Sigma^\infty, \Foo)$
such that $q(u) = Q(\Gamma_u)$ of all $u \in \Sigma^*$.
Therefore, for all $v \in \Sigma^\infty$ and
for all $t \in \mathbb{N}$ with $P(\Gamma_{v_{1:t}}) > 0$,
\[
  X_t(v)
= \frac{q(v_{1:t})}{P(\Gamma_{v_{1:t}})}
= \frac{Q(\Gamma_{v_{1:t}})}{P(\Gamma_{v_{1:t}})}.
\]
Moreover,
$P$ is locally absolutely continuous with respect to $Q$ since
$P(\Gamma_u) = 0$ implies
\[
Q(\Gamma_u) = q(u) = X_{|u|}(u) P(\Gamma_u) = 0.
\qedhere
\]
\end{proof}


\addcontentsline{toc}{chapter}{Bibliography}
\bibliography{ai,sv}


\makeatletter
\chapter*{List of Notation\@mkboth{List of Notation}{List of Notation}}
\makeatother
\addcontentsline{toc}{chapter}{List of Notation}
\label{cha:notation}

\subsubsection*{Abbreviations}

\begin{longtable}{lp{0.85\textwidth}}
AIXI & Bayesian RL agent with a Solomonoff prior, see \autoref{ssec:AIXI} \\
AI & artificial intelligence \\
HLAI & human-level artificial intelligence \\
MDL & minimum description length, see \autoref{ex:MDL} \\
MDP & Markov decision process, see \autoref{ssec:typical-environment-classes} \\
POMDP & partially observable Markov decision process,
	see \autoref{ssec:typical-environment-classes} \\
RL & reinforcement learning \\
UTM & universal Turing machine, \autoref{sec:AIT} \\
\end{longtable}

\subsubsection*{Mathematical notation}

\begin{longtable}{lp{0.77\textwidth}}
$:=$
	& defined to be equal \\
$:\in$
	& defined to be an element of \\
$A$, $B$, $\Omega$
	& sets \\
$\#A$
	& the cardinality of the set $A$, i.e., the number of elements \\
$\Delta\Omega$
	& the set of probability distributions over a finite or countable set $\Omega$ \\
$\one_x$
	& the characteristic function that is $1$ for $x$
	and $0$ otherwise. \\
$f$, $g$
	& functions \\
$f \timesgeq g$
	& there is a constant $c > 0$ such that $f \geq cg$ \\
$f \timeseq g$
	& $f \timesgeq g$ and $g \timesgeq f$ \\
$\mathbb{N}$
	& the set of natural numbers, starting with $1$ \\
$\mathbb{Q}$
	& the set of rational numbers \\
$\mathbb{R}$
	& the set of real numbers \\
$n$, $k$, $t$, $m$, $i$, $j$
	& natural numbers \\
$t$
	& (current) time step, $t \in \mathbb{N}$ \\
$k$
	& some other time step, $k \in \mathbb{N}$ \\
$q, q'$
	& rational numbers \\
$r$
	& real number \\
$\X$
	& a finite nonempty alphabet \\
$\X^*$
	& the set of all finite strings over the alphabet $\X$ \\
$\X^\infty$
	& the set of all infinite strings over the alphabet $\X$ \\
$\X^\sharp$
	& $\X^\sharp := \X^* \cup \X^\infty$,
	the set of all finite and infinite strings over the alphabet $\X$ \\
$x, y, z$
	& (typically finite) strings from $\X^\sharp$ \\
$x_{<t}$
	& the first $t - 1$ symbols of the string $x$ \\
$x \sqsubseteq y$
	& the string $x$ is a prefix of the string $y$ \\
$\mathrm{zeros}(x)$
	& the number of zeros in the binary string $x \in \{ 0, 1 \}^*$ \\
$\mathrm{ones}(x)$
	& the number of ones in the binary string $x \in \{ 0, 1 \}^*$ \\
$\phi$, $\psi$
	& computable functions \\
$\varphi$
	& formula of first-order logic \\
$\eta$
	& computable relation/quantifier-free formula \\
$T$
	& a Turing machine \\
$p, p'$
	& programs on a universal Turing machine in the form of finite binary strings \\
$|p|$
	& length of the program $p$ in bits \\
$K$
	& the Kolmogorov complexity of a string or a semimeasure \\
$\Km$
	& the monotone Kolmogorov complexity of a string \\
$\Ent$
	& entropy \\
$\KL_m$
	& KL-divergence \\
$D_m$
	& total variation distance \\
$\IG$
	& information gain \\
$F$
	& expected total variation distance \\
$\F$, $\F_t$, $\F_\infty$
	& $\sigma$-algebras \\
$\Gamma_x$
	& the cylinder set of all strings starting with $x$ \\
$A$, $H$, $E$
	& measurable sets \\
$X$, $Y$
	& real-valued random variables \\
$P$ &
	a distribution over $X^\infty$, the \emph{true} distribution \\
$Q$ &
	a distribution over $X^\infty$,
	the learning algorithm or belief distribution \\
$\Bernoulli(r)$
	& a Bernoulli process with parameter $r$ \\
$\lambda$
	& the uniform measure or Lebesgue measure \\
$\rho_L$
	& Laplace rule \\
$M$ & Solomonoff's prior \\
$\MM$
	& the measure mixture \\
$S_{Kt}$
	& the speed prior \\
$\nu$
	& a semimeasure \\
$\nu\norm$
	& the Solomonoff normalization of the semimeasure $\nu$ \\
$\gg$
	& absolute continuity \\
$\timesgeq_W$
	& weak dominance \\
$\gg_L$
	& local absolute continuity \\
$\A$
	& the finite set of possible actions \\
$\O$
	& the finite set of possible observations \\
$\E$
	& the finite set of possible percepts,
	$\E \subset \O \times \mathbb{R}$ \\
$\alpha, \beta$
	& two different actions, $\alpha, \beta \in \A$ \\
$a_t$
	& the action in time step $t$ \\
$o_t$
	& the observation in time step $t$ \\
$r_t$
	& the reward in time step $t$, bounded between $0$ and $1$ \\
$e_t$
	& the percept in time step $t$, we use $e_t = (o_t, r_t)$ implicitly \\
$\ae_{<t}$
	& the first $t - 1$ interactions,
	$a_1 e_1 a_2 e_2 \ldots a_{t-1} e_{t-1}$
	(a history of length $t - 1$) \\
$h$
	& a history, $h \in \H$ \\
$\epsilon$
	& the history of length $0$ \\
$\eps$, $\delta$
	& small positive real numbers \\
$\gamma$
	& the discount function $\gamma: \mathbb{N} \to \mathbb{R}_{\geq0}$,
	defined in \autoref{def:discounting} \\
$\Gamma_t$
	& a discount normalization factor,
	$\Gamma_t := \sum_{k=t}^\infty \gamma(k)$ \\
$m$
	& horizon of the agent (how many steps it plans ahead) \\
$H_t(\eps)$
	& an $\eps$-effective horizon \\
$\nu, \mu, \rho$
	& environments \\
$\pi, \tilde\pi$
	& policies, $\pi, \tilde\pi: \H \to \A$ \\
$\pi^*_\nu$
	& an optimal policy for environment $\nu$ \\
$\nu^\pi$
	& the history distribution generated by policy $\pi$ in environment $\nu$ \\
$\EE^\pi_\nu$
	& the expectation with respect to the history distribution $\nu^\pi$ \\
$V^\pi_\nu$
	& the $\nu$-expected value of the policy $\pi$ \\
$V^*_\nu$
	& the optimal value in environment $\nu$ \\
$W^\pi_\nu$
	& the iterative value of the policy $\pi$ in environment $\nu$ \\
$W^*_\nu$
	& the optimal iterative value in environment $\nu$ \\
$\expectimax{}$
	& the max-sum-operator \\
$R_m(\pi, \mu)$
	& regret of policy $\pi$ in environment $\mu$ for horizon $m$ \\
$\Upsilon_\xi(\pi)$
	& the Legg-Hutter intelligence of policy $\pi$
	measured in the universal mixture $\xi$ \\
$\underline\Upsilon_\xi$
	& the minimal Legg-Hutter intelligence
	measured in the universal mixture $\xi$ \\
$\overline\Upsilon_\xi$
	& the maximal Legg-Hutter intelligence
	measured in the universal mixture $\xi$ \\
$\M$
	& a class of environments \\
$\Mccs$
	& the class of all chronological contextual semimeasures \\
$\Mlscccs$
	& the class of all lower semicomputable chronological contextual semimeasures \\
$\Mcomp$
	& the class of all computable chronological contextual measures \\
$\Mrefl^O$
	& the class of all reflective-oracle-computable environments \\
$U$
	& reference universal Turing machine \\
$U'$
	& a `bad' universal Turing machine \\
$w$
	& a positive prior over the environment class \\
$w'$
	& a `bad' positive prior over the environment class \\
$\xi$
	& the universal mixture over all environments $\Mlscccs$
	given by the reference UTM $U$ \\
$\xi'$
	& a `bad' universal mixture over all environments $\Mlscccs$
	given by the `bad' UTM $U'$ \\
$\mathcal{T}$
	& the set of all probabilistic Turing machines \\
$O$
	& an oracle \\
$\tilde O$
	& a partial oracle \\
$\lambda_T$
	& the semimeasure generated by Turing machine $T$ \\
$\lambda_T^O$
	& the semimeasure generated by Turing machine $T$ running with oracle $O$ \\
$\overline\lambda_T^O$
	& the completion of $\lambda_T^O$ into a measure using oracle $O$ \\
$\sigma$
	& a multi-agent environment \\
$\sigma_i$
	& the subjective environment of agent $i$
	  acting in multi-agent environment $\sigma$ \\
\end{longtable}

\cleardoublepage
\phantomsection
\printindex

\end{document}